\documentclass[a4paper]{ociamthesis_modified}
\correctionstrue

\usepackage{pdfpages}
\usepackage{natbib}

\usepackage{subcaption}
\usepackage[ruled,vlined]{algorithm2e}
\usepackage{caption}
\usepackage{graphicx}
\usepackage{amsmath}
\usepackage{amsthm}
\usepackage{amsfonts}
\usepackage{thmtools,thm-restate}
\usepackage{bbm}
\usepackage{mathtools}
\usepackage{authblk}

\newcommand{\indep}{\perp \!\!\! \perp}
\usepackage{capt-of}
\usepackage{comment}
\newcommand{\R}{\mathbb{R}}

\usepackage{amsmath}
\DeclareMathOperator*{\argmax}{arg\,max}

\usepackage{comment}
\usepackage{microtype}
\usepackage{booktabs}
\usepackage{centernot}

\usepackage[many]{tcolorbox}
\usepackage{xcolor}
\usepackage{color}
\usepackage{colortbl}

\definecolor{MyLightGray}{gray}{0.925}
\definecolor{MyDarkGray}{gray}{0.55}
\definecolor{myblue}{rgb}{0.00, 0.45, 0.70}
\definecolor{mygreen}{rgb}{0.01, 0.62, 0.45}
\definecolor{myred}{rgb}{0.84, 0.37, 0.00}

\PassOptionsToPackage{hyphens}{url}
\usepackage{hyperref}
\definecolor{blue}{RGB}{0,114,178}
\definecolor{red}{RGB}{213,80,20}
\definecolor{green}{RGB}{0,158,115}
\definecolor{purple}{RGB}{204,121,167}
\definecolor{orange}{RGB}{230,159, 0}
\definecolor{pink}{RGB}{204,121,167}

\definecolor{ourmethod}{gray}{0.93}

\definecolor{sns0}{HTML}{1f77b4}
\definecolor{sns1}{HTML}{ff7f0e}
\definecolor{sns2}{HTML}{2ca02c}
\definecolor{sns3}{HTML}{d62728}
\definecolor{sns4}{HTML}{9467bd}
\definecolor{sns5}{HTML}{8c564b}

\definecolor{sns-orange}{HTML}{ff7f0e}
\definecolor{sns-ambiguous}{HTML}{00528C}
\definecolor{sns-nonambiguous}{HTML}{2CA9FF}
\definecolor{sns-blue}{HTML}{1f77b4}
\definecolor{solarized@base03}{HTML}{002B36}
\definecolor{solarized@base02}{HTML}{073642}
\definecolor{solarized@base01}{HTML}{586e75}
\definecolor{solarized@base00}{HTML}{657b83}
\definecolor{solarized@base0}{HTML}{839496}
\definecolor{solarized@base1}{HTML}{93a1a1}
\definecolor{solarized@base2}{HTML}{EEE8D5}
\definecolor{solarized@base3}{HTML}{FDF6E3}
\definecolor{solarized@yellow}{HTML}{B58900}
\definecolor{solarized@orange}{HTML}{CB4B16}
\definecolor{solarized@red}{HTML}{DC322F}
\definecolor{solarized@magenta}{HTML}{D33682}
\definecolor{solarized@violet}{HTML}{6C71C4}
\definecolor{solarized@blue}{HTML}{268BD2}
\definecolor{solarized@cyan}{HTML}{2AA198}
\definecolor{solarized@green}{HTML}{859900}

\theoremstyle{plain}

\newtcolorbox{mainresult}{colback=solarized@violet!5!white,
colframe=solarized@violet,parbox, left=0.5mm, right=0.5mm,top=0.5mm,bottom=0.5mm}

\newtcolorbox{mainresultwithtitle}[2][]{colback=solarized@violet!7!white,
colframe=solarized@violet!7, 
colbacktitle=solarized@violet,
parbox=false, 
left=0.5mm, 
right=0.5mm, 
top=0.5mm, 
bottom=0.5mm, 
title={#2}, 
#1}

\newtcolorbox{importantresultwithtitle}[2][]{
  colback=solarized@cyan!7!white,
  colframe=solarized@cyan!7,
  colbacktitle=solarized@cyan,
  parbox=false,
  left=0.5mm,
  right=0.5mm,
  top=0.5mm,
  bottom=0.5mm,
  title={#2}, 
  #1
}

\newtcolorbox{whiteresult}{colback=solarized@violet!2!white,
colframe=solarized@violet,parbox, left=0.5mm, right=0.5mm,top=0.5mm,bottom=0.5mm}

\newtheorem{theorem}{Theorem}[section]
\renewenvironment{theorem}[1][]
 {\refstepcounter{theorem}
  \begin{mainresultwithtitle}[title={Theorem \thetheorem\ifx#1\empty\else~(#1)\fi}]\noindent}
 {\end{mainresultwithtitle}}

\newcounter{proposition}[section]
\renewcommand{\theproposition}{\thesection.\arabic{proposition}}

\newenvironment{proposition}[1][]
 {\refstepcounter{proposition}
  \begin{mainresultwithtitle}[title={Proposition \theproposition\ifx#1\empty\else~(#1)\fi}]\noindent}
 {\end{mainresultwithtitle}}

\newcounter{lemma}[section]
\renewcommand{\thelemma}{\thesection.\arabic{lemma}}

\newenvironment{lemma}[1][]
 {\refstepcounter{lemma}
  \begin{mainresultwithtitle}[title={Lemma \thelemma\ifx#1\empty\else~(#1)\fi}]\noindent}
 {\end{mainresultwithtitle}}

\newcounter{corollary}[section]
\renewcommand{\thecorollary}{\thesection.\arabic{corollary}}

\newenvironment{corollary}[1][]
 {\refstepcounter{corollary}
  \begin{mainresultwithtitle}[title={Corollary \thecorollary\ifx#1\empty\else~(#1)\fi}]\noindent}
 {\end{mainresultwithtitle}}

\newcounter{definition}[section]
\renewcommand{\thedefinition}{\thesection.\arabic{definition}}

\newenvironment{definition}[1][]
 {\refstepcounter{definition}
  \begin{importantresultwithtitle}[title={Definition \thedefinition\ifx#1\empty\else~(#1)\fi}]\noindent}
 {\end{importantresultwithtitle}}

\newtheorem{assumption}[theorem]{Assumption}

 \newcounter{limitation}[section]
\renewcommand{\thelimitation}{\thesection.\arabic{limitation}}

\newcounter{remark}[section]
\renewcommand{\theremark}{\thesection.\arabic{remark}}

\newcommand{\red}[1]{\textcolor{red}{#1}}
\newcommand{\p}[0]{\mathbb{P}}
\newcommand{\E}[0]{\mathbb{E}}
\newcommand{\tarprob}[0]{\mathbb{P}_{(X,Y)\sim P^{\pi^*}_{X,Y}}}
\newcommand{\ttar}[0]{\mathbb{P}_{(X,Y)\sim \tilde{P}^{\pi^*}_{X,Y}}}
\newcommand{\behprob}[0]{\mathbb{P}_{(X,Y)\sim P^{\pi^b}_{X,Y}}}
\newcommand{\expb}[0]{\mathbb{E}_{(X,Y)\sim P^{\pi^b}_{X,Y}}}
\newcommand{\expt}[0]{\mathbb{E}_{(X,Y)\sim P^{\pi^*}_{X,Y}}}
\newcommand{\exptt}[0]{\mathbb{E}_{(\tilde{X},\tilde{Y})\sim P^{\pi^*}_{X,Y}}}

\newcommand{\expatt}[0]{\mathbb{E}_{(\tilde{X},\tilde{Y})\sim \tilde{P}^{\pi^*}_{X,Y}}}

\newcommand{\expbcal}[0]{\mathbb{E}_{(X_i, Y_i) \sim P^{\pi^b}_{X,Y}}}

\newcommand{\ind}{\mathbbm{1}}
\newcommand{\Aspace}{\mathcal{A}}
\newcommand{\Xspace}{\mathcal{X}}

\newcommand{\abs}[1]{|#1|}

\newcommand{\Prob}{\mathbb{P}}

\newcommand{\eqas}{\overset{\mathrm{a.s.}}{=}}
\newcommand{\Law}{\mathrm{Law}}

\DeclarePairedDelimiterX{\infdivx}[2]{(}{)}{#1\;\delimsize\|\;#2}

\newcommand\ci{\perp\!\!\!\perp}

\newcommand{\dee}{\mathrm{d}}

\newcommand{\Xspacedim}{d}
\newcommand{\A}{A}
\newcommand{\X}{X}
\newcommand{\xx}{x}
\newcommand{\Xt}{\widehat{\X}}

\newcommand{\ax}{a}

\newcommand{\sx}{s}

\newcommand{\nx}{n}

\newcommand{\gx}{g}

\newcommand{\tx}{t}
\newcommand{\B}{B}
\newcommand{\N}{N}
\newcommand{\T}{T}

\newcommand{\fx}{f}

\newcommand{\lo}[1]{#1_{\mathrm{lo}}}
\newcommand{\up}[1]{#1_{\mathrm{up}}}

\newcommand{\Y}{Y}
\newcommand{\Yt}{\widehat{Y}}
\newcommand{\ylo}{\lo{y}}
\newcommand{\yup}{\up{y}}

\newcommand{\Ylo}{\lo{\Y}}
\newcommand{\Yup}{\up{\Y}}

\newcommand{\Q}{Q}
\newcommand{\Qt}{\widehat{\Q}}
\newcommand{\Qlo}{\lo{\Q}}
\newcommand{\Qup}{\up{\Q}}

\newcommand{\qlo}[1]{\lo{R}^{#1}}
\newcommand{\qup}[1]{\up{R}^{#1}}
\newcommand{\qt}[1]{\widehat{R}^{#1}}

\newcommand{\YClmean}{{\lo{\mu}}}

\newcommand{\Ytmean}{\widehat{\mu}}
\newcommand{\CIlen}{\Delta}

\newcommand{\Hyp}{\mathcal{H}}
\newcommand{\Hlo}{\lo{\Hyp}}
\newcommand{\plo}{p_{\textup{lo}}}
\newcommand{\Hup}{\up{\Hyp}}
\newcommand{\pup}{p_{\textup{up}}}

\newcommand{\D}{\mathcal{D}}
\newcommand{\Dt}{\widehat{\D}}

\newcommand{\twinfunction}{h}
\newcommand{\twinnoise}{U}
\newcommand{\ux}{u}

\newcommand{\AppendixName}{Appendix\xspace}
\usepackage{tabularx}

\newcommand{\tar}[0]{\pi^*}
\newcommand{\beh}[0]{\pi^b}
\newcommand{\hatbeh}[0]{\widehat{\pi}^b}
\newcommand{\ptar}[0]{p_{\pi^*}}
\newcommand{\Etar}[0]{\mathbb{E}_{\pi^*}}
\newcommand{\Etarred}[0]{\mathbb{E}_{\textcolor{red}{\pi^*}}}
\newcommand{\Ebeh}[0]{\mathbb{E}_{\pi^b}}
\newcommand{\Ebehblue}[0]{\mathbb{E}_{\textcolor{blue}{\pi^b}}}

\newcommand{\Vbeh}[0]{\textup{Var}_{\pi^b}}
\newcommand{\V}[0]{\textup{Var}}

\newcommand{\pbeh}[0]{p_{\pi^b}}

\newcommand{\thetadr}[0]{\hat{\theta}_{\textup{DR}}}
\newcommand{\thetamr}[0]{\hat{\theta}_{\textup{MR}}}
\newcommand{\thetaipw}[0]{\hat{\theta}_{\textup{IPW}}}
\newcommand{\thetaswitch}[0]{\hat{\theta}_{\textup{SwitchDR}}}
\newcommand{\thetadros}[0]{\hat{\theta}_{\textup{DRos}}}
\newcommand{\approxipw}[0]{\tilde{\theta}_{\textup{IPW}}}
\newcommand{\approxmr}[0]{\tilde{\theta}_{\textup{MR}}}
\newcommand{\thetagmdr}[0]{\tilde{\theta}_{\textup{DM-DR}}}
\newcommand{\ipw}[0]{\textup{IPW}}

\newcommand{\mr}[0]{\textup{MR}}
\newcommand{\dr}[0]{\textup{DR}}
\newcommand{\ate}[0]{\textup{ATE}}
\newcommand{\var}[0]{\textup{Var}}
\newcommand{\Yspace}[0]{\mathcal{Y}}

\newcommand{\Dtr}[0]{\mathcal{D}_{\textup{tr}}}

\newcommand{\gt}[0]{{\textup{gt}}}

\newcommand{\ateipw}{\widehat{\ate}_{\ipw}}
\newcommand{\atemr}{\widehat{\ate}_{\mr}}
\newcommand{\atedr}{\widehat{\ate}_{\dr}}
\newcommand{\thetasnmr}{\theta_{\textup{SNMR}}}
\newcommand{\tr}{{\textup{tr}}}
\newcommand{\myparagraph}[1]{\paragraph{#1}}
\newcommand{\flag}[1]{#1}
\usepackage{wrapfig}
\usepackage{tikz}
\usepackage{ifthen}
\usepackage{rotating}

\newboolean{compilePapers}
\newboolean{compileAppendices}

\AtBeginEnvironment{subappendices}{%
\chapter*{Appendix}
\addcontentsline{toc}{chapter}{Appendices}
\counterwithin{figure}{section}
\counterwithin{table}{section}
}

\title{Uncertainty Quantification and Causal Considerations for Off-Policy Decision Making}
\author{Muhammad Faaiz Taufiq}
\college{Wolfson College}

\degree{Doctor of Philosophy}
\degreedate{Michaelmas 2024}

\begin{document}

\setlength{\textbaselineskip}{22pt plus2pt}
\setlength{\frontmatterbaselineskip}{17pt plus1pt minus1pt}

\setlength{\baselineskip}{\textbaselineskip}

\setcounter{secnumdepth}{2}
\setcounter{tocdepth}{1}

\begin{romanpages}

\maketitle

\begin{originality}
    I hereby declare that except where specific reference is made to the work of others, the intellectual contents of this dissertation are original and have not been submitted in whole or in part for consideration for any other degree or qualification. My personal contributions are detailed in the authorship forms at the end of each chapter. This dissertation is my own work except as specified in the text and authorship forms.
\end{originality}

    \begin{dedication}
    \textit{This thesis is dedicated to my family.}\\
    \textit{I would not be where I am today without their love, support and sacrifices.}\\
    \end{dedication}

\begin{acknowledgements}
 	First and foremost I would like to thank my supervisors Rob Cornish, Arnaud Doucet and Yee Whye Teh 
for their support, patience and invaluable mentorship throughout my DPhil journey. 
Their guidance taught me the importance of perseverance, rigour and not compromising on quality -- 
traits which I will cherish for the rest of my life.
They were present for me whenever I needed their help, while also providing me all the flexibility and independence that I needed to conduct my research. 

I would also like to sincerely thank the collaborators who I have had the privilege of working with throught my DPhil. 
Special thanks to Jean-Francois Ton, Chris Holmes, Lenon Minorics and Yang Liu whose guidance and insights have not only contributed to the success of my research projects but also taught me valuable skills which I will draw upon in future as well.

I am deeply grateful to Google DeepMind for funding my DPhil. 
It would not have been possible for me to pursue my studies without their generous support.

My experience at Oxford would be incomplete without the friendships I developed here. 
I would like to specially mention the amazing people who enriched my time at the department with laughter and camaraderie: Jef, Sahra, Emi, Bobby, Jin, Robert, Michael, Alan, Carlo, Andrew, Valentin, Chris, Sheh, Tyler, Jessie, Desi and Guneet.  
Not only did they make my DPhil an enjoyable journey, these are also some of the most smart, talented and supportive individuals that I have ever met and they continue to inspire me everyday. 

To my sister, thank you for being the one unwavering constant in my ever-changing life. 
And to my parents, I cannot express the gratitude I feel for your untiring dedication towards my and my sister's education.
It is impossible for me to list the countless sacrifices that you have made to set us up for success in our lives. 

Finally, I am forever indebted to God for blessing me with this incredible life.
\end{acknowledgements}

\begin{abstract}
	Off-policy evaluation (OPE) is a critical challenge in robust decision-making that seeks to assess the performance of a new policy using data collected under a different policy. 
However, the existing OPE methodologies suffer from several limitations arising from statistical uncertainty as well as causal considerations. 
In this thesis, we address these limitations by presenting three different works. 

Firstly, we consider the problem of high variance in the importance-sampling-based OPE estimators. 
We propose a novel off-policy evaluation estimator, the Marginal Ratio (MR) estimator, to alleviate this problem.
By focusing on the marginal distribution of outcomes rather than the policy shift directly, the MR estimator achieves significant variance reduction compared to state-of-the-art methods, while maintaining unbiasedness. 

Next, we shift our attention towards uncertainty quantification in off-policy evaluation.
To this end, we propose Conformal Off-Policy Prediction (COPP) as a novel approach to quantify this uncertainty with finite-sample guarantees.
Unlike traditional methods focusing on point estimates of expected outcomes, COPP provides reliable predictive intervals for outcomes under a target policy. This enables robust decision-making in risk-sensitive applications and offers a more comprehensive understanding of policy performance. 

Finally, we address the fundamental challenge of causal inference in off-policy evaluation. 
Recognizing the limitations of traditional OPE methods under unmeasured confounding, we develop novel causal bounds for sequential decision settings that remain valid under arbitrary confounding.
We apply these bounds for the assessment of digital twin models without relying on strong causal assumptions. 
We propose a framework for causal falsification, allowing us to identify scenarios where the digital twin's predictions diverge from real-world behavior. This approach provides valuable insights into model reliability and helps ensure safe and effective decision-making.

We conclude this thesis with a discussion of our contributions and
limitations of the presented work, and outline interesting avenues for future research arising from our work.
\end{abstract}

\dominitoc %
\flushbottom

\tableofcontents
\end{romanpages}

\setboolean{compilePapers}{true}

\setboolean{compileAppendices}{true}

\flushbottom

\chapter{\label{ch:1-intro}Introduction} 

\minitoc

The ability to make well-informed decisions is crucial across a variety of domains. 
Whether it is a doctor prescribing the most effective treatment for a patient or a company launching a marketing campaign that resonates with its target audience \citep{xu2020contextual,li2010contextual,bastani2019online}, we constantly strive to take actions that lead to desirable outcomes. 
However, achieving this goal becomes increasingly challenging in the face of uncertainty. Real-world data is often noisy and incomplete, and the systems we interact with are complex and constantly evolving. As machine learning models become more integrated into critical applications, the need for robust decision-making under these challenging conditions becomes paramount.

This thesis explores the key challenges of robust decision-making in machine learning, specifically focusing on \emph{off-policy evaluation} \citep{uncertainty5, adaptive-ope, uncertainty2, uncertainty3, uncertainty4, doubly-robust}. Consider the example of a doctor who wants to assess a new treatment for a disease. Ideally, they would conduct a randomized controlled trial \citep{tsiatis2019dynamic} where patients are randomly assigned the new treatment or a standard one. However, such trials can be expensive, time-consuming or worse, ethically problematic. Off-policy evaluation (OPE) offers a compelling alternative by allowing us to evaluate the performance of a new decision-making policy (the new treatment) using data collected under a different policy (the standard treatment). This eliminates the need for costly experimentation and allows for quicker implementation of potentially more effective strategies.

However, off-policy evaluation presents its own set of challenges. These challenges stem from two main sources of uncertainty:

\begin{itemize}
    \item \textbf{Statistical uncertainty:} This arises from the inherent randomness in the data we have access to and the limitations of the models we use to represent the real world. 
    For instance, the doctor might have a limited number of patients in their historical dataset, and their model might not perfectly capture a patient's response to treatment due to model misspecification.
    In these circumstances, the conventional OPE methods may suffer from high variance and/or bias, thereby potentially resulting in misleading conclusions \citep{saito2021evaluating,su2020doubly,saito2022off}. 
    \item \textbf{Causal unidentifiability:} In many cases, it may be impossible to definitively establish the causal effects of actions even if we had access to infinite data. This arises due to factors like confounding variables, which can influence both the treatment and the outcome. Imagine the existence of some unmeasured factors, such as a patient's pre-existing conditions, that can influence both their initial treatment and their response to the treatment. This makes it challenging to isolate the true causal effect of the new treatment from the influence of these confounding variables \citep{tsiatis2019dynamic,kallus2018confounding,namkoong2020offpolicy}.
\end{itemize}
This thesis tackles these challenges head-on, proposing novel methods for off-policy evaluation that address both statistical and causal uncertainties.
Before we go into the specifics of these challenges, we introduce the problem of off-policy evaluation in contextual bandits which forms the basis of the setting considered in Chapters \ref{ch:3-mr} and \ref{ch:4-copp}.

\section{Off-policy evaluation in contextual bandits}
\subsection{Contextual bandits}
Contextual bandits \citep{Lattimore_Szepesvári_2020} provide a powerful framework for tackling decision-making problems where the effectiveness of an action depends on the specific context in which it is chosen. 
For instance, in medical decision-making, the optimal treatment for a patient might depend on various factors such as their age, medical history, and current symptoms. 
Contextual bandits allow us to model these complex decision-making scenarios by incorporating the notion of context.

In this setting, we use covariates $X \in \mathcal{X}$ to denote features which encapsulate the contextual information such as the patient's age and medical history. We use $A \in \mathcal{A}$ to represent the action chosen by some real-world agent (such as a doctor), and $Y \in \mathcal{Y}$ to denote the outcome/reward observed as a result of taking action $A$. For example, $Y \in \{0, 1\}$ might represent whether a patient survives ($Y=1$) or not ($Y=0$). The goal of a learner in contextual bandits is to choose actions $A$ for a context $X$ which maximises the reward $Y$. 

\subsection{Off-policy evaluation}
Off-policy evaluation (OPE) tackles a crucial challenge in decision-making: assessing the performance of a new policy using data collected under a different policy \citep{swaminathan2015counterfactual, wang2017optimal, farajtabar2018more, su2019continuous, metelli2021subgaussian, liu2019triply, sugiyama2012machine, swaminathan2017off}. This is particularly valuable when conducting controlled experiments with the new policy is impractical or unethical. Here, we formally define the OPE problem in contextual bandits which will set up the challenges tackled in Chapters \ref{ch:3-mr} and \ref{ch:4-copp} of this thesis. 

To be more concrete, let $\D\coloneqq \{(x_i, a_i, y_i)\}_{i=1}^n$ be a historically logged dataset with $n$ observations, generated by a (possibly unknown) \emph{behaviour policy} $\beh(a\mid x)$, i.e. the conditional distribution of agent's actions is $A\mid X=x \sim \beh(\,\cdot\mid x)$.
Next, suppose that we are given a different \emph{target policy}, which we denote by $\tar(a\mid x)$. Our goal is to estimate what the expected outcome \emph{would} be if actions were instead sampled from this target policy $\tar$.

\begin{importantresultwithtitle}[title=Off-policy evaluation (OPE)]\noindent
    The main objective of off-policy evaluation (OPE) is to estimate the expectation of the outcome $Y$ under a given target policy $\tar$ using only the logged data $\D$.
\end{importantresultwithtitle}

The key challenge of OPE arises from the fact that we do not
have access to samples from the target distribution which makes the estimation of off-policy value non-trivial in general.
To tackle this problem, the standard OPE methods make the following assumption.

\begin{assumption}[No unmeasured confounding]\label{assum:no-unmeasured-confounding}
    The agent's action in the observational data $A$ depends only on the context $X$ and possibly additional randomness independent of everything else. This means that when choosing the action, the agent does not rely on additional information relevant to the outcome which is not captured in the context. For instance, in a medical context, this assumption means that all of the information that clinicians use to make treatment decisions is captured in the data. This assumption is also referred to as the \emph{strong ignorability assumption} \citep{tsiatis2019dynamic}.
\end{assumption}

Then, under Assumption \ref{assum:no-unmeasured-confounding}, the off-policy value can be estimated using importance-sampling-based methods \citep{horvitz1952generalization}. 
However, these estimators come with their own set of limitations, which are described in the following section. 

\section{Limitations of existing OPE methods}
\subsection{High variance of OPE estimators}\label{subsec:high-variance}
The conventional off-policy value estimators use policy ratios $\rho(a, x) \coloneqq \tar(a\mid x)/\beh(a\mid x)$ as importance weights. 
In cases where the two policies are significantly different, the policy ratios $\rho(a, x)$ attain extreme values leading to a high variance in the OPE estimators. 
To alleviate this high variance, \cite{dudik2014doubly} proposed a Doubly Robust (DR) estimator for OPE which uses a control variate to decrease the variance of conventional OPE estimators. 
However, DR still relies on policy ratios as importance weights and as a result, also suffers from high variance when the policy shift is large. 
This problem is further exacerbated as the sizes of the action and context spaces grow \citep{sachdeva2020off, saito2022off}.
Chapter \ref{ch:3-mr} of this thesis specifically focuses on this limitation of OPE. 

Besides using control variates (as in DR estimator), several techniques have been proposed to address the variance issues associated with importance weights. 

\paragraph{Trading off variance for bias}
\cite{swaminathan2015counterfactual, swaminathan2015the, chaudhuri2019london} attempt to bound the importance weights within a certain range to prevent them from becoming excessively large. 
Besides this, the Direct Method (DM) \citep{Beygelzimer2008Offset} avoids the use of importance-sampling by estimating the reward function from observational data.
Similarly, Switch-DR \citep{wang2017optimal} aims to circumvent the high variance in conventional DR estimator by switching to DM when the importance weights are large.
However, these approaches introduce a bias-variance trade-off, as clipping the weights or using the learned reward function can introduce bias into the estimates. 

\paragraph{Marginalization-based techniques}
Several works explore marginalisation techniques for variance reduction. For example, \cite{saito2022off} propose Marginalized Inverse Propensity Score (MIPS), which considers the marginal shift in the distribution of a lower dimensional embedding of the action space, denoted by $E$, instead of considering the shift in the policies explicitly. 
While this approach reduces the variance, we show in Chapter \ref{ch:3-mr} that MIPS relies on an additional assumption regarding the action embeddings $E$ which does not hold in general.

In addition, various marginalisation ideas have also been proposed in the context of Reinforcement Learning (RL). For example, \cite{liu2018breaking, xie2019advances, kallus2020off} use methods which consider the shift in the marginal distribution of the states, and apply importance weighting with respect to this marginal shift rather than the trajectory distribution. Similarly, \cite{Fujimoto2021deep} use marginalisation for OPE in deep RL, where the goal is to consider the shift in marginal distributions of state and action. Although marginalization is a key trick of these estimators, these techniques are aimed at resolving the curse of horizon, a problem specific to RL.

\subsection{Lack of uncertainty quantification}\label{subsec:uncertainty-quantification}
Most techniques for OPE in contextual bandits focus on evaluating policies based on their \emph{expected} outcomes \citep{uncertainty5, adaptive-ope, uncertainty2, uncertainty3, uncertainty4, doubly-robust}. However, this can be problematic as
methods that are only concerned with the average outcome do not take into account any notions of
uncertainty in the outcome. Therefore, in risk-sensitive settings such as econometrics, where we want
to minimize the potential risks, metrics such as CVaR (Conditional Value at Risk) might be more
appropriate \citep{keramati2020being}. Additionally, when only small sample sizes of observational data are available, the average outcomes under finite data can be misleading, as they are prone to outliers and hence, metrics such as medians or quantiles are more robust in these scenarios \citep{altschuler2019best}. Next, we outline some recent works which tackle this challenge by developing methodologies to account for the uncertainty in off-policy performance using available data. 

\paragraph{Off-policy risk assessment in contextual bandits}
Instead of estimating bounds on the expected outcomes, \cite{risk-assessment, chandak2021universal} establish finite-sample bounds for a general class of metrics (e.g., Mean, CVaR, CDF) on the outcome. Their methods can be used to estimate quantiles of the outcomes under the target policy and are therefore robust to outliers. 
For example, \cite{chandak2021universal} proposed a non-parametric Weighted Importance Sampling (WIS) estimator for the empirical CDF of $Y$ under $\pi^*$,
which can be used to construct predictive intervals on the outcome under target policy. This can help us quantify the range of plausible outcomes $Y$ that are likely to occur if actions are chosen according to target policy $\tar$. However, the resulting bounds do not depend on the context $X$ (i.e., are not adaptive w.r.t. $X$). This can lead to overly conservative intervals, which may not be very informative. In Chapter \ref{ch:4-copp}, we circumvent this problem by proposing a methodology of constructing predictive intervals on $Y$ under target policy $\tar$ which are adaptive w.r.t. context $X$ and are therefore considerably more informative.

\section{Causal considerations for sequential decisions}
Having outlined some of the key limitations of OPE methods in contextual bandits, we now move on to the causal considerations for off-policy decision-making which will set up our contribution in Chapter \ref*{ch:5-causal}. 
Before we dive deeper into this topic, we introduce the sequential decision setting which generalises the contextual bandits framework.

\subsection{Sequential decision setting}
Contextual bandits encapsulate the single-decision regimes where, for each observed context, we take a single action and observe the resulting outcome. This is analogous to a doctor choosing a single treatment for a patient based on their current state. However, many real-world decision-making scenarios involve multiple interventions over time, where each action not only affects the immediate outcome but also influences the context for future decisions. To capture this complexity, we introduce the sequential decision setting in this section. This setting extends the framework of contextual bandits to handle sequential decision-making problems, allowing us to model more complex scenarios where interventions unfold over time and the context evolves dynamically.

We consider a setting with a fixed number of decisions per episode (i.e., a fixed time horizon) $T \in \{1, 2, \ldots\}$. For each $t\in \{0, \ldots, T\}$, we assume that the process gives rise to an observation at time $t$, denoted by $X_t$ which takes values in some space $\mathcal{X}_t \coloneqq \R^{\Xspacedim_\tx}$. 
Moreover, at time $t\in  \{1, \ldots, T\}$ a real-world agent (such as a doctor) chooses an action $A_t$ which takes values in some space $\mathcal{A}_t$. The agent's choice of $A_t$ may depend on the historical observations $(X_0, \ldots, X_{t-1})$ or any additional information not captured in historical observations that the agent can access. For example, in a medical context, the observations may consist of a patient's vital signs, and the actions may consist of possible treatments or interventions that the doctor chooses based on patient history.
The actions taken up to time $t$, i.e. $(A_1, A_2, \ldots, A_t)$ can influence the future observations $(X_t, X_{t+1}, \ldots, X_T)$. 
This describes the sequential decision setting, of which the contextual bandits are a special case when $T=1$. 

\subsection{Causal unidentifiability under unmeasured confounding}\label{subsec:unmeasured-confounding}
Most of the standard OPE methods for contextual bandits can be straightforwardly extended to sequential decision settings \citep{uehara2022reviewoffpolicyevaluationreinforcement}. 
However, like in contextual bandits, these estimators assume no unmeasured confounding (outlined in Assumption \ref{assum:no-unmeasured-confounding}) in the available observational data. 
Informally, in the sequential decision setting, this assumption holds when each action $\A_\tx$ is chosen by the behavioural agent solely on the basis of the information available at time $\tx$ that is actually recorded in the dataset, namely $\X_{0}, \A_1, \X_1, \ldots, \A_{\tx-1}, \X_{\tx-1}$, as well as possibly some additional randomness that is independent of the real-world process, such as the outcome of a coin toss. 
Unobserved confounding is present whenever this does not hold, i.e.\ whenever some unmeasured factor simultaneously influences both the agent's choice of action and the observation produced by the real-world process.
This can happen when the real-world agent has access to more information than is captured in the data. 
In such circumstances, the causal effect of a given action sequence may be unidentifiable from the available observational data, making it impossible to accurately estimate the value of a target policy.
To make this concrete, we provide an intuitive illustration of this phenomenon below using a toy example where the available observational data suffers from unmeasured confounding.

\begin{importantresultwithtitle}[title=Toy example: Unmeasured confounding in medical decision-making]\noindent
Suppose that we are interested in estimating the effect of a drug on the weight of patients in a certain population. 
Moreover, assume that this drug interacts with an enzyme that is only present in part of the population.
Denote by $U \in \{0, 1\}$ the presence or absence of the enzyme in a patient, and assume that when $U = 1$ the patient's weight increases after action the drug is administered, and that when $U = 0$ the drug has no effect.
Additionally, suppose that, among the patients whose data we have obtained, the drug was only prescribed to those for whom $U = 1$, perhaps on the basis of some initial lab reports available to the prescriber.
Finally, suppose that these lab results were \emph{not} included in the context $X$ captured in the observational dataset $\D$, so that the value of $U$ for each patient cannot be determined from the data we have available. 

In this setup, since the drug was only administered to patients with $U=1$, it would appear from the data that the drug causes patient weight to increase. 
However, when the drug is administered to the general population, i.e.\ regardless of the value of $U$, we would observe that the drug has no effect on patients for whom $U=0$. Figure \ref{fig:syn_ex_intro} illustrates this discrepancy under a toy model for
this scenario. In this example, since the data $\D$ contains no information about the presence or absence of the enzyme in patients, $U$, it is impossible to determine using the data $\D$ alone how the drug will affect a given population of patients.

\end{importantresultwithtitle}

\begin{figure}[t]
    \centering
    \includegraphics[height=3.3cm]{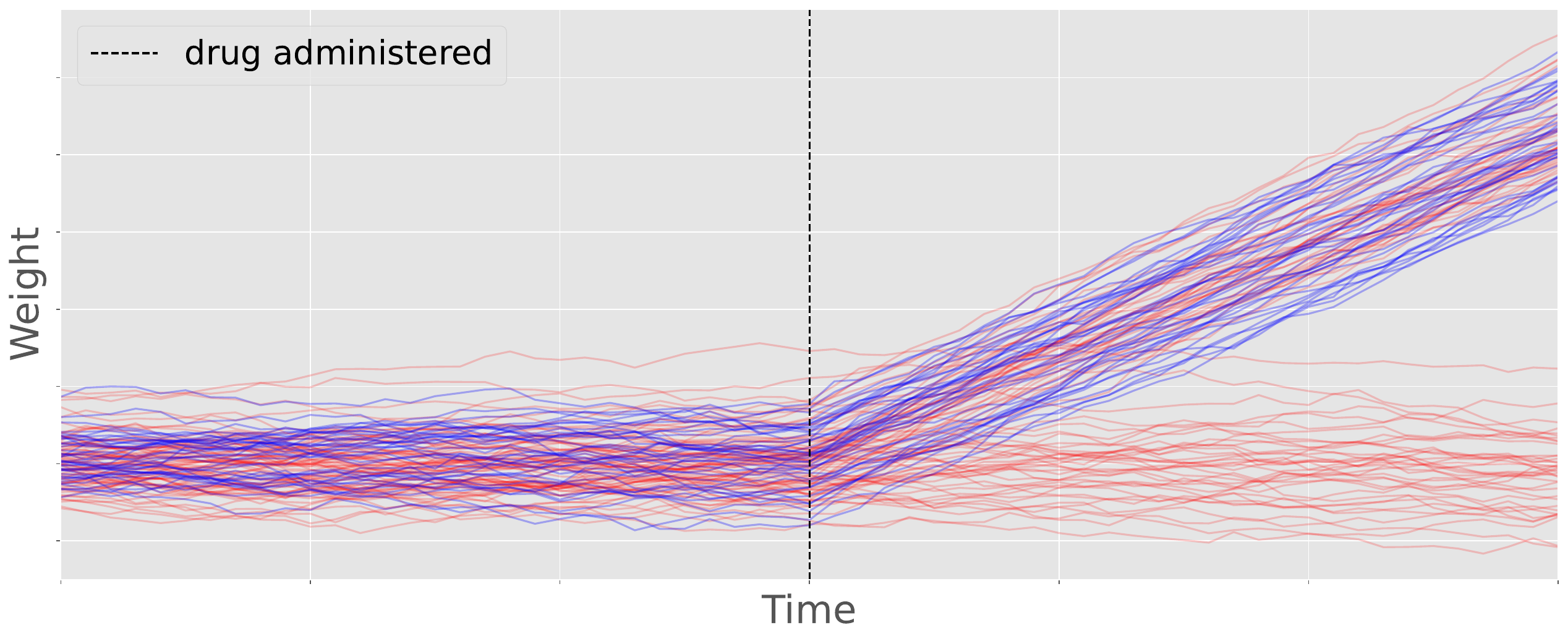}
    \caption{The discrepancy between observational data and interventional behaviour in the presence of unmeasured confounding: the range of outcomes observed in the data for patients who were administered the drug (blue) differs from what \emph{would} be observed if the drug were administered to the general population (red).}
    \label{fig:syn_ex_intro}
\end{figure}

In certain contexts it may be reasonable to assume that the data are unconfounded. For example, in certain situations it may be possible to gather data in a way that specifically guarantees there is no confounding.
Randomised controlled trials, which ensure that each $\A_\tx$ is chosen via a carefully designed randomisation procedure \citep{lavori2004dynamic,murphy2005experimental}, constitute a widespread example of this approach. However, for typical datasets, it is widely acknowledged that the assumption of no unmeasured confounding will rarely hold, and so OPE procedures based on this assumption may yield unreliable results in practice \citep{murphy2003optimal,tsiatis2019dynamic}. 
This is formalised in a foundational result from the causal inference literature, often referred to as the \emph{fundamental problem of causal inference} \citep{holland1986statistics}.
\begin{mainresultwithtitle}[title=Fundamental problem of causal inference (informal statement)]\noindent
    The causal effect of an action is not uniquely identified by the observational data distribution without additional assumptions.
\end{mainresultwithtitle}

\paragraph{Partial identification}
Since the precise identifiability of causal effects is not possible in the presence of unmeasured confounding, a notable line of work instead explores partial identification techniques \citep{manski,manski1989anatomy, manski2003partial}. Instead of the point identification of causal effects which may require strong unconfounding assumptions, partial identification typically considers the range of causal effects which may occur in the presence of confounding. For example, \cite{manski} constructs sharp bounds on the causal effects which can be readily estimated using the available observational data. While these bounds do not require any strong assumptions, they can be conservative. 

\paragraph{Sensitivity analysis}
Slightly stronger assumptions yield inferences that may be more powerful but less credible. To this end, \cite{rosenbaum2002observational} proposes a classical model of confounding for a single binary decision setting which posits that the unobserved confounders have a limited influence on the agent's actions in the real world.  \cite{namkoong2020offpolicy} extend this model to the multi-action sequential decision-making setting, and subsequently use this to obtain bounds on the off-policy value.

The Rosenbaum model is also closely related to (albeit different from) the marginal sensitivity model introduced by \cite{tan2006distributional} which also assumes bounds on the strength of unmeasured confounding on agent's actions. Subsequently, \cite{kallus2020minimax} use the marginal sensitivity model to develop a policy learning algorithm which remains robust to unmeasured confounding. However, these models impose assumptions on the strength of unmeasured confounding which can be impossible to verify using observational data alone, and therefore the inferences obtained may be misleading in many cases. 

\paragraph{Proxy causal learning}
This comprises methodologies for estimating the causal effect of actions on outcomes in the presence of unobserved confounding, using \emph{proxy variables} which contain relevant side information about the unmeasured confounders \citep{xu2021deep, tchetgen2020introduction, xu2024kernel}. This usually involves a two-stage regression. First, the relationship between action and proxies is modelled and subsequently, this model is used to learn the causal effect of actions on the outcomes. \cite{kuroki2014measurement} outline the necessary conditions on proxy variables to obtain the true causal effects. While proxy causal learning may be effective in cases where such proxy variables are available, in many real-world settings the available proxy variables may not satisfy the necessary conditions for identification of true causal effects.

Chapter \ref{ch:5-causal} of this thesis considers the challenges posed by unmeasured confounding in sequential decision setting. We propose a set of novel bounds on the causal effects in this setting, which remain valid in the presence of arbitrary unmeasured confounding and rely on minimal assumptions making them highly applicable to a wide variety of real-world settings.

\section{Contributions and thesis outline}
Having outlined some of the key challenges associated with off-policy evaluation, we dedicate the rest of this thesis to addressing each of these individually.
Specifically, this thesis is organised as follows:

\paragraph*{Chapter \ref*{ch:3-mr}: Variance reduction \citep{taufiq2023marginal}}
The first challenge we consider is that of high variance in existing OPE estimators based on importance sampling. 
As we mentioned in Section \ref{subsec:high-variance}, this variance is exacerbated in cases where there is low overlap between behaviour and target policies, or where the action or context space is high-dimensional.
To address this challenge, we propose a novel OPE estimator for contextual bandits, the Marginal Ratio (MR) estimator, which uses a marginalisation technique to focus on the shift in the marginal distribution of outcomes $Y$
directly, instead of the policies themselves. 
Unlike the conventional approaches which use policy ratios as importance weights, intuitively, our proposed estimator treats actions $A$ and contexts $X$ as latent variables.
Consequently, the resulting estimator is significantly more robust to the overlap between policies and the sizes of action and/or context spaces.
This chapter also includes extensive theoretical and empirical analyses demonstrating the benefits of the MR estimator compared to the state-of-the-art OPE estimators for contextual bandits.

\paragraph*{Chapter \ref*{ch:4-copp}: Uncertainty quantification \citep{taufiq2022conformal}}
As explained in Section \ref{subsec:uncertainty-quantification}, most OPE methods have focused on the expected outcome of a policy which does not capture the variability of the outcome $Y$.
In addition, many of these methods provide only asymptotic guarantees of validity at best. 
In this chapter, we address these limitations by considering a novel application of conformal prediction \citep{vovk2005algorithmic} to contextual bandits. 
Given data collected under a behavioral policy, we propose \emph{conformal off-policy prediction} (COPP), which can output reliable predictive intervals for the outcome under a new target policy. We provide theoretical finite-sample guarantees without making any additional assumptions beyond the standard contextual bandit setup, and empirically demonstrate the utility of COPP compared with existing methods on synthetic and real-world data.

\paragraph*{Chapter \ref*{ch:5-causal}: Causal considerations \citep{cornish2023causalfalsificationdigitaltwins}}
In this chapter we consider the sequential decision setting, where the available observational data may suffer from unmeasured confounding. 
As mentioned in Section \ref{subsec:unmeasured-confounding}, fundamental results from causal inference mean that in this setting the causal effect of interventions is unidentifiable from the observational distribution.  
To address this challenge, we provide a novel set of longitudinal causal bounds that remain valid under arbitrary unmeasured confounding.

Chapter \ref*{ch:5-causal} focuses on the application of these 
bounds for assessing the accuracy of \emph{digital twin models}.
These models are virtual systems designed to predict how a real-world process will evolve in response to interventions. 
To be considered accurate, these models must correctly capture the true causal effects of interventions.
Unfortunately, the causal unidentifiability results mean that observational data cannot be used to certify a twin in this sense if the data are confounded.
To circumvent this, we instead use our proposed causal bounds to find situations in which the twin \emph{is not} correct, and present a general-purpose statistical procedure for doing so.
Our approach yields reliable and actionable information about the twin under only the assumption of an i.i.d.\ dataset of observational trajectories, and remains sound even if the data are confounded.

\paragraph*{Chapter \ref*{ch:6-conclusion}: Conclusion} 
Finally, we conclude by summarising the main findings of the works presented in this thesis. 
In this chapter, we also discuss some of the limitations of our proposed methodologies and mention some interesting avenues for future research arising from these works. 

\section{An overview of work conducted during the DPhil}
In this section, we provide an overview of the research conducted during the doctoral studies by listing the papers which are included in this thesis, as well those which have been omitted.

\subsection{Works included in the thesis}
This is an integrated thesis where each chapter comprises a paper and therefore is self-contained.
These papers are listed here in chronological order for completeness.
\begin{enumerate}
    \item \textbf{Muhammad Faaiz Taufiq}*, Jean-Francois Ton*, Rob Cornish, Yee Whye Teh, and Arnaud Doucet.
    Conformal Off-Policy Prediction in Contextual Bandits. In \textit{\textcolor{purple}{Advances in Neural Information Processing
    Systems, 2022}}. \citep{taufiq2022conformal}
    \item Rob Cornish*, \textbf{Muhammad Faaiz Taufiq}*, Arnaud Doucet, and Chris Holmes. Causal Falsification
    of Digital Twins, 2023. \textit{Preprint}. \citep{cornish2023causalfalsificationdigitaltwins}
    \item \textbf{Muhammad Faaiz Taufiq}, Arnaud Doucet, Rob Cornish, and Jean-Francois Ton. Marginal Density Ratio for Off-Policy Evaluation in Contextual Bandits. In 
    \textit{\textcolor{purple}{Advances in Neural Information Processing Systems, 2023}}. \citep{taufiq2023marginal}
\end{enumerate}

\subsection{Works omitted from the thesis}
For the purposes of coherence and conciseness, several works which were part of the doctoral research have been omitted from this thesis. 
Here, we list these papers along with a brief description in chronological order for completeness. 
\begin{enumerate}
    \item \textbf{Muhammad Faaiz Taufiq}, Patrick Blöbaum, and Lenon Minorics. Manifold Restricted
    Interventional Shapley Values. In \textit{\textcolor{purple}{International Conference on Artificial Intelligence and
    Statistics, 2023}}.  \citep{taufiq2023manifold}
    \item  \textbf{Muhammad Faaiz Taufiq}, Jean-Francois Ton, and Yang Liu. Achievable Fairness on your Data
    with Utility Guarantees. \textit{Preprint}. \citep{taufiq2024achievablefairnessdatautility}
    \item Yang Liu, Yuanshun Yao, Jean-Francois Ton, Xiaoying Zhang, Ruocheng Guo, Hao Cheng, Yegor
    Klochkov, \textbf{Muhammad Faaiz Taufiq}, and Hang Li. 
    Trustworthy LLMs: A Survey and Guideline for Evaluating Large Language Models' Alignment, 2023. 
    In \textit{\textcolor{purple}{NeurIPS 2023 Workshop on Socially Responsible Language Modelling Research (SoLaR)}}.    
    \citep{liu2024trustworthyllmssurveyguideline}
\end{enumerate}
In \cite{taufiq2023manifold}, we consider the robustness of Shapley values, which are model-agnostic methods for explaining model predictions.
Many commonly used methods of computing Shapley values, known as off-manifold methods, are sensitive to model behaviour outside the data distribution.
This makes Shapley explanations highly sensitive to off-manifold perturbations of models, resulting in misleading explanations.
To circumvent this problem, we propose \emph{ManifoldShap}, which respects the model’s domain of validity by restricting model evaluations to the data manifold.
We show, theoretically and empirically, that ManifoldShap is robust to off-manifold perturbations of the model and leads to
more accurate and intuitive explanations than existing state-of-the-art Shapley methods.

Beyond this, \cite{taufiq2024achievablefairnessdatautility} considers fairness within the context of machine learning models. 
In this setting, training models that minimize disparity across different sensitive groups often leads to diminished accuracy, a phenomenon known as the fairness-accuracy tradeoff. 
The severity of this trade-off inherently depends on dataset characteristics such as dataset
imbalances or biases and therefore, using a uniform fairness requirement across diverse datasets
remains questionable. To address this, we present a computationally efficient approach to approximate the fairness-accuracy trade-off curve tailored to individual datasets, backed by rigorous
statistical guarantees.  Crucially, we introduce a novel methodology for quantifying uncertainty in our
estimates, thereby providing practitioners with a robust framework for auditing model fairness
while avoiding false conclusions due to estimation errors.

Finally, \cite{liu2024trustworthyllmssurveyguideline} presents a comprehensive survey of key dimensions that are crucial to consider when assessing the trustworthiness of Large Language Models (LLMs). 
The survey covers seven major categories of LLM trustworthiness: reliability, safety, fairness, resistance to misuse, explainability and reasoning, adherence to social norms, and robustness. 
The empirical results presented in this study indicate that, in general, more aligned models tend to perform better in terms of overall trustworthiness. However, the effectiveness of alignment varies across the different trustworthiness categories considered. This highlights the importance of conducting more fine-grained analyses, testing, and making continuous improvements on LLM alignment. 

\chapter{\label{ch:3-mr}Marginal Density Ratio for Off-Policy Evaluation in
Contextual Bandits} 

\minitoc

\ifthenelse{\boolean{compilePapers}}
    { %

\begin{abstract}

Off-Policy Evaluation (OPE) in contextual bandits is crucial for assessing new policies using existing data without costly experimentation. However, current OPE methods, such as Inverse Probability Weighting (IPW) and Doubly Robust (DR) estimators, suffer from high variance, particularly in cases of low overlap between target and behavior policies or large action and context spaces. In this paper, we introduce a new OPE estimator for contextual bandits, the Marginal Ratio (MR) estimator, which focuses on the shift in the marginal distribution of outcomes $Y$ instead of the policies themselves. Through rigorous theoretical analysis, we demonstrate the benefits of the MR estimator compared to conventional methods like IPW and DR in terms of variance reduction. Additionally, we establish a connection between the MR estimator and the state-of-the-art Marginalized Inverse Propensity Score (MIPS) estimator, proving that MR achieves lower variance among a generalized family of MIPS estimators. We further illustrate the utility of the MR estimator in causal inference settings, where it exhibits enhanced performance in estimating Average Treatment Effects (ATE). Our experiments on synthetic and real-world datasets corroborate our theoretical findings and highlight the practical advantages of the MR estimator in OPE for contextual bandits. 
\end{abstract}
\section{Introduction} 
In contextual bandits, the objective is to select an action $A$, guided by contextual information $X$, to maximize the resulting outcome $Y$. This paradigm is prevalent in many real-world applications such as healthcare, personalized recommendation systems, or online advertising \citep{li2010contextual, bastani2019online, xu2020contextual}. The objective is to perform actions, such as prescribing medication or recommending items, which lead to desired outcomes like improved patient health or higher click-through rates. Nonetheless, updating the policy presents challenges, as na\"ively implementing a new, untested policy may raise ethical or financial concerns. For instance, prescribing a drug based on a new policy poses risks, as it may result in unexpected side effects. As a result, recent research \citep{swaminathan2015counterfactual, wang2017optimal, farajtabar2018more, su2019continuous, metelli2021subgaussian, liu2019triply, sugiyama2012machine, swaminathan2017off} has concentrated on evaluating the performance of new policies (target policy) using only existing data that was generated using the current policy (behaviour policy). This problem is known as Off-Policy Evaluation (OPE).

Current OPE methods in contextual bandits, such as the Inverse Probability Weighting (IPW) \citep{horvitz1952generalization} and Doubly Robust (DR) \citep{dudik2014doubly} estimators primarily account for the policy shift by re-weighting the data using the ratio of the target and behaviour polices to estimate the target policy value. This can be problematic as it may lead to high variance in the estimators in cases of substantial policy shifts. The issue is further exacerbated in situations with large action or context spaces \citep{saito2022off}, since in these cases the estimation of policy ratios is even more difficult leading to extreme bias and variance.

In this work we show that this problem of high variance in OPE can be alleviated by using methods which directly consider the shift in the marginal distribution of the outcome $Y$ resulting from the policy shift, instead of considering the policy shift itself (as in IPW and DR). To this end, we propose a new OPE estimator for contextual bandits called the Marginal Ratio (MR) estimator, which weights the data directly based on the shift in the marginal distribution of outcomes $Y$ and consequently is much more robust to increasing sizes of action and context spaces than existing methods like IPW or DR. 
Our extensive theoretical analyses show that MR enjoys better variance properties than the existing methods making it highly attractive for a variety of applications in addition to OPE. One such application is the estimation of Average Treatment Effect (ATE) in causal inference, for which we show that MR provides greater sample efficiency than the most commonly used methods.

Our contributions in this paper are as follows:

\begin{itemize}
    \item Firstly, we introduce MR, an OPE estimator for contextual bandits, that focuses on the shift in the marginal distribution of $Y$ rather than the joint distribution of $(X, A, Y)$. 
    \flag{We show that MR has favourable theoretical properties compared to existing methods like IPW and DR. Our analysis also encompasses theory on the approximation errors of our estimator. 
    }

    \item Secondly, we explicitly lay out the connection between MR and  Marginalized Inverse Propensity Score (MIPS) \citep{saito2022off}, a recent state-of-the-art contextual bandits OPE method, and prove that MR attains lowest variance among a generalized family of MIPS estimators. 
    \item Thirdly, we show that the MR estimator can be applied in the setting of causal inference to estimate average treatment effects (ATE), and theoretically prove that the resulting estimator is more data-efficient with higher accuracy and lower variance than commonly used methods. 
    \item Finally, we verify all our theoretical analyses through a variety of experiments on synthetic and real-world datasets and empirically demonstrate that the MR estimator achieves better overall performance compared to current state-of-the-art methods. 
\end{itemize}

\section{Background}
\subsection{Setup and Notation} \label{sec:setup_notation}
We consider the standard contextual bandit setting. Let $X\in\Xspace$ be a context vector (e.g., user features), $A\in \Aspace$ denote an action (e.g., recommended website to the user), and $Y\in \Yspace$ denote a scalar reward or outcome (e.g., whether the user clicks on the website). The outcome and context are sampled from unknown probability distributions $p(y\mid x, a)$ and $p(x)$ respectively. Let $\D\coloneqq \{(x_i, a_i, y_i)\}_{i=1}^n$ be a historically logged dataset with $n$ observations, generated by a (possibly unknown) \emph{behaviour policy} $\beh(a\mid x)$.
Specifically, $\D$ consists of i.i.d. samples from the joint density under\textit{ behaviour policy},
\begin{align}
    \pbeh(x, a, y) \coloneqq p(y\mid x, a)\, \textcolor{blue}{\beh(a\mid x)}\,p(x). \label{eq:behav-joint-factorisation}
\end{align}
We denote the joint density of $(X, A, Y)$ under the \textit{target policy} as
\begin{align}
    \ptar(x, a, y) \coloneqq p(y\mid x, a)\, \textcolor{red}{\tar(a\mid x)}\,p(x). \label{eq:tar-joint-factorisation}
\end{align}

Moreover, we use $\pbeh(y)$ to denote the marginal density of $Y$ under the behaviour policy, 
\begin{align*}
    \pbeh(y) &= \int_{\Aspace \times \Xspace} \pbeh(x, a, y)\, \mathrm{d}a \, \mathrm{d}x,
\end{align*}
and likewise for the target policy $\tar$. Similarly, we use $\Ebeh$ and $\Etar$ to denote the expectations under the joint densities $\pbeh(x, a, y)$ and $\ptar(x, a, y)$ respectively.

\myparagraph{Off-policy evaluation (OPE)}
The main objective of OPE is to estimate the expectation of the outcome $Y$ under a given target policy $\tar$, i.e., $\Etar [Y]$, using only the logged data $\D$.

Throughout this work, we assume that the support of the target policy $\tar$ is included in the support of the behaviour policy $\beh$. This is to ensure that importance sampling yields unbiased off-policy estimators, and is satisfied for exploratory behaviour policies such as the $\epsilon$-greedy policies. 
\begin{assumption}[Support]
    For any $x \in \Xspace, a \in \Aspace$,  $\tar(a\mid x) >0 \implies \beh(a\mid x) >0$. 
\end{assumption}

\subsection{Existing off-policy evaluation methodologies}
Next, we will present some of the most commonly used OPE estimators before outlining the limitations of these methodologies. This motivates our proposal of an alternative OPE estimator. 

The value of the target policy can be expressed as the expectation of outcome $Y$ under the target data distribution $\ptar(x, a, y)$.
However in most cases, we do not have access to samples from this target distribution and hence we have to resort to importance sampling methods.
\paragraph{Inverse Probability Weighting (IPW) estimator}
One way to compute the target policy value, $\Etar[Y]$, when only given data generated from $\pbeh(x, a, y)$ is to rewrite the policy value as follows:

\begin{small}
\begin{align*}
    \Etarred[Y] =
    \int y \, \ptar(x, a, y) \,\mathrm{d}y \, \mathrm{d}a\, \mathrm{d}x   =
    \int y \, \underbrace{\frac{\ptar(x, a, y)}{\pbeh(x, a, y)}}_{\rho(a,x)}\, \pbeh(x, a, y) \,\mathrm{d}y \, \mathrm{d}a\, \mathrm{d}x =
    \Ebehblue\left[Y\,\rho(A, X)\right],
\end{align*}
\end{small} 
where 
$
\rho(a, x) \coloneqq \frac{\ptar(x, a, y)}{\pbeh(x, a, y)} = \frac{\tar(a|x)}{\beh(a|x)}
$, given the factorizations in Eqns. \eqref{eq:behav-joint-factorisation} and \eqref{eq:tar-joint-factorisation}.
This leads to the commonly used \emph{Inverse Probability Weighting (IPW)} \citep{horvitz1952generalization} estimator:
\[
\thetaipw \coloneqq \frac{1}{n}\sum_{i=1}^n \rho(a_i, x_i)\,y_i.
\]
When the behaviour policy is known, IPW is an unbiased and consistent estimator. However, it can suffer from high variance, especially as the overlap between the behaviour and target policies decreases. 

\myparagraph{Doubly Robust (DR) estimator} 
To alleviate the high variance of IPW, \cite{dudik2014doubly} proposed a \emph{Doubly Robust (DR)} estimator for OPE. 
DR uses an estimate of the conditional mean $\hat{\mu}(a, x) \approx\E[Y\mid X=x, A=a]$ (\emph{outcome model}), as a control variate to decrease the variance of IPW. It is also doubly robust in that it yields accurate value estimates if either the importance weights $\rho(a, x)$ or the outcome model $\hat{\mu}(a, x)$ is well estimated \citep{dudik2014doubly, jiang2016doubly}. 
The DR estimator for $\Etar[Y]$ can be written as follows:
\[
\thetadr = \frac{1}{n} \sum_{i=1}^n \rho(a_i, x_i)\,(y_i - \hat{\mu}(a_i, x_i)) + \hat{\eta}(\tar),
\]
where
$
\hat{\eta}(\tar) = \frac{1}{n} \sum_{i=1}^n \sum_{a'\in \Aspace} \hat{\mu}(a', x_i) \tar(a'\mid x_i) \approx \E_{\tar}[\hat{\mu}(A, X)]$. Here, $\hat{\eta}(\tar)$ is referred to as the Direct Method (DM) as it uses $\hat{\mu}(a, x)$ directly to estimate target policy value. 

\subsection{Limitation of existing methodologies} 
To estimate the value of the target policy $\tar$, the existing methodologies consider the shift in the joint distribution of $(X, A, Y)$  as a result of the policy shift (by weighting samples by policy ratios). As we show in Section \ref{subsec:comparison}, considering the joint shift can lead to inefficient policy evaluation and high variance especially as the policy shift increases \citep{li2018addressing}.
Since our goal is to estimate $\Etar[Y]$, we will show in the next section that considering only the shift in the marginal distribution of the outcomes $Y$ from $\pbeh(Y)$ to $\ptar(Y)$, leads to a more efficient OPE methodology compared to existing approaches.

To better comprehend why only considering the shift in the marginal distribution is advantageous, let us examine an extreme example where we assume that $Y \indep A \mid X$, i.e., the outcome $Y$ for a user $X$ is independent of the action $A$ taken. In this specific instance, $\Etar[Y] = \Ebeh[Y] \approx 1/n\sum_{i=1}^n y_i,$ indicating that an unweighted empirical mean serves as a suitable unbiased estimator of $\Etar[Y]$. However, IPW and DR estimators use policy ratios $\rho(a, x)  = \frac{\tar(a \mid x)}{\beh(a \mid x)}$ as importance weights. In case of large policy shifts, these ratios may vary significantly, leading to high variance in IPW and DR.

In this particular example, the shift in policies is inconsequential as it does not impact the distribution of outcomes $Y$. Hence, IPW and DR estimators introduce additional variance due to the policy ratios when they are not actually required. This limitation is not exclusive to this special case; in general, methodologies like IPW and DR exhibit high variance when there is low overlap between target and behavior policies \citep{li2018addressing} even if the resulting shift in marginals of the outcome $Y$ is not significant.

Therefore, we propose the \emph{Marginal Ratio (MR)} OPE estimator for contextual bandits in the subsequent section, which circumvents these issues by focusing on the shift in the marginal distribution of the outcomes $Y$. Additionally, we provide extensive theoretical insights on the comparison of MR to existing state-of-the-art methods, such as IPW and DR.

\section{Marginal Ratio (MR) estimator}

Our method's key insight involves weighting outcomes by the marginal density ratio of outcome $Y$:
\begin{align*}
\Etarred[Y] &= \int_{\Yspace} y \, \ptar(y)\, \mathrm{d}y = \int_\Yspace y\, \frac{\ptar(y)}{\pbeh(y)} \, \pbeh(y) \, \mathrm{d}y = \Ebehblue\left[Y\, w(Y) \right],
\end{align*}
where 
$
w(y) \coloneqq \frac{\ptar(y)}{\pbeh(y)}.
$
This leads to the Marginal Ratio OPE estimator:
\begin{align*}
    \thetamr \coloneqq \frac{1}{n}\sum_{i=1}^n w(y_i) \, y_i.
\end{align*}

In Section \ref{subsec:comparison} we prove that by only considering the shift in the marginal distribution of outcomes, the MR estimator achieves a lower variance than the standard OPE methods. In fact, this estimator does not depend on the shift between target and behaviour policies directly. Instead, it depends on the shift between the marginals $\pbeh(y)$ and $\ptar(y)$.

\myparagraph{Estimation of $w(y)$} When the weights $w(y)$ are known exactly, the MR estimator is unbiased and consistent. However, in practice the weights $w(y)$ are often not known and must be estimated using the logged data $\D$. Here, we outline an efficient way to estimate $w(y)$ by first representing it as a conditional expectation, which can subsequently be expressed as the solution to a regression problem.
\begin{lemma}\label{lemma:weights-est}
Let $w(y)=\frac{\ptar(y)}{\pbeh(y)}$ and $\rho(a, x) = \frac{\tar(a\mid x)}{\beh(a\mid x)}$, then $w(y) = \Ebeh\left[ \rho(A, X) \mid \,Y=y \right]$, and,
\begin{align}
 w = \arg\min_{f} \, \Ebeh \left[(\rho(A, X)-f(Y))^2\right]. \label{eq:weights-obj-main}
\end{align}
\end{lemma}
Lemma \ref{lemma:weights-est} allows us to approximate $w(y)$ using a parametric family $\{f_\phi: \mathbb{R}\rightarrow \mathbb{R} \mid \phi \in \Phi\}$ (e.g.\ neural networks) and defining $\hat{w}(y)\coloneqq f_{\phi^*}(y)$, where $\phi^*$ solves the regression problem in Eq. \eqref{eq:weights-obj-main}.

Note that MR can also be estimated alternatively by directly estimating $h(y) \coloneqq w(y)\,y$ 
using a similar regression technique as above and computing $\thetamr = 1/n \sum_{i=1}^n h(y_i)$. We include additional details along with empirical comparisons in Appendix \ref{sec:alt-estimation-method}. 

\subsection{Theoretical analysis}\label{subsec:comparison}
Recall that the traditional OPE estimators like IPW and DR use importance weights which account for the the shift in the joint distributions of $(X, A, Y)$. In this section, we prove that by considering only the shift in the marginal distribution of $Y$ instead, MR achieves better variance properties than these estimators.
Our analysis in this subsection assumes that the ratios $\rho(a, x)$ and $w(y)$ are known exactly. Since the OPE estimators considered are unbiased in this case, 
our analysis of variance is analogous to that of the mean squared error (MSE) here.
We address the case where the weights are not known exactly in Section \ref{subsec:weight-estimation-error}.
First, we make precise our intuition that the shift in the joint distribution of $(X, A, Y)$ is `greater' than the shift in the marginal distribution of outcomes $Y$. 
We formalise this using the notion of $f$-divergences.
\begin{proposition}\label{tv_prop}
Let $f:[0, \infty) \rightarrow \mathbb{R}$ be a convex function with $f(1)=0$, and $\textup{D}_f(P || Q)$ denotes the $f$-divergence between distributions $P$ and $Q$. Then,
\[
\textup{D}_f\left(\ptar(x,a,y)\,||\, \pbeh(x,a,y)\right) \geq \textup{D}_f\left(\ptar(y)\,||\, \pbeh(y)\right).
\]
\end{proposition}

\paragraph{Intuition}
Proposition \ref{tv_prop} shows that the shift in the joint distributions is at least as `large' as the shift in the marginals of the outcome $Y$. Traditional OPE estimators, therefore take into consideration more of a distribution shift than needed, and consequently lead to inefficient estimators. In contrast, the MR estimator mitigates this problem by only considering the shift in the marginal distributions of outcomes resulting from the policy shift. 
This provides further intuition on why the MR estimator has lower variance compared to existing methods. 
%
%
%
%

%
%
%

\begin{proposition}[Variance comparison with IPW estimator]\label{prop:var_mr}
When the weights $\rho(a, x)$ and $w(y)$ are known exactly, we have that $\Vbeh[\thetamr] \leq \Vbeh[\thetaipw]$. In particular,
\begin{align*}
    \Vbeh[\thetaipw] - \Vbeh[\thetamr]
    = \frac{1}{n} \Ebeh \left[ \Vbeh\left[ \rho(A, X) \mid Y \right]\, Y^2 \right] \geq 0.
\end{align*}
\end{proposition}
\myparagraph{Intuition}
Proposition \ref{prop:var_mr} shows that the variance of MR estimator is smaller than that of the IPW estimator when the weights are known exactly. 
Moreover, the proposition also shows that the difference between the two variances will increases as the variance $\Vbeh\left[ \rho(A, X) \mid Y \right]$ increases. This variance is likely to be large when the policy shift between $\beh$ and $\tar$ is large, or when the dimensions of contexts $X$ and/or the actions $A$ is large, and therefore in these cases the MR estimator will perform increasingly better than the IPW estimator.
A similar phenomenon occurs for DR as we show next, even though in this case the variance of MR is not in general smaller than that of DR. 

\begin{proposition}[Variance comparison with DR estimator]\label{prop:var_dr}
When the weights $\rho(a, x)$ and $w(y)$ are known exactly and $\mu(A, X) \coloneqq \E[Y\mid X, A]$, we have that,
\begin{align*}
     \Vbeh[\thetadr] - \Vbeh[\thetamr]
    \geq \frac{1}{n} \Ebeh \left[ \Vbeh\left[\rho(A, X)\,Y \mid Y \right] -  \Vbeh\left[\rho(A, X)\,\mu(A, X) \mid X \right] \right].
\end{align*}
\end{proposition}
\paragraph{Intuition}
Proposition \ref{prop:var_dr} shows that if the conditional variance $\Vbeh\left[ \rho(A, X)\,Y \mid Y \right]$ is greater than $\Vbeh\left[ \rho(A, X)\,\mu(A, X) \mid X \right]$ on average, the variance of the MR estimator will be less than that of the DR estimator. 
Intuitively, this will occur when the dimension of context space $\Xspace$ is high because in this case the conditional variance over $X$ and $A$, $\Vbeh\left[\rho(A, X)\,Y \mid Y \right]$ is likely to be greater than the conditional variance over $A$, $\Vbeh\left[ \rho(A, X)\,\mu(A, X) \mid X \right]$. Our empirical results in Appendix \ref{subsec:mips-empirical} are consistent with this intuition.
Additionally, we also provide theoretical comparisons with other extensions of DR, such as Switch-DR \citep{wang2017optimal} and DR with Optimistic Shrinkage (DRos) \citep{su2020doubly} in Appendix \ref{sec:dr-extensions}, and show that a similar intuition applies for these results. 
\flag{We emphasise that the well known results in \cite{wang2017optimal} which show that IPW and DR estimators achieve the optimal \emph{worst case} variance (where the worst case is taken over a class of possible outcome distributions $Y\mid X, A$) are not at odds with our results presented here (as the distribution of $Y\mid X, A$ is fixed in our setting).}

\subsubsection{Comparison with Marginalised Inverse Propensity Score (MIPS) \citep{saito2022off}}\label{subsec:mips-comparison}
In this section, we compare MR against the recently proposed Marginalised Inverse Propensity Score (MIPS) estimator \citep{saito2022off}, which uses a marginalisation technique to reduce variance and provides a robust OPE estimate specifically in contextual bandits with large action spaces. We prove that the MR estimator achieves lower variance than the MIPS estimator and doesn't require new assumptions.

\myparagraph{MIPS estimator}
As we mentioned earlier, the variance of the IPW estimator may be high when the action $A$ is high dimensional. To mitigate this, the MIPS estimator assumes the existence of a (potentially lower dimensional) action embedding $E$, which summarises all `relevant' information about the action $A$. Formally, this assumption can be written as follows: 
\begin{assumption}\label{assum:indep-mips}
    The action $A$ has no direct effect on the outcome $Y$, i.e., 
    $$Y \indep A \mid X, E.$$
\end{assumption}
For example, in the setting of a recommendation system where $A$ corresponds to the items recommended, $E$ may correspond to the item categories. Assumption \ref{assum:indep-mips} then intuitively means that item category $E$ encodes all relevant information about the item $A$ which determines the outcome $Y$. Assuming that such action embedding $E$ exists, \cite{saito2022off} prove that the MIPS estimator $\hat{\theta}_{\textup{MIPS}}$, defined as
\[
\hat{\theta}_{\textup{MIPS}} \coloneqq \frac{1}{n}\sum_{i=1}^n \frac{\ptar(e_i, x_i)}{\pbeh(e_i, x_i)}\, y_i = \frac{1}{n}\sum_{i=1}^n \frac{\ptar(e_i\mid x_i)}{\pbeh(e_i \mid x_i)}\, y_i,
\]
provides an unbiased estimator of target policy value $\Etar[Y]$. Moreover,
$\Vbeh[\hat{\theta}_{\textup{MIPS}}] \leq \Vbeh[\thetaipw]$.

\begin{figure}[ht]
\centering
\begin{tikzpicture}

\node[circle,draw, minimum size=1.2cm] (R0) at (0,0) {\begin{small}$(X, A)$\end{small}
};
\node[circle,draw, minimum size=1.2cm] (R1) at (2,0) {\begin{small}$(X, E)$\end{small}};
\node[circle,draw, minimum size=1.2cm] (Y) at (4,0) {$Y$};

\path[->, thick] (R0) edge (R1);
\path[->, thick] (R1) edge (Y);

\end{tikzpicture}
\caption{Bayesian network corresponding to Assumption \ref{assum:indep-mips}.}
\label{fig:embedding_mips}
\vspace{-0.2cm}
\end{figure}

\myparagraph{Intuition}
The context-embedding pair $(X, E)$ can be seen as a representation of the context-action pair $(X, A)$ which contains less `redundant information' regarding the outcome $Y$. Intuitively, the MIPS estimator, which only considers the shift in the distribution of $(X, E)$ is therefore more efficient than the IPW estimator (which considers the shift in the distribution of $(X, A)$ instead). 

\myparagraph{MR achieves lower variance than MIPS}
Given the intuition above, we should achieve greater variance reduction as the amount of redundant information in the representation $(X, E)$ decreases. We formalise this in Appendix \ref{app:gmips} and show that the variance of MIPS estimator decreases as the representation gets closer to $Y$ in terms of information content. As a result, we achieve the greatest variance reduction by considering the marginal shift in the outcome $Y$ itself (as in MR) rather than the shift in the representation $(X, E)$ (as in MIPS). The following result formalizes this finding. 
\begin{theorem}\label{prop:mips_main_text}
    When the weights $w(y)$, $\frac{\ptar(e, x)}{\pbeh(e, x)}$ and $\rho(a, x)$ are known exactly, then under Assumption \ref{assum:indep-mips}, 
    \begin{align*}
        \Ebeh[\thetamr] = \Ebeh[\hat{\theta}_{\textup{MIPS}}] = \Etar[Y], \quad \textup{and} \quad \Vbeh[\thetamr] \leq \Vbeh[\hat{\theta}_{\textup{MIPS}}] \leq \Vbeh[\thetaipw].
    \end{align*}
\end{theorem}
This analysis provides a link between the MR and MIPS estimators in the framework of contextual bandits, and shows that the MR estimator achieves lower variance than MIPS estimator while not requiring any additional assumptions (e.g.\ Assumption \ref{assum:indep-mips} as in MIPS). We also verify this empirically in Section \ref{sec:exp-synth} by reproducing the experimental setup in \cite{saito2022off} along with the MR baseline.

\subsubsection{Weight estimation error}\label{subsec:weight-estimation-error}
Our analysis so far assumes prior knowledge of the behavior policy $\beh$ and the marginal ratios $w(y)$. However, in practice, both quantities are often unknown and must be estimated from data. To this end, we assume access to an additional training dataset $\Dtr = \{(x^\tr_i, a^\tr_i, y^\tr_i)\}_{i=1}^m$ (for weight estimation), in addition to the evaluation dataset $\D = \{(x_i, a_i, y_i)\}_{i=1}^n$ (for computing the OPE estimate). 
The estimation of $\hat{w}(y)$ involves a two-step process that exclusively utilizes data from $\Dtr$:
\begin{enumerate}[label=(\roman*)]
    \item First, we estimate the policy ratio $\hat{\rho}(a, x) \approx \frac{\tar(a | x)}{\beh(a | x)}$. This can be achieved by estimating the behaviour policy $\hatbeh$, and defining $\hat{\rho}(a, x)\coloneqq \frac{\tar(a\mid x)}{\hatbeh(a\mid x)}$. Alternatively, $\hat{\rho}(a, x)$ can also be estimated directly by using density ratio estimation techniques as in \cite{sondhi2020balanced}.
    \item Secondly, we estimate the weights $\hat{w}(y)$ using Eq. \eqref{eq:weights-obj-main} with $\hat{\rho}$ instead of $\rho$.
\end{enumerate}

In practice, one may consider splitting $\Dtr$ for each estimation step outlined above. Moreover,
each approximation step may introduce bias and therefore, the MR estimator may have two sources of bias.
While classical OPE methods like IPW and DR also suffer from bias because of $\hat{\rho}$ estimation, the estimation of $\hat{w}(y)$ is specific to MR. However, we show below
that given any policy ratio estimate $\hat{\rho}$, if $\hat{w}(y)$ approximates $\Ebeh[\hat{\rho}(A, X)\mid Y=y]$ `well enough' (i.e., the estimation step (ii) shown above is `accurate enough'), 
then MR achieves a lower variance than IPW and incurs little extra bias.

\begin{proposition}\label{prop:bias-and-var-main}
Suppose that the IPW and MR estimators are defined as,
\[
\approxipw \coloneqq \frac{1}{n}\sum_{i=1}^n\hat{\rho}(a_i, x_i)\, y_i, \quad \textup{and }\quad \approxmr \coloneqq \frac{1}{n}\sum_{i=1}^n\hat{w}(y_i)\, y_i,
\]
and define the approximation error as $\epsilon \coloneqq \hat{w}(Y) - \tilde{w}(Y)$, where $\tilde{w}(Y) \coloneqq \Ebeh[\hat{\rho}(A, X)\mid Y]$. Then we have that, $\textup{Bias}(\approxmr) - \textup{Bias}(\approxipw) = \Ebeh[\epsilon\,Y]$. Moreover,
\begin{small}
\begin{align}
    \Vbeh[\approxipw] - \Vbeh[\approxmr]
    &= \frac{1}{n}(\underbrace{\Ebeh[\Vbeh[\hat{\rho}(A, X)\,Y\mid Y]]}_{\geq 0} - \Vbeh[\epsilon\,Y] - 2\,\textup{Cov}(\tilde{w}(Y)\,Y, \epsilon\,Y)). \label{eq:var-difference-approximate-weights}
\end{align}
\end{small}
\end{proposition}
\myparagraph{Intuition} The $\epsilon$ term defined in Proposition \ref{prop:bias-and-var-main} denotes the error of the second approximation step outlined above. 
As a direct consequence of this result, we show in Appendix \ref{sec:wide_nns_weight_estimation} that as the error $\epsilon$ becomes small (specifically as $\Ebeh[\epsilon^2]\rightarrow 0$), the difference between biases of MR and IPW estimator becomes negligible.
Likewise, the terms $\Vbeh[\epsilon\,Y]$ and $\textup{Cov}(\tilde{w}(Y)\,Y, \epsilon\,Y)$ in Eq. \eqref{eq:var-difference-approximate-weights} will also be small and as a result the variance of MR will be lower than that of IPW (as the first term is positive). 

In fact, using recent results regarding the generalisation error of neural networks \citep{lai2023generalization}, we show that when using 2-layer wide neural networks to approximate the weights $\hat{w}(y)$, the estimation error $\epsilon$ declines with increasing training data size $m$. Specifically, under certain regularity assumptions we obtain $\Ebeh[\epsilon^2] = O(m^{-2/3})$. Using this we show that as the training data size $m$ increases, the biases of MR and IPW estimators become roughly equal with a high probability, and
\[
\Vbeh[\approxipw] - \Vbeh[\approxmr] = \frac{1}{n}\,\Ebeh[\Vbeh[\hat{\rho}(A, X)\,Y\mid Y]] + O(m^{-1/3}).
\]
Therefore the variance of MR estimator falls below that of IPW for large enough $m$. The empirical results shown in Appendix \ref{subsec:mips-empirical} are consistent with this result. Due to space constraints, the main technical result has been included in Appendix \ref{sec:wide_nns_weight_estimation}.

\subsection{Application to causal inference}\label{subsec:application-to-causal-inference}
 Beyond contextual bandits, the variance reduction properties of the MR estimator make it highly useful in a wide variety of other applications. Here, we show one such application in the field of causal inference, where MR can be used for the estimation of average treatment effect (ATE) \citep{pearl2009causality} and leads to some desirable properties in comparison to the conventional ATE estimation approaches. Specifically, we illustrate that the MR estimator for ATE utilizes the evaluation data $\D$ more efficiently and achieves lower variance than state-of-the-art ATE estimators and consequently provides more accurate ATE estimates.
To be concrete, the goal in this setting is to estimate ATE, defined as follows:
\[
\ate \coloneqq \E[Y(1)-Y(0)].
\]
Here $Y(a)$ corresponds to the outcome under a deterministic policy $\pi_a(a'\mid x) \coloneqq \ind(a'=a)$. Hence any OPE estimator can be used to estimate $\E[Y(a)]$ (and therefore ATE) by considering target policy $\tar = \pi_a$.
An important distinction between MR and existing approaches (like IPW or DR) is that, when estimating $\E[Y(a)]$, the existing approaches only use datapoints in $\D$ with $A=a$. To see why this is the case, we note that the policy ratios $\frac{\tar(A|X)}{\beh(A|X)} = \frac{\ind(A=a)}{\beh(A|X)}$ are zero when $A\neq a$. In contrast, the MR weights $\frac{\ptar(Y)}{\pbeh(Y)}$ are not necessarily zero for datapoints with $A\neq a$, and therefore the MR estimator uses all evaluation datapoints when estimating $\E[Y(a)]$. 

As such we show that MR applied to ATE estimation leads to a smaller variance than the existing approaches. Moreover, because MR is able to use all datapoints when estimating $\E[Y(a)]$, MR will generally be more accurate than the existing methods especially in the setting where the data is imbalanced, i.e., the number of datapoints with $A=a$ is small for a specific action $a$.
In Appendix \ref{app:causal-inference}, we formalise this variance reduction of the MR ATE estimator compared to IPW and DR estimators, by deriving analogous results to Propositions \ref{prop:var_mr} and \ref{prop:var_dr}. In addition, we also show empirically in Section \ref{subsec:causal-experiments} that the MR ATE estimator outperforms the most commonly used ATE estimators.

%

%

%

%
%
%
%

%
%
%
%
%
%
%
%
%
%
%
%
%
%
%
%
%
%
%

\section{Related work}
Off-policy evaluation is a central problem both in contextual bandits \citep{dudik2014doubly, wang2017optimal, liu2018breaking, farajtabar2018more, su2019continuous, su2020doubly, kallus2021optimal, metelli2021subgaussian, saito2020open} and in RL \citep{thomas2016data, xie2019advances, kallus2020off, liu2020understanding}. 
Existing OPE methodologies can be broadly categorised into Direct Method (DM), Inverse Probability Weighting (IPW), and Doubly Robust (DR). 
While DM typically has a low variance, it suffers from high bias when the reward model is misspecified \citep{voloshin2021empirical}. 
On the other hand, IPW \citep{horvitz1952generalization} and DR \citep{dudik2014doubly, wang2017optimal, su2020doubly} use policy ratios as importance weights when estimating policy value and suffer from high variance as overlap between behaviour and target policies increases or as the action/context space gets larger \citep{sachdeva2020off, saito2022off}. To circumvent this problem, techniques like weight clipping or normalisation \citep{swaminathan2015counterfactual, swaminathan2015the, chaudhuri2019london} are often employed, however, these can often increase bias.

In contrast to these approaches, \cite{saito2022off} propose MIPS, which considers the marginal shift in the distribution of a lower dimensional embedding of the action space. While this approach reduces the variance associated with IPW, we show in Section \ref{subsec:mips-comparison} that the MR estimator achieves a lower variance than MIPS while not requiring any additional assumptions (like Assumption \ref{assum:indep-mips}).

In the context of Reinforcement Learning (RL), various marginalisation techniques of importance weights have been used to propose OPE methodologies.
\cite{liu2018breaking, xie2019advances, kallus2020off} use methods which considers the shift in the marginal distribution of the states, and applies importance weighting with respect to this marginal shift rather than the trajectory distribution. Similarly, \cite{Fujimoto2021deep} use marginalisation for OPE in deep RL, where the goal is to consider the shift in marginal distributions of state and action. Although marginalization is a key trick of these estimators, these techniques do not consider the marginal shift in reward as in MR and are aimed at resolving the curse of horizon, a problem specific to RL. Apart from this, \cite{rowland2020conditional} propose a general framework of OPE based on conditional expectations of importance ratios for variance reduction. While their proposed framework includes reward conditioned importance ratios, this is not the main focus and there is little theoretical and empirical comparison of their proposed methodology with existing state-of-the-art methods like DR. 

Finally we note that the idea of approximating the ratio of intractable marginal densities via leveraging the fact that this ratio can be reformulated as the conditional expectation of a ratio of tractable densities is a standard idea in computational statistics \citep{meng1996simulating} and has been exploited more recently to perform likelihood-free inference \citep{brehmer2020mining}. In particular, while  \cite{meng1996simulating} typically approximates this expectation through Markov chain Monte Carlo, \cite{brehmer2020mining} uses regression instead, however without any theory.

\section{Empirical evaluation}
In this section, we provide empirical evidence to support our theoretical results by investigating the performance of our MR estimator against the current state-of-the-art OPE methods. The code to reproduce our experiments has been made available at: \href{https://github.com/faaizT/MR-OPE}{\textcolor{blue}{github.com/faaizT/MR-OPE}}.

\subsection{Experiments on synthetic data}\label{sec:exp-synth}
For our synthetic data experiment, we reproduce the experimental setup for the synthetic data experiment in \cite{saito2022off} by reusing their code with minor modifications.
Specifically, $\Xspace \subseteq \mathbb{R}^d$, for various values of $d$ as described below. Likewise, the action space $\Aspace = \{0, \dots, n_a-1\}$, with $n_a$ taking a range of different values. Additional details regarding the reward function, behaviour policy $\beh$, and the estimation of weights $\hat{w}(y)$ have been included in Appendix \ref{subsec:mips-empirical} for completeness.

\myparagraph{Target policies} 
To investigate the effect of increasing policy shift, we define a class of policies,
\[
\pi^{\alpha^\ast}(a | x) = \alpha^\ast\,\ind(a = \arg\max_{a'\in \Aspace} q(x, a')) + \frac{1-\alpha^\ast}{|\Aspace|} \quad \textup{where} \quad q(x, a) \coloneqq \E[Y\mid X=x, A=a],
\]
where $\alpha^\ast \in [0, 1]$ allows us to control the shift between $\beh$ and $\tar$. In particular, as we show later, the shift between $\beh$ and $\tar$ increases as $\alpha^\ast \rightarrow 1$. Using the ground truth behaviour policy $\beh$, we generate a dataset which is split into training and evaluation datasets of sizes $m$ and $n$ respectively.

\myparagraph{Baselines} 
We compare our estimator with DM, IPW, DR and MIPS estimators. Our setup includes action embeddings $E$ satisfying Assumption \ref{assum:indep-mips}, and so MIPS is unbiased.
Additional baselines have been considered in Appendix \ref{subsec:mips-empirical}.
For MR, we split the training data to estimate $\hatbeh$ and $\hat{w}(y)$, whereas for all other baselines we use the entire training data to estimate $\hatbeh$ for a fair comparison.
\begin{figure}[t]
     \centering
     \begin{subfigure}[b]{0.5\textwidth}
         \centering
         \includegraphics[height=1.06in]{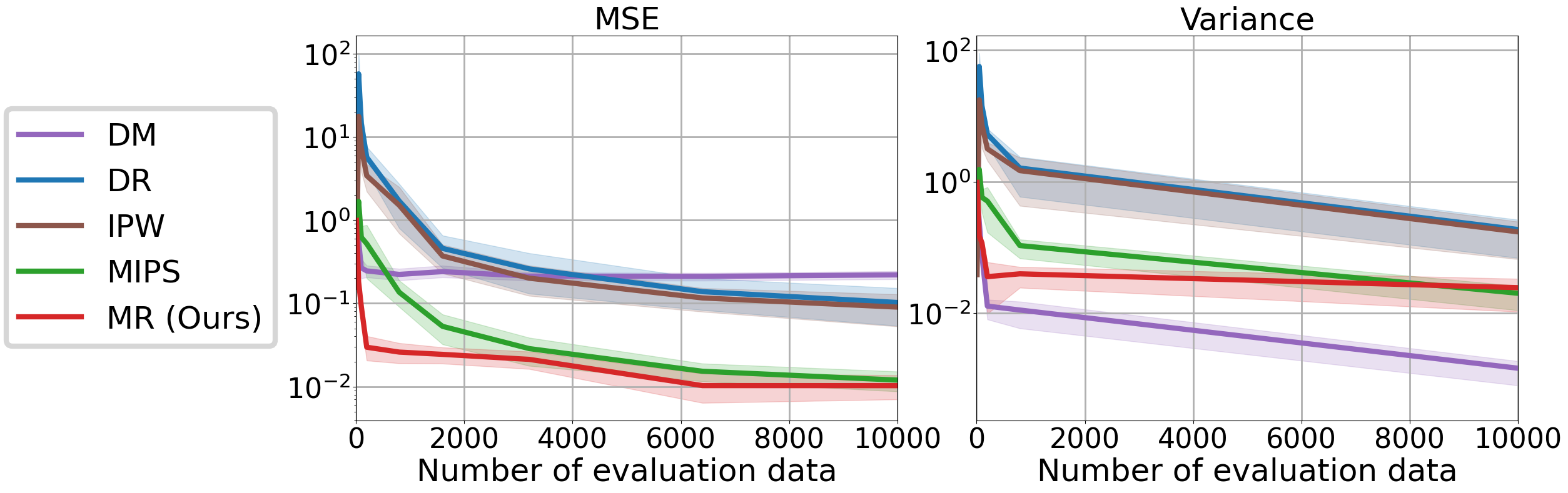}
         \caption{Results with varying evaluation data size $n$.}
         \label{fig:mse-vs-neval}
     \end{subfigure}%
     \begin{subfigure}[b]{0.5\textwidth}
         \centering
         \includegraphics[height=1.06in]{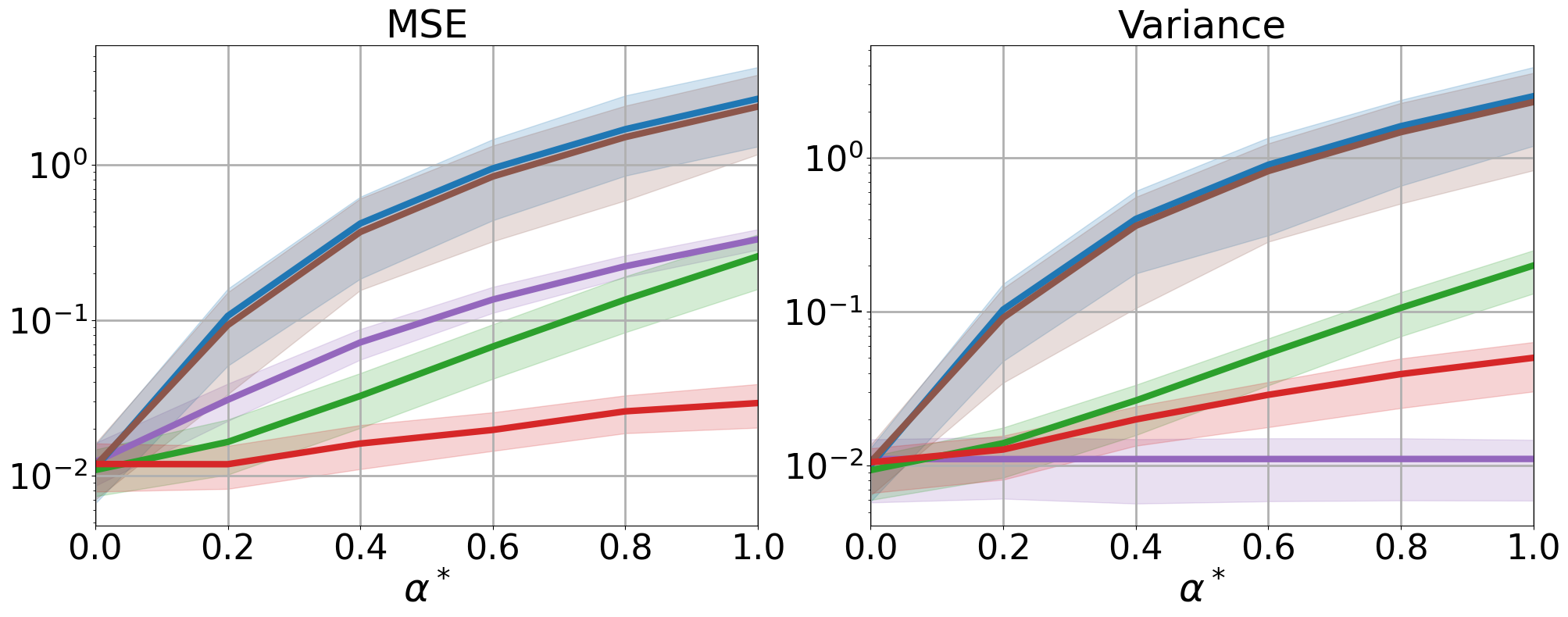}
         \caption{Results with varying $\alpha^\ast$.}
         \label{fig:mse-vs-betatar}
     \end{subfigure}\\
    \caption{Results for synthetic data experiment. In \ref{fig:mse-vs-neval} we have $\alpha^\ast=0.8$ and in \ref{fig:mse-vs-betatar} we have $n = 800$.}
    \label{fig:syn_results1}
\end{figure}

\myparagraph{Results}
We compute the target policy value using the $n$ evaluation datapoints. Here, the MSE of the estimators is computed over 10 different sets of logged data replicated with different seeds. The results presented have context dimension $d=1000$, number of actions $n_a=100$ and training data size $m=5000$. More experiments for a variety of parameter values are included in Appendix \ref{subsec:mips-empirical}.

\myparagraph{Varying number of evaluation data $n$} 
In Figure \ref{fig:mse-vs-neval} we plot the results with increasing size of evaluation data $n$ increases. MR achieves the smallest MSE among all the baselines considered when $n$ is small, with the MSE of MR being at least an order of magnitude smaller than every baseline for $n\leq 500$. This shows that MR is significantly more accurate than the baselines when the size of the evaluation data is small. As $n\rightarrow \infty$, the difference between the results for MR and MIPS decreases. However, MR attains smaller variance and MSE than MIPS generally, verifying our analysis in Section \ref{subsec:mips-comparison}.
Moreover, Figure \ref{fig:mse-vs-neval} shows that while the variance of MR is greater than that of DM, it still achieves the lowest MSE overall, owing to the high bias of DM.

\begin{wrapfigure}{r}{0.35\textwidth}
    \centering
    \includegraphics[width=0.35\textwidth]{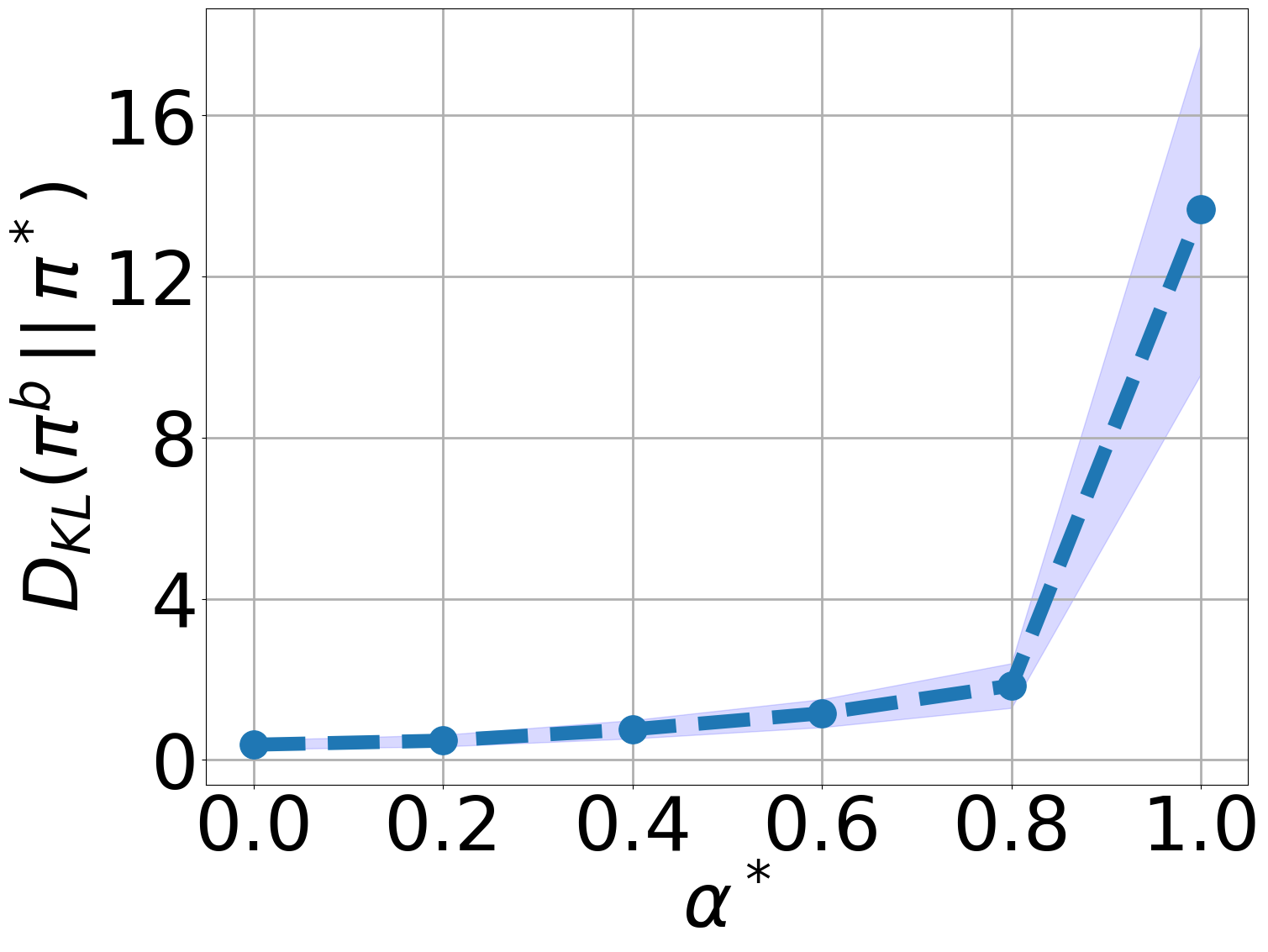}
\end{wrapfigure}
\myparagraph{Varying $\alpha^\ast$}
As $\alpha^\ast$ parameter of the target policy increases, so does the shift between the policies $\beh$ and $\pi^{\alpha^\ast}$ as illustrated by the figure on the right, which plots the KL-divergence $D_{\textup{KL}}(\beh\, || \, \pi^{\alpha^\ast})$ as a function of $\alpha$.
Figure \ref{fig:mse-vs-betatar} plots the results for increasing policy shift. 
Overall, the MSE of MR estimator is lowest among all the baselines. Moreover, while the MSE and variance of all estimators increase with increasing $\alpha^\ast$ the increase in these quantities is lower for the MR estimator than for the other baselines. Therefore, the relative performance of MR estimator improves with increasing policy shift and MR remains robust to increase in policy shift.

\myparagraph{Additional ablation studies}
In Appendix \ref{subsec:mips-empirical}, we investigate the effect of varying context dimensions $d$, number of actions $n_a$ and number of training data $m$. In every case, we observe that the MR estimator has a smaller MSE than all other baselines considered. In particular, MR remains robust to increasing $n_a$ whereas the MSE and variance of IPW and DR estimators degrade substantially when $n_a \geq 2000$. Likewise, MR outperforms the baselines even when the training data size $m$ is small.

\begin{sidewaystable}[!htp]
    \centering
    \caption{Mean squared error of target policy value with standard errors over 10 different seeds for different classification datasets. Here, number of evaluation data $n=1000$, and $\alpha^\ast=0.6$.}
    \label{tab:classification-dataset-results}
    \begin{footnotesize}
    \begin{scshape}

\begin{tabular}{llllllll}
\toprule
Dataset &             Digits &               Letter &          OptDigits &          PenDigits &           SatImage  &              Mnist & CIFAR-100\\
\midrule
DM        &  0.1508$\pm$0.0015 &    0.0886$\pm$0.0026 &  0.0485$\pm$0.0016 &   0.0520$\pm$0.0016 &  0.0208$\pm$0.0009  &  0.1109$\pm$0.0014 & 0.0020$\pm$0.0001 \\
DR        &    0.1334$\pm$0.0400 &    \red{35.085$\pm$17.768} &  0.0464$\pm$0.0061 &  0.2343$\pm$0.1404 &   0.0560$\pm$0.0395 &  0.2617$\pm$0.0139 & \red{3823.9$\pm$2023.2} \\
DRos      &  0.0847$\pm$0.0025 &    0.2363$\pm$0.0586 &  0.0384$\pm$0.0025 &  0.0138$\pm$0.0029 &  0.0078$\pm$0.0008 &  0.2151$\pm$0.0061 & 0.2628$\pm$0.1087 \\
IPW       &  0.1632$\pm$0.0462 &  \red{45.253$\pm$22.057} &   0.0844$\pm$0.0056 &  0.1342$\pm$0.0531 &    0.0900$\pm$0.0676 & 0.3359$\pm$0.0118 & \red{4116.9$\pm$2097.9}\\
SwitchDR  &  0.0982$\pm$0.0032 &    0.2387$\pm$0.0507 &  0.0557$\pm$0.0047 &   0.0342$\pm$0.0090 &  0.0136$\pm$0.0012  &   0.2750$\pm$0.0102 & 1.1644$\pm$0.8227 \\
MR (Ours) &  \textbf{0.0034$\pm$0.0001} &    \textbf{0.0018$\pm$0.0004} &  \textbf{0.0006$\pm$0.0002} &  \textbf{0.0008$\pm$0.0002} &  \textbf{0.0016$\pm$0.0003} &  \textbf{0.0121$\pm$0.0009} &  \textbf{0.0007$\pm$0.0002}\\
\bottomrule
\end{tabular}

\end{scshape}
\end{footnotesize}
\end{sidewaystable}
\subsection{Experiments on classification datasets}
Following previous works on OPE in contextual bandits \citep{dudik2014doubly, kallus2021optimal, mehrdad2018more,wang2017optimal}, we transform classification datasets into contextual bandit feedback data in this experiment.
We consider five UCI classification datasets \citep{dua2019uci} as well as Mnist \citep{deng2012mnist} and CIFAR-100 \citep{krizhevsky2009learning} datasets, each of which comprises $\{(x_i, a^\gt_i)\}_{i}$, where $x_i\in \Xspace$ are feature vectors and $a^\gt_i\in \Aspace$ are the ground-truth labels.
In the contextual bandits setup, the feature vectors $x_i$ are considered to be the contexts, whereas the actions correspond to the possible class of labels. For the context vector $x_i$ and the action $a_i$, the reward $y_i$ is defined as $y_i \coloneqq \ind(a_i = a^\gt_i)$, i.e., the reward is 1 when the action is the same as the ground truth label and 0 otherwise. Here, the baselines considered include the DM, IPW and DR estimators as well as Switch-DR \citep{wang2017optimal} and DR with Optimistic Shrinkage (DRos) \citep{su2020doubly}. We do not consider a MIPS baseline here as there is no natural embedding $E$ of $A$. Additional details are provided in Appendix \ref{subsec:additional-experiments-classification}. 

In Table \ref{tab:classification-dataset-results}, we present the results with number of evaluation data $n=1000$ and number of training data $m=500$. 
The table shows that across all datasets, MR achieves the lowest MSE among all methods. \flag{Moreover, for the Letter and CIFAR-100 datasets the IPW and DR yield large bias and variance arising from poor policy estimates $\hatbeh$. Despite this, the MR estimator which utilizes the \emph{same} $\hatbeh$ for the estimation of $\hat{w}(y)$ leads to much more accurate results.} We also verify that MR outperforms the baselines for increasing policy shift and evaluation data $n$ in Appendix \ref{subsec:additional-experiments-classification}.

\subsection{Application to ATE estimation}\label{subsec:causal-experiments}
In this experiment, we investigate the empirical performance of the MR estimator for ATE estimation. 

\myparagraph{Twins dataset}
We use the Twins dataset studied in \cite{louizos2017causal}, which comprises data from twin births in the USA between 1989-1991. The treatment $a=1$ corresponds to being born the heavier twin and the outcome $Y$ corresponds to the mortality of each of the twins in their first year of life. Specifically, $Y(1)$ corresponds to the mortality of the heavier twin (and likewise for $Y(0)$). To simulate the observational study, we follow a similar strategy as in \cite{louizos2017causal} to selectively hide one of the two twins as explained in Appendix \ref{app:ate-empirical}. We obtain a total of 11,984 datapoints, of which 5000 datapoints are used to train the behaviour policy $\hatbeh$ and outcome model $\hat{q}(x, a)$.

\begin{table}[t]
    \centering
    \caption{Mean absolute ATE estimation error $\epsilon_\ate$ with standard errors over 10 different seeds, for increasing number of evaluation data $n$.}
    \label{tab:ate_errors-main}
    \begin{small}
    \begin{tabular}{lllll}
\toprule
$n$ &             50   &             200  &             1600 &             3200 \\
\midrule
DM       &  0.092$\pm$0.003 &  0.092$\pm$0.003 &  0.092$\pm$0.004 &  0.092$\pm$0.004 \\
DR       &  0.101$\pm$0.024 &  \textbf{0.065$\pm$0.009} &  0.071$\pm$0.005 &  0.069$\pm$0.004 \\
\textsc{DRos}     &    0.100$\pm$0.017 &  0.089$\pm$0.006 &   0.093$\pm$0.004 &  0.087$\pm$0.004 \\
IPW      &  0.092$\pm$0.024 &  0.088$\pm$0.014 &  0.067$\pm$0.007 &  0.067$\pm$0.007 \\
\textsc{SwitchDR} &  0.101$\pm$0.024 &  \textbf{0.065$\pm$0.009} &  0.071$\pm$0.005 &  0.069$\pm$0.004 \\
MR (Ours)      &  \textbf{0.062$\pm$0.007} &  \textbf{0.065$\pm$0.007} &  \textbf{0.061$\pm$0.005} &  \textbf{0.061$\pm$0.006} \\
\bottomrule
\end{tabular}
\end{small}
\end{table}
Here, we consider the same baselines as the classification data experiments in previous section.
For our evaluation, we consider the absolute error in ATE estimation, $\epsilon_\ate$, defined as:
$
\epsilon_\ate \coloneqq | \hat{\theta}^{(n)}_\ate - \theta_\ate |.
$
Here, $\hat{\theta}^{(n)}_\ate$ denotes the value of the ATE estimated using $n$ evaluation datapoints.
We compute the ATE value using the $n$ evaluation datapoints, over 10 different sets of observational data (using different seeds). Table \ref{tab:ate_errors-main} shows that MR achieves the lowest estimation error $\epsilon_\ate$ for all values of $n$ considered here. While the performance of other baselines improves with increasing $n$, MR outperforms them all. 

\section{Discussion}

In this paper, we proposed an OPE method for contextual bandits called marginal ratio (MR) estimator, which considers only the shift in the marginal distribution of the outcomes resulting from the policy shift. Our theoretical and empirical analysis showed that MR achieves better variance and MSE compared to the current state-of-the-art methods and is more data efficient overall. Additionally, we demonstrated that MR applied to ATE estimation provides more accurate results than most commonly used methods. Next, we discuss limitations of our methodology and possible avenues for future work.

\myparagraph{Limitations}
The MR estimator requires the additional step of estimating $\hat{w}(y)$ which may introduce an additional source of bias in the value estimation. However, $\hat{w}(y)$ can be estimated by solving a simple 1d regression problem, and as we show empirically in Appendix \ref{app:experiments}, MR achieves the smallest bias among all baselines considered in most cases. Most notably, our ablation study in Appendix \ref{subsec:mips-empirical} shows that even when the training data is reasonably small, MR outperforms the baselines considered. 

\myparagraph{Future work}
The MR estimator can also be applied to policy optimisation problems, where the data collected using an `old' policy is used to learn a new policy. This approach has been used in Proximal Policy Optimisation (PPO) \citep{schulman2017proximal} for example, which has gained immense popularity and has been applied to reinforcement learning with human feedback (RLHF) \citep{lambert2022illustrating}. We believe that the MR estimator applied to these methodologies could lead to improvements in the stability and convergence of these optimisation schemes, given its favourable variance properties.

        \section*{Acknowledgements}
We would like to thank Jake Fawkes, Siu Lun Chau, Shahine Bouabid and Robert Hu for their useful feedback. 
We also appreciate the insights and constructive criticisms provided by the anonymous reviewers.
MFT acknowledges his PhD funding from Google DeepMind.

    }
    { %
    }

    \includepdf[pages=-]{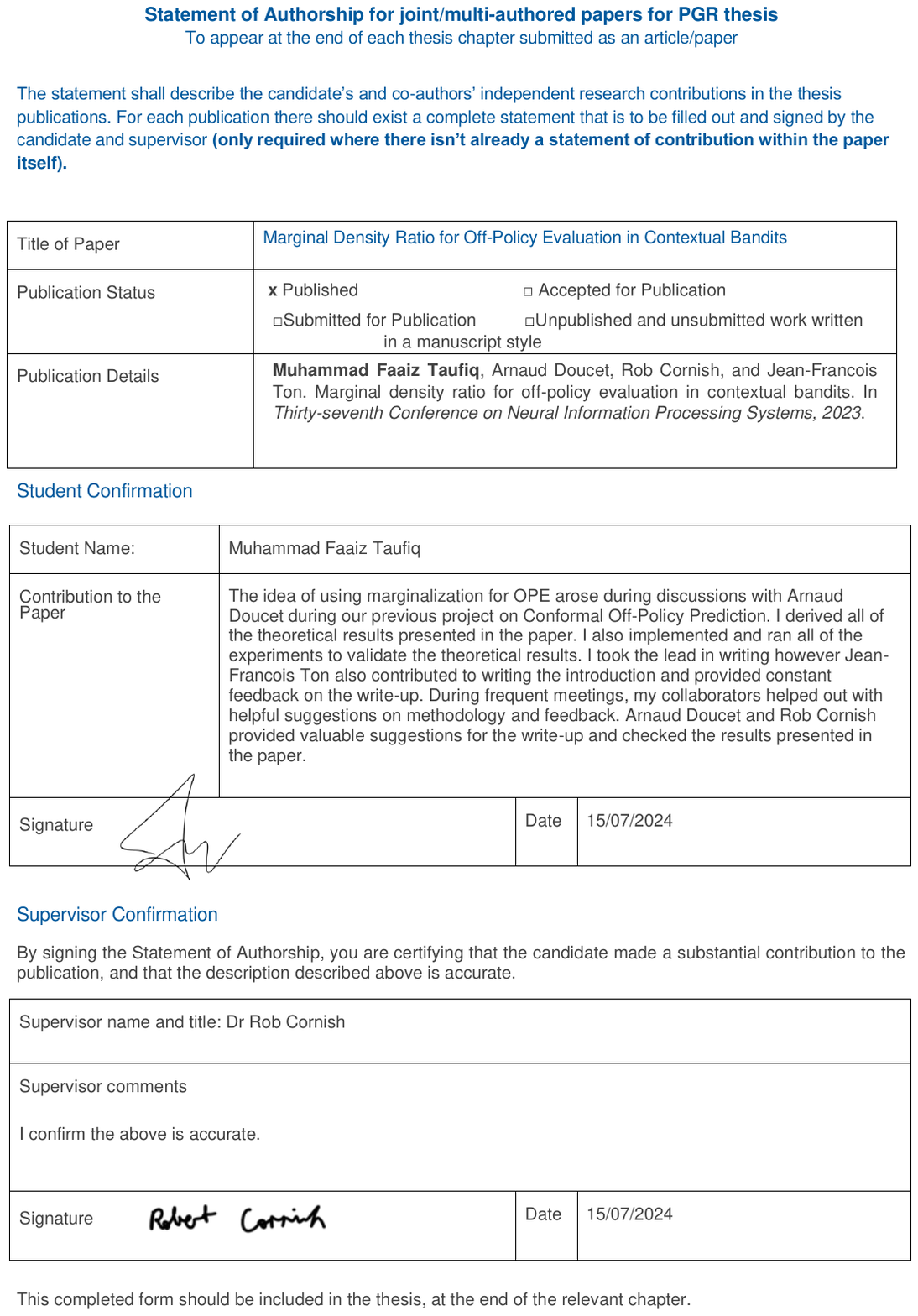}

\clearpage

\chapter{\label{ch:4-copp}Conformal Off-Policy Prediction in Contextual Bandits} 

\minitoc

\ifthenelse{\boolean{compilePapers}}
    { %

\begin{abstract}
  Most off-policy evaluation methods for contextual bandits have focused on the expected outcome of a policy, which is estimated via methods that at best provide only asymptotic guarantees. However, in many applications, the expectation may not be the best measure of performance as it does not capture the variability of the outcome. In addition, particularly in safety-critical settings, stronger guarantees than asymptotic correctness may be required. To address these limitations, we consider a novel application of conformal prediction to contextual bandits. Given data collected under a behavioral policy, we propose \emph{conformal off-policy prediction} (COPP), which can output reliable predictive intervals for the outcome under a new target policy. We provide theoretical finite-sample guarantees without making any additional assumptions beyond the standard contextual bandit setup, and empirically demonstrate the utility of COPP compared with existing methods on synthetic and real-world data. 
\end{abstract}

\section{Introduction}

Before deploying a decision-making policy to production, it is usually important to understand the plausible range of outcomes that it may produce.
However, due to resource or ethical constraints, it is often not possible to obtain this understanding by testing the policy directly in the real-world.
In such cases we have to rely on observational data collected under a different policy than the target.
Using this observational data to evaluate the target policy is known as off-policy evaluation (OPE).%

Traditionally, most techniques for OPE in contextual bandits focus on evaluating policies based on their \textbf{expected} outcomes; see e.g., \cite{uncertainty5, adaptive-ope, uncertainty2, uncertainty3, uncertainty4, doubly-robust}.
However, this can be problematic as methods that are only concerned with the average outcome do not take into account any notions of variance, for example. Therefore, in risk-sensitive settings such as econometrics, where we want to minimize the potential risks, metrics such as CVaR (Conditional Value at Risk) might be more appropriate \citep{keramati2020being}. Additionally, when only small sample sizes of observational data are available, the average outcomes under finite data can be misleading, as they are prone to outliers and hence, metrics such as medians or quantiles are more robust in these scenarios \citep{altschuler2019best}.

\begin{figure}
     \centering
     \begin{subfigure}[t]{0.5\textwidth}
         \centering
         \includegraphics[height=1.85in]{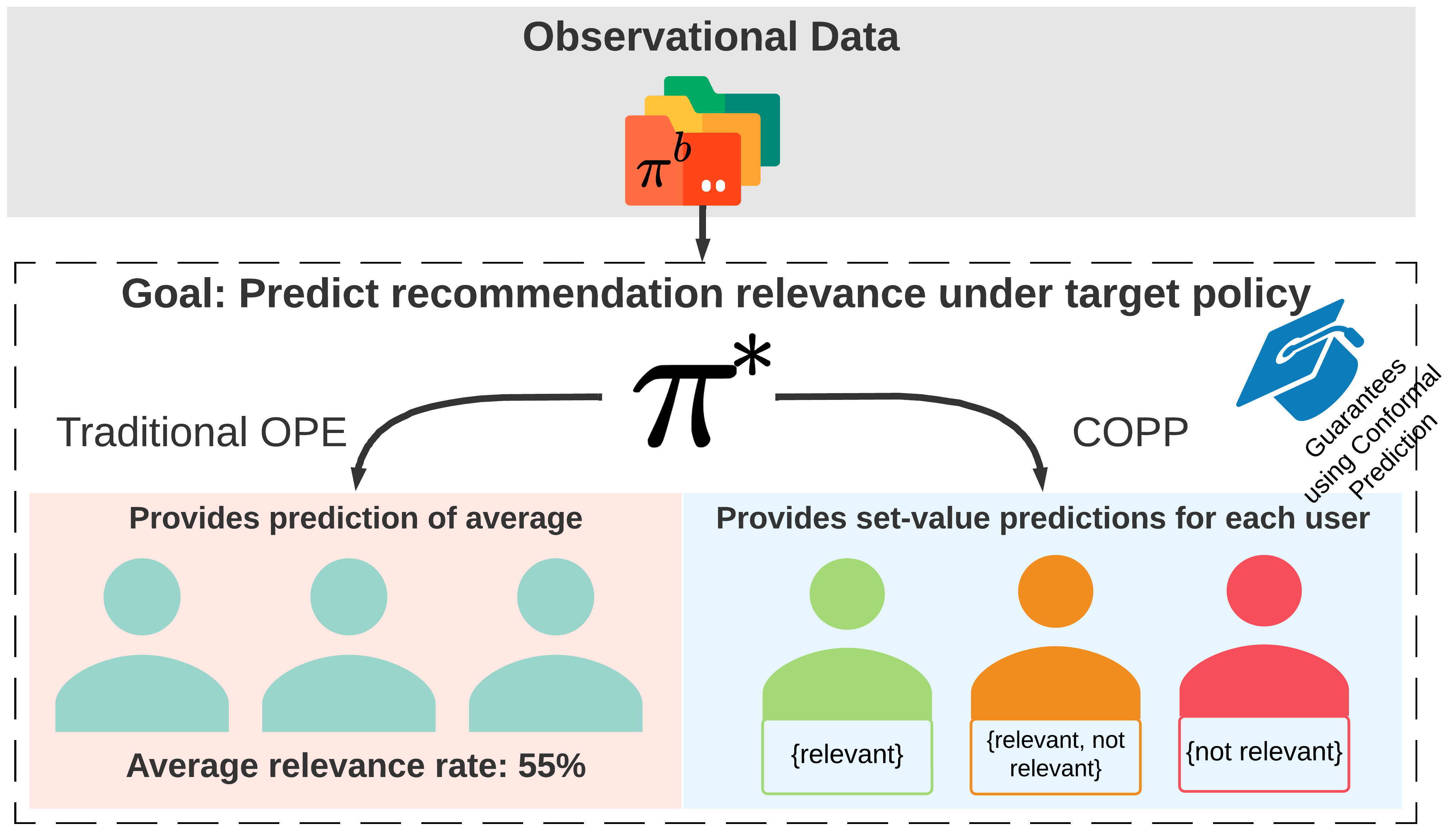}
     \end{subfigure}%
     \begin{subfigure}[t]{0.5\textwidth}
         \centering
         \includegraphics[height=1.85in]{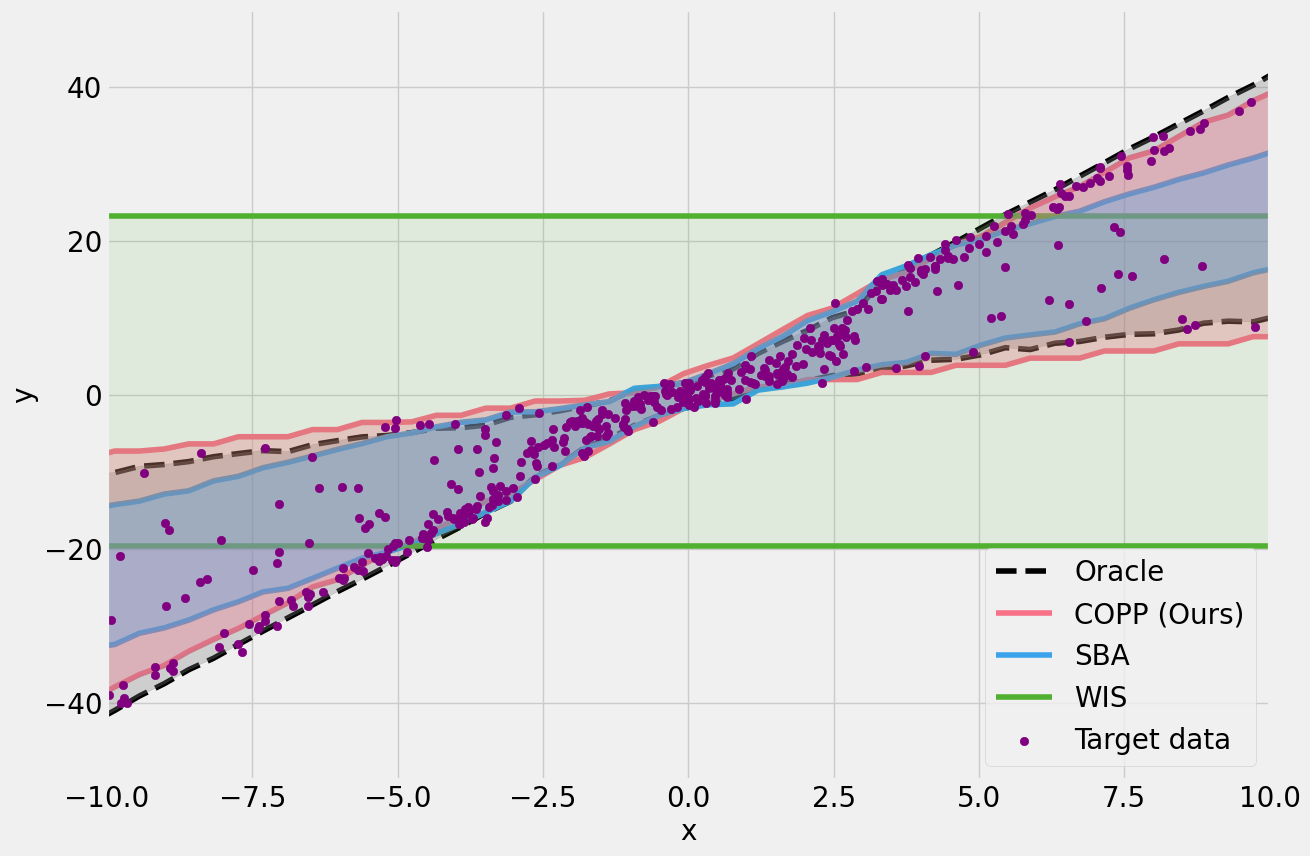}
     \end{subfigure}

    \caption{\textbf{Left (a):} Conformal Off-Policy Prediction against standard off-policy evaluation methods. \textbf{Right (b):} $90\%$ predictive intervals for $Y$ against $X$ for COPP, competing methods and the oracle.}\label{fig:copp}
\end{figure}

Notable exceptions in the OPE literature are \cite{risk-assessment, chandak2021universal}. Instead of estimating bounds on the expected outcomes, \cite{risk-assessment, chandak2021universal} establish finite-sample bounds for a general class of metrics (e.g., Mean, CVaR, CDF) on the outcome. Their methods can be used to estimate quantiles of the outcomes under the target policy and are therefore robust to outliers. However, the resulting bounds do not depend on the covariates $X$ (not adaptive w.r.t. $X$). This can lead to overly conservative intervals, as we will show in our experiments and can become uninformative when the data are heteroscedastic (see Fig. \ref{fig:copp}b).

In this paper, we propose Conformal Off-Policy Prediction (COPP), a novel algorithm that uses Conformal Prediction (CP) \citep{vovk2005algorithmic} to construct predictive interval/sets for outcomes in contextual bandits (see Fig.\ref{fig:copp}a) using an observational dataset.
COPP enjoys both finite-sample theoretical guarantees and adaptivity w.r.t.\ the covariates $X$, and, to the best of our knowledge, is the first such method based on CP that can be applied to stochastic policies and continuous action spaces.
In summary, our contributions are: 
\begin{enumerate}[label=(\roman*)]
  \item We propose an application of CP to construct predictive intervals for bandit outcomes that is more general (applies to stochastic policies and continuous actions) than previous work.
  \item We provide theoretical guarantees for COPP, including finite-sample guarantees on marginal coverage and asymptotic guarantees on conditional coverage.
  \item We show empirically that COPP outperforms standard methods in terms of coverage and predictive interval width when assessing new policies. 
\end{enumerate}

\subsection{Problem setup}\label{sec:problem_setup}

Let $\mathcal{X}$ be the covariate space (e.g., user  data), $\mathcal{A}$ the action space (e.g., recommended items) and $\mathcal{Y}$ the outcome space (e.g., relevance to the user).
We allow both $\mathcal{A}$ and $\mathcal{Y}$ to be either discrete or continuous.
In our setting, we are given logged observational data $\mathcal{D}_{obs}=\{x_i, a_i, y_i \}_{i=1}^{n_{obs}}$ where actions are sampled from a behavioural policy $\pi^{b}$, i.e. $A \mid x \sim \pi^{b}(\cdot\mid x)$ and $Y \mid x,a \sim P(\cdot \mid x, a)$. We assume that we do not suffer from unmeasured confounding. At test time, we are given a state $x^{test}$ and a new policy $\pi^*$. While $\pi^{b}$ may be unknown, we assume the target policy $\pi^*$ to be known.

We consider the task of rigorously quantifying the performance of $\pi^*$ without any distributional assumptions on $X$ or $Y$. Many existing approaches estimate $\mathbb{E}_{\pi^*}[Y]$, which is useful for comparing two policies directly as they return a single number. However, the expectation does not convey fine-grained information about how the policy performs for a specific value of $X$, nor does it account for the uncertainty in the outcome $Y$.

Here, we aim to construct intervals/sets on the outcome $Y$ which are (i) adaptive w.r.t. $X$, (ii) capture the variability in the outcome $Y$ and (iii) provide finite-sample guarantees. Current methods lack at least one of these properties (see Sec. \ref{sec:related_work}). One way to achieve these properties is to construct a set-valued function of $x$, $\hat{C}(x)$ which outputs a \emph{subset} of $\mathcal{Y}$. Given any finite dataset $\mathcal{D}_{obs}$, this subset is guaranteed to contain the true value of $Y$ with any pre-specified probability, i.e.
\begin{align}
     \hspace{-0.24cm}1- \alpha \hspace{-0.05cm}\leq  \tarprob(Y \in \hat{C}(X)) \hspace{-0.05cm}\leq 1- \alpha + o_{n_{obs}}(1) \label{guarantee}
\end{align}
where $n_{obs}$ is the size of available observational data and $P^{\pi^*}_{X,Y}$ is the joint distribution of $(X,Y)$ under target policy $\pi^*$. In practice, $\hat{C}(x)$ can be used as a diagnostic tool downstream for a granular assessment of likely outcomes under a target policy. The probability in (\ref{guarantee}) is taken over the joint distribution of $(X, Y)$, meaning that \eqref{guarantee} holds marginally in $X$ (marginal coverage) and not for a given $X=x$ (conditional coverage). In Sec. \ref{sec:cond_cov}, we provide additional regularity conditions under which not only marginal but also conditional coverage holds. Next, we introduce the Conformal Prediction framework, which allows us to construct intervals $\hat{C}(x)$ that satisfy \eqref{guarantee} along with properties (i)-(iii). 

\section{Background}
Conformal prediction \citep{vovk2005algorithmic, shafer2008tutorial} is a methodology that was originally used to compute distribution-free prediction sets for regression and classification tasks. Before introducing COPP, which applies CP to contextual bandits, we first illustrate how CP can be used in standard regression.

\subsection{Standard conformal prediction} 

Consider the problem of regressing $\mbox{Y} \in \mathcal{Y}$ against $X\in \mathcal{X}$.
Let $\hat{f}$ be a model trained on the \emph{training} data $\mathcal{D}_{tr} = \{X_i^0, Y_i^0\}_{i=1}^m \overset{\textup{i.i.d.}}{\sim} P_{X,Y}$ and let the \emph{calibration} data $\mathcal{D}_{cal} = \{X_i, Y_i\}_{i=1}^n \overset{\textup{i.i.d.}}{\sim} P_{X,Y}$ be independent of $\mathcal{D}_{tr}$. Given a desired coverage rate $1-\alpha \in (0,1)$, we construct a band $\hat{C}_n:\mathcal{X}\rightarrow \{\text{subsets of }\mathcal{Y}\}$, based on the calibration data such that, for a new i.i.d. test data $(X,Y) \sim P_{X,Y}$,
\begin{align}
    1-\alpha \leq \mathbb{P}_{(X,Y)\sim P_{X,Y}}(Y\in \hat{C}_n(X)) \leq 1-\alpha + \frac{1}{n+1}, \label{cp_guarantee}
\end{align}
where the probability is taken over $X,Y$ and $\mathcal{D}_{cal} = \{X_i, Y_i\}_{i=1}^n$ and is conditional upon $\mathcal{D}_{tr}$.

In order to obtain $\hat{C}_n$ satisfying \eqref{cp_guarantee}, we introduce a non-conformity score function $V_i = s(X_i, Y_i)$, e.g., $(\hat{f}(X_i) - Y_i)^2$. We assume here $\{V_i\}_{i=1}^n$ have no ties almost surely. Intuitively, the non-conformity score $V_i$ uses the outputs of the predictive model $\hat{f}$ on the calibration data, to measure how far off these predictions are from the ground truth response. Higher scores correspond to worse fit between $x$ and $y$ according to $\hat{f}$. We define the empirical distribution of the scores $\{V_i\}_{i=1}^n \cup \{\infty\}$
\begin{align}\label{eq:std_emp_score}
 \hat{F}_{n} \coloneqq \frac{1}{n+1} \sum_{i=1}^n \delta_{V_i} + \frac{1}{n+1}\delta_{\infty}  
\end{align}
with which we can subsequently construct the conformal interval $\hat{C}_n$ that satisfies \eqref{cp_guarantee} as follows:
\begin{align}
    \hat{C}_n(x) \coloneqq \{y: s(x,y) \leq \eta\} \label{eq:interval}
\end{align}
where $\eta$ is an empirical quantile of $\{V_i\}_{i=1}^n$, i.e. $\eta = \text{Quantile}_{1-\alpha}(\hat{F}_{n})$ is the $1-\alpha$ quantile.

Intuitively, for roughly $100\cdot(1-\alpha) \%$ of the calibration data, the score values will be below $\eta$. Therefore, if the new datapoint $(X, Y)$ and $\mathcal{D}_{cal}$ are i.i.d., the probability $\p(s(X,Y) \leq \eta)$ (which is equal to $\p(Y \in \hat{C}_n(X))$ by \eqref{eq:interval}) will be roughly $1-\alpha$. Exchangeability of the data is crucial for the above to hold. In the next section we will explain how \cite{tibshirani2020conformal} relax the exchangeability assumption.

\subsection{Conformal prediction under covariate shift}\label{CP_cov_shift}
\cite{tibshirani2020conformal} extend the CP framework beyond the setting of exchangeable data, by constructing valid intervals even when the calibration data and test data are not drawn from the same distribution. The authors focus on the \textit{covariate shift} scenario i.e. the distribution of the covariates changes at test time:
\begin{align}
    &(X_i, Y_i) \overset{\textup{i.i.d}}{\sim} P_{X,Y} = P_X \times P_{Y\mid X}, \quad i = 1, \dots, n \nonumber \\
    &(X, Y) \sim \tilde{P}_{X,Y} = \tilde{P}_{X} \times P_{Y\mid X},~\text{independently}\nonumber
\end{align}

where the ratio $w(x)\coloneqq\mathrm{d}\tilde{P}_{X}/\mathrm{d}P_{X}(x)$ is known.
The key realization in \cite{tibshirani2020conformal} is that the requirement of \textit{exchangeability} in CP can be relaxed to a more general property, namely \textit{weighted exchangeability} (see Def. \ref{def:weighted_exch}). 
They propose a weighted version of conformal prediction, which shifts the empirical distribution of non-conformity scores, $\hat{F}_{n}$, at a point $x$, using weights $w(x)$. This adjusts $\hat{F}_{n}$ for the covariate shift, before picking the quantile $\eta$: $$\hat{F}_{n}^{x} \coloneqq  \sum_{i=1}^n p_i^w(x) \delta_{V_i} + p_{n+1}^w(x)\delta_{\infty}\quad \textup{ where,}$$ 
\begin{align*}
    p_i^{w}(x) = \frac{w(X_i)}{\sum_{j=1}^n w(X_j) + w(x)}, \quad p_{n+1}^{w}(x) = \frac{w(x)}{\sum_{j=1}^n w(X_j) + w(x)}.
\end{align*}

In standard CP (without covariate shift), the weight function satisfies $w(x)=1$ for all $x$, and we recover \eqref{eq:std_emp_score}. Next, we construct the conformal prediction intervals $\hat{C}_n$ as in standard CP using \eqref{eq:interval} where $\eta$ now depends on $x$ due to $p^w_i(x)$. The resulting intervals, $\hat{C}_n$, satisfy: 
\begin{align*}
     \mathbb{P}_{(X,Y)\sim \tilde{P}_{X, Y}}(Y\in \hat{C}_n(X)) \geq 1-\alpha    
\end{align*}
As mentioned previously in Sec. \ref{sec:problem_setup}, the above demonstrates marginal coverage guarantees over test point $X$ and calibration dataset $\mathcal{D}_{cal}$, not conditional on a given $X=x$ or a fixed $\mathcal{D}_{cal}$.  We will discuss this nuance later on in Sec. \ref{sec:cond_cov}. In addition, previous work by \citeauthor{vovk2012} shows that conditioned on a single calibration dataset, standard CP can achieve coverage that is `close' to the required coverage with high probability. However, this has not been extended to the case where the distribution shifts. This is out of the scope of this paper and an interesting future direction.

\begin{algorithm}[!htp]
\SetAlgoLined
\textbf{Inputs:} Observational data $\mathcal{D}_{obs}=\{X_i, A_i, Y_i\}_{i=1}^{n_{obs}}$, conf. level $\alpha$, a score function $s(x,y)\in\mathbb{R}$, new data point $x^{test}$, target policy $\pi^*$ \;
\textbf{Output:} Predictive interval $\hat{C}_n(x^{test})$\;
Split $\mathcal{D}_{obs}$ into training data ($\mathcal{D}_{tr}$) and calibration data ($\mathcal{D}_{cal}$) of sizes $m$ and $n$ respectively\;
Use $\mathcal{D}_{tr}$ to estimate weights $\hat{w}(\cdot, \cdot)$ using \eqref{weight-est}\;
Compute $V_i \coloneqq s(X_i, Y_i)$ for $(X_i, A_i, Y_i) \in \mathcal{D}_{cal}$\;
Let $\hat{F}_{n}^{x, y}$ be the weighted distribution of scores 
$$\hat{F}_{n}^{x, y} \coloneqq  \sum_{i=1}^n p_i^{\hat{w}}(x, y) \delta_{V_i} + p_{n+1}^{\hat{w}}(x, y)\delta_{\infty}$$\\
where $p_i^{\hat{w}}(x, y) = \frac{\hat{w}(X_i, Y_i)}{\sum_{j=1}^n \hat{w}(X_j, Y_j) + \hat{w}(x, y)}$ and $p_{n+1}^{\hat{w}}(x, y) = \frac{\hat{w}(x, y)}{\sum_{j=1}^n \hat{w}(X_j, Y_j) + \hat{w}(x, y)}$\;
For $x^{test}$ construct:
$
    \hat{C}_n(x^{test})\hspace{-0.1cm} \coloneqq \{y: s(x^{test},y) \leq \text{Quantile}_{1-\alpha}(\hat{F}_{n}^{x^{test}, y})\} \nonumber
$

\textbf{Return} $\hat{C}_n(x^{test})$
  \caption{Conformal Off-Policy Prediction (COPP)}
  \label{cp_covariate_shift}
\end{algorithm}

Thus \cite{tibshirani2020conformal} show that the CP algorithm can be extended to the setting of covariate shift with the resulting predictive intervals satisfying the coverage guarantees when the weights are known. The extension of these results to approximate weights was proposed in  \cite{lei2020conformal} and is generalized to our setting in Sec. \ref{sec:theory}. 
\section{Conformal Off-Policy Prediction (COPP)}
In the contextual bandits introduced in Sec. \ref{sec:problem_setup}, we assume that the observational data $\mathcal{D}_{obs} = \{x_i, a_i, y_i\}_{i=1}^{n_{obs}}$ is generated from a behavioural policy $\pi^b$. At inference time we are given a new target policy $\pi^*$ and want to provide intervals on the outcomes $Y$ for covariates $X$ that satisfy \eqref{guarantee}.

The key insight of our approach is to consider the following joint distribution of $(X,Y)$:
\begin{align*}
    P^{\pi^{b}}(x, y)=& P(x) \int P(y| x, a) \textcolor{red}{\pi^{b}(a|x)} \mathrm{d}a  =P(x) \textcolor{red}{P^{\pi^{b}}(y|x)} \\
    P^{\pi^*}(x, y) =& P(x)\int P(y| x, a) \textcolor{red}{\pi^*(a|x)}  \mathrm{d}a = P(x)\textcolor{red}{P^{\pi^*}(y|x)}
\end{align*}

Therefore, the change of policies from $\pi^b$ to $\pi^*$ causes a shift in the joint distributions of $(X, Y)$ from $P^{\pi^{b}}_{X, Y}$ to $P^{\pi^*}_{X, Y}$. More precisely, a shift in the conditional distribution of $Y|X$. As a result, our problem boils down to using CP in the setting where the conditional distribution $P^{\pi^{b}}_{Y\mid X}$ changes to $P^{\pi^{*}}_{Y \mid X}$ due to the different policies, while the covariate distribution $P_X$ remains the same. 

Hence our problem is not concerned about covariate shift as addressed in \cite{tibshirani2020conformal}, but instead uses the idea of \textit{weighted exchangeability} to extend CP to the setting of policy shift. To account for this distributional mismatch, our method shifts the empirical distribution of non-conformity scores at a point $(x, y)$ using the weights $w(x,y) = \mathrm{d}P^{\pi^{*}}_{X,Y}/\mathrm{d}P^{\pi^{b}}_{X,Y}(x,y) = \mathrm{d}P^{\pi^{*}}_{Y|X}/\mathrm{d}P^{\pi^{b}}_{Y|X}(x,y)$ as follows:
\begin{align}
   \textstyle  \hat{F}_{n}^{x, y} &\coloneqq \sum_{i=1}^n p_i^w(x, y)\delta_{V_i} + p_{n+1}^w(x,y)\delta_\infty, \label{score-dist-pshift}
\end{align}
where,
\begin{align*}
p_i^{w}(x, y) &= \frac{w(X_i, Y_i)}{\sum_{j=1}^n w(X_j, Y_j) + w(x, y)} \quad \textup{and,} \\
p_{n+1}^{w}(x, y) &= \frac{w(x, y)}{\sum_{j=1}^n w(X_j, Y_j) + w(x, y)}. 
\end{align*}

The intervals are then constructed as below which we call Conformal Off-Policy Prediction (see Algorithm \ref{cp_covariate_shift}).
\begin{align}
    \hat{C}_n(x^{test}) \coloneqq \{y: s(x^{test},y) \leq \eta(x^{test}, y)\} \hspace{0.2cm} \textup{where, }  \eta(x, y) \coloneqq \text{Quantile}_{1-\alpha}( \hat{F}_{n}^{x, y}). \label{cp-sets}
\end{align}

\paragraph{Remark}
    The weights $w(x, y)$ in \eqref{score-dist-pshift} depend on $x$ and $y$, as opposed to only $x$. In particular, finding the set of $y$'s satisfying \eqref{cp-sets} becomes more complicated than for the standard covariate shifted CP which only requires a single computation of $\eta(x)$ for a given $x$ as shown in \eqref{eq:interval}. In our case however, we have to create a $k$ sized grid of potential values of $y$ for every $x$ to find $\hat{C}_n(x)$. This operation is embarrassingly parallel and hence does not add much computational overhead compared to the standard CP, especially because CP mainly focuses on scalar predictions.

\subsection{Estimation of weights $w(x, y)$}\label{sec:weights}
So far we have been assuming that we know the weights $w(x, y)$ exactly. However, in most real-world settings, this will not be the case. Therefore, we must resort to estimating $w(x, y)$ using observational data. In order to do so, we first split the observational data into training ($\mathcal{D}_{tr}$) and calibration ($\mathcal{D}_{cal}$) data. Next, using $\mathcal{D}_{tr}$, we estimate $\hat{\pi}^b(a\mid x) \approx \pi^b(a \mid x)$ and $\hat{P}(y \mid x, a) \approx P(y \mid x, a)$ (which is independent of the policy). We then compute a Monte Carlo estimate of weights using the following:
\begin{align}
    \hat{w}(x, y) &= \frac{\tfrac{1}{h}\sum_{k=1}^{h} \hat{P}(y|x, A^*_k)}{\tfrac{1}{h} \sum_{k=1}^{h} \hat{P}(y|x, A_k)} \approx \frac{\int P(y|x, a) \textcolor{red}{\pi^*(a|x)} \mathrm{d}a}{\int P(y| x, a) \textcolor{red}{\pi^b(a|x)} \mathrm{d}a},  \label{weight-est}
\end{align}
where $A_k\sim \hat{\pi}^b(\cdot \mid x),~ A_k^* \sim  \pi^*(\cdot \mid x)$ and $h$ is the number of Monte Carlo samples.

\begin{importantresultwithtitle}[title=Why not construct intervals using \text{$\hat{P}(y|x, a)$} directly?]\noindent
    We could directly construct predictive intervals $\hat{C}_n(x)$ over outcomes by sampling $$Y_j \overset{\textup{i.i.d.}}{\sim} \hat{P}^{\pi^*}(y|x) = \int \hat{P}(y|x, a)\pi^*(a|x)\mathrm{d}a.$$ However, the coverage of these intervals directly depends on the estimation error of $\hat{P}(y|x, a)$. This is not the case in COPP, as the coverage does not depend on $\hat{P}(y|x, a)$ directly but rather on the estimation of $\hat{w}(x, y)$ (see Prop. \ref{prop2}). We hypothesize that this indirect dependence of COPP on $\hat{P}(y|x, a)$ makes it less sensitive to the estimation error. In Sec. \ref{sec:exp}, our empirical results support this hypothesis as COPP provides more accurate coverage than directly using $\hat{P}(y|x, a)$ to construct intervals. Lastly, in Appendix \ref{sec:alternate_weights_est} we show how we can avoid estimating $\hat{P}(y|x, a)$ by proposing an alternative method for estimating the weights directly. We leave this for future work.
\end{importantresultwithtitle}

\section{Theoretical guarantees}\label{sec:theory}
\subsection{Marginal coverage}

In this section we provide theoretical guarantees on marginal coverage $\tarprob(Y \in \hat{C}_n(X))$ for the cases where the weights $w(x, y)$ are known exactly as well as when they are estimated. Using the idea of \textit{weighted exchangeability}, we extend \cite[Theorem 2]{tibshirani2020conformal} to our setting. 

\begin{proposition}\label{coverage_theorem}
Let $\{X_i, Y_i\}_{i =1}^n \overset{\textup{i.i.d.}}{\sim}P^{\pi^b}_{X,Y}$ be the calibration data. For any score function $s$, and any $\alpha \in (0,1)$, define the conformal predictive interval at a point $x\in \mathbb{R}^d$ as 
$$\hat{C}_n(x) \coloneqq \left\{y \in \mathbb{R}: s(x, y) \leq \eta(x,y) \right\}$$
where $\eta(x, y) \coloneqq \text{Quantile}_{1-\alpha}( \hat{F}_{n}^{x, y})$, and $\hat{F}_{n}^{x, y}$ is as defined in \eqref{score-dist-pshift} with exact weights $w(x,y)$.
If $P^{\pi^*}(y| x)$ is absolutely continuous w.r.t. $P^{\pi^b}(y| x)$,
then $\hat{C}_{n}$ satisfies
$$\tarprob(Y \in \hat{C}_{n}(X)) \geq 1-\alpha \nonumber.$$
\end{proposition}
%
%
%
%
%
%
%
%
%
%
%
%
%
%
%
%
%
%
%
%
%
%
%
%
%
%
%
%
%
%
%
%
%
%
%
%
%
%
Proposition \ref{coverage_theorem} assumes  exact weights $w(x, y)$, which is usually not the case. For CP under covariate shift, \cite{lei2020conformal} showed that even when the weights are approximated, i.e., $\hat{w}(x, y) \neq w(x, y)$, we can still provide finite-sample upper and lower bounds on the coverage, albeit with an error term $\Delta_w$ (defined in \eqref{delta_w}). Next, we extend this result to our setting when the weight function $w(x, y)$ is approximated as in Section \ref{sec:weights}.

\begin{proposition}\label{prop2}
Let $\hat{C}_n$ be the conformal predictive intervals obtained as in Proposition \ref{coverage_theorem}, with weights $w(x,y)$ replaced by approximate weights $\hat{w}(x,y) = \hat{w}(x,y;\mathcal{D}_{tr})$, where the training data, $\mathcal{D}_{tr}$, is fixed. Assume that $\hat{w}(x, y)$ satisfies $(\expb[\hat{w}(X,Y)^r])^{1/r} \leq M_r < \infty$ for some $r \geq 2$.
Define $\Delta_w$ as,
\begin{align}
    \Delta_w \coloneqq \tfrac{1}{2}\expb \mid \hat{w}(X,Y) - w(X,Y)\mid  \label{delta_w}.\\
    \text{Then, } \hspace{0.2cm} \tarprob(Y\in \hat{C}_n(X)) \geq 1-\alpha - \Delta_w.\nonumber
\end{align}
If, in addition, non-conformity scores $\{V_i\}_{i=1}^n$ have no ties almost surely, then we also have
\begin{align}
    \tarprob(Y\in \hat{C}_n(X)) \leq 1-\alpha + \Delta_w + cn^{1/r-1}, \nonumber
\end{align}
for some positive constant $c$ depending only on $M_r$ and $r$.
\end{proposition}
Proposition \ref{prop2} provides finite-sample guarantees with approximate weights $\hat{w}(\cdot, \cdot)$. Note that if the weights are known exactly then the above proposition can be simplified by setting $\Delta_w =0$. In the case where the weight function is estimated \textit{consistently}, we recover the exact coverage asymptotically. A natural question to ask is whether the consistency of $\hat{w}(x, y)$ implies the consistency of $\hat{P}(y|x, a)$; in which case one could use $\hat{P}(y|x, a)$ directly to construct the intervals. We prove that this is not the case in general and provide detailed discussion in Appendix \ref{sec:weights_estimation_app}. 

\subsection{Conditional coverage}\label{sec:cond_cov}
So far we only considered marginal coverage \eqref{guarantee}, where the probability is over both $X$ and $Y$. Here, we provide results on conditional coverage $\p_{Y \sim P^{\pi^*}_{Y \mid X}}(Y \in \hat{C}_n(X) \mid X)$ which is a strictly stronger notion of coverage than marginal coverage \citep{foygel2021limits}. \cite{vovk2012, lei2014distribution} prove that exact conditional coverage cannot be achieved without making additional assumptions. However, we show that, in the case where $Y$ is a continuous random variable and we can estimate the quantiles of $P^{\pi^*}_{Y \mid X}$ consistently, we get an approximate conditional coverage guarantee using the below proposition.
\begin{proposition}[Asymptotic conditional coverage]\label{conditional-res}
Let $m, n$ be the number of training and calibration data respectively, $\hat{q}_{\beta, m} (x)= \hat{q}_{\beta, m} (x; \mathcal{D}_{tr})$ be an estimate of the $\beta$-th conditional quantile $q_\beta (x)$ of $P^{\pi^*}_{Y \mid X=x}$, $\hat{w}_m(x, y) = \hat{w}_m(x, y; \mathcal{D}_{tr})$ be an estimate of $w(x,y)$ and $\hat{C}_{m,n}(x)$ be the conformal interval resulting from algorithm \ref{cp_covariate_shift} with score function $s(x, y) = \max \{y - \hat{q}_{\alpha_{hi}} (x), \hat{q}_{\alpha_{lo}} (x) - y \}$ where $\alpha_{hi} - \alpha_{lo} = 1 - \alpha$. Assume that the following hold:
\begin{enumerate}
    \item $\lim_{m \rightarrow \infty} \expb |\hat{w}_{m}(X, Y) -  w(X, Y)|  = 0$.
    \item there exists $r, b_1, b_2 > 0$ such that $P^{\pi^*}(y \mid x) \in [b_1, b_2]$ uniformly over all $(x, y)$ with $y \in [q_{\alpha_{lo}}(x) - r, q_{\alpha_{lo}}(x) + r] \cup [q_{\alpha_{hi}}(x) - r, q_{\alpha_{hi}}(x) + r]$,
    \item  $\exists k > 0$ s.t. $\lim_{m\rightarrow\infty} \mathbb{E}_{X\sim P_X}[H^k_{m}(X)] = 0$
    where $$H_m(x) = \max\{|\hat{q}_{\alpha_{lo}, m}(x) - q_{\alpha_{lo}}(x)|, |\hat{q}_{\alpha_{hi}, m}(x) - q_{\alpha_{hi}}(x)|\}$$
\end{enumerate}
Then for any $t > 0$, we have that $ \lim_{m, n \rightarrow \infty} \p(\p_{Y \sim P^{\pi^*}_{Y\mid X} }(Y\in \hat{C}_{m, n}(X) \mid X) \leq 1 - \alpha - t) = 0.$
\end{proposition}

One caveat of Prop. \ref{conditional-res} is that Assumption 3 is rather strong. In general, consistently estimating the quantiles under the target policy $\pi^*$ is not straightforward given that we only have access to observational data from $\pi^b$. While one can use a weighted pinball loss to estimate quantiles under $\pi^*$, consistent estimation of these quantiles would require a consistent estimate of the weights (see Appendix  \ref{sec:estimating_target_quantiles}). Hence, unlike \cite[Theorem 1]{lei2020conformal}, our Prop. \ref{conditional-res} is not a ``\textit{doubly robust}'' result.

\begin{importantresultwithtitle}[title=Towards group balanced coverage]\noindent \label{sec:group_balanced_cov}
As pointed out by \cite{conf-bates}, we may want predictive intervals that have the same coverage across different groups, e.g., across male and female users \citep{Romano2020With}. Standard CP will not necessarily achieve this, as the coverage guarantee \eqref{guarantee} is over the entire population of users.
However, we can use COPP on each subgroup separately to obtain group balanced coverage. A more detailed discussion on how to construct such intervals has been included in Appendix \ref{sec:grp-bal}.
\end{importantresultwithtitle}

\section{Related work}\label{sec:related_work}
\paragraph{Conformal prediction} A number of works have explored the use of CP under distribution shift. The works of \cite{tibshirani2020conformal} and \cite{lei2020conformal} are particularly notable as they extend CP to the general setting of \textit{weighted exchangeability}.  In particular, \cite{lei2020conformal} use CP for counterfactual inference where the goal is to obtain predictive intervals on the outcomes of treatment and control groups. The authors formulate the counterfactual setting into that of covariate shift in the input space $\mathcal{X}$ and show that under certain assumptions, finite-sample coverage can be guaranteed.

Fundamentally, our work differs from \cite{lei2020conformal} by framing the problem as a shift in the conditional $P_{Y\mid X}$ rather than as a shift in the marginal $P_X$.
The resulting methodology we obtain from this then differs from \cite{lei2020conformal} in a variety of ways.
For example, while \cite{lei2020conformal} assume a deterministic target policy, COPP can also be applied to stochastic target policies, which have been used in a variety of applications, such as recommendation systems or RL applications \citep{swaminathan2016off, su2020doubly, farajtabar2018more}. 
Likewise, unlike \cite{lei2020conformal}, COPP is applicable to continuous action spaces, e.g., doses of medication administered.

In addition, when the target policy is deterministic, there is an important methodological difference between COPP and \cite{lei2020conformal}.
In particular, \cite{lei2020conformal} construct the intervals on outcomes by splitting calibration data w.r.t.\ actions.
In contrast, it can be shown that COPP uses the entire calibration data when constructing intervals on outcomes.
This is a consequence of integrating out the actions in the weights $w(x, y)$ \eqref{weight-est}, and empirically leads to smaller variance in coverage compared to \cite{lei2020conformal}.
See \ref{sec:comp_lc} for the experimental results comparing COPP to \cite{lei2020conformal} for deterministic policies.

\cite{osama2020learning} propose using CP to \textit{construct} robust policies in contextual bandits with discrete actions. Their methodology uses CP to choose actions and does not involve evaluating target policies. Hence, the problem being considered is orthogonal to ours. There has also been concurrent work adapting CP to individual treatment effect (ITE) sensitivity analysis model \citep{jin2021sensitivity, yin2021conformal}. Similar to our approach, these works formulate the sensitivity analysis problem as one of CP under the joint distribution shift $P_{X, Y}$. While our methodologies are related, the application of CP explored in these works, i.e. ITE estimation under unobserved confounding, is fundamentally different. 

\paragraph{Uncertainty in contextual bandits} Recall from the introduction, that most works in this area have focused on quantifying uncertainty in expected outcome (policy value) \citep{doubly-robust, uncertainty5}. Despite providing finite sample-guarantees on the expectation, these methods do not account for the variability in the outcome itself and in general are not adaptive w.r.t. $X$, i.e. they do not satisfy properties (i), (ii) from Sec. \ref{sec:problem_setup}. \cite{risk-assessment, chandak2021universal} on the other hand, propose off-policy assessment algorithms for contextual bandits w.r.t. a more general class of risk objectives such as Mean, CVaR etc. Their methodologies can be applied to our problem, to construct predictive intervals for off-policy outcomes. However, unlike COPP, these intervals are not adaptive w.r.t. $X$, i.e. do not satisfy property (i) in Sec. \ref{sec:problem_setup}. Moreover, they do not provide upper bounds on coverage probability, which often leads to overly conservative intervals, as shown in our experiments. Lastly, while distributional perspective has been explored in reinforcement learning \citep{distributional-rl}, no finite sample-guarantees are available to the best of our knowledge.
\section{Experiments} \label{sec:exp}

\paragraph{Baselines for comparison}
Given our problem setup, there are no established baselines. Instead, we compare our proposed method COPP to the following competing methods, which were constructed to capture the uncertainty in the outcome distribution and take into account the policy shift. 

\paragraph{Weighted Importance Sampling (WIS) CDF estimator} Given observational dataset $\mathcal{D}_{obs} = \{x_i, a_i, y_i\}_{i=1}^{n_{obs}}$, \cite{risk-assessment} proposed a non-parametric WIS-based estimator for the empirical CDF of $Y$ under $\pi^*$, 
$
\hat{F}_{WIS}(t) \coloneqq \frac{\sum_{i=1}^{n_{obs}} \hat{\rho}(a_i, x_i) \mathbbm{1}(y_i \leq t)}{\sum_{i=1}^{n_{obs}} \hat{\rho}(a_i, x_i)}
$
where $\hat{\rho}(a, x) \coloneqq \frac{\pi^*(a \mid x)}{\hat{\pi}^b(a \mid x)}$ are the importance weights. We can use $\hat{F}_{WIS}$ to get predictive intervals $[y_{\alpha/2}, y_{1-\alpha/2}]$ where $y_\beta \coloneqq \text{Quantile}_\beta(\hat{F}_{WIS})$. The intervals $[y_{\alpha/2}, y_{1-\alpha/2}]$ do not depend on $x$.

\begin{table}[t]
    \caption{Toy experiment results with required coverage $90\%$. While WIS intervals provide required coverage, the mean interval length is huge compared to COPP (see table \ref{tab:length_toy}).}
    \begin{minipage}[b]{.48\linewidth}
      \centering
      \subcaption{Mean coverage as a function of policy shift with 2 standard errors over 10 runs.}\label{tab:coverage_toy}
      \resizebox{1\columnwidth}{!}{%
        \begin{tabular}{lccc}
\toprule
Coverage &  $\Delta_{\epsilon}=0.0$ &  $\Delta_{\epsilon}=0.1$ &  $\Delta_{\epsilon}=0.2$ \\
\midrule
COPP (Ours)            &                    \textbf{0.90 $\pm$ 0.01}&                    \textbf{0.90 $\pm$ 0.01}&                    \textbf{0.91 $\pm$ 0.01}\\
WIS                  &                    \textbf{0.89 $\pm$ 0.01}&                     \textbf{0.91 $\pm$ 0.02}&                     0.94 $\pm$ 0.02\\
SBA                  &                     \textbf{0.90 $\pm$ 0.01}&                     0.88 $\pm$ 0.01&                     0.87 $\pm$ 0.01\\
\midrule
\midrule
COPP (GT weights Ours)      &                     \textbf{0.90 $\pm$ 0.01}&                     \textbf{0.90 $\pm$ 0.01}&                     \textbf{0.90 $\pm$ 0.01}\\
CP (no policy shift) &                     \textbf{0.90 $\pm$ 0.01}&                     0.87 $\pm$ 0.01&                     0.85 $\pm$ 0.01\\
CP (union) &                      0.96 $\pm$ 0.01 &         0.96 $\pm$ 0.01 &         0.96 $\pm$ 0.01 \\
\bottomrule
\end{tabular}
}
    \end{minipage}%
    \hspace{0.5cm}
    \begin{minipage}[b]{.48\linewidth}
      \centering
      \subcaption{Mean interval length as a function of policy shift with 2 standard errors over 10 runs.}\label{tab:length_toy}
      \resizebox{1\columnwidth}{!}{%
        \begin{tabular}{lccc}
\toprule
Interval Lengths &  $\Delta_{\epsilon}=0.0$ &  $\Delta_{\epsilon}=0.1$ &  $\Delta_{\epsilon}=0.2$ \\
\midrule
COPP (Ours)           &                     9.08 $\pm$ 0.10&                     9.48 $\pm$ 0.22&                     9.97 $\pm$ 0.38\\
WIS                  &                    \red{24.14 $\pm$ 0.30}&               \red{32.96 $\pm$ 1.80}&             \red{43.12 $\pm$ 3.49}\\
SBA                  &                     8.78 $\pm$ 0.12&                     8.94 $\pm$ 0.10&                     8.33 $\pm$ 0.09\\
\midrule
\midrule
COPP (GT weights Ours)      &                     8.91 $\pm$ 0.09&                     9.25 $\pm$ 0.12&                     9.59 $\pm$ 0.20\\
CP (no policy shift) &                     9.00 $\pm$ 0.10&                     9.00 $\pm$ 0.10&                     9.00 $\pm$ 0.10\\
CP (union) &                     10.66 $\pm$ 0.18 &         11.04 $\pm$ 0.2 &         11.4 $\pm$ 0.26 \\
\bottomrule
\end{tabular}%
}
    \end{minipage} 
\end{table}

\paragraph{Sampling Based Approach (SBA)} As mentioned in Sec. \ref{sec:weights}, we can directly use the estimated $\hat{P}(y\mid x, a)$ to construct the predictive intervals as follows. For a given $x^{test}$, we generate $A_i \overset{\textup{i.i.d.}}{\sim} \pi^*(\cdot \mid x^{test})$, and $Y_i \sim \hat{P}(\cdot \mid x^{test}, A_i)$ for $i \leq \ell$. We then define the predictive intervals for $x^{test}$ using the $\alpha/2$ and $1-\alpha/2$ quantiles of $\{Y_i\}_{i \leq \ell}$. While SBA is not a standard baseline, it is a natural comparison to make to answer the question of ``why not construct the intervals using $\hat{P}(y|x, a)$ directly''?

\subsection{Toy experiment}\label{sec:exp_toy} 
 We start with synthetic experiments and an ablation study, in order to dissect and understand our proposed methodology in more detail. We assume that our policies are stationary and there is overlap between the behaviour and target policy, both of which are standard assumptions \citep{risk-assessment, drobust, ope-rl}.
\subsubsection{Synthetic data experiments setup}

In order to understand how COPP works, we construct a simple experimental setup where we can control the amount of \textit{``policy shift''} and know the ground truth. In this experiment, $X \in \mathbb{R}$, $A \in \{1, 2, 3, 4\}$ and $Y \in \mathbb{R}$, where $X$ and $Y\mid x, a$ are normal random variables. Further details and additional experiments on continuous action spaces are given in Appendix \ref{sec:toy_experiments_descrip}.   

\paragraph{Behaviour and target policies}
We define a family of policies $\pi_\epsilon(a \mid x)$, where we use the parameter $\epsilon \in (0,1/3)$ to control the policy shift between target and behaviour policies. Exact form of $\pi_\epsilon(a \mid x)$ is given in \ref{sec:toy_experiments_descrip}. For the behaviour policy $\pi^b$, we use $\epsilon^b = 0.3$ (i.e. $\pi^b(a \mid x) \equiv  \pi_{0.3}(a \mid x)$), and for target policies $\pi^*$, we use $\epsilon^* \in \{0.1, 0.2, 0.3\}$. Using the true behaviour policy, $\pi^b$, we generate observational data $\mathcal{D}_{obs} = \{x_i, a_i, y_i\}_{i=1}^{n_{obs}}$ which is then split into training ($\mathcal{D}_{tr}$) and calibration ($\mathcal{D}_{cal}$) datasets, of sizes $m$ and $n$ respectively.

\paragraph{Estimation of ratios, $\hat{w}(x, y)$\,}
Using the training dataset $\mathcal{D}_{tr}$, we estimate $P(y | x, a)$ as $\hat{P}(y | x, a) = \mathcal{N}(\mu(x, a), \sigma(x, a))$, where $\mu(x, a), \sigma(x, a)$ are both neural networks (NNs). Similarly, we use NNs to estimate the behaviour policy $\hat{\pi}^b$ from $\mathcal{D}_{tr}$. Next, to estimate $\hat{w}(x, y)$, we use \eqref{weight-est} with $h = 500$.

\paragraph{Score}
For the score function, we use the same formulation as in \cite{romano2019conformalized}, i.e. $s(x, y) = \max\{ \hat{q}_{\alpha_{lo}}(x) - y, y - \hat{q}_{\alpha_{hi}}(x) \}$, where $\hat{q}_\beta(x)$ denotes the $\beta$ quantile estimate of $P^{\pi^b}_{Y\mid X=x}$ trained using pinball loss.

Lastly, our weights $w(x, y)$ depend on $x$ \textbf{and} $y$ and hence we use a grid of $100$ equally spaced out $y$'s in our experiments to determine the predictive interval which satisfies $\hat{C}_n(x) \coloneqq \{y: s(x,y) \leq \text{Quantile}_{1-\alpha}(\hat{F}_{n}^{x, y})\}$. This is parallelizable and hence does not add much computational overhead.

\paragraph{Results} Table \ref{tab:coverage_toy} shows the coverages of different methods as the policy shift $\Delta_{\epsilon}=\epsilon^b - \epsilon^*$ increases. The behaviour policy $\pi^b = \pi_{0.3}$ is fixed and we use $n=5000$ calibration datapoints, across 10 runs. Table \ref{tab:coverage_toy} shows, how COPP stays very close to the required coverage of $90\%$ across all target policies compared to WIS and SBA. WIS intervals are overly conservative i.e. above the required coverage, while the SBA intervals suffer from under-coverage i.e. below the required coverage. These results supports our hypothesis from Sec. \ref{sec:weights}, which stated that COPP is less sensitive to estimation errors of $\hat{P}(y|x, a)$ compared to directly using $\hat{P}(y|x, a)$ for the intervals, i.e. SBA. 

Next, Table \ref{tab:length_toy} shows the mean interval lengths and even though WIS has reasonable coverage for $\Delta_{\epsilon}=0.0$ and $0.1$, the average interval length is huge compared to COPP. Fig. \ref{fig:copp}b shows the predictive intervals for one such experiment with $\pi^* = \pi_{0.1}$ and $\pi^b = \pi_{0.3}$. We can see that SBA intervals are overly optimistic, while WIS intervals are too wide and are not adaptive w.r.t. $X$. COPP produces intervals which are much closer to the oracle intervals. 

\subsubsection{Ablation study} 

To isolate the effect of weight estimation error and policy shift, we conduct an ablation study, comparing COPP with estimated weights to COPP with Ground Truth (GT) weights and standard CP (assuming no policy shift). Table \ref{tab:coverage_toy} shows that at $\Delta_\epsilon = 0$, i.e. no policy shift, standard CP achieves the required coverage as expected. However the coverage of standard CP intervals decreases as the policy shift $\Delta_\epsilon$ increases. COPP, on the other hand, attains the required coverage of $90\%$, by adapting the predictive intervals with increasing policy shift. Table \ref{tab:length_toy} shows that the average interval length of COPP increases with increasing policy shift $\Delta_\epsilon$. Furthermore, Table \ref{tab:coverage_toy} illustrates that while COPP achieves the required coverage for different target policies, on average it is slightly more conservative than using COPP with GT weights. This can be explained by the estimation error in $\hat{w}(x,y)$. Additionally, to investigate the effect of integrating out the actions in \eqref{weight-est}, we also perform CP for each action $a$ separately (as in \cite{lei2020conformal}) and then take the union of the intervals across these actions. In the union method, the probability of an action being chosen is not taken into account, (i.e., intervals are independent of $\pi^*$) and hence the coverage is overly conservative as expected.

Lastly, we investigate how increasing the number of calibration data $n$ affects the coverage for all the methodologies. We observe that coverage of COPP is closer to the required coverage of $90\%$ compared to the competing methodologies. Additionally, the coverage of COPP converges to the required coverage as $n$ increases; see Appendix \ref{app:N-cal_exp_toy} for detailed experimental results.

\begin{table}[t]
\centering
\caption{Mean coverage as a function of policy shift $\Delta_\epsilon$ and 2 standard errors over 10 runs. COPP attains the required coverage of $90\%$, whereas the competing methods, WIS and SBA, are over-conservative i.e. coverage above $90\%$. In addition, when we do not account for the policy shift, standard CP becomes progressively worse with increasing policy shift.}\label{tab:MSR}
\resizebox{0.7\columnwidth}{!}{%
\begin{tabular}{lccccc}
\toprule
 &  $\Delta_{\epsilon}=0.0$ &  $\Delta_{\epsilon}=0.1$ &  $\Delta_{\epsilon}=0.2$ &  $\Delta_{\epsilon}=0.3$ &  $\Delta_{\epsilon}=0.4$ \\
\midrule
COPP (Ours)            &                  \textbf{0.90 $\pm$ 0.00}&                  \textbf{0.90 $\pm$ 0.02}&                  \textbf{0.90 $\pm$ 0.01}&                  \textbf{0.89 $\pm$ 0.01}&                  \textbf{0.91 $\pm$ 0.01}\\
WIS                  &                  1.00 $\pm$ 0.00&                  1.00 $\pm$ 0.00&                  0.92 $\pm$ 0.00&                  0.94 $\pm$ 0.00&                  0.91 $\pm$ 0.00\\
SBA                  &                  0.99 $\pm$ 0.00&                  0.99 $\pm$ 0.00&                  0.98 $\pm$ 0.00&                  0.97 $\pm$ 0.00&                  0.96 $\pm$ 0.00\\
\midrule
\midrule
CP (no policy shift) &                  \textbf{0.91 $\pm$ 0.02}&                 \textbf{ 0.92 $\pm$ 0.02}&                  0.93 $\pm$ 0.01&                  0.94 $\pm$ 0.01&                  0.96 $\pm$ 0.01\\
\bottomrule
\end{tabular}%
}
\end{table}

\subsection{Experiments on Microsoft Ranking Dataset}

We now apply COPP onto a real dataset i.e. the Microsoft Ranking dataset 30k \citep{msr, swaminathan2016off, bietti2018contextual}. Due to space constraints, we have added additional extensive experiments on UCI datasets in Appendix \ref{sec:UCI}.

\paragraph{Dataset}
The dataset contains relevance scores for websites recommended to different users, and comprises of $30,000$ user-website pairs. For each user-website pair, the data contains a $136$-dimensional feature vector, which consists of user's attributes corresponding to the website, such as length of stay or number of clicks on the website. Furthermore, for each user-website pair, the dataset also contains a relevance score, i.e. how relevant the website was to the user.

First, given a user, we sample (with replacement) $5$ websites from the data corresponding to that user. Next, we reformulate this into a contextual bandit where $a_i \in \{1,2,3,4,5\}$ corresponds to the action of recommending the $a_i$'th website to the user $i$. $x_i$ is obtained by combining the $5$ user-website feature vectors corresponding to the user $i$ i.e. $x_i \in \mathbb{R}^{5 \times 136}$. $y_i \in\{0,1,2,3,4\}$ corresponds to the relevance score for the $a_i$'th website, i.e. the recommended website. The goal is to construct prediction sets that are guaranteed to contain the true relevance score with a probability of $90\%$.

\paragraph{Behaviour and target policies} 
We first train a NN classifier model, $\hat{f}_\theta$, mapping each 136-dimensional user-website feature vector to the softmax scores for each relevance score class. We use this trained model $\hat{f}_\theta$ to define a family of policies which pick the most relevant website as predicted by $\hat{f}_\theta$ with probability $\epsilon$ and the rest uniformly with probability $(1-\epsilon)/4$ (see Appendix \ref{sec:MSR_experiments_decrip} for more details). Like the previous experiment, we use $\epsilon$ to control the shift between behaviour and target policies. For $\pi^b$, we use $\epsilon^b = 0.5$ and for $\pi^*$, $\epsilon^* \in \{0.1, 0.2, 0.3, 0.4, 0.5\}$. 

\paragraph{Estimation of ratios $\hat{w}(X, Y)$}
To estimate the $\hat{P}(y \mid x, a)$ we use the trained model $\hat{f}_\theta$ as detailed in Appendix \ref{sec:MSR_experiments_decrip}. To estimate the behaviour policy $\hat{\pi}^b$, we train a neural network classifier model $\mathcal{X} \rightarrow \mathcal{A}$, and we use \eqref{weight-est} to estimate the weights $\hat{w}(x, y)$.

\paragraph{Score} The space of outcomes $\mathcal{Y}$ in this experiment is discrete. We define $\hat{P}^{\pi^b}(y \mid x) = \sum_{i = 1}^5 \hat{\pi}^b(A = i|x) \hat{P}(y|x, A = i)$. Using similar formulation as in \cite{conf-bates}, we define the score:
$$
s(x, y) = \sum_{y' = 0}^4 \hat{P}^{\pi^b}(y' \mid x) \mathbbm{1}(\hat{P}^{\pi^b}(y' \mid x) \geq \hat{P}^{\pi^b}(y \mid x)).
$$
Since $\mathcal{Y}$ is discrete, we no longer need to construct a grid of $y$ values on which to compute $\text{Quantile}_{1-\alpha}(\hat{F}_{n}^{x, y})$. Instead, we will simply compute this quantity on each $y \in \mathcal{Y}$, when constructing the predictive sets $\hat{C}_{n}(x^{test})$.

\paragraph{Results}
Table \ref{tab:MSR} shows the coverages of different methodologies across varying target policies $\pi_{\epsilon^*}$. The behaviour policy $\pi^b = \pi_{0.5}$ is fixed and we use $n=5000$ calibration datapoints, across 10 runs. Table \ref{tab:MSR} also shows that the coverage of WIS and SBA sets is dependent upon the policy shift, with both being overly conservative across the different target policies as compared to COPP. Recall that the WIS sets do not depend on $x^{test}$ and as a result we get the same set for each test data point. This becomes even more problematic when $Y$ is discrete -- if, for each label $y$, $\tarprob(Y = y)>10\%$, then WIS sets (with the required coverage of $90\%$) are likely to contain every label $y \in \mathcal{Y}$.
In comparison, COPP is able to stay much closer to the required coverage of $90\%$ across all target policies. We have also added standard CP without policy shift as a sanity check, and observed that the sets get increasingly conservative as the policy shift increases.

Finally, we also plotted how the coverage changes as the number of calibration data $n$ increases. We observe again that the coverage of COPP is closer to the required coverage of $90\%$ compared to the competing methodologies. Due to space constraints, we have added the plots in Appendix \ref{app:N-cal_exp_msr}.

\paragraph{Class-balanced conformal prediction}
Using the methodology described in Sec. \ref{sec:group_balanced_cov}, we construct predictive sets, $\hat{C}^{\mathcal{Y}}_n(x)$, which offer label conditioned coverage guarantees (see \ref{sec:grp-bal}), i.e. for all $y\in \mathcal{Y}$, 
$$
\tarprob(Y \in \hat{C}^{\mathcal{Y}}_n(X) \mid Y = y) \geq 1- \alpha.
$$
We empirically demonstrate that $\hat{C}^{\mathcal{Y}}_n$ provides label conditional coverage, while $\hat{C}_n$ obtained using alg. \ref{cp_covariate_shift} may not. Due to space constraints, details on construction of $\hat{C}^{\mathcal{Y}}_n$ as well as experimental results have been included in Appendix \ref{sec:results_class_bal_coverage}.

\section{Conclusion and limitations}\label{sec:lims}

In this paper, we propose COPP, an algorithm for constructing predictive intervals on off-policy outcomes, which are adaptive w.r.t. covariates $X$. We theoretically prove that COPP can guarantee finite-sample coverage by adapting the framework of conformal prediction to our setup.
Our experiments show that conventional methods cannot guarantee any user pre-specified coverage, whereas COPP can.
For future work, it would be interesting to apply COPP to policy training. This could be a step towards robust policy learning by optimising the worst case outcome \citep{stutz2021learning}.

We conclude by mentioning several limitations of COPP. 
Firstly, we do not guarantee conditional coverage in general.
We outline conditions under which conditional coverage holds asymptotically (Prop. \ref{conditional-res}), however, this relies on somewhat strong assumptions.
Secondly, our current method estimates the weights $w(x, y)$ through $P(y \mid x, a)$, which can be challenging.
We address this limitation in Appendix \ref{sec:alternate_weights_est}, where we propose an alternative method to estimate the weights directly, without having to model $P(y \mid x, a)$. %
Lastly, reliable estimation of our weights $\hat{w}(x, y)$ requires sufficient overlap between behaviour and target policies. The results from COPP may suffer in cases where this assumption is violated, which we illustrate empirically in Appendix \ref{subsec:cts_act}.
We believe these limitations suggest interesting research questions that we leave to future work.

\section*{Acknowledgements}
We would like to thank Andrew Jesson, Sahra Ghalebikesabi, Robert Hu, Siu Lun Chau and Tim Rudner for useful feedback.
JFT is supported by the EPSRC and MRC through the OxWaSP CDT programme (EP/L016710/1).
MFT acknowledges his PhD funding from Google DeepMind.
RC and AD are supported by the Engineering and Physical Sciences Research Council (EPSRC) through the Bayes4Health programme [Grant number EP/R018561/1].  

\newpage

    }
    { %
    }

    \includepdf[pages=-]{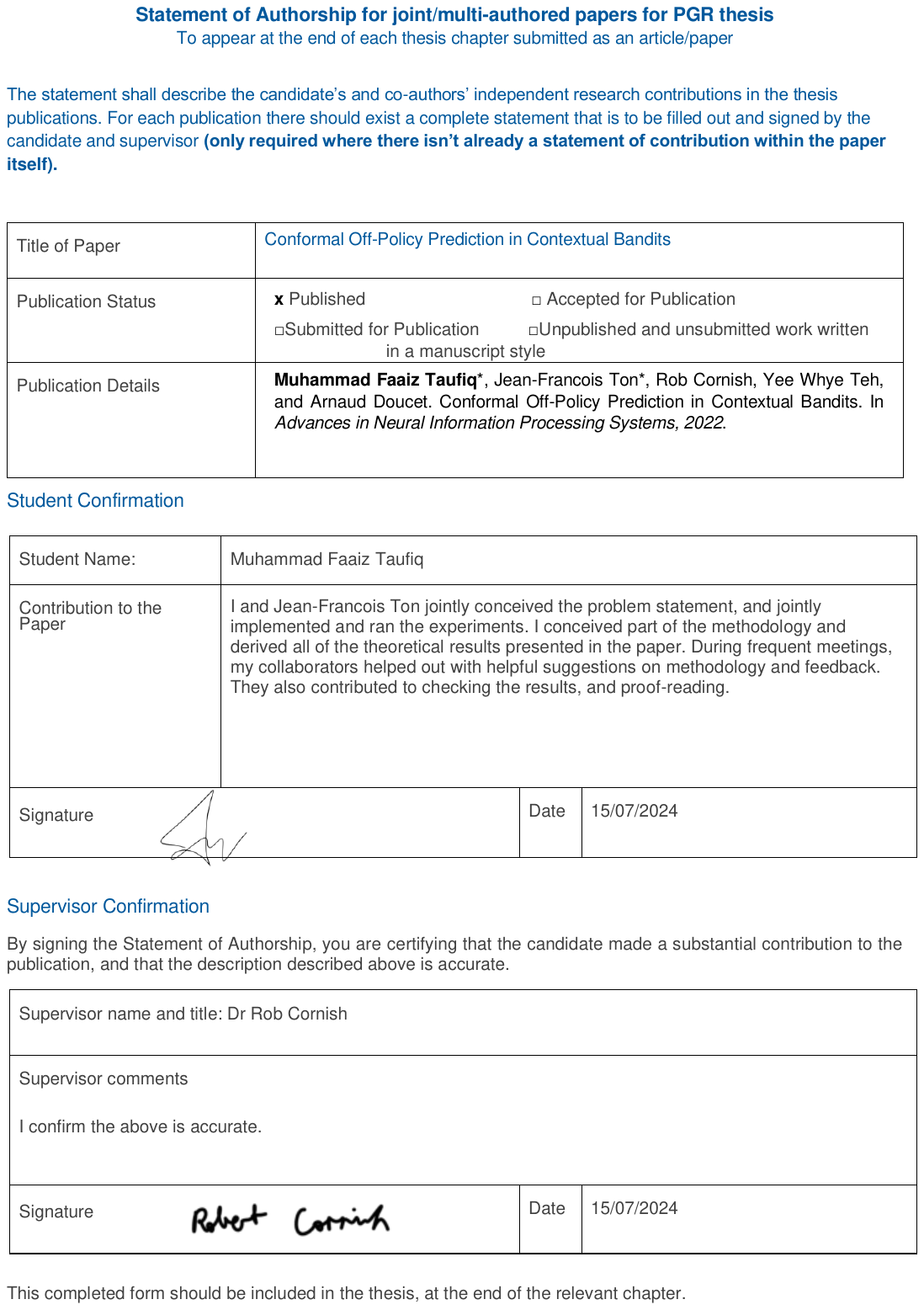}

\chapter{\label{ch:5-causal}Causal Falsification of Digital Twins}

\minitoc

\ifthenelse{\boolean{compilePapers}}
    { %

\begin{abstract}
\emph{Digital twins} are simulation-based models designed to predict how a real-world process will evolve in response to interventions.
This modelling paradigm holds substantial promise in many applications, but rigorous procedures for assessing their accuracy are essential for safety-critical settings.
We consider how to assess the accuracy of a digital twin using real-world data.
We formulate this as causal inference problem, which leads to a precise definition of what it means for a twin to be ``correct''. %
Unfortunately, fundamental results from causal inference mean observational data cannot be used to certify a twin in this sense unless potentially tenuous assumptions are made, such as that the data are unconfounded.
To avoid these assumptions, we propose instead to find situations in which the twin \emph{is not} correct, and present a general-purpose statistical procedure for doing so.
Our approach yields reliable and actionable information about the twin under only the assumption of an i.i.d.\ dataset of observational trajectories, and remains sound even if the data are confounded.
We apply our methodology to a large-scale, real-world case study involving sepsis modelling within the Pulse Physiology Engine, which we assess using the MIMIC-III dataset of ICU patients.
\end{abstract}
        \section{Introduction}

\subsection{Motivation}

There is increasing interest in the use of simulation-based models for obtaining \emph{causal} insights.
Such models aim to describe what \emph{would} occur when different actions or interventions are applied to some real-world process of interest, thereby allowing planning and decision-making to be done with a fuller understanding of the different outcomes that may result.
In many applications, models of this kind are referred to as \emph{digital twins} \citep{barricelli2019survey,jones2020characterising,niederer2021scaling}.
These have been considered for a wide range of use-cases including aviation \citep{tuegel2011reengineering}, manufacturing \citep{lu2020digital}, healthcare \citep{corral2020digital,coorey2022health}, civil engineering \citep{sacks2020construction}, and agriculture \citep{jans2020digital}.

Many applications of digital twins are considered safety-critical, which means the cost of deploying an inaccurate twin to production is potentially very high.
As such, methodology for assessing the performance of a twin before its deployment is essential for the safe, widespread adoption of digital twins in practice \citep{niederer2021scaling}.
In this work, we consider the problem of assessing twin accuracy and propose a concrete, theoretically grounded, and general-purpose methodology to this end.
We focus specifically on the use of statistical methods that leverage data obtained from the real-world process that the twin is designed to model.
Such strategies are increasingly viable for many applications as datasets grow larger, and offer the promise of lower overheads compared with alternative strategies that rely for instance on domain expertise.

We formulate twin assessment as a problem of \emph{causal inference} \citep{rubin1974estimating,rubin2005causal,pearl2009causality,hernan2020causal}.
In particular, we consider a twin to be accurate if it correctly captures the behaviour of a real-world process of interest in response to certain interventions, rather than the behaviour of the process as it evolves on its own.
This seems in keeping with the overall (if often implicit) design objectives that underlie many applications, including those cited above.
Our causal assessment approach also highlights certain pitfalls associated with conventional methods that do not account for causal factors, and that can give rise to misleading inferences about the twin as a result.

In most cases, it is desirable for an assessment procedure to be reliable and robust, and for its conclusions about the twin to be highly trustworthy.
As such, our goal in this paper is to obtain a methodology that is always \emph{sound}, even possibly at the expense of being conservative: we prefer not to draw any conclusion about the accuracy of the twin at all than to draw some conclusion that is potentially misleading.
To this end, we rely on minimal assumptions about the twin and the real-world process of interest.
In addition to improving robustness, this also means our resulting methodology is very general, and may be applied to a wide variety of twins across application domains.

\subsection{Contribution}

We begin by providing a causal model for a general-purpose twin and the data we have available for assessment.
We use this to show precisely that it is not possible to use observational data to \emph{certify} that the twin is causally accurate unless strong and often tenuous assumptions are made about the data-generating process, such as that the data are free of unmeasured confounding.
To avoid these assumptions, we propose an assessment paradigm instead based on \emph{falsification}: we search for specific cases when the twin is \emph{not} accurate, rather than trying to quantify its accuracy in a more holistic sense. %

To obtain a practical methodology suitable for real twins, we provide a novel set of longitudinal causal bounds that hold without additional causal assumptions.
These bounds generalize the classical bounds of \cite{manski}, and can be considerably more informative in comparison.
We use this result as the basis for a general-purpose statistical testing procedure for falsifying a twin.
Overall, our method relies on only the assumption of an independent and identically distributed (i.i.d.)\ dataset of observational trajectories: it does not require modelling the dynamics of the real-world process or any internal implementation details of the twin, and remains sound in the presence of arbitrary unmeasured confounding. %
We demonstrate the effectiveness of our procedure through a large-scale, real-world case study in which we use the MIMIC-III ICU dataset \citep{mimic} to assess the Pulse Physiology Engine \citep{pulse}, an open-source model for human physiology simulation.

\subsection{Related work}

Various high-level guidelines and workflows have been proposed for the assessment of digital twins in the literature to-date \citep{roy2011comprehensive,grieves2017digital,khan2018digital,corral2020digital,kochunas2021digital,niederer2021scaling,dahmen2022verification}. %
In some cases, these guidelines have been codified as standards: for example, the ASME V\&V40 Standard \citep{amse2018assessing} provides a risk-based framework for assessing the credibility of a model from a variety of factors that include source code quality and the mathematical form of the model \citep{galappaththige2022credibility}.
However, a significant gap still exists between these guidelines and a practical implementation that could be deployed for real twins, and the need for a rigorous lower-level framework to enable the systematic assessment of twins has been noted in this literature \citep{corral2020digital,niederer2021scaling,kapteyn2021probabilistic,masison2021modular}.
We contribute towards this effort by describing a precise statistical methodology for twin assessment that can be readily implemented in practice, and which is accompanied by theoretical guarantees of robustness that hold under minimal assumptions.

In addition, a variety of concrete assessment procedures have been applied to certain specific digital twin models in the literature.
For example, the Pulse Physiology Engine \citep{pulse}, which we consider in our empirical case study, as well as the related BioGears \citep{sepsis-modelling,biogears} were both assessed by comparing their outputs with ad hoc values based either on available medical literature or the opinions of subject matter experts.
Other twins have been assessed by comparing their outputs with real-world data through a variety of bespoke numerical schemes \citep{larrabide2012fast,hemmler2019patient,DT-patient,jans2020digital,galappaththige2022credibility}.
In contrast, our paper proposes a general-purpose statistical procedure for assessing twins that may be applied generically across many applications and architectures. %
To the best of our knowledge, our paper is also the first to identify the need for a causal approach to twin assessment and the pitfalls that arise when causal considerations are not properly accounted for.

\section{Causal formulation} \label{sec:causal-formulation}

\subsection{The real-world process}

We begin by providing a causal model the real-world process that the twin is designed to simulate.
We do so in the language of \emph{potential outcomes} \citep{rubin1974estimating,rubin2005causal}, although we note that we could have used the alternative framework of directed acyclic graphs and structural causal models \citep{pearl2009causality} (see also \cite{imbens2020potential} for a comparison of the two).
We assume the real-world process operates over a fixed time horizon $\T \in \{1, 2, \ldots\}$.
This simplifies our presentation in what follows, and it is straightforward to generalize our methodology to variable length time horizons if needed.
For each $\tx \in \{0, \ldots, \T\}$, we assume the process gives rise to an \emph{observation} at time $\tx$, which takes values in some real-valued space $\Xspace_\tx \coloneqq \R^{\Xspacedim_\tx}$.
We also assume that the process can be influenced by some \emph{action} taken at each time $\tx \in \{1, \ldots, \T\}$.
We denote the space of actions available at time $\tx$ by $\Aspace_\tx$, which in this work we assume is always finite.
For example, in a robotics context, the observations may consist of all the readings of all the sensors of the robot, and the actions may consist of commands that can be input by an external user.
In a medical context, the observations may consist of the vital signs of a patient, and the actions may consist of possible treatments or interventions.
To streamline notation, we will index these spaces using vector notation, so that e.g.\ $\Aspace_{1:\tx}$ denotes the cartesian product $\Aspace_1 \times \cdots \times \Aspace_\tx$, and $\ax_{1:\tx} \in \Aspace_{1:\tx}$ is a choice of $\ax_1 \in \Aspace_1, \ldots, \ax_\tx \in \Aspace_\tx$.

We model the dynamics of the real-world process via the longitudinal potential outcomes framework proposed by \cite{robins1986new}, which imposes only a weak temporal structure on the underlying phenomena of interest and so may be applied across a wide range of applications in practice.
In particular, for each $\ax_{1:\T} \in \Aspace_{1:\T}$, we posit the existence of random variables or \emph{potential outcomes} $\X_0, \X_1(\ax_1), \ldots, \X_\T(\ax_{1:\T})$, where $\X_\tx(\ax_{1:\tx})$ takes values in $\Xspace_\tx$.
We will denote this sequence more concisely as $\X_{0:\T}(\ax_{1:\T})$.
Intuitively, $\X_0$ represents data available before the first action, while $\X_{1:\T}(\ax_{1:\T})$ represents the sequence of real-world outcomes that \emph{would} occur if actions $\ax_{1:\T}$ were taken successively.
These quantities are therefore of fundamental interest for planning a course of actions to achieve some desired result.

As random variables, each $\X_{\tx}(\ax_{1:\tx})$ may depend on additional randomness that is not explicitly modelled, and so in particular may be influenced by all the previous potential outcomes $\X_{0:\tx-1}(\ax_{1:\tx-1})$, and possibly other random quantities.
This models a process whose initial state is determined by external factors, such as when a patient from some population first presents at a hospital, and where the process then evolves according to specific actions chosen from $\Aspace_{1:\T}$ as well as additional external factors.
It is clear that this structure applies to a wide range of phenomena occurring in practice.

\subsection{The digital twin}

We think of the twin as a computational device that, when executed, outputs a sequence of values intended to simulate a possible future trajectory of the real-world process when certain actions in $\Aspace_{1:\T}$ are chosen, conditional on some initial data in $\Xspace_0$.
We allow the twin to make use of an internal random number generator to produce outputs that vary stochastically even under fixed inputs (although our framework encompasses twins that evolve deterministically also).
By executing the twin repeatedly, a user may therefore estimate the range of behaviours that the real-world process may exhibit under different action sequences, which can then inform planning and decision-making downstream.

Precisely, we model the output the twin would produce at timestep $\tx \in \{1, \ldots, \T\}$ after receiving initialisation $\xx_0 \in \Xspace_0$ and successive inputs $\ax_{1:\tx} \in \Aspace_{1:\tx}$ as the quantity $\twinfunction_\tx(\xx_0, \ax_{1:\tx}, \twinnoise_{1:\tx})$, where $\twinfunction_\tx$ is a measurable function taking values in $\Xspace_{\tx}$, and each $\twinnoise_\sx$ is some (possibly vector-valued) random variable.
We will denote $\Xt_{\tx}(\xx_0, \ax_{1:\tx}) \coloneqq \twinfunction_\tx(\xx_0, \ax_{1:\tx}, \twinnoise_{1:\tx})$, which we also refer to as a potential outcome.
A full twin trajectory therefore consists of $\Xt_1(\xx_0, \ax_1), \ldots, \Xt_\T(\xx_0, \ax_{1:\T})$, which we write more compactly by $\Xt_{1:\T}(\xx_0, \ax_{1:\T})$.
Conceptually, $\twinfunction_1, \ldots, \twinfunction_\T$ constitute the program that executes inside the twin, and $\twinnoise_{1:\T}$ may be thought of as the collection of all outputs of the internal random number generator that the twin uses.
We assume these random numbers $\twinnoise_{1:\T}$ and the real-world outcomes $(\X_{0:\T}(\ax_{1:\T}) : \ax_{1:\T} \in \Aspace_{1:\T})$ are independent, which is mild in practice.
We also assume that repeated executions of the twin give rise to i.i.d.\ copies of $\twinnoise_{1:\T}$.
This means that, given fixed inputs $\xx_0$ and $\ax_{1:\T}$, repeated executions of the twin produce i.i.d.\ copies of $\Xt_{1:\T}(\xx_0, \ax_{1:\T})$.
Otherwise, we make no assumptions about the precise form of either the $\twinfunction_\tx$ or the $\twinnoise_\tx$, which allows our model to encompass a wide variety of possible twin implementations.

\subsection{Correctness}

Before we can consider how to assess the twin, we must first define how we want the twin ideally to behave.
The following condition seems appropriate for many applications.

\begin{definition}[Correctness] \label{eq:interventional-correctness}
    The twin is \emph{interventionally correct} if, for $\Law[\X_0]$-almost all $\xx_0$ and $\ax_{1:\T} \in \Aspace_{1:\T}$, distribution of $\Xt_{1:\T}(\xx_0, {\ax}_{1:\T})$ is equal to the conditional distribution of $\X_{1:\T}(\ax_{1:\T})$ given $\X_0 = \xx_0$.
\end{definition}

Operationally, if a twin is interventionally correct, then by repeatedly executing the twin and applying Monte Carlo techniques, it is possible to approximate arbitrarily well the conditional distribution of the future of the real-world process under each possible choice of action sequence.
The same can also be shown to hold when each action at each time $\tx$ is chosen dynamically on the basis of previous observations in $\Xspace_{0:\tx}$. %
As a result, an interventionally correct twin may be used for \emph{planning}, or in other words may be used to select a policy for choosing actions that will yield a desirable distribution over observations at each step.
We emphasise that interventional correctness does not mean the twin will accurately predict the behaviour of any \emph{specific} trajectory of the real-world process in an almost sure sense (unless the real-world process is deterministic), but only the distribution of outcomes that will be observed over repeated independent trajectories.
However, this is sufficient for many applications, and appears to be the strongest guarantee possible when dealing with real-world phenomena whose underlying behaviour is stochastic.

Definition \ref{eq:interventional-correctness} introduces some technical difficulties that arise in the general case when conditioning on events with probability zero (e.g.\ $\{\X_0 = \xx_0\}$ if $\X_0$ is continuous).
In what follows, it is more convenient to consider an unconditional formulation of interventional correctness.
This is supplied by the following result, which considers the behaviour of the twin when it is initialised with the (random) value of $\X_0$ taken from the real-world process, rather than with a fixed choice of $\xx_0$.
See Section \ref{sec:unconditional-interventional-correctness-proof} of the \AppendixName for a proof.

\begin{proposition} \label{prop:interventional-correctness-alternative-characterisation}
    The twin is interventionally correct if and only if, for all choices of $\ax_{1:\T} \in \Aspace_{1:\T}$, the distribution of $(\X_0, \Xt_{1:\T}(\X_0, \ax_{1:\T}))$ is equal to the distribution of $\X_{0:\T}(\ax_{1:\T})$.
\end{proposition}

\subsection{Online prediction}

Our model here represents a twin at time $\tx = 0$ making predictions about all future timesteps $\tx \in \{1, \ldots, \T\}$ under different choices of inputs $\ax_{1:\T}$.
In practice, many twins are designed to receive new information at each timestep in an online fashion and update their predictions for subsequent timesteps accordingly \citep{grieves2017digital,niederer2021scaling}.
Various notions of correctness can be devised for this online setting.
We describe two possibilities in Section \ref{sec:online-prediction} of the \AppendixName, and show that these notions of correctness essentially reduce to Definition \ref{eq:interventional-correctness}, which motivates our focus on that notion in what follows.

\section{Data-driven twin assessment} \label{sec:data-driven-twin-assessment}

\subsection{Overall setup}

There are many conceivable methods for assessing the accuracy of a twin, including static analysis of the twin's source code and the solicitation of domain expertise, and in practice it seems most robust to use a combination of different techniques rather than relying on any single one \citep{amse2018assessing,niederer2021scaling}.
However, in this paper, we focus on what we will call \emph{data-driven assessment}, which we see as an important component of a larger assessment pipeline. %
That is, we consider the use of statistical methods that rely solely on a dataset of trajectories obtained from the real-world process and the twin.
We show in this section that without further assumptions, it is not possible to obtain a data-driven assessment procedure that can \emph{certify} that a twin is interventionally correct.
We instead propose a strategy based on \emph{falsifying} the twin, which we develop into a concrete statistical testing procedure in later sections.

We will assume access to a dataset of trajectories obtained by observing the interaction of some behavioural agents with the real-world process.
We model each trajectory as follows.
First, we represent the action chosen by the agent at time $\tx \in \{1, \ldots, \T\}$ as an $\Aspace_\tx$-valued random variable $\A_\tx$.
We then obtain a trajectory in our dataset by recording at each step the action $\A_\tx$ chosen and the observation $\X_\tx(\A_{1:\tx})$ corresponding to this choice of action.
As a result, each observed trajectory has the following form:
\begin{equation} \label{eq:observed-data-trajectory}
    \X_0, \A_1, \X_1(\A_1), \ldots, \A_\T, \X_\T(\A_{1:\T}).
\end{equation}
This corresponds to the standard \emph{consistency} assumption in causal inference \citep{hernan2020causal}, and intuitively means that the potential outcome $\X_\tx(\ax_{1:\tx})$ is observed in the data when the agent actually chose $\A_{1:\tx} = \ax_{1:\tx}$.
We model our full dataset as a set of i.i.d.\ copies of \eqref{eq:observed-data-trajectory}. %

\subsection{Certification is unsound in general}

A natural high-level strategy for twin assessment has the following structure.
First, some hypothesis $\Hyp$ is chosen with the following property:
 \begin{equation} \label{eq:verification-hypothesis-property}
    \text{If $\Hyp$ is true, then the twin is interventionally correct.}%
\end{equation}
Data is then used to try to show $\Hyp$ is true, perhaps up to some level of confidence. %
If successful, it follows by construction that the twin is interventionally correct.
Assessment procedures designed to \emph{certify} the twin in this way are appealing because they promise a strong guarantee of accuracy for certified twins. %
Unfortunately, the following foundational result from the causal inference literature (often referred to as the \emph{fundamental problem of causal inference} \citep{holland1986statistics}) means that data-driven certification procedures of this kind are in general unsound, as we explain next.
For completeness, Section \ref{sec:non-identifiability-result-proof-supp} of the \AppendixName includes a self-contained proof of this result in our notation.

\begin{theorem} \label{prop:nonidentifiability}
    If $\Prob(\A_{1:\T} \neq \ax_{1:\T}) > 0$, then the distribution of $\X_{0:\T}(\ax_{1:\T})$ is not uniquely identified by the distribution of the data in \eqref{eq:observed-data-trajectory} without further assumptions.
\end{theorem}

Since the distribution of the data encodes the information that would be contained in an infinitely large dataset of trajectories, Theorem \ref{prop:nonidentifiability} imposes a fundamental limit on what can be learned about the distribution of $\X_{0:\T}(\ax_{1:\T})$ from the data we have assumed.
It follows that if $\Hyp$ is any hypothesis satisfying \eqref{eq:verification-hypothesis-property}, then $\Hyp$ cannot be determined to be true from even an infinitely large dataset.
This is because, if we could do so, then we could also determine the distribution of $\X_{0:\T}(\ax_{1:\T})$, since by Proposition \ref{prop:interventional-correctness-alternative-characterisation} this would be equal to the distribution of $(\X_0, \Xt_{\T}(\X_0, \ax_{1:\T}))$.
In other words, we cannot use the data alone to certify that the twin is interventionally correct.

\subsection{The assumption of no unmeasured confounding} \label{sec:no-unmeasured-confounding-assumption}

Theorem \ref{prop:nonidentifiability} is true in the general case, when no additional assumptions about the data-generating process are made.
One way forward is therefore to introduce assumptions under which the distribution of $\X_{0:\T}(\ax_{1:\T})$ \emph{can} be identified.
This would mean it is possible to certify that the twin is interventionally correct, since, at least in principle, we could simply check whether this matches the distribution of $(\X_0, \Xt_{1:\T}(\X_{0}, \ax_{1:\T}))$ produced by the twin.

The most common such assumption in the causal inference literature is that the data are free of \emph{unmeasured confounding}.
Informally, this holds when each action $\A_\tx$ is chosen by the behavioural agent solely on the basis of the information available at time $\tx$ that is actually recorded in the dataset, namely $\X_{0}, \A_1, \X_1(\A_1), \ldots, \A_{\tx-1}, \X_{\tx-1}(\A_{1:\tx-1})$, as well as possibly some additional randomness that is independent of the real-world process, such as the outcome of a coin toss. (This can be made precise via the \emph{sequential randomisation assumption} introduced by \cite{robins1986new}.)
Unobserved confounding is present whenever this does not hold, i.e.\ whenever some unmeasured factor simultaneously influences both the agent's choice of action and the observation produced by the real-world process.

It is reasonable to assume that the data are unconfounded in certain contexts.
For example, in certain situations it may be possible to gather data in a way that specifically guarantees there is no confounding.
Randomised controlled trials, which ensure that each $\A_\tx$ is chosen via a carefully designed randomisation procedure \citep{lavori2004dynamic,murphy2005experimental}, constitute a widespread example of this approach.
Likewise, it is possible to show that the data are unconfounded if each $\X_\tx(\ax_{1:\tx})$ is a deterministic function of $\X_{0:\tx-1}(\ax_{1:\tx-1})$ and $\ax_{\tx}$, which may be reasonable to assume for example in certain low-level physics or engineering contexts.
(See Section \ref{eq:deterministic-potential-outcomes-are-unconfounded} of the \AppendixName for a proof.)
However, for stochastic phenomena and for typical datasets, it is widely acknowledged that the assumption of no unmeasured confounding will rarely hold, and so assessment procedures based on this assumption may yield unreliable results in practice \citep{murphy2003optimal,tsiatis2019dynamic}.
Section \ref{sec:motivating-example} of the \AppendixName illustrates this concretely with a toy scenario.

\subsection{General-purpose assessment via falsification}

Our goal is to obtain an assessment methodology that is general-purpose, and as such we would like to avoid introducing assumptions such as unconfoundedness that do not hold in general.
To achieve this, borrowing philosophically from \cite{popper2005logic}, we propose a strategy that replaces the goal of verifying the interventional correctness of the twin with that of \emph{falsifying} it.
Specifically, we consider hypotheses $\Hyp$ with the dual property to \eqref{eq:verification-hypothesis-property}, namely:
\begin{equation} \label{eq:falsification-hypothesis-property}
    \text{If the twin is interventionally correct, then $\Hyp$ is true.}
\end{equation}
We will then try to show that each such $\Hyp$ is \emph{false}.
Whenever we are successful, we will thereby have gained some knowledge about a failure mode of the twin, since by construction the twin can only be correct if $\Hyp$ is true.
In effect, each $\Hyp$ we falsify will constitute a \emph{reason} that the twin is not correct, and may suggest concrete improvements to its design, or may identify cases where its output should not be trusted.

Importantly, unlike for \eqref{eq:verification-hypothesis-property}, Theorem \ref{prop:nonidentifiability} does not preclude the possibility of data-driven assessment procedures based on \eqref{eq:falsification-hypothesis-property}.
As we show below, there do exist hypotheses $\Hyp$ satisfying \eqref{eq:falsification-hypothesis-property} that can in principle be determined to be false from the data alone without additional assumptions.
In this sense, falsification provides a means for \emph{sound} data-driven twin assessment, %
whose results can be relied upon across a wide range of circumstances.
On the other hand, falsification approaches cannot provide a \emph{complete} guarantee about the accuracy of a twin: even if we fail to falsify many $\Hyp$ satisfying \eqref{eq:falsification-hypothesis-property}, we cannot then infer that the twin is correct.
As such, in situations where (for example) it is reasonable to believe that the data are in fact unconfounded, it may be desirable to use this assumption to obtain additional information about the twin than is possible from falsification alone.

\section{Longitudinal causal bounds} \label{sec:causal-bounds}

\subsection{Statement of result} \label{sec:causal-bounds-statement}

One possible means for obtaining interventional information about the twin is via the classical bounds proposed by \cite{manski}.
These bounds hold without further assumptions, and so could in principle give rise to a sound falsification procedure of the kind we are seeking.
However, although they have been successfully applied in various cases, Manski's bounds are often very conservative, and so would not lead to very informative results if used directly.
To address this, we propose a novel generalisation of these bounds that explicitly accounts for the temporal structure of our setting.
As we explain below, our bounds can become considerably more informative than those of Manski, while also not requiring the addition of untestable causal assumptions.
We provide these bounds next, along with several theoretical results about their behaviour and optimality.
In Section \ref{sec:hypotheses-from-causal-bounds}, we use these bounds to define a class of $\Hyp$ with the desired property \eqref{eq:falsification-hypothesis-property}, which then yields a procedure for falsifying twins through hypothesis testing techniques.

\begin{theorem} \label{thm:causal-bounds}
    Suppose $(\Y(\ax_{1:\tx}) : \ax_{1:\tx} \in \Aspace_{1:\tx})$ are real-valued potential outcomes defined jointly with $(\X_{0:\T}(\ax_{1:\T}) : \ax_{1:\T} \in \Aspace_{1:\T})$ and $\A_{1:\T}$. If for some $\tx \in \{1, \ldots, \T\}$, $\ax_{1:\tx} \in \Aspace_{1:\tx}$, measurable $\B_{0:\tx} \subseteq \Xspace_{0:\tx}$, and $\ylo, \yup \in \R$ we have
\noindent
    \begin{gather}
        \Prob(\X_{0:\tx}(\ax_{1:\tx}) \in \B_{0:\tx}) > 0 \label{eq:B-positivity-assumption} \\
        \Prob(\ylo \leq \Y(\ax_{1:\tx}) \leq \yup \mid \X_{0:\tx}(\ax_{1:\tx}) \in \B_{0:\tx}) = 1,  \label{eq:Y-boundedness-assumption}
    \end{gather}
    \noindent then it holds that
    \begin{equation} \label{eq:causal-bounds-fully-written}
        \E[\Ylo \mid \X_{0:\N}(\A_{1:\N}) \in \B_{0:\N}]
        \leq \E[\Y(\ax_{1:\tx}) \mid \X_{0:\tx}(\ax_{1:\tx}) \in \B_{0:\tx}]
        \leq  \E[\Yup \mid \X_{0:\N}(\A_{1:\N}) \in \B_{0:\N}],
    \end{equation}
    where we define $\N \coloneqq \max \{0 \leq \sx \leq \tx \mid \A_{1:\sx} = \ax_{1:\sx}\}$, and similarly
    \begin{align*}
        \Ylo \coloneqq \ind(\A_{1:\tx} &= \ax_{1:\tx}) \, \Y(\A_{1:\tx}) + \ind(\A_{1:\tx} \neq \ax_{1:\tx}) \, \ylo \\
        \Yup \coloneqq \ind(\A_{1:\tx} &= \ax_{1:\tx}) \, \Y(\A_{1:\tx}) + \ind(\A_{1:\tx} \neq \ax_{1:\tx}) \, \yup.
    \end{align*}

\end{theorem}
(See Section \ref{sec:causal-bounds-proofs} of the \AppendixName for a proof.)

For brevity, in what follows we will write the terms in \eqref{eq:causal-bounds-fully-written} as $\Qlo$, $\Q$, and $\Qup$ respectively, so the conclusion of this result becomes $\Qlo \leq \Q \leq \Qup$.

Intuitively, $\Y(\ax_{1:\tx})$ here may be thought of as some quantitative outcome of interest.
For example, in a medical context, $\Y(\ax_{1:\tx})$ might represent the heart rate of a patient at time $\tx$ after receiving some treatments $\ax_{1:\tx}$.
When defining our hypotheses below, we consider the specific form $\Y(\ax_{1:\tx}) \coloneqq \fx(\X_{0:\tx}(\ax_{1:\tx}))$, where $\fx : \Xspace_{0:\tx} \to \R$ is some scalar function.
The value $\Q$ is then simply the (conditional) average behaviour of this outcome.
By Theorem \ref{prop:nonidentifiability}, $\Q$ is in general not identified by the data since it depends on $\X_{0:\tx}(\ax_{1:\tx})$.
On the other hand, both $\Qlo$ and $\Qup$ \emph{are} identified, since the relevant random variables $\Ylo$, $\Yup$, $\N$, and $\X_{0:\N}(\A_{1:\N})$ can all be expressed as functions of the observed data $\X_{0:\tx}(\A_{1:\tx})$ and $\A_{1:\tx}$.
In this way, Theorem \ref{thm:causal-bounds} bounds the behaviour of a non-identifiable quantity in terms of identifiable ones.
At a high level, this is achieved by replacing $\Y(\ax_{1:\tx})$, whose value is only observed on the event $\{\A_{1:\tx} = \ax_{1:\tx}\}$, with $\Ylo$ and $\Yup$, which are equal to $\Y(\ax_{1:\tx})$ when its value is observed (i.e.\ when $\A_{1:\tx} = \ax_{1:\tx}$), and which fall back to the worst-case values of $\ylo$ and $\yup$ otherwise.
We emphasise that Theorem \ref{thm:causal-bounds} does not require any additional causal assumptions, and in particular remains true under arbitrary unmeasured confounding.

In the structural causal modelling framework \citep{pearl2009causality}, a related result to Theorem \ref{thm:causal-bounds} was given as Corollary 1 by \citet{bareinboim}.
However, their result involves a complicated ratio of unknown quantities that makes estimation of their bounds difficult, since it is not obvious how to obtain an unbiased estimator for their ratio term.
In contrast, our proposed causal bounds are considerably simpler, since both $\Qlo$ and $\Qup$ here are expressed as (conditional) expectations. 
This makes their unbiased estimation straightforward, which we use to obtain exact confidence intervals for both terms in Section \ref{sec:statistical-methodology}.

\subsection{Informativeness} \label{sec:informativeness-of-our-bounds}

For Theorem \ref{thm:causal-bounds} to be useful in practice, we would like the bounds $[\Qlo, \Qup]$ to be relatively narrower than the worst-case bounds $[\ylo, \yup]$ that are trivially implied by \eqref{eq:Y-boundedness-assumption}.
We can quantify the extent to which this occurs by the ratio
\begin{equation} \label{eq:tightness-of-bounds-measure}
    \frac{\Qup - \Qlo}{\yup - \ylo} = 1 - \Prob(\A_{1:\tx} = \ax_{1:\tx} \mid \X_{0:\N}(\A_{1:\N}) \in \B_{0:\N}),
\end{equation}
where the equality here follows from the definitions of $\Qlo$ and $\Qup$ together with some straightforward manipulations.
In other words, the (relative) tightness of our bounds is determined by the value of $\Prob(\A_{1:\tx} = \ax_{1:\tx} \mid \X_{0:\N}(\A_{1:\N}) \in \B_{0:\N})$, which is itself closely related to the classical \emph{propensity score} in the causal inference literature \citep{rosenbaum1983central}. %
Intuitively, as this probability grows larger, so too does $\Prob(\Y(\ax_{1:\tx}) = \Y(\A_{1:\tx}) \mid \X_{0:\N}(\A_{1:\N}) \in \B_{0:\N})$, which means the effect of unmeasured confounding on the value of $\Q$ is reduced, leading to tighter bounds.

Theorem \ref{thm:causal-bounds} is a generalisation of the bounds proposed by \cite{manski}, which can be recovered as the case where $\B_{0:\tx} = \Xspace_{0:\tx}$, so that \eqref{eq:causal-bounds-fully-written} becomes $\E[\Ylo] \leq \E[\Y(\ax_{1:\tx})] \leq \E[\Yup]$.
In practice, Manski's result is often regarded as quite uninformative.
From \eqref{eq:tightness-of-bounds-measure}, this is true whenever $\Prob(\A_{1:\tx} = \ax_{1:\tx})$ is small, which often occurs in many applications, particularly for longer action sequences.
On the other hand, in many contexts it seems reasonable to anticipate that certain longer action sequences will be fairly likely to occur when conditioned on some intermediate observations.
In other words, $\Prob(\A_{1:\tx} = \ax_{1:\tx} \mid \X_{0:\N}(\A_{1:\N}) \in \B_{0:\N})$ may be large, even if $\Prob(\A_{1:\tx} = \ax_{1:\tx})$ is not.
By choosing $\B_{0:\tx}$ carefully, we can therefore obtain tighter bounds than would be possible by using Manski's original result.
The following straightforward result provides a sufficient condition for this to hold.
In Section \ref{sec:case-study} below, we also show empirically that Theorem \ref{thm:causal-bounds} yields more informative results in our case study compared with Manski's original bounds.

\begin{proposition} \label{prop:our-bounds-vs-manskis}
    Consider the same setup as Theorem \ref{thm:causal-bounds}, where also $\ylo \leq \Y(\ax_{1:\tx}) \leq \yup$ almost surely.
    If $\Prob(\A_{1:\tx} = \ax_{1:\tx} \mid \X_{0:\N}(\A_{1:\N}) \in \B_{0:\N}) > \Prob(\A_{1:\tx} = \ax_{1:\tx})$, then the width of Manski's bounds exceeds that of Theorem \ref{thm:causal-bounds}, i.e.\ $\E[\Yup] - \E[\Ylo] > \Qup - \Qlo$.
\end{proposition}

Beyond allowing us to obtain tighter bounds, the conditional nature of Theorem \ref{thm:causal-bounds} also appears of interest simply for its own sake.
In particular, Theorem \ref{thm:causal-bounds} describes the interventional behaviour of $\Y(\ax_{1:\tx})$ conditional on the behaviour of the trajectory $\X_{0:\tx}(\ax_{1:\tx})$, thereby providing more granular information than can be obtained from unconditional bounds of \cite{manski} alone.

\subsection{Optimality}

The following result shows that Theorem \ref{thm:causal-bounds} cannot be improved without further assumptions. 
Intuitively speaking, there always exists \emph{some} family of potential outcomes that produces the same observational data as our model, but that attains the worst-case bounds $\Qlo$ or $\Qup$.
Therefore, we cannot rule out the possibility that the true potential outcomes achieve $\Qlo$ or $\Qup$ from the observational data alone.

\begin{proposition} \label{prop:sharpness-of-bounds}
    Under the same setup as in Theorem \ref{thm:causal-bounds}, there always exists potential outcomes $(\tilde{\X}_{0:\T}(\ax_{1:\T}'), \tilde{\Y}(\ax_{1:\tx}') : \ax_{1:\T}' \in \Aspace_{1:\T})$ also satisfying \eqref{eq:Y-boundedness-assumption} (mutatis mutandis) with 
    $$(\tilde{\X}_{0:\T}(\A_{1:\T}), \tilde{\Y}(\A_{1:\tx}), \A_{1:\T}) \eqas (\X_{0:\T}(\A_{1:\T}), \Y(\A_{1:\tx}), \A_{1:\T})$$ 
    but for which $$\E[\tilde{\Y}(\ax_{1:\tx}) \mid \tilde{\X}_{0:\tx}(\ax_{1:\tx}) \in \B_{0:\tx}] = \Qlo.$$
    The corresponding statement is also true for $\Qup$.
\end{proposition}

Apart from attempting to tighten our bounds on $\Q$, in some cases we may wish to consider bounding the alternative quantity $\E[\Y(\ax_{1:\tx}) \mid \X_{0:\tx}(\ax_{1:\tx})]$ that conditions on the \emph{value} of $\X_{0:\tx}(\ax_{1:\tx})$ rather than on the event $\{\X_{0:\tx}(\ax_{1:\tx}) \in \B_{0:\tx}\}$.
To achieve this, it is natural to generalize our assumption \eqref{eq:Y-boundedness-assumption} by supposing we now have measurable functions $\ylo, \yup : \Xspace_{0:\tx} \to \R$ such that
\begin{equation} \label{eq:worst-case-behaviour-functional-assumption}
    \ylo(\X_{0:\tx}(\ax_{1:\tx})) \leq \Y(\ax_{1:\tx}) \leq \yup(\X_{0:\tx}(\ax_{1:\tx})) \qquad \qquad \text{almost surely,}
\end{equation}
and our goal is to obtain measurable functions $\lo{\gx}, \up{\gx} : \Xspace_{0:\tx} \to \R$ such that
\begin{equation} \label{eq:almost-sure-bound-desired-result}
    \lo{\gx}(\X_{0:\tx}(\ax_{1:\tx})) \leq \E[\Y(\ax_{1:\tx}) \mid \X_{0:\tx}(\ax_{1:\tx})] \leq \up{\gx}(\X_{0:\tx}(\ax_{1:\tx})) \qquad \qquad \text{almost surely.}
\end{equation}
As we describe in Section \ref{sec:impossibility-of-bounds-for-continuous-data} of the \AppendixName, bounds of this kind can be obtained directly from Theorem \ref{thm:causal-bounds} if $\X_{0:\tx}(\ax_{1:\tx})$ is discrete, or by a simple modification of the proof of Theorem \ref{thm:causal-bounds} if $\X_{1:\tx}(\ax_{1:\tx})$ is discrete (but $\X_0$ is possibly continuous).
However, somewhat surprisingly, in general we cannot obtain nontrivial bounds of this kind without further assumptions beyond the discrete case.
To make this precise, we will say that a given $\lo{\gx}$ and $\up{\gx}$ are \emph{permissible} if \eqref{eq:almost-sure-bound-desired-result} holds when $(\X_{0:\T}(\ax_{1:\tx}), \Y(\ax_{1:\tx}), \A_{1:\T} : \ax_{1:\T}' \in \Aspace_{1:\T})$ are replaced by any potential outcomes $(\tilde{\X}_{0:\T}(\ax_{1:\T}'), \tilde{\Y}(\ax_{1:\tx}'), \tilde{\A}_{1:\T} : \ax_{1:\T}' \in \Aspace_{1:\T})$  for which $\Law[\tilde{\X}_{0:\T}(\tilde{\A}_{1:\T}), \tilde{\A}_{1:\T}, \tilde{\Y}(\tilde{\A}_{1:\tx})] = \Law[\X_{0:\T}(\A_{1:\T}), \A_{1:\T}, \Y(\A_{1:\tx})]$, and which also satisfy \eqref{eq:worst-case-behaviour-functional-assumption} (mutatis mutandis).
Intuitively, this means that $\lo{\gx}$ and $\up{\gx}$ depend only on the information we have available, i.e.\ the observational distribution and our assumed worst-case values.
We then have the following:

\begin{theorem} \label{thm:no-causal-bounds-for-continuous-data}
    Suppose $\X_0$ is almost surely constant, $\Prob(\A_1 \neq \ax_1) > 0$, and for some $\sx \in \{1, \ldots, \tx\}$ we have $\Prob(\X_{\sx}(\A_{1:\sx}) = \xx_{\sx}) = 0$ for all $\xx_{\sx} \in \Xspace_{\sx}$.
    Then $\lo{\gx}, \up{\gx} : \Xspace_{0:\tx} \to \R$ are permissible bounds only if they are trivial, i.e.\ 
    \begin{align*}
        \lo{\gx}(\X_{0:\tx}(\ax_{1:\tx})) \leq \ylo(\X_{0:\tx}(\ax_{1:\tx})) \quad \textup{and} \quad \up{\gx}(\X_{0:\tx}(\ax_{1:\tx})) \geq \yup(\X_{0:\tx}(\ax_{1:\tx})) \quad \textup{almost surely}.
    \end{align*}
\end{theorem}

Here the assumption that $\X_0$ is constant essentially means we consider a special case of our model where there are no covariates available before the first action is taken, and serves mainly to simplify the proof.
We conjecture that Theorem \ref{thm:no-causal-bounds-for-continuous-data} holds more generally, provided the other assumptions are accordingly made to be conditional on $\X_0$ also.
In any case, this result shows that general purpose bounds on $\E[\Y(\ax_{1:\tx}) \mid \X_{0:\tx}(\ax_{1:\tx})]$ are not forthcoming, and Theorem \ref{thm:causal-bounds} is the best we can hope for in general.
We show below that this result is nevertheless powerful enough to obtain useful information in practice.

\section{Falsification methodology} \label{sec:hypotheses-from-causal-bounds}

\subsection{Hypotheses derived from causal bounds} \label{sec:hypotheses-from-causal-bounds-setup}

We now use Theorem \ref{thm:causal-bounds} to obtain a hypothesis testing procedure that can be used to falsify the twin, and that does not rely on any further assumptions than we have already provided.
To this end, we first define the hypotheses $\Hyp$ satisfying \eqref{eq:falsification-hypothesis-property} that we will consider. %
Each of these will depend on a specific choice of the following parameters:

\begin{multicols}{2}
    \begin{itemize}
        \item A timestep $\tx \in \{1, \ldots, \T\}$
        \item A measurable function $\fx : \Xspace_{0:\tx} \to \R$
    \end{itemize}
    \columnbreak
    \begin{itemize}
        \item A sequence of actions $\ax_{1:\tx} \in \Aspace_{1:\tx}$
        \item A sequence of subsets $B_{0:t} \subseteq \Xspace_{0:t}$. 
    \end{itemize}
\end{multicols}
\noindent
To streamline notation, in this section, we will consider these parameters to be fixed.
However, we emphasize that our construction can be instantiated for many different choices of these parameters, and indeed we will do so in our case study below.
We think of $\fx$ as expressing a specific outcome of interest at time $\tx$ in terms of the data we have assumed.
Accordingly, for each $\ax_{1:\tx}' \in \Aspace_{1:\tx}$, we define new potential outcomes
$\Y(\ax_{1:\tx}') \coloneqq \fx(\X_{0:\tx}(\ax_{1:\tx}'))$.
For example, in a medical context, if $\X_{0:\tx}(\ax_{1:\tx}')$ represents a full patient history at time $\tx$ after treatments $\ax_{1:\tx}'$, then $\Y(\ax_{1:\tx}')$ might represent the patient's heart rate after these treatments. 
Likewise, $\B_{0:\tx}$ selects a subgroup of patients of interest, e.g.\ elderly patients whose blood pressure values were above some threshold at some timesteps before $\tx$.

The hypotheses we consider are based on the corresponding outcome produced by the twin when initialised at $\X_0$, which we define for $\ax_{1:\tx}' \in \Aspace_{1:\tx}$ as $\Yt(\ax_{1:\tx}') \coloneqq \fx(\X_0, \Xt_{1:\tx}(\X_0, \ax_{1:\tx}'))$.
Supposing it holds that
\begin{equation} \label{eq:B-twin-positivity-assumption}
    \Prob(\Xt_{1:\tx}(\X_0, \ax_{1:\tx}) \in \B_{1:\tx}) > 0,
\end{equation}
we may then define
$\Qt \coloneqq \E[\Yt(\ax_{1:\tx}) \mid \X_0 \in \B_0, \Xt_{1:\tx}(\X_0, \ax_{1:\tx}) \in \B_{1:\tx}]$, i.e.\
the analogue of $\Q$ for the twin.
By Proposition \ref{prop:interventional-correctness-alternative-characterisation}, if the twin is interventionally correct, then $\Qt = \Q$.
Theorem \ref{thm:causal-bounds} therefore implies that the following hypotheses have our desired property \eqref{eq:falsification-hypothesis-property}:
\begin{gather*}
\text{$\Hlo$: If \eqref{eq:B-positivity-assumption}, \eqref{eq:Y-boundedness-assumption}, and \eqref{eq:B-twin-positivity-assumption} hold, then $\Qt \geq \Qlo$} \\
\text{$\Hup$: If \eqref{eq:B-positivity-assumption}, \eqref{eq:Y-boundedness-assumption}, and \eqref{eq:B-twin-positivity-assumption} hold, then $\Qt \leq \Qup$.}
\end{gather*}
Moreover, $\Hlo$ and $\Hup$ can in principle be determined to be true or false from the information we have assumed available, since $\Qlo$ and $\Qup$ depend only on the observational data, and $\Qt$ can be estimated by generating trajectories from the twin.

When either $\Hlo$ or $\Hup$ is falsified, it immediately follows that the twin is not interventionally correct.
However, even more than this, a falsification describes a concrete failure mode with various potential implications downstream, which is considerably more useful information about the twin in practice.
For example, if $\Hlo$ is false (i.e.\ if $\Qt < \Qlo$), it follows that, among those trajectories for which $(\X_0, \Xt_{1:\tx}(\X_0, \ax_{1:\tx})) \in \B_{0:\tx}$, the mean of $\Yt(\ax_{1:\tx})$ is too small.
(Higher moments could also be considered by choosing $\fx$ appropriately.) %
In light of this, a user might choose not to rely on outputs of the twin produced under these circumstances, while a developer seeking to improve the twin could focus their attention on the specific parts of its implementation that give rise to this behaviour.
We illustrate this concretely through our case study in Section \ref{sec:case-study}.

\subsection{Exact testing procedure} \label{sec:statistical-methodology}

We now describe a procedure for testing $\Hlo$ and $\Hup$ using a finite dataset that obtains exact control over type I error without relying on additional assumptions or asymptotic approximations.
We show in our case study below that this procedure is nevertheless powerful enough to obtain useful information about a twin in practice.

We focus here on obtaining a p-value for $\Hlo$ given a fixed choice of parameters $(\tx, \fx, \ax_{1:\tx}, \B_{0:\tx})$.
Our procedure for $\Hup$ is symmetrical, or may be regarded as a special case of testing $\Hlo$ by replacing $\fx$ with $-\fx$.
Multiple hypotheses may then be handled via standard techniques such as the method of \cite{holm1979simple} (which we use in our case study) or of \cite{benjamini2001control}, both of which apply without additional assumptions.

As above, we assume access to a finite dataset $\D$ of i.i.d.\ copies of \eqref{eq:observed-data-trajectory}.
We also assume access to a dataset $\Dt(\ax_{1:\tx})$ of i.i.d.\ copies of
\begin{equation} \label{eq:twin-trajecs}
    \X_0, \Xt_1(\X_0, \ax_1), \ldots, \Xt_\tx(\X_0, \ax_{1:\tx}).
\end{equation}
In practice, these copies can be obtained by initialising the twin with some value $\X_0$ taken from $\D$ without replacement and supplying inputs $\ax_{1:\tx}$.
If each $\X_0$ in $\D$ is used to initialize the twin at most once, then the resulting trajectories in $\Dt(\ax_{1:\tx})$ are guaranteed to be i.i.d., since we assumed in Section \ref{sec:causal-formulation} that the potential outcomes $\Xt_\tx(\xx_0, \ax_{1:\tx})$ produced by the twin are independent across runs.
We adopt this approach in our case study.

Observe that $\Hlo$ is false only if \eqref{eq:B-positivity-assumption}, \eqref{eq:Y-boundedness-assumption}, and \eqref{eq:B-twin-positivity-assumption} all hold.
We account for this in our testing procedure as follows.
First, \eqref{eq:Y-boundedness-assumption} immediately follows if 
\begin{equation} \label{eq:fx-boundedness-assumption}
    \ylo \leq \fx(\xx_{0:\tx}) \leq \yup \qquad \text{for all $\xx_{0:\tx} \in \B_{0:\tx}$.}
\end{equation}
This holds automatically in certain cases, such as for binary outcomes (e.g.\ patient survival), or otherwise can be enforced simply by clipping the value of $\fx$ to live within $[\ylo, \yup]$.
We describe a practical means for choosing $\fx$ in this way in Section \ref{sec:case-study}.

To account for \eqref{eq:B-positivity-assumption} and \eqref{eq:B-twin-positivity-assumption}, we simply check whether there exists some trajectory in $\D$ with $\A_{1:\tx} = \ax_{1:\tx}$ and $\X_{0:\tx}(\A_{1:\tx}) \in \B_{0:\tx}$, and some trajectory in $\Dt(\ax_{1:\tx})$ with $(\X_0, \Xt_{1:\tx}(\X_0, \ax_{1:\tx})) \in \B_{0:\tx}$.
If there are not, then we refuse to reject $\Hlo$ at any significance level; otherwise, we proceed to test $\Qt \geq \Qlo$ as described next.
It easily follows that there is zero probability we will reject $\Hlo$ if \eqref{eq:B-positivity-assumption} and \eqref{eq:B-twin-positivity-assumption} do not in fact hold, and so our overall type I error is controlled at the desired level.

To test $\Qt \geq \Qlo$, we begin by constructing a one-sided lower confidence interval for $\Qlo$, and a one-sided upper confidence interval for $\Qt$.
In detail, for each significance level $\alpha \in (0, 1)$, we obtain $\qlo{\alpha}$ and $\qt{\alpha}$ as functions of $\D$ and $\Dt(\ax_{1:\tx})$ such that
\begin{equation}
    \Prob(\Qlo \geq \qlo{\alpha}) \geq 1 - \frac{\alpha}{2} \qquad\,\, \Prob(\Qt \leq \qt{\alpha}) \geq 1 - \frac{\alpha}{2}. \label{eq:q-confidence-interval-new}
\end{equation}
We will also ensure that these are nested, i.e.\ $\qlo{\alpha} \leq \qlo{\alpha'}$ and $\qt{\alpha'} \leq \qt{\alpha}$ if $\alpha \leq \alpha'$.
We describe two methods for obtaining $\qlo{\alpha}$ and $\qt{\alpha}$ satisfying these conditions below.

From these confidence intervals, we obtain a test for the hypothesis $\Qt \geq \Qlo$ that rejects when $\qt{\alpha} < \qlo{\alpha}$.
A straightforward argument given in Section \ref{sec:hyp-testing-supplement} of the \AppendixName shows that this controls the type I error at the desired level $\alpha$.
Nestedness also yields a $p$-value obtained as the smallest value of $\alpha$ for which this test rejects, i.e.\
$\plo \coloneqq \inf\{\alpha \in (0, 1) \mid \text{$\qt{\alpha} < \qlo{\alpha}$}\}$,
or $1$ if the test does not reject at any level.

We now consider how to obtain confidence intervals for $\Qlo$ and $\Qt$ satisfying the desired conditions above.
To this end, observe that both quantities are (conditional) expectations involving random variables that can be computed from $\D$ or $\Dt(\ax_{1:\tx})$.
This allows both to be estimated unbiasedly, which in turn can be used to derive confidence intervals via standard techniques.
For example, consider the subset of trajectories in $\Dt(\ax_{1:\tx})$ with $(\X_0, \Xt_{\tx}(\X_0, \ax_{1:\tx})) \in \B_{0:\tx}$.
For each such trajectory, we obtain a corresponding value of $\Yt(\ax_{1:\tx})$ that is i.i.d.\ and has expectation $\Qt$.
Similarly, for $\Qlo$, we extract the subset of trajectories in $\D$ for which $\X_{0:\N}(\A_{1:\N}) \in \B_{0:\N}$ holds. %
The values of $\Ylo$ obtained from each such trajectory are then i.i.d.\ and have expectation $\Qlo$.

At this point, our problem now reduces to that of constructing a confidence interval for the expectation of a random variable using i.i.d.\ copies of it.
Various techniques exist for this, and we consider two possibilities in our case study.
The first leverages the fact that $\Yt(\ax_{1:\tx})$ and $\Ylo$ are bounded in $[\ylo, \yup]$, which gives rise to $\qlo{\alpha}$ and $\qt{\alpha}$ via an application of Hoeffding's inequality.
This approach has the appealing property that \eqref{eq:q-confidence-interval-new} holds exactly, although often at the expense of conservativeness.
In practice, this could be mitigated by instead obtaining confidence intervals via (for example) the bootstrap \citep{efron1979bootstrap}, although at the expense of requiring (often mild) asymptotic assumptions.
Section \ref{sec:confidence-intervals-methodology-supplement} of the \AppendixName describes both methods in greater detail.
Our empirical results reported in the next section all use Hoeffding's inequality are hence exact, but we also provide additional results using bootstrapping in Section \ref{sec:boostrapping-details-supplement} of the \AppendixName.

\section{Case Study: Pulse Physiology Engine} \label{sec:case-study}

\subsection{Experimental setup} \label{sec:experimental-setup}

We applied our assessment methodology to the Pulse Physiology Engine \citep{pulse}, an open-source model for human physiology simulation.
Pulse simulates trajectories of various physiological metrics for patients with conditions like sepsis, COPD, and ARDS.
We describe the main steps of our experimental procedure and results below, with full details given in the Section \ref{sec:experiments-supplement} of the \AppendixName. 

We utilized the MIMIC-III dataset \citep{mimic}, a comprehensive collection of longitudinal health data from critical care patients at the Beth Israel Deaconess Medical Center (2001-2012).
We focused on patients meeting the sepsis-3 criteria \citep{sepsis-criteria}, following the methodology of \cite{ai-clinician} for selecting these.
This yielded 11,677 sepsis patient trajectories.
We randomly selected 5\% of these (583 trajectories, denoted as $\D_0$) to use for choosing the parameters of our hypotheses via a sample splitting approach \citep{cox1975note}, with the remaining 95\% (11,094 trajectories, denoted as $\D$) reserved for the actual testing.

We considered hourly observations of each patient over the the first four hours of their ICU stay, i.e.\ $\T = 4$.
We defined the observation spaces $\Xspace_{0:\T}$ using a total of 17 features included in our extracted MIMIC trajectories for this time period, including static demographic quantities and patient vitals.
Following \cite{ai-clinician}, the actions we considered involved the administration of intravenous fluids and vasopressors, which both play a primary role in the treatment of sepsis in clinical practice.
Since these are recorded in MIMIC as continuous doses, we discretised their values via the same procedure as \cite{ai-clinician}, obtaining finite action spaces $\Aspace_1 = \cdots = \Aspace_4$, each with $25$ distinct actions.

We defined a collection of hypothesis parameters $(\tx, \fx, \ax_{1:\tx}, \B_{0:\tx})$, each of which we then used to define an $\Hlo$ and $\Hup$ to test.
For this, we chose 14 different physiological quantities of interest to assess, including heart rate, skin temperature, and respiration rate (see Table \ref{tab:hypotheses_hoeffding_full} in the \AppendixName for a complete list).
For each of these, we selected combinations of $\tx$, $\ax_{1:\tx}$, and $\B_{0:\tx}$ observed for at least one patient trajectory in $\D_0$.
We took \(\ylo\) and \(\yup\) to be the .2 and .8 quantiles of the same physiological quantity as was recorded in $\D_0$, and defined $\fx$ as the function that extracts this quantity from \(\Xspace_\tx\) and clips its value between \(\ylo\) and \(\yup\), so that \eqref{eq:fx-boundedness-assumption} holds.
We describe this procedure in full in Section \ref{sec:hypothesis-parameters-supplement} of the \AppendixName.
We also investigated the sensitivity of our procedure to the choice of $\ylo$ and $\yup$ and found it to be relatively stable: see Section \ref{sec:sensitity-analysis-appendix} of the \AppendixName.
We obtained 721 unique parameter choices, each of which produced two hypotheses $\Hlo$ and $\Hup$, leading to 1,442 hypotheses in total.

We generated data from Pulse to test the chosen hypotheses.
For each $\ax_{1:\tx}$ occurring in any of our hypotheses, we obtained the dataset $\Dt(\ax_{1:\tx})$ as described in Section \ref{sec:statistical-methodology}.
Specifically, we sampled $\X_0$ without replacement from $\D$, and used this to initialize a twin trajectory.
Then, at each hour $\tx' \in \{1, \ldots, 4\}$ in the simulation, we administered a dose of intravenous fluids and vasopressors corresponding to $\ax_{\tx'}$ and recorded the resulting patient features generated by Pulse. 
We describe this procedure in full in Section \ref{sec:pulse-trajectories-supplement} in the \AppendixName.
This produced a total of 26,115 simulated trajectories.

\subsection{Hypothesis rejections} 

We tested the hypotheses just described using our methodology from Section \ref{sec:statistical-methodology}.
Here we report the results when using Hoeffding's inequality to obtain confidence intervals for $\Qlo$, $\Qup$, and $\Qt$.
We also tried confidence intervals obtained via bootstrapping, and obtained similar if less conservative results (see Section \ref{sec:experiments-supplement} of the \AppendixName).
We used the Holm-Bonferroni method to adjust for multiple tests, with family-wise error rate of 0.05.

\begin{table}%
    \centering
\begin{footnotesize}
\begin{tabular}{l|cc|cc}
& \multicolumn{2}{c}{Ours} & \multicolumn{2}{c}{Manski} \\
                   Physiological quantity &  Rejs. &  Hyps. &  Rejs. &  Hyps. \\
\midrule
  Chloride Blood Concentration (Chloride) &            24 &            94 & 1 & 46  \\
      Sodium Blood Concentration (Sodium) &            21 &            94 & 9 & 46  \\
Potassium Blood Concentration (Potassium) &            13 &            94 & 0 & 46  \\
                  Skin Temperature (Temp) &            10 &            86 & 9 & 46  \\
    Calcium Blood Concentration (Calcium) &             5 &            88 & 0 & 46  \\
    Glucose Blood Concentration (Glucose) &             5 &            96 & 1 & 46  \\
      Arterial CO$_2$ Pressure (paCO$_2$) &             3 &            70 & 0 & 46  \\
Bicarbonate Blood Concentration (HCO$_3$) &             2 &            90 & 1 & 46  \\
       Systolic Arterial Pressure (SysBP) &             2 &           154 & 0 & 46  \\
\bottomrule
\end{tabular}
\end{footnotesize}
    \caption{Total hypotheses (Hyps.) and rejections (Rejs.) per physiological quantity using our causal bounds, as well as those of \cite{manski}} \label{tab:hypotheses}
\end{table}

We obtained rejections for hypotheses corresponding to 10 different physiological quantities shown in Table \ref{tab:hypotheses}.
(Table \ref{tab:hypotheses_hoeffding_full} in the \AppendixName shows all hypotheses we tested, including those not rejected.)
We may therefore infer that, at a high level, Pulse does not simulate these quantities accurately for the population of sepsis patients we consider.
This appears of interest in a variety of downstream settings: for example, a developer could use this information when considering how to improve the accuracy of Pulse, while a practitioner using Pulse may wish to rely less on these outputs as a result.

To assess the relative performance of our bounds from Theorem \ref{thm:causal-bounds} compared with the unconditional bounds of \citet{manski}, we also reran this analysis with each $\tx$, $\fx$, and $\ax_{1:\tx}$ chosen as before, but with each $\B_{0:\tx}$ now set to the whole space $\Xspace_{0:\tx}$.
This in turn led to fewer hypotheses, namely 690 in total, which were evenly divided between hypotheses of the form $\E[\Yt(\ax_{1:\tx})] \geq \E[\Ylo]$ and those of the form $\E[\Yt(\ax_{1:\tx})] \leq \E[\Yup]$.
We kept all other aspects of our procedure the same as just described, including our method for obtaining confidence intervals and controlling for multiple testing.
The number of rejections that we obtained in this case is also shown in Table \ref{tab:hypotheses}.
As anticipated by our discussion in Section \ref{sec:informativeness-of-our-bounds}, the conservativeness of Manski's bounds led to considerably fewer rejections than our more general result given in Theorem \ref{thm:causal-bounds}, even when considered as a proportion of the total hypotheses we tested.

\subsection{p-value plots}

To obtain more granular information about the failure modes of the twin just identified, we examined the $p$-values obtained for each hypothesis $\Hlo$ and $\Hup$ tested using our causal bounds, which we denote here by $\plo$ and $\pup$.
Figure \ref{fig:p_values} shows the distributions of $-\log_{10}{\plo}$ and $-\log_{10}{\pup}$ that we obtained for all physiological quantities for which some hypothesis was rejected.
(The remaining $p$-values are shown in Figure \ref{fig:p_values_hoeff_complete} in the \AppendixName.)
Notably, in each row, one distribution is always tightly concentrated at $-\log_{10} p = 0$ (i.e.\ $p = 1$).
This means that, for all physiological outcomes of interest, there was either very little evidence in favour of rejecting any $\Hlo$, or very little in favour of rejecting any $\Hup$.
In other words, across configurations of $(\tx, \fx, \ax_{1:\tx}, \B_{0:\tx})$ that were rejected, the twin consistently either underestimated or overestimated each quantity on average.
For example, Pulse consistently underestimated chloride blood concentration and skin temperature, while it consistently overestimated sodium and glucose blood concentration levels.
Like Table \ref{tab:hypotheses}, this information appears of interest and actionable in a variety of downstream tasks.

\begin{figure}[t]
    \centering
    \includegraphics[height=4.5cm]{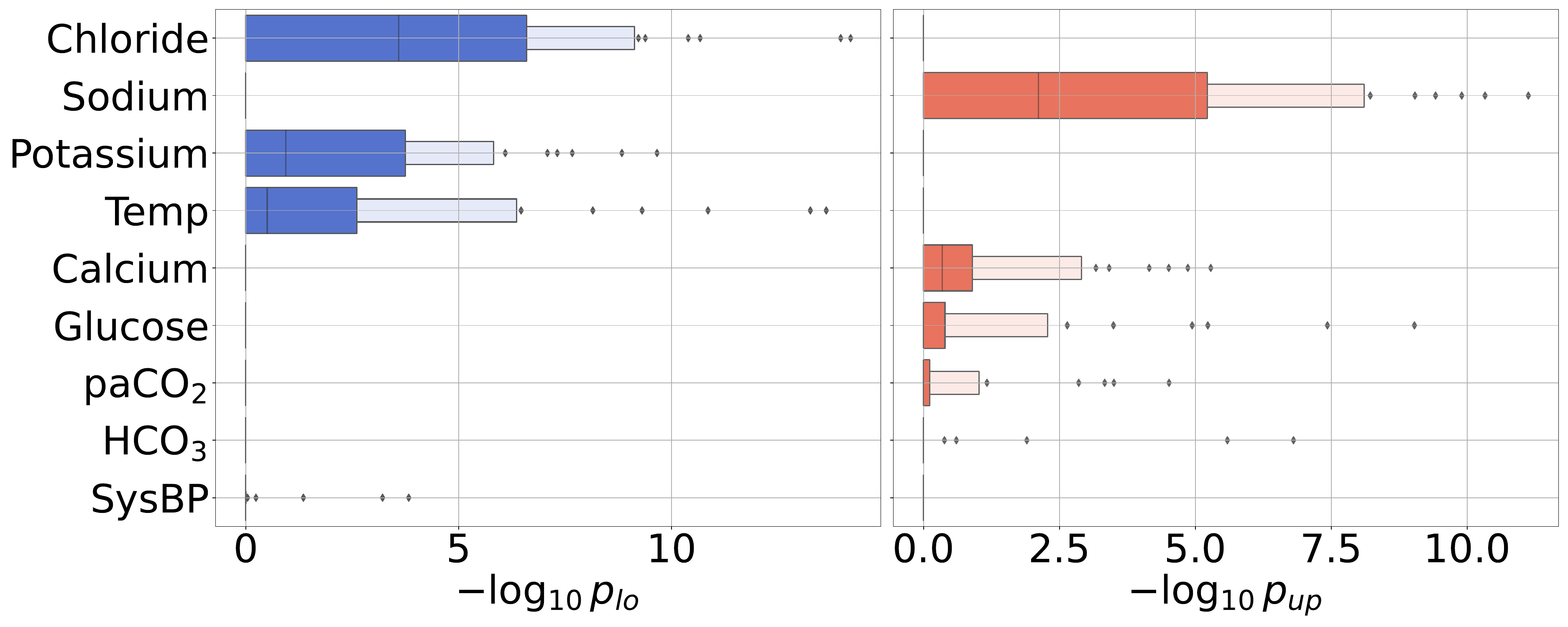}
    \caption{Distributions of $-\log_{10}{p_\textup{lo}}$ and $-\log_{10}{p_\textup{up}}$ across hypotheses, grouped by physiological quantity. Higher values indicate greater evidence in favour of rejection.}
    \label{fig:p_values}
\end{figure}

\subsection{Pitfalls of naive assessment}

A naive approach to twin assessment involves simply comparing the output of the twin with the observational data directly, without accounting for causal considerations.
We now show that, unlike our methodology, the results produced in this way can be potentially misleading.
In Figure \ref{fig:longitudinal_plots}, for two different choices of $(\ax_{1:4}, \B_{1:4})$, we plot estimates of $\Qt_\tx$ and $\Q^{\textup{obs}}_\tx$ for $\tx \in \{1, \ldots, 4\}$, where
\begin{align*}
    \Qt_\tx &\coloneqq \E[\Yt(\ax_{1:\tx})\mid \X_0 \in \B_0, \Xt_{1:\tx}(\X_0, \ax_{1:\tx}) \in \B_{1:\tx}] \\
    \Q^{\textup{obs}}_\tx &\coloneqq \E[\Y(\A_{1:\tx})\mid \X_{0:\tx}(\A_{1:\tx})\in \B_{0:\tx}, \A_{1:\tx}=\ax_{1:\tx}].
\end{align*}
\begin{figure}%
    \centering
    \includegraphics[height=4cm]{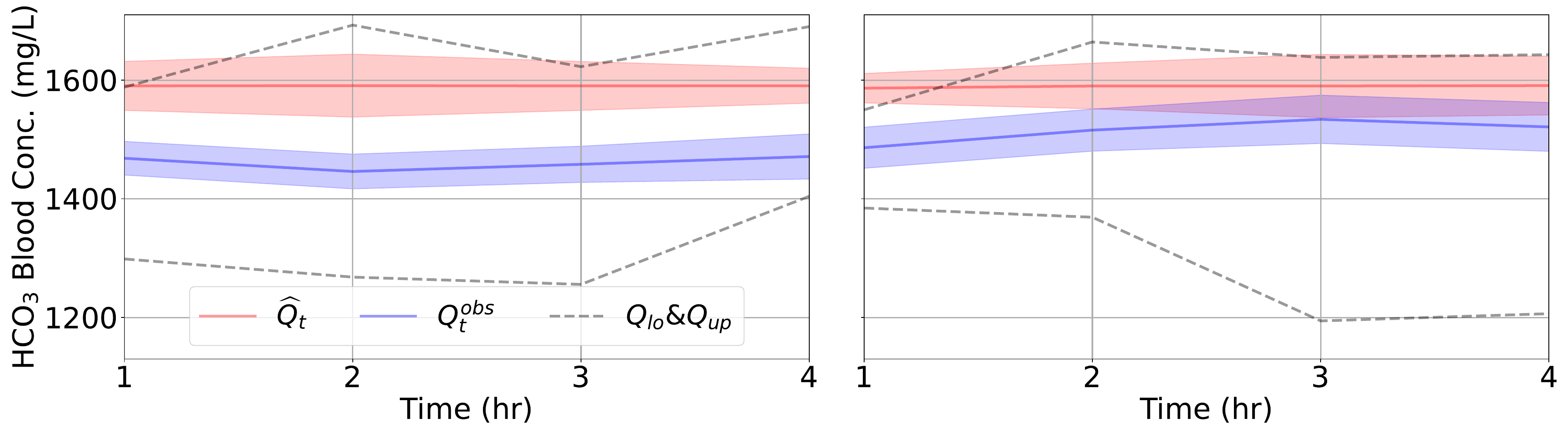}
    \caption{Estimates and 95\% confidence intervals for $\Qt_\tx$ and $\Q^{\textup{obs}}_\tx$ at each $1\leq \tx \leq 4$ for two choices of $(\B_{0:4}, \ax_{1:4})$, where $\Yt(\ax_{1:\tx})$ and $\Y(\ax_{1:\tx})$ correspond to HCO$_3$ concentration.
    The dashed lines indicate lower and upper 95\% confidence intervals for $\Qlo, \Qup$ respectively.  }
    \label{fig:longitudinal_plots}
\end{figure}
Here $\Qt_\tx$ is just $\Qt$ as defined above with its dependence on $\tx$ made explicit.
Each plot also shows one-sided 95\% confidence intervals on $\Qlo$ and $\Qup$ at each $\tx \in \{1, \ldots, 4\}$ obtained from Hoeffding's inequality. 
Directly comparing the estimates of $\Qt_\tx$ and $\Q^{\textup{obs}}_\tx$ would suggest that the twin is comparatively more accurate for the right-hand plot, as these estimates are closer to one another in that case.
However, the output of the twin in the right-hand plot is falsified at $\tx = 1$, as can be seen from the fact that confidence interval for $\Qt_1$ lies entirely above the one-sided confidence interval for $\Qup$ at that timestep.
On the other hand, the output of the twin in the left-hand plot is not falsified at any of the timesteps shown, so that the twin may in fact be accurate for these $(\ax_{1:4}, \B_{1:4})$, contrary to what a naive assessment strategy would suggest.
Our methodology provides a principled means for twin assessment that avoids drawing potentially misleading inferences like this.

A similar phenomenon appears in Figure \ref{fig:histograms}, which for two choices of $\B_{0:\tx}$ and $\ax_{1:\tx}$ shows histograms of raw glucose values obtained from the observational data conditional on $\A_{1:\tx}=\ax_{1:\tx}$ and $\X_{0:\tx}(\A_{1:\tx})\in \B_{0:\tx}$, and from the twin conditional on $\Xt_{0:\tx}(\ax_{1:\tx})\in \B_{0:\tx}$.
Below each histogram we also show 95\% confidence intervals for $\Qup$ and $\Qt$ obtained from Hoeffding's inequality.
While Figures \ref{fig:glucosea} and \ref{fig:glucoseb} appear visually very similar, the inferences produced by our testing procedure are different: the hypothesis corresponding to the right-hand plot is rejected, since there is no overlap between the confidence intervals underneath, while the hypothesis corresponding to the left-hand plot is not.
This was not an isolated case and several other examples of this phenomenon are shown in Figure \ref{fig:histograms-supplement} in the \AppendixName. 
This demonstrates that the inferences obtained from our procedure do not depend only on the distribution of observed outcomes (which is essentially the same for both cases).
Instead, as discussed in Section \ref{sec:causal-bounds-statement}, these also account for the worst-case effects of unmeasured confounding that may exist in the observational data.

\begin{figure}%
    \centering
    \begin{subfigure}[b]{0.5\textwidth}
    \centering
    \includegraphics[height=5cm]{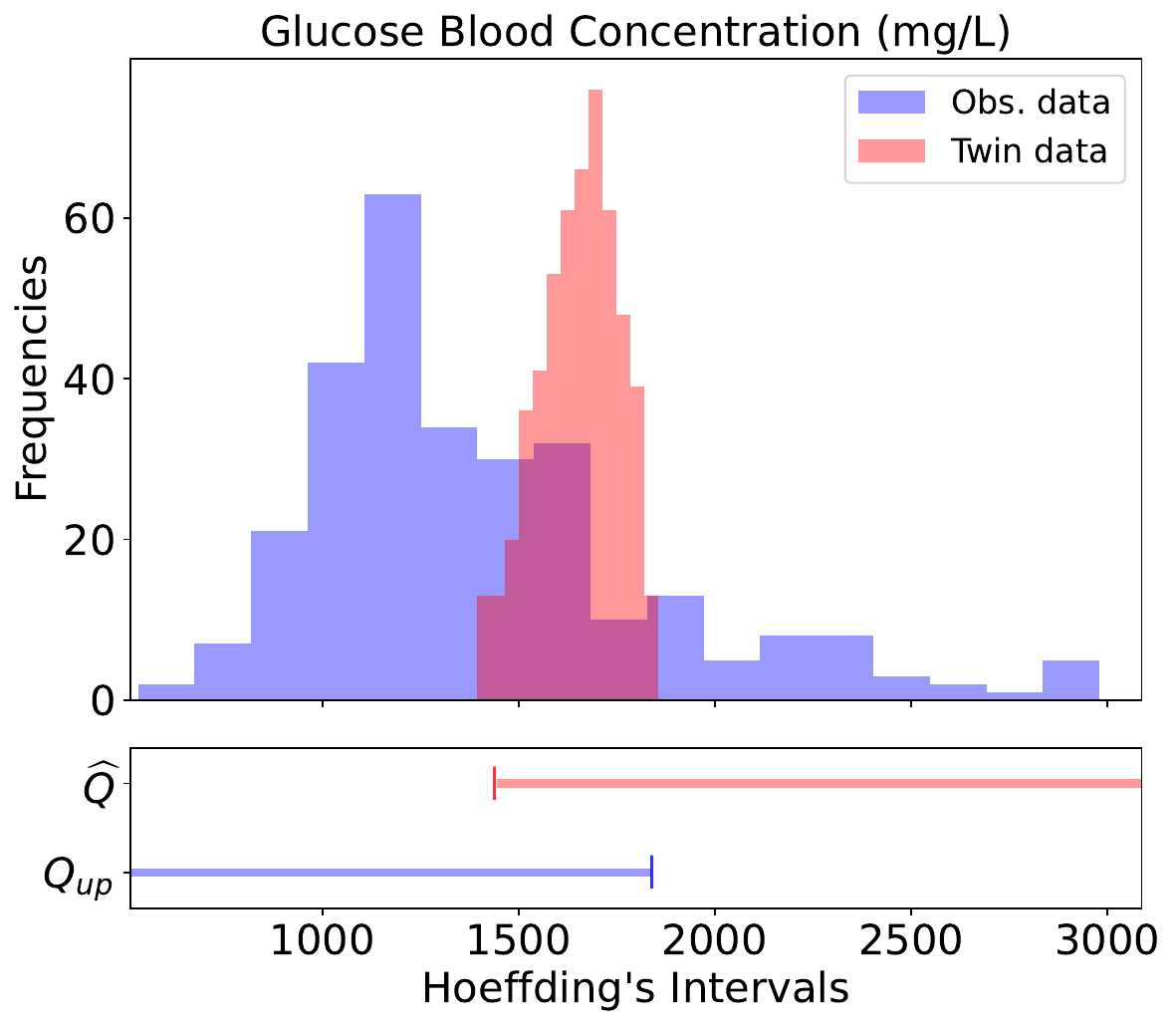}
    \subcaption{Not rejected}
    \label{fig:glucosea}
    \end{subfigure}%
    \begin{subfigure}[b]{0.5\textwidth}
    \centering
    \includegraphics[height=5cm]{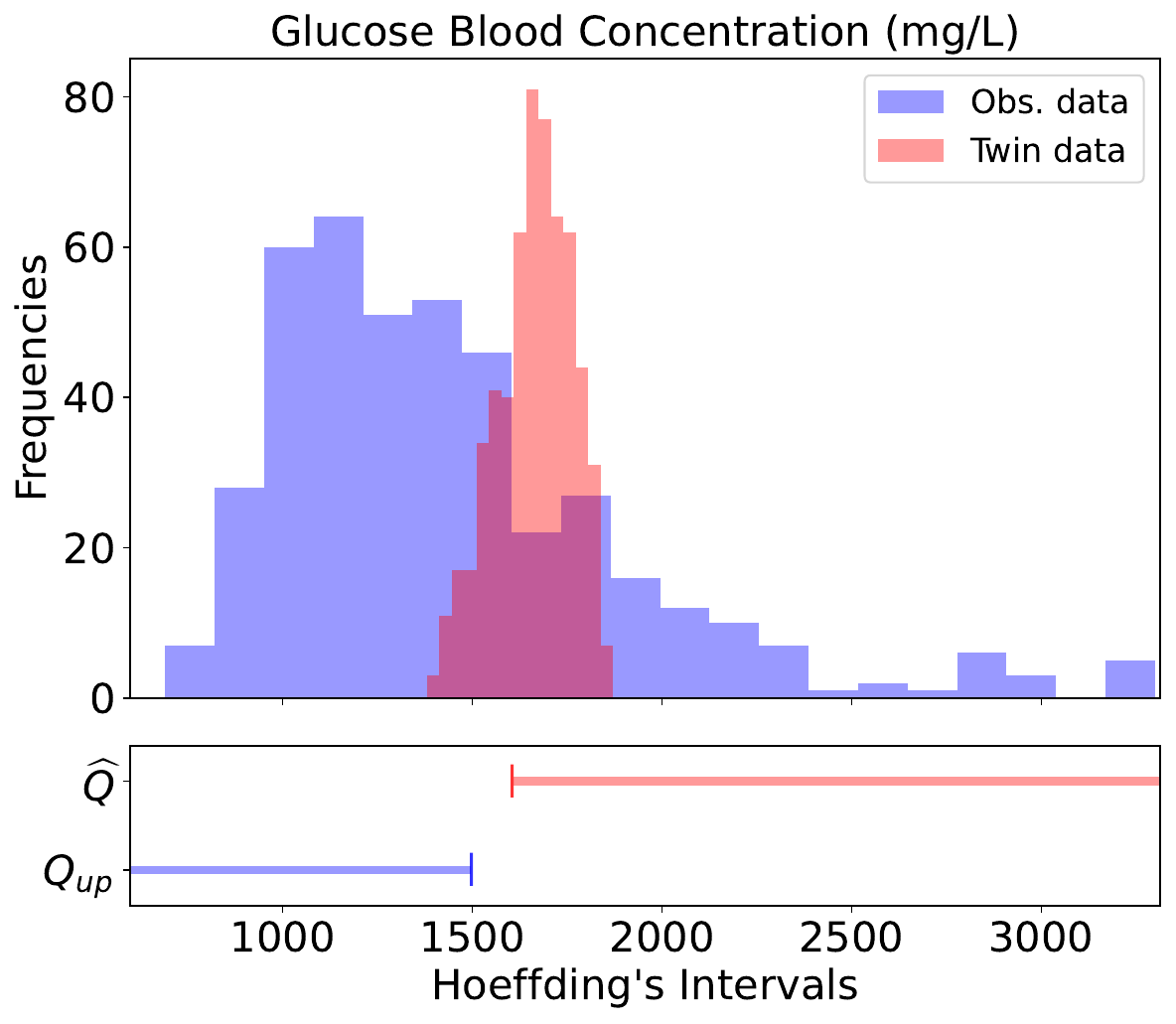}
    \subcaption{Rejected}
    \label{fig:glucoseb}
    \end{subfigure}
    \caption{Raw glucose values from the observational data and twin for two choices of $(\B_{0:\tx}, \ax_{1:\tx})$, with confidence intervals for $\Qt$ and $\Qup$ shown below. The horizontal axes are truncated to the .025 and .975 quantiles of the observational data for clarity. Untruncated plots are shown in Figure \ref{fig:histograms-supplement} of the \AppendixName.}
    \label{fig:histograms}
\end{figure}

\section{Discussion}

We have advocated for a causal approach to digital twin assessment, and have presented a statistical procedure for doing so that obtains rigorous theoretical guarantees under minimal assumptions.
We now highlight the key limitations of our approach.
Importantly, our methodology implicitly assumes that there is no distribution shift between testing and deployment time.
If the conditional distribution of $\X_{1:\T}(\ax_{1:\T})$ given $\X_0$ changes at deployment time, then so too does the set of twins that are interventionally correct, and if this change is significant enough, our assessment procedure may yield misleading results.
Distribution shift in this sense is a separate issue to unobserved confounding, and arises in a wide variety of statistical problems beyond ours.

Additionally, the procedure we used in our case study to choose the hypothesis parameters $\B_{0:\tx}$ was ad hoc.
For scalability, it would likely be necessary to obtain $\B_{0:\tx}$ via a more automated procedure.
It may also be desirable to choose $\B_{0:\tx}$ dynamically in light of previous hypotheses tested, zooming in to regions containing possible failure modes to obtain increasingly granular information about the twin.
We see opportunities here for using machine learning techniques, but leave this to future work.

Various other extensions and improvements appear possible.
For example, 
one can leverage ideas from the literature on partial identification \citep{manski2003partial} to obtain greater statistical efficiency, for example by building on the line of work initiated by \cite{imbens2004confidence} for obtaining more informative confidence intervals.
Beyond this, it may sometimes be useful to consider additional assumptions that lead to less conservative assessment results.
For example, various methods for \emph{sensitivity analysis} have been proposed that model the \emph{degree} to which the actions of the behavioural agent are confounded \citep{rosenbaum2002observational,tan2006distributional,yadlowsky2022bounds}.
This can yield tighter bounds on $\Q$ than are implied by Theorem \ref{thm:causal-bounds}, albeit at the expense of less robustness if these assumptions are violated.

        \section*{Acknowledgements}

The authors are highly appreciative of the troubleshooting and development assistance provided by the Pulse team.
RC, AD, and CH were supported by the Engineering and Physical Sciences Research Council (EPSRC) through the Bayes4Health programme [Grant number EP/R018561/1].
MFT was funded by Google DeepMind.
The authors declare there are no competing interests.

    }
    { %
    }

    \includepdf[pages=-]{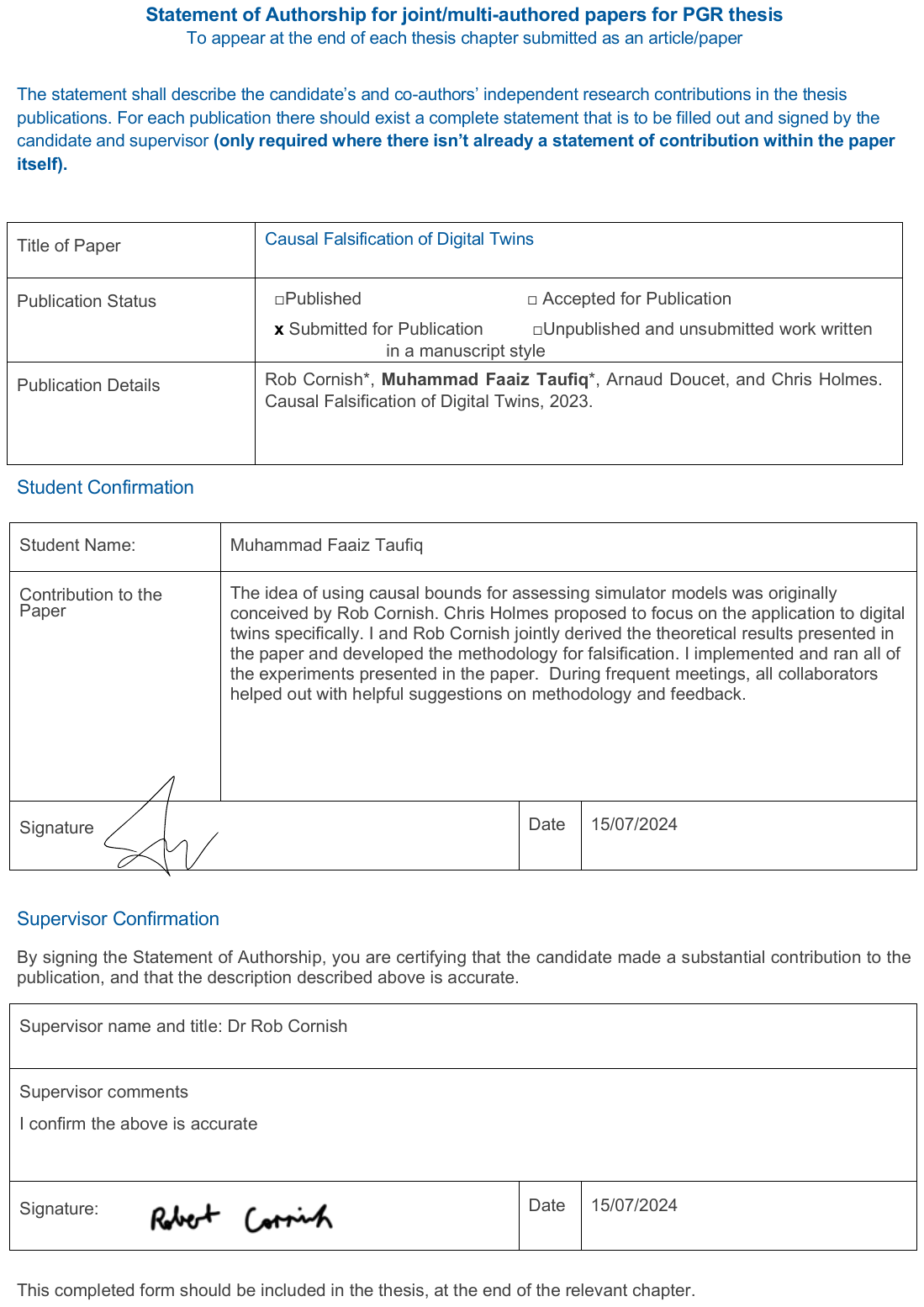}

\chapter{\label{ch:6-conclusion}Conclusion and Future Work} 

\minitoc

\section{Discussion}

Before deploying a decision-making policy to production, it is usually important to understand the
plausible range of outcomes that it may produce. However, due to resource or ethical constraints, it is
often not possible to obtain this understanding by testing the policy directly in the real-world. In such
cases we have to rely on off-policy evaluation (OPE), which uses observational data collected under a different behavioural policy to evaluate the target policy in some way. 
In this thesis, we have considered some of the challenges posed by the current OPE methodologies and proposed novel solutions to each of these individually.
We have also demonstrated the practical utility of these solutions by applying them to a range of real-world problems involving large scale datasets.
To be more specific, below we provide a brief summary of the challenges tackled in this thesis:
\begin{itemize}
    \item In Chapter \ref*{ch:3-mr} we consider the problem of variance reduction in OPE estimators. To address this, we propose the Marginal Ratio (MR) estimator, which uses a marginalization technique to provide a more efficient and robust estimator for contextual bandits, and may also be of interest in other domains such as causal inference. 
    \item Next, in Chapter \ref*{ch:4-copp} we provide a novel methodology of uncertainty quantification in off-policy outcomes based on conformal prediction \citep{vovk2005algorithmic}. Our proposed technique can help practitioners quantify the plausible range of outcomes that are likely to occur under the target policy, and comes with sound finite-sample probabilistic guarantees.
    \item Finally, in Chapter \ref*{ch:5-causal} we explore the case when the assumption of no unmeasured confounding (needed for the existing OPE methodologies) is violated. 
    We provide a set of novel causal bounds which remain valid in this case, and subsequently use these bounds to develop a procedure for robust assessment of digital twin models using observational data which remains valid under only the assumption of an i.i.d. dataset of observational trajectories. 
\end{itemize}

\section{Limitations}
Here, we outline some of the limitations of the methodologies described in this thesis. 

\paragraph*{Distributional shift in data generating mechanism}
Our methodologies highlighted in this thesis implicitly assume that the data generating process remains unchanged between testing and deployment times. 
Technically, for contextual bandits this assumption means that the conditional distribution of $Y$ given $(X, A)$ does not shift when the target policy is deployed. 
If this distribution changes at deployment time, then so too does the distribution of outcomes that \emph{would} be observed under the target policy. 
If this shift is significant enough, our methodologies may yield misleading results.  

\paragraph*{Estimation errors}
The techniques mentioned in Chapters \ref*{ch:3-mr} and \ref*{ch:4-copp} involve an additional estimation of marginal density ratios for importance sampling. 
While we outline straightforward regression methodologies for estimating these importance ratios directly, this step may still introduce an additional source of bias in the value estimation. 

\paragraph*{Scalability limitations}
Increasing data dimensionality may pose additional challenges for our solutions, especially those described in Chapters \ref*{ch:4-copp} and \ref*{ch:5-causal}. 
For example, in Chapter \ref*{ch:4-copp}, the estimation of importance ratios for our conformal off-policy prediction (COPP) algorithm may become more challenging when $(X, A)$ is high-dimensional, thereby yielding biased results.
Likewise in Chapter \ref*{ch:5-causal}, the procedure we used in our case study to choose the hypothesis parameters was ad hoc, which may not scale to high-dimensional datasets.
For scalability, it would likely be necessary to obtain these parameters via a more automated procedure.

\section{Directions for future work}
Our work in this thesis opens up several interesting avenues for future research. We highlight some of these below.

\paragraph*{Off-policy learning}
This thesis has largely focused on robust off-policy assessment methodologies. However, our findings are highly applicable to robust policy optimisation problems, where the data collected using an `old’ policy is used to learn a new policy. 
Proximal Policy Optimisation \citep{schulman2017proximal} is one such approach which has gained immense popularity and has been applied to reinforcement learning with human feedback \citep{lambert2022illustrating}. 
We believe that our MR estimator proposed in Chapter \ref*{ch:3-mr} applied to these methodologies could lead to improvements in the stability and convergence
of these optimisation schemes, given its favourable variance properties. 
Similarly, our conformal off-policy prediction (COPP) algorithm when applied to off-policy learning could be a
step towards robust policy learning by optimising the worst-case outcome \citep{stutz2021learning}.

\paragraph*{Addressing the curse of horizon in sequential decision-making}
Chapters \ref*{ch:3-mr} and \ref*{ch:4-copp} of this thesis specifically consider OPE in contextual bandits. 
This setting offers a strong foundational framework for conducting rigorous theoretical and empirical analyses, 
however, it would be interesting to extend the application of these methodologies to sequential decision frameworks.
While some follow-up works have attempted to apply our COPP algorithm to Markov Decision Processes \citep{foffano2023conformal, zhang2023conformal, kuipers2024conformal}, the obtained confidence sets become increasingly conservative with increasing time horizon. 
It is worth exploring methodologies for obtaining intervals which remain valid and informative even in sequential decision settings with large time horizons.

\paragraph*{Application to transfer learning}
Finally, our solutions in Chapters \ref*{ch:3-mr} and \ref*{ch:4-copp} involves learning importance ratios which may also be of interest in other domains beyond OPE. 
One such area is \emph{transfer learning} which considers cases where the testing data distribution is different from the training data distribution. 
Classical transfer learning methods rely on importance weighting to handle the distribution mismatch \citep{SHIMODAIRA2000Improving, sugiyama2007covariate, huang2007correcting, sugiyama2008direct,lu2021rethinking}. 
Our proposed regression techniques may be of interest for obtaining these weights efficiently in high-dimensional datasets in this setting.

\ifthenelse{\boolean{compileAppendices}}
    { %
        \startappendices
        \chapter{\label{app:mr}Marginal Density Ratio for Off-Policy
Evaluation in Contextual Bandits}

\minitoc

\newpage

\section{Proofs}
\begin{proof}[Proof of Lemma \ref{lemma:weights-est}]
First, we express the weights $w(y)$ as the conditional expectation as follows:
\begin{align*}
    w(y) &= \frac{\ptar(y)}{\pbeh(y)} \\
    &= \int_{\Xspace, \Aspace} \frac{\ptar(x,a,y)}{\pbeh(y)}\,\mathrm{d}a\, \mathrm{d}x\\
    &= \int_{\Xspace, \Aspace} \frac{\ptar(x,a,y)}{\pbeh(y)}\,\frac{\pbeh(x, a\mid y)}{\pbeh(x, a\mid y)}\,\mathrm{d}a\, \mathrm{d}x\\
    &= \int_{\Xspace, \Aspace} \frac{\ptar(x,a,y)}{\pbeh(x, a, y)}\,\pbeh(x, a\mid y)\,\mathrm{d}a\, \mathrm{d}x\\
    &= \int_{\Xspace, \Aspace} \rho(a, x)\,\pbeh(x, a\mid y)\,\mathrm{d}a\, \mathrm{d}x\\
    &= \Ebeh[\rho(A, X)\mid Y=y],
\end{align*}
where $\rho(a, x) = \frac{\ptar(x, a, y)}{\pbeh(x, a, y)} = \frac{\tar(a\mid x)}{\beh(a\mid x)}$.
Since conditional expectations can be defined as the solution of regression problem, the result follows. 
\end{proof}

\begin{proof}[Proof of Proposition \ref{tv_prop}] We have
\begin{align*}
    \textup{D}_f\left(\ptar(x,a,y)\,||\, \pbeh(x,a,y)\right) &= \Ebeh\left[ f\left( \frac{\ptar(X,A,Y)}{\pbeh(X,A,Y)}\right) \right]\\
    &= \Ebeh\left[ f\left( \frac{\tar(A\mid X)}{\beh(A\mid X)}\right) \right]\\
    &= \Ebeh\left[\Ebeh\left[ f\left( \frac{\tar(A\mid X)}{\beh(A\mid X)}\right) \Bigg| Y \right]\right]\\
    &\geq \Ebeh\left[ f\left( \Ebeh\left[\frac{\tar(A\mid X)}{\beh(A\mid X)}\Bigg| Y \right]\right) \right]\quad\text{(Jensen's inequality)}\\
    &= \Ebeh\left[ f\left( \frac{\ptar(Y)}{\pbeh(Y)} \right) \right]\\
    &= \textup{D}_f\left(\ptar(y)\,||\, \pbeh(y)\right).
\end{align*}
\end{proof}

\begin{proof}[Proof of Proposition \ref{prop:var_mr}]
Since $\Ebeh[\thetaipw] = \Ebeh[\thetamr] = \Etar[Y]$, we have that,  
\begin{align*}
    \Vbeh[\thetaipw] - \Vbeh[\thetamr] &= \Ebeh[\thetaipw]^2 - \Ebeh[\thetamr]^2 \\
    &= \frac{1}{n} \left(\Ebeh\left[\rho(A, X)^2\, Y^2 \right] - \Ebeh\left[w(Y)^2\, Y^2 \right] \right)\\
    &= \frac{1}{n} \left(\Ebeh\left[\Ebeh[\rho(A, X)^2\mid Y]\, Y^2 \right] - \Ebeh\left[w(Y)^2\, Y^2 \right] \right)\\
    &= \frac{1}{n} \left(\Ebeh\left[\Ebeh[\rho(A, X)^2\mid Y]\, Y^2 \right] - \Ebeh\left[\Ebeh[\rho(A, X)\mid Y]^2\, Y^2 \right] \right)\\
    &= \frac{1}{n} \Ebeh \left[ \Vbeh\left[ \rho(A, X) \mid Y \right]\, Y^2 \right].
\end{align*}
In the second last step above, we use the fact that $w(y) = \Ebeh[\rho(A, X)\mid Y=y]$. 
\end{proof}

\begin{proof}[Proof of Proposition \ref{prop:var_dr}]
Let $\hat{\mu}(a, x) \approx \E[Y\mid X=x, A=a]$ denote the outcome model in DR estimator. Then, using multiple applications of the law of total variance we get that
\begin{align*}
    n\, \Vbeh[\thetadr] &= \Vbeh\left[\rho(A, X)\,(Y - \hat{\mu}(A, X)) + \sum_{a'\in \Aspace} \hat{\mu}(a', X)\,\tar(a'\mid X)\right]\\
    &= \Vbeh\left[\rho(A, X)\,(Y - \hat{\mu}(A, X)) + \Etar[\hat{\mu}(A, X)\mid X]\right]\\
    &= \Ebeh[\Vbeh[\rho(A, X)\,(Y - \hat{\mu}(A, X)) + \Etar[\hat{\mu}(A, X)\mid X]\mid X, A]]\\
    &\quad+ \Vbeh[\Ebeh[\rho(A, X)\,(Y - \hat{\mu}(A, X)) + \Etar[\hat{\mu}(A, X)\mid X]\mid X, A]]\\
    &= \Ebeh[\rho(A, X)^2 \var[Y\mid X, A]] \\
    &\quad+ \Vbeh[\Ebeh[\rho(A, X)\,(Y - \hat{\mu}(A, X)) + \Ebeh[\rho(A, X)\,\hat{\mu}(A, X)\mid X]\mid X, A]]\\
    &= \Ebeh[\rho(A, X)^2 \var[Y\mid X, A]] \\
    &\quad+ \Vbeh[\rho(A, X)\,(\mu(A, X) - \hat{\mu}(A, X)) + \Ebeh[\rho(A, X)\,\hat{\mu}(A, X)\mid X]]\\
    &= \Ebeh[\rho(A, X)^2 \var[Y\mid X, A]] \\
    &\quad+ \Vbeh[\Ebeh[\rho(A, X)\,(\mu(A, X) - \hat{\mu}(A, X)) + \Ebeh[\rho(A, X)\,\hat{\mu}(A, X)\mid X]\mid X]] \\
    &\quad+ \Ebeh[\Vbeh[\rho(A, X)\,(\mu(A, X) - \hat{\mu}(A, X)) + \Ebeh[\rho(A, X)\,\hat{\mu}(A, X)\mid X]\mid X]]\\
    &= \Ebeh[\rho(A, X)^2 \var[Y\mid X, A]]+ \Vbeh[\Ebeh[\rho(A, X)\,\mu(A, X)\mid X]] \\
    &\quad+ \Ebeh[\Vbeh[\rho(A, X)\,(\mu(A, X) - \hat{\mu}(A, X))\mid X]]\\
    &\geq \Ebeh[\rho(A, X)^2 \var[Y\mid X, A]] + \Vbeh[\Ebeh[\rho(A, X)\,\mu(A, X)\mid X]].
\end{align*}
Using this, we get that
\begin{align*}
    &n (\Vbeh[\thetadr] - \Vbeh[\thetamr]) \\
    &\quad\geq \Ebeh[\rho(A, X)^2 \var[Y\mid X, A]] +  \Vbeh \left[\Ebeh[\rho(A, X) \,\mu(A, X)\mid X]\right] - \Vbeh[w(Y)\,Y].
\end{align*}
Again, using the law of total variance,
\begin{align*}
    \Vbeh[\rho(A, X)\,Y] &= \Ebeh[\Vbeh[\rho(A, X)\,Y \mid X, A]] + \Vbeh[\Ebeh[\rho(A, X)\,Y\mid X, A]]\\
    &= \Ebeh[\rho(A, X)^2\var[Y \mid X, A]] + \Vbeh[ \rho(A, X)\,\mu(A, X)]\\
    &= \Ebeh[\rho(A, X)^2\var[Y \mid X, A]] + \Vbeh \left[\Ebeh[\rho(A, X) \,\mu(A, X)\mid X]\right] \\
    &\quad+ \Ebeh \left[\Vbeh[\rho(A, X) \,\mu(A, X)\mid X]\right].
\end{align*}
Rearranging and substituting back into the expression earlier, we get that
\begin{align*}
    &n (\Vbeh[\thetadr] - \Vbeh[\thetamr]) \\
    &\quad\geq \Vbeh[\rho(A, X)\,Y] - \Ebeh \left[\Vbeh[\rho(A, X) \,\mu(A, X)\mid X]\right] - \Vbeh[w(Y)\,Y].
\end{align*}
Now, from Proposition \ref{prop:var_mr} we know that \begin{align*}
    n (\Vbeh[\thetaipw] - \Vbeh[\thetamr]) = \Vbeh[\rho(A, X)\,Y] - \Vbeh[w(Y)\,Y] = \Ebeh \left[ \Vbeh\left[ \rho(A, X) \mid Y \right]\, Y^2 \right].
\end{align*}
Therefore, 
\begin{align*}
    &n (\Vbeh[\thetadr] - \Vbeh[\thetamr]) \\
    &\quad\geq \Ebeh \left[ \Vbeh\left[ \rho(A, X) \mid Y \right]\, Y^2 \right] - \Ebeh \left[\Vbeh[\rho(A, X) \,\mu(A, X)\mid X]\right]\\
    &\quad= \Ebeh \left[\Vbeh\left[ \rho(A, X)\,Y \mid Y \right] - \Vbeh[\rho(A, X) \,\mu(A, X)\mid X] \right].
\end{align*}
\end{proof}

\begin{proof}[Proof of Theorem \ref{prop:mips_main_text}]
This result follows straightforwardly from Proposition \ref{prop:mips_generalised} in Appendix \ref{app:gmips}.    
\end{proof}

\begin{proof}[Proof of Proposition \ref{prop:bias-and-var-main}]
\begin{align*}
    \textup{Bias}(\thetaipw) &= \Ebeh[\hat{\rho}(A, X)\, Y] - \Etar[Y] \\
    &= \Ebeh\left[\Ebeh[\hat{\rho}(A, X)\mid Y]\,Y \right] - \Etar[Y]  \\
    &= \Ebeh[\hat{w}(Y)\, Y] - \Ebeh[\epsilon\, Y] - \Etar[Y] \\
    &= \textup{Bias}(\thetamr) - \Ebeh[\epsilon\, Y].
\end{align*}
Next, to prove the variance result, we first use the law of total variance to obtain
\begin{align*}
    \Vbeh[\thetaipw] &= \frac{1}{n} \Vbeh[\hat{\rho}(A, X)\,Y]\\
    &= \frac{1}{n} \left( \Vbeh[\Ebeh[\hat{\rho}(A, X)\,Y\mid Y]] + \Ebeh[\Vbeh[\hat{\rho}(A, X)\,Y\mid Y]]\right)\\
    &= \frac{1}{n} \left( \Vbeh[\tilde{w}(Y)\,Y] + \Ebeh[\Vbeh[\hat{\rho}(A, X)\,Y\mid Y]]\right).
\end{align*}
Moreover, using the fact that $\hat{w}(Y) = \tilde{w}(Y) + \epsilon$ we get that,
\begin{align*}
    \Vbeh[\thetamr] &= \frac{1}{n} \Vbeh[\hat{w}(Y)\,Y]\\
    &= \frac{1}{n} \Vbeh[\left(\tilde{w}(Y) + \epsilon \right)\,Y]\\
    &= \frac{1}{n} \left( \Vbeh[\tilde{w}(Y)\,Y] + \Vbeh[\epsilon\,Y] + 2\,\textup{Cov}(\tilde{w}(Y)\,Y, \epsilon\,Y)\right).
\end{align*}
Putting together the two variance expressions derived above, we get that
\begin{align*}
    &\Vbeh[\thetaipw] - \Vbeh[\thetamr]\\
    &\quad=
    \frac{1}{n}\left(\Ebeh[\Vbeh[\hat{\rho}(A, X)\mid Y]\,Y^2] - \Vbeh[\epsilon\,Y] - 2\,\textup{Cov}(\tilde{w}(Y)\,Y, \epsilon\,Y) \right).
\end{align*}

\end{proof}

\section{Comparison with extensions of the doubly robust estimator}\label{sec:dr-extensions}
In this section, we theoretically investigate the variance of MR against the commonly used extensions of the DR estimator, namely Switch-DR \citep{wang2017optimal} and DR with Optimistic Shrinkage (DRos) \citep{su2020doubly}. At a high level, these estimators seek to reduce the variance of the vanilla DR estimator by considering modified importance weights, thereby trading off the variance for additional bias.
Below, we provide the explicit definitions of these estimators for completeness.

\paragraph{Switch-DR estimator}
The original DR estimator can still have a high variance when the importance weights are large due to a large policy shift. Switch-DR \citep{wang2017optimal} aims to circumvent this problem by switching to DM when the importance weights are large:
\[
\thetaswitch \coloneqq \frac{1}{n} \sum_{i=1}^n \rho(a_i, x_i)\,(y_i - \hat{\mu}(a_i, x_i))\ind(\rho(a_i, x_i) \leq \tau) + \hat{\eta}(\tar),
\]
where $\tau \geq 0$ is a hyperparameter, $\hat{\mu}(a, x) \approx \E[Y \mid X=x, A=a]$ is the outcome model, and 
$$
\hat{\eta}(\tar) = \frac{1}{n} \sum_{i=1}^n \sum_{a'\in \Aspace} \hat{\mu}(a', x_i) \tar(a'\mid x_i) \approx \E_{\tar}[\hat{\mu}(A, X)]
$$
where $a_i^* \sim \tar(\cdot \mid x_i)$.

\paragraph{Doubly Robust with Optimal Shrinkage (DRos)}
DRos proposed by \citep{su2020doubly} uses new weights $\hat{\rho}_\lambda(a_i, x_i)$ which directly minimises sharp bounds on the MSE of the resulting estimator,
\[
\thetadros \coloneqq \frac{1}{n} \sum_{i=1}^n \hat{\rho}_\lambda(a_i, x_i)\,(y_i - \hat{\mu}(a_i, x_i)) + \hat{\eta}(\tar),
\]
where $\lambda \geq 0$ is a pre-defined hyperparameter and $\hat{\rho}_\lambda$ is defined as
\[
\hat{\rho}_\lambda(a, x) \coloneqq \frac{\lambda}{\rho^2(a, x) + \lambda}\, \rho(a, x).
\]
When $\lambda = 0$, $\hat{\rho}_\lambda(a, x) = 0$ leads to DM, whereas as $\lambda \rightarrow \infty$, $\hat{\rho}_\lambda(a, x) \rightarrow \rho(a, x)$ leading to DR.

More generally, both of these estimators can be written as follows:
\[
\hat{\theta}_{\textup{DR}}^{\tilde{\rho}} \coloneqq \frac{1}{n} \sum_{i=1}^n \tilde{\rho}(a_i, x_i)\,(y_i - \hat{\mu}(a_i, x_i)) + \hat{\eta}(\tar).
\]
Here, when $\tilde{\rho}(a, x) = \rho(a, x)\ind(\rho(a_i, x_i) \leq \tau)$, we recover the Switch-DR estimator and likewise when $\tilde{\rho}(a, x) = \hat{\rho}_\lambda(a, x)$, we recover DRos. 

\subsection{Variance comparison with the DR extensions}
Next, we provide a theoretical result comparing the variance of the MR estimator with these DR extension methods.
\begin{proposition}\label{prop:var_dr_extensions}
    When the weights $w(y)$ are known exactly and the outcome model is exact, i.e., $\hat{\mu}(a, x) = \mu(a, x) = \E[Y \mid X=x, A=a]$ in the DR estimator $\hat{\theta}_{\textup{DR}}^{\tilde{\rho}}$ defined above,
    \begin{align*}
    &\Vbeh[\hat{\theta}_{\textup{DR}}^{\tilde{\rho}}] - \Vbeh[\thetamr]\\
    &\qquad \qquad \qquad  \geq \frac{1}{n} \Ebeh \left[ \Vbeh\left[ \rho(A, X) \mid Y \right]\, Y^2 -  \Vbeh\left[ \rho(A, X)\mu(A, X) \mid X \right] \right] - \Delta,
\end{align*}
where $\Delta \coloneqq \frac{1}{n}\Ebeh\left[(\rho^2(A, X) - \tilde{\rho}^2(A, X))\,\V[Y\mid X, A]\right]$. 
\end{proposition}

\begin{proof}[Proof of Proposition \ref{prop:var_dr_extensions}]
Using the fact that $\hat{\mu}(a, x) = \mu(a, x) $ and the law of total variance, we get that
\begin{align*}
    n\,\Vbeh[\hat{\theta}_{\textup{DR}}^{\tilde{\rho}}] &= \Vbeh[\tilde{\rho}(A, X)\,(Y -\hat{\mu}(A, X)) + \sum_{a'\in \Aspace}\hat{\mu}(a', X)\tar(a'\mid X) ]\\
    &= \Vbeh[\tilde{\rho}(A, X)\,(Y -\hat{\mu}(A, X)) + \Etar[\hat{\mu}(A, X)\mid X] ]\\
    &= \Vbeh[\tilde{\rho}(A, X)\,(Y -\mu(A, X)) + \Etar[\mu(A, X)\mid X] ]\\
    &= \Vbeh[\Ebeh[\tilde{\rho}(A, X)\,(Y -\mu(A, X)) + \Etar[\mu(A, X)\mid X] \mid X, A]] \\
    &\qquad+ \Ebeh[\Vbeh[\tilde{\rho}(A, X)\,(Y -\mu(A, X)) + \Etar[\mu(A, X)\mid X]\mid X, A]]\\
    &= \Vbeh[\Etar[\mu(A, X)\mid X]] + \Ebeh[\tilde{\rho}^2(A, X)\V[Y\mid X, A]]\\
    &= \Vbeh[\Etar[\mu(A, X)\mid X]] + \Ebeh[\rho^2(A, X)\,\V[Y\mid X, A]] \\
    &\qquad+ \underbrace{\Ebeh[(\tilde{\rho}^2(A, X) - \rho^2(A, X))\,\V[Y\mid X, A]]}_{-n\,\Delta}\\
    &= \Vbeh[\Ebeh[\rho(A, X)\,\mu(A, X)\mid X]] + \Ebeh[\rho^2(A, X)\,\V[Y\mid X, A]] - n\,\Delta.
\end{align*}
    Again, using the law of total variance we can rewrite the second term on the RHS above as,
    \begin{align*}
        &\Ebeh[\rho^2(A, X)\,\V[Y\mid X, A]] \\
        &\quad= \Vbeh[\rho(A, X)\, Y] - \Vbeh[\rho(A, X)\,\mu(A, X)]\\
        &\quad= \Vbeh[\Ebeh[\rho(A, X)\mid Y]\, Y] + \Ebeh[\Vbeh[\rho(A, X)\mid Y]\,Y^2] \\
        &\qquad- \Vbeh[\rho(A, X)\,\mu(A, X)]\\
        &\quad= \Vbeh[w(Y)\, Y] + \Ebeh[\Vbeh[\rho(A, X)\mid Y]\,Y^2] - \Vbeh[\rho(A, X)\,\mu(A, X)]\\
        &\quad= n\,\Vbeh[\thetamr] + \Ebeh[\Vbeh[\rho(A, X)\mid Y]\,Y^2] - \Vbeh[\rho(A, X)\,\mu(A, X)].
    \end{align*}
    Putting this together, we get that
    \begin{align*}
        &n\,\Vbeh[\hat{\theta}_{\textup{DR}}^{\tilde{\rho}}] \\
        &\quad= n\,\Vbeh[\thetamr] + \Ebeh[\Vbeh[\rho(A, X)\mid Y]\,Y^2] - \Vbeh[\rho(A, X)\,\mu(A, X)] \\
        &\qquad+ \Vbeh[\Ebeh[\rho(A, X)\,\mu(A, X)\mid X]] - n\,\Delta\\
        &\quad= n\,\Vbeh[\thetamr] + \Ebeh[\Vbeh[\rho(A, X)\mid Y]\,Y^2] - \Ebeh[\Vbeh[\rho(A, X)\,\mu(A, X)\mid X]] - n\,\Delta,
    \end{align*}
    where in the last step above, we again use the law of total variance. Rearranging the above leads us to the result. 
\end{proof}
\paragraph{Intuition}
Note that for both of the DR extensions under consideration, the modified ratios $\tilde{\rho}(a, x)$ satisfy $0\leq \tilde{\rho}(a, x)\leq \rho(a, x)$ and hence $\Delta \geq 0$ (using the definition of $\Delta$ in Proposition \ref{prop:var_dr_extensions}).
When the modified ratios $\tilde{\rho}(a, x)$ are `close' to the true policy ratios $\rho(a, x)$, then using the definition of $\Delta$, we have that $\Delta \approx 0$. In this case, the result above provides a similar intuition to Proposition \ref{prop:var_dr} in the main text. Specifically, in this case we have that if $\Vbeh\left[ \rho(A, X)\,Y \mid Y \right]$ is greater than $\Vbeh\left[ \rho(A, X)\,\mu(A, X) \mid X \right]$ on average, the variance of the MR estimator will be less than that of the DR extension under consideration. 
Intuitively, this will occur when the dimension of context space $\Xspace$ is high because in this case the conditional variance over $X$ and $A$, $\Vbeh\left[\rho(A, X)\,Y \mid Y \right]$ is likely to be greater than the conditional variance over $A$, $\Vbeh\left[ \rho(A, X)\,\mu(A, X) \mid X \right]$.

In contrast if the modified ratios $\tilde{\rho}(a, x)$ differ substantially from $\rho(a, x)$, then $\Delta$ will be large and the variance of MR may be higher than that of the resulting DR extension. However, this comes at the cost of significantly higher bias in the DR extension and consequently MSE of the DR extension will be high in this case.

\section{Weight estimation error} \label{sec:wide_nns_weight_estimation}
In this section, we theoretically investigate the effects of using the estimated importance weights $\hat{w}(y)$ rather than $\hat{\rho}(a, x)$ on the bias and variance of the resulting OPE estimator. Further to our discussion in Section \ref{subsec:weight-estimation-error}, we focus in this section on the approximation error when using a wide neural network to estimate the weights $\hat{w}(y)$. To this end, we use recent results regarding the generalization of wide neural networks \citep{lai2023generalization} to show that the estimation error of the approximation step (ii) in the Section \ref{subsec:weight-estimation-error} declines with increasing number of training data when $\hat{w}(y)$ is estimated using wide neural networks. Before providing the main result, we explicitly lay out the assumptions needed.

\subsection{Using wide neural networks to approximate the weights $\hat{w}(y)$}
\begin{assumption}\label{assumption:weights-in-rkhs}
    Let 
    $
    \tilde{w}(y) \coloneqq \Ebeh[\hat{\rho}(A, X)\mid Y=y].
    $
    Suppose $\tilde{w}  \in \mathcal{H}_1$ and $||\tilde{w}||_{\mathcal{H}_1} \leq R$ for some constant $R$, where $\mathcal{H}_1$ is the reproducing kernel Hilbert space (RKHS) associated with the Neural Tangent Kernel $K_1$ associated with 2 layer neural network defined on $\mathbb{R}$.
\end{assumption}
\begin{assumption}\label{assumption:outcome-bounded}
    There exists an $M \in [0, \infty)$ such that $\mathbb{P}_{\beh}(|Y| \leq M) = 1$.
\end{assumption}

\begin{assumption}\label{assumption:pol-ratios-bounded}
$\hat{\rho}(a_i, x_i)$ satisfies
    \begin{align*}
        \hat{\rho}(a_i, x_i) = \tilde{w}(y_i) + \eta_i,
    \end{align*}
    where $\eta_i \overset{\textup{iid}}{\sim} \mathcal{N}(0, \sigma^2)$ for some $\sigma > 0$. 
\end{assumption}

\begin{theorem}\label{prop:bias-and-var-v3}
Suppose that the IPW and MR estimators are defined as,
\[
\approxipw \coloneqq \frac{1}{n}\sum_{i=1}^n\hat{\rho}(a_i, x_i)\, y_i, \quad \text{and }\quad \approxmr \coloneqq \frac{1}{n}\sum_{i=1}^n\hat{w}_m(y_i)\, y_i,
\]
where $\hat{w}_m(y)$ is obtained by regressing to the estimated policy ratios $\hat{\rho}(a, x)$ using $m$ i.i.d. training samples $\Dtr \coloneqq \{(x^\tr_i, a^\tr_i, y^\tr_i)\}_{i=1}^m$, i.e., by minimising the loss
\begin{align*}
    \mathcal{L}(\phi) = \E_{(X, A, Y)\sim \Dtr} \left[\left(\hat{\rho}(A, X) - f_{\phi}(Y)\right)^2\right].
\end{align*}
Suppose Assumptions \ref{assumption:weights-in-rkhs}-\ref{assumption:pol-ratios-bounded} hold, then for any given $\delta \in (0, 1)$, if $f_\phi$ is a two-layer neural network with width $k$ that is sufficiently large and stops the gradient flow at time $t_* \propto m^{2/3}$, then for sufficiently large $m$, there exists a constant $C_1$ independent of $\delta$ and $m$, such that  
\[
|\textup{Bias}(\approxmr) - \textup{Bias}(\approxipw)| \leq C_1\, m^{-1/3} \log{\frac{6}{\delta}}
\]
holds with probability at least $(1-\delta)(1-o_{k}(1))$. Moreover, 
there exist constants $C_2, C_3$ independent of $\delta$ and $m$ such that 
\begin{align*}
    &n (\Vbeh[\approxipw] - \Vbeh[\approxmr]) \\
    &\qquad \qquad \qquad \geq \underbrace{\Ebeh[\Vbeh[\hat{\rho}(A, X)\,Y\mid Y]]}_{\geq 0} - C_2\,m^{-2/3}\,\log^2{\frac{6}{\delta}} - C_3\,m^{-1/3}\,\log{\frac{6}{\delta}}
\end{align*}
holds with probability at least $(1-\delta)(1-o_{k}(1))$. Here, the randomness comes from the joint distribution of training samples and random initialization of parameters in the neural network $f_{\phi}$. 
\end{theorem}

\begin{proof}[Proof of Theorem \ref{prop:bias-and-var-v3}]
The proof of this theorem relies on \cite[Theorem 4.1]{lai2023generalization}. 
Recall the definition $\tilde{w}(Y)\coloneqq \Ebeh[\hat{\rho}(A, X) \mid Y]$. 
We can rewrite our setup in the setting of \cite[Theorem 4.1]{lai2023generalization}, by relabelling $\hat{\rho}(a, x)$ in our setup as $y$ in their setup and relabelling $y$ in our setup as $x$ in their setup. 
Then, given $\delta \in (0, 1)$, from \cite[Theorem 4.1]{lai2023generalization}, it follows that under Assumptions \ref{assumption:weights-in-rkhs}-\ref{assumption:pol-ratios-bounded} that there exists a constant $C$ independent of $\delta$ and $m$, such that
    \begin{align}
        \Ebeh[\epsilon^2] \leq C\, m^{-2/3} \, \log^2{\frac{6}{\delta}} \label{eqn:theorem-statement}
    \end{align}
    holds with probability at least $(1-\delta)(1-o_{k}(1))$, where $\epsilon \coloneqq \hat{w}_m(Y) - \tilde{w}(Y)$. Recall from Proposition \ref{prop:bias-and-var-main} that
    \[|\textup{Bias}(\approxmr) - \textup{Bias}(\approxipw)|  = |\Ebeh[\epsilon\,Y] |. \]
    From this it follows using Cauchy-Schwarz inequality that,
    \begin{align*}
        |\textup{Bias}(\approxmr) - \textup{Bias}(\approxipw)| &=|\Ebeh[\epsilon\,Y] | \leq \left(\Ebeh[\epsilon^2] \Ebeh[Y^2]\right)^{1/2}.
    \end{align*}
    Combining the above with Eqn. \eqref{eqn:theorem-statement}, it follows that,
    \begin{align*}
        |\textup{Bias}(\approxmr) - \textup{Bias}(\approxipw)| \leq C^{1/2}\, m^{-1/3}\, \log{\frac{6}{\delta}} (\Ebeh[Y^2])^{1/2} = C_1 \, m^{-1/3}\, \log{\frac{6}{\delta}}
    \end{align*}
    holds with probability at least $(1-\delta)(1-o_{k}(1))$, where $C_1 = C^{1/2}\, (\Ebeh[Y^2])^{1/2}$. 

    Next, to prove the variance result, we recall from Proposition \ref{prop:bias-and-var-main} that
    \begin{align*}
        n (\Vbeh[\approxipw] - \Vbeh[\approxmr]) &= \Ebeh[\Vbeh[\hat{\rho}(A, X) \mid Y]\, Y^2] - \Vbeh[\epsilon\,Y] - 2\,\textup{Cov}(\epsilon\,Y, \tilde{w}(Y)\,Y)
    \end{align*}
    Now note that, under Assumption \ref{assumption:outcome-bounded},
    \begin{align*}
         \Vbeh[\epsilon\,Y]  \leq \Ebeh[(\epsilon\,Y)^2] \leq M^2 \Ebeh[\epsilon^2] \leq C\, M^2\, m^{-2/3} \, \log^2{\frac{6}{\delta}} = C_2\, m^{-2/3} \, \log^2{\frac{6}{\delta}},
    \end{align*}
    holds with probability at least $(1-\delta)(1-o_{k}(1))$, where $C_2 = C\, M^2$. Similarly, we have that with probability at least $(1-\delta)(1-o_{k}(1))$,
    \begin{align*}
        |\textup{Cov}(\epsilon\,Y, \tilde{w}(Y)\,Y)| &= |\Ebeh[\epsilon\,\tilde{w}(Y)\,Y^2] - \Ebeh[\epsilon\,Y]\Ebeh[\tilde{w}(Y)\,Y] |\\ 
        &\leq |\Ebeh[\epsilon \,\tilde{w}(Y)\,Y^2]| + |\Ebeh[\epsilon\,Y]\Ebeh[\tilde{w}(Y)\,Y]|\\
        &\leq \left(\Ebeh[\epsilon^2]\Ebeh[\tilde{w}(Y)^2\,Y^4]\right)^{1/2} + (\Ebeh[\epsilon^2]\Ebeh[Y^2])^{1/2} |\Ebeh[\tilde{w}(Y)\,Y]|\\
        &= (\Ebeh[\epsilon^2])^{1/2}\,\left( (\Ebeh[\tilde{w}(Y)^2\,Y^4])^{1/2} + (\Ebeh[Y^2])^{1/2}\,|\Ebeh[\tilde{w}(Y)\,Y]|\right)\\
        &\leq C_3\,m^{-1/3}\,\log{\frac{6}{\delta}},
    \end{align*}
    where $C_3 = C\,(\Ebeh[\tilde{w}(Y)^2\,Y^4])^{1/2} + (\Ebeh[Y^2])^{1/2}\,|\Ebeh[\tilde{w}(Y)\,Y]|$, and we use Cauchy-Schwarz inequality in the third step above. Putting this together, we obtain the required result.
\end{proof}

\paragraph{Intuition} This theorem shows that as the number of training samples $m$ increases, the biases of MR and IPW estimators become roughly equal, whereas the variance of MR estimator falls below that of the IPW estimator. The empirical results shown in Appendix \ref{subsec:mips-empirical} are consistent with this result.
Moreover, in Theorem \ref{prop:bias-and-var-v3}, the estimated policy ratio $\hat{\rho}(a, x)$ is fixed for increasing $m$, i.e., we do not update $\hat{\rho}(a, x)$ as more training data becomes available. While this may seem as a disadvantage for the IPW estimator, we point out that the result also holds when the policy ratio is exact (i.e., $\hat{\rho}(a, x) = \rho(a, x)$) and hence the IPW estimator is unbiased.

\paragraph{Relaxing Assumption \ref{assumption:pol-ratios-bounded}}
\cite{lai2023generalization}[Theorem 4.1] suppose that the data has the relationship shown in Assumption \ref{assumption:pol-ratios-bounded}. However, the theorem relies on Corollary 4.4 in \cite{lin2020optimal}, which requires a strictly weaker assumption (Assumption 1 in \cite{lin2020optimal}). Therefore, we can relax Assumption \ref{assumption:pol-ratios-bounded} to the following assumption.
\begin{assumption}\label{assum:relaxed-assumption}
    There exists positive constants $Q$ and $M$ such that for all $l \geq 2$ with $l \in \mathbb{N}$
    \[
\Ebeh[\hat{\rho}(A, X)^l\mid Y] \leq \frac{1}{2} \,l!\,M^{l-2}\,Q^2 
    \]
    $\pbeh$-almost surely.
\end{assumption}
It is easy to check that Assumption \ref{assum:relaxed-assumption} is strictly weaker than Assumption \ref{assumption:pol-ratios-bounded}, and is also satisfied if the policy ratio $\hat{\rho}(A, X)$ is almost surely bounded. For simplicity, we use the stronger assumption in our Proposition \ref{prop:bias-and-var-v3}.  

\section{Generalised formulation of the MIPS estimator \citep{saito2022off}}\label{app:gmips}
As described in Section \ref{subsec:mips-comparison}, the MIPS estimator proposed by \cite{saito2022off} assumes the existence of \emph{action embeddings} $E$ which summarise all relevant information about the action $A$, and achieves a lower variance than the IPW estimator. To achieve this, the MIPS estimator only considers the shift in the distribution of $(X, E)$ as a result of policy shift, instead of considering the shift in $(X, A)$ (as in IPW estimator). In this section, we show that this idea can be generalised to instead consider general representations $R$ of the context-action pair $(X, A)$, which encapsulate all relevant information about the outcome $Y$. The MIPS estimator is a special case of this generalised setting where the representation $R$ is of the form $(X, E)$.

\paragraph{Generalised MIPS (G-MIPS) estimator}
Suppose that there exists an embedding $R$ of the context-action pair $(X, A)$, with the Bayesian network shown in Figure \ref{fig:embedding_single}. Here, $R$ may be a lower-dimensional representation of the $(X, A)$ pair which contains all the information necessary to predict the outcome $Y$. This corresponds to the following conditional independence assumption:
\begin{assumption}\label{assum:indep-general}
    The context-action pair $(X, A)$ has no direct effect on the outcome $Y$ given $R$, i.e., 
    $Y \indep (X, A) \mid R$.
\end{assumption}

\begin{figure}[ht]
\centering
\begin{tikzpicture}

\node[circle,draw, minimum size=1.2cm] (R0) at (0,0) {\begin{small}$(X, A)$\end{small}
};
\node[circle,draw, minimum size=1.2cm] (R1) at (2,0) {$R$};
\node[circle,draw, minimum size=1.2cm] (Y) at (4,0) {$Y$};

\path[->, thick] (R0) edge (R1);
\path[->, thick] (R1) edge (Y);

\end{tikzpicture}
\caption{Bayesian network corresponding to Assumption \ref{assum:indep-general}.}
\label{fig:embedding_single}
\end{figure}
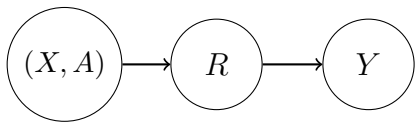
As illustrated in Figure \ref{fig:embedding_single}, Assumption \ref{assum:indep-general} means that the embedding $R$ fully mediates every possible effect of $(X, A)$ on $Y$. The generalised MIPS estimator $\hat{\theta}_{\textup{G-MIPS}}$ of target policy value, $\Etar[Y]$, is defined as
\[
\hat{\theta}_{\textup{G-MIPS}} \coloneqq \frac{1}{n}\sum_{i=1}^n \frac{\ptar(r_i)}{\pbeh(r_i)}\, y_i,
\]
where $\pbeh(r)$ denote the density of $R$ under the behaviour policy (likewise for $\ptar(r)$). Under assumption \ref{assum:indep-general}, $\hat{\theta}_{\textup{G-MIPS}}$ provides an unbiased estimator of target policy value. 
Similar to Lemma \ref{lemma:weights-est}, the density ratio $\frac{\ptar(r)}{\pbeh(r)}$ can be estimated by solving the regression problem
\begin{align}
    \arg \min_f \Ebeh \left(\frac{\tar(A\mid X)}{\beh(A\mid X)} - f\left(R\right)\right)^2. \label{eq:embedding-ratio-estimation}
\end{align}

\subsection{Variance reduction of G-MIPS estimator}\label{app:gmips-var-reduction}
By only considering the shift in the embedding $R$, the G-MIPS estimator achieves a lower variance relative to the vanilla IPW estimator. The following result, which is a straightforward extension of \cite[Theorem 3.6]{saito2022off}, formalises this.

\begin{proposition}[Variance reduction of G-MIPS]\label{prop:mips_var_reduction}
    When the ratios $\rho(a, x)$ and $\frac{\ptar(r)}{\pbeh(r)}$ are known exactly then under Assumption \ref{assum:indep-general}, we have that $\Ebeh[\thetaipw] = \Ebeh[\hat{\theta}_{\textup{G-MIPS}}] = \Etar[Y]$. Moreover,
\begin{align*}
     \Vbeh[\thetaipw] - \Vbeh[\hat{\theta}_{\textup{G-MIPS}}]
    \geq \frac{1}{n}\Ebeh \left[ \E[Y^2\mid R] \Vbeh[\rho(A, X)\mid R] \right] \geq 0.
\end{align*}
\end{proposition}

\begin{proof}[Proof of Proposition \ref{prop:mips_var_reduction}]
The following proof, which is included for completeness, is a straightforward extension of \cite[Theorem 3.6]{saito2022off}. 
\begin{align*}
    &n (\Vbeh[\thetaipw] - \Vbeh[\hat{\theta}_{\textup{MIPS}}])\\
    &\quad= \Vbeh\left[\frac{\tar(A|X)}{\beh(A|X)}\,Y\right] - \Vbeh\left[\frac{\ptar(R)}{\pbeh(R)}\,Y \right]\\
     &\quad= \Vbeh\left[\Ebeh\left[\frac{\tar(A|X)}{\beh(A|X)}\,Y \Bigg| R\right]\right] + \Ebeh\left[ \Vbeh\left[\frac{\tar(A|X)}{\beh(A|X)}\,Y\Bigg| R \right]\right] - \Vbeh\left[\Ebeh\left[ \frac{\ptar(R)}{\pbeh(R)}\,Y \Bigg| R\right]\right]\\
     &\qquad- \Ebeh\left[\Vbeh\left[ \frac{\ptar(R)}{\pbeh(R)}\,Y \Bigg| R\right]\right]
\end{align*}
Now using the conditional independence Assumption \ref{assum:indep-general}, the first term on the RHS above becomes,
\begin{align*}
    \Vbeh\left[\Ebeh\left[\frac{\tar(A|X)}{\beh(A|X)}\,Y \Bigg| R\right]\right] &= \Vbeh\left[\Ebeh\left[\frac{\tar(A|X)}{\beh(A|X)}\Bigg| R\right]\,\Ebeh\left[Y | R\right]\right]\\
    &= \Vbeh\left[\frac{\ptar(R)}{\pbeh(R)}\,\Ebeh\left[Y | R\right]\right],
\end{align*}
where in the last step above we use the fact that
\begin{align*}
    \Ebeh\left[\frac{\tar(A|X)}{\beh(A|X)}\Bigg| R\right] = \frac{\ptar(R)}{\pbeh(R)}.
\end{align*}
Putting this together, we get that
\begin{align}
    &n (\Vbeh[\thetaipw] - \Vbeh[\hat{\theta}_{\textup{MIPS}}]) \nonumber\\ 
    &\quad= \Ebeh\left[ \Vbeh\left[\frac{\tar(A|X)}{\beh(A|X)}\,Y\Bigg| R \right]\right] - \Ebeh\left[\Vbeh\left[ \frac{\ptar(R)}{\pbeh(R)}\,Y \Bigg| R\right]\right]. \label{eq:variance-difference}
\end{align}
Since we have that 
\begin{align*}
    \Ebeh\left[\frac{\tar(A|X)}{\beh(A|X)}\,Y\Bigg| R \right] = \Ebeh\left[\frac{\tar(A|X)}{\beh(A|X)}\Bigg| R \right]\,\Ebeh\left[Y| R \right] = \frac{\ptar(R)}{\pbeh(R)}\,\Ebeh\left[Y| R \right],
\end{align*}
Eq. \eqref{eq:variance-difference} becomes,
\begin{align*}
    &\Ebeh\left[ \Vbeh\left[\frac{\tar(A|X)}{\beh(A|X)}\,Y\Bigg| R \right]\right] - \Ebeh\left[\Vbeh\left[ \frac{\ptar(R)}{\pbeh(R)}\,Y \Bigg| R\right]\right] \\
    &\quad= \Ebeh\left[ \Ebeh\left[ \left(\frac{\tar(A|X)}{\beh(A|X)}\,Y \right)^2\Bigg| R  \right] - \Ebeh\left[ \left(\frac{\ptar(R)}{\pbeh(R)}\,Y \right)^2\Bigg| R \right] \right]\\
    &\quad= \Ebeh\left[ \Ebeh\left[ \left(\frac{\tar(A|X)}{\beh(A|X)} \right)^2 \Bigg| R  \right] \, \Ebeh\left[Y^2 | R  \right] - \left(\frac{\ptar(R)}{\pbeh(R)}\right)^2\,\Ebeh\left[Y^2 | R \right] \right]\\
    &\quad= \Ebeh\left[\Ebeh\left[Y^2 | R \right]\, \left( \Ebeh\left[ \left(\frac{\tar(A|X)}{\beh(A|X)} \right)^2 \Bigg| R  \right] - \left(\Ebeh\left[ \frac{\tar(A|X)}{\beh(A|X)} \Bigg| R  \right]\right)^2\, \right) \right]\\
    &\quad=\Ebeh\left[\Ebeh\left[Y^2 | R \right]\, \Vbeh\left[\frac{\tar(A|X)}{\beh(A|X)} \Bigg| R \right]\right].
\end{align*}
\end{proof}

\paragraph{Intuition}
Here, $R$ contains all relevant information regarding the outcome $Y$. Moreover, intuitively $R$ can be thought of as the state obtained by `filtering out' relevant information about $Y$ from $(X, A)$. Therefore, $R$ contains less `redundant' information regarding the outcome $Y$ as compared to the covariate-action pair $(X, A)$. As a result, the G-MIPS estimator which only considers the shift in the marginal distribution of $R$ due to the policy shift is more efficient than the IPW estimator, which considers the shift in the joint distribution of $(X, A)$ instead.
In fact, as the amount of `redundant' information regarding $Y$ decreases in the embedding $R$, the G-MIPS estimator becomes increasingly efficient with decreasing variance. We formalise this as follows:
\begin{assumption}\label{assum:two-embeddings}
    Assume there exist embeddings $R^{(1)}, R^{(2)}$ of the covariate-action pair $(X, A)$, with Bayesian network shown in Figure \ref{fig:embedding_double}. 
    This corresponds to the following conditional independence assumptions:
    \[
    R^{(2)} \indep (X, A) \mid R^{(1)}, \qquad \textup{and} \qquad Y \indep (R^{(1)}, X, A) \mid R^{(2)}.
    \]
\end{assumption}
\begin{figure}[h!]
\centering
\begin{tikzpicture}

\node[circle,draw, minimum size=1.2cm] (R0) at (0,0) {\begin{small}$(X, A)$\end{small}};
\node[circle,draw, minimum size=1.2cm] (R1) at (2,0) {$R^{(1)}$};
\node[circle,draw, minimum size=1.2cm] (R2) at (4,0) {$R^{(2)}$};
\node[circle,draw, minimum size=1.2cm] (Y) at (6,0) {$Y$};

\path[->, thick] (R0) edge (R1);
\path[->, thick] (R1) edge (R2);
\path[->, thick] (R2) edge (Y);

\end{tikzpicture}
\caption{Bayesian network corresponding to Assumption \ref{assum:two-embeddings}.}
\label{fig:embedding_double}
\end{figure}  
We can define G-MIPS estimators for these embeddings to obtain unbiased OPE estimators under Assumption \ref{assum:two-embeddings} as follows:
\[    \hat{\theta}^{(j)}_{\textup{G-MIPS}} \coloneqq \frac{1}{n}\sum_{i=1}^n \frac{\ptar(r_i^{(j)})}{\pbeh(r_i^{(j)})}\, y_i,
\]
for $j\in \{1, 2\}$. Here, $\frac{\ptar(r^{(j)})}{\pbeh(r^{(j)})}$ is the ratio of marginal densities of $R^{(j)}$ under target and behaviour policies. 
We next show that the variance of $\hat{\theta}^{(j)}_{\textup{G-MIPS}}$ decreases with increasing $j$.
\begin{proposition}\label{prop:mips_generalised}
    When the ratios $\rho(a, x)$, $w(y)$ and $\frac{\ptar(r^{(j)})}{\pbeh(r^{(j)})}$ are known exactly for $j \in \{1, 2\}$, then under Assumption \ref{assum:two-embeddings} we get that
    \[
    \Ebeh[\thetaipw] = \Ebeh[\hat{\theta}^{(1)}_{\textup{G-MIPS}}] = \Ebeh[\hat{\theta}^{(2)}_{\textup{G-MIPS}}] = \Ebeh[\thetamr] = \Etar[Y].
    \]
    Moreover, 
    \[
    \Vbeh[\thetaipw] \geq \Vbeh[\hat{\theta}^{(1)}_{\textup{G-MIPS}}] \geq  \Vbeh[\hat{\theta}^{(2)}_{\textup{G-MIPS}}] \geq \Vbeh[\thetamr].
    \]
\end{proposition}
\begin{proof}[Proof of Proposition \ref{prop:mips_generalised}]
    First, we prove that the G-MIPS estimators are unbiased using induction on $j$. We define $R^{(0)} \coloneqq (X, A)$ and $\hat{\theta}^{(0)}_{\textup{G-MIPS}}$ defined as
    \[
    \hat{\theta}^{(0)}_{\textup{G-MIPS}} \coloneqq \frac{1}{n}\sum_{i=1}^n \frac{\ptar(r_i^{(0)})}{\pbeh(r_i^{(0)})}\, y_i,
    \]
    recovers the IPW estimator $\thetaipw$. When $j=0$, we know that $\hat{\theta}^{(0)}_{\textup{G-MIPS}} = \thetaipw$ is unbiased. 
    Now, assume that $\Ebeh[\hat{\theta}^{(j)}_{\textup{G-MIPS}}] = \Etar[Y]$.

    Conditional on $R^{(j)}$, $R^{(j+1)}$ does not depend on the policy. Therefore, 
    \begin{align*}
        \frac{\ptar(r^{(j)})}{\pbeh(r^{(j)})} = \frac{\ptar(r^{(j)})\,p(r^{(j+1)}\mid r^{(j)}) }{\pbeh(r^{(j)})\,p(r^{(j+1)}\mid r^{(j)})} = \frac{\ptar(r^{(j)}, r^{(j+1)})}{\pbeh(r^{(j)}, r^{(j+1)})}.
    \end{align*}
    And therefore,
    \begin{align*}
        \frac{\ptar(r^{(j+1)})}{\pbeh(r^{(j+1)})} &= \int_{r^{(j)}} \frac{\ptar(r^{(j)}, r^{(j+1)})}{\pbeh(r^{(j)}, r^{(j+1)})} \, \pbeh(r^{(j)} \mid r^{(j+1)}) \,\mathrm{d} r^{(j)}\\ 
        &= \int_{r^{(j)}} \frac{\ptar(r^{(j)})}{\pbeh(r^{(j)})} \,\pbeh(r^{(j)} \mid r^{(j+1)}) \,\mathrm{d} r^{(j)}\\ 
        &= \Ebeh\left[\frac{\ptar(R^{(j)})}{\pbeh(R^{(j)})} \Bigg|  R^{(j+1)}=r^{(j+1)}\right].
    \end{align*}
    Using this and the fact that $R^{(j)}\indep Y \mid R^{(j+1)}$, we get that
    \begin{align*}
        \Ebeh\left[\hat{\theta}^{(j+1)}_{\textup{G-MIPS}} \right] &= \Ebeh\left[\frac{\ptar(R^{(j+1)})}{\pbeh(R^{(j+1)})}\, Y \right]\\
        &= \Ebeh\left[\frac{\ptar(R^{(j+1)})}{\pbeh(R^{(j+1)})}\, \Ebeh[Y| R^{(j+1)}] \right]\\
        &= \Ebeh\left[\Ebeh\left[\frac{\ptar(R^{(j)})}{\pbeh(R^{(j)})} \Bigg|  R^{(j+1)}\right]\, \Ebeh[Y| R^{(j+1)}] \right] \\
        &= \Ebeh\left[\Ebeh\left[\frac{\ptar(R^{(j)})}{\pbeh(R^{(j)})} \, Y \Bigg|  R^{(j+1)}\right]\right]\\
        &= \Ebeh\left[ \frac{\ptar(R^{(j)})}{\pbeh(R^{(j)})} \, Y \right]\\
        &= \Ebeh\left[\hat{\theta}^{(j)}_{\textup{G-MIPS}} \right] = \Etar[Y].
    \end{align*}
    Next, to prove the variance result we consider the difference
    \begin{align*}
        &\Vbeh[\hat{\theta}^{(j)}_{\textup{G-MIPS}}] - \Vbeh[\hat{\theta}^{(j+1)}_{\textup{G-MIPS}}] \\
        &= \frac{1}{n}\left(\Vbeh\left[\frac{\ptar(R^{(j)})}{\pbeh(R^{(j)})}\, Y\right] - \Vbeh\left[\frac{\ptar(R^{(j+1)})}{\pbeh(R^{(j+1)})}\, Y\right]\right) \\
        &= \frac{1}{n}\Bigg(\Vbeh\left[ \Ebeh\left[\frac{\ptar(R^{(j)})}{\pbeh(R^{(j)})}\, Y \Bigg| R^{(j+1)} \right] \right] + \Ebeh\left[ \Vbeh\left[\frac{\ptar(R^{(j)})}{\pbeh(R^{(j)})}\, Y \Bigg| R^{(j+1)} \right] \right] \\
        &\qquad- \Vbeh\left[\frac{\ptar(R^{(j+1)})}{\pbeh(R^{(j+1)})}\, \Ebeh[Y\mid R^{(j+1)}]\right] - \Ebeh \left[\left(\frac{\ptar(R^{(j+1)})}{\pbeh(R^{(j+1)})}\right)^2\,\Vbeh[Y\mid R^{(j+1)}] \right] \Bigg)
    \end{align*}
    where in the last step we use the law of total variance. Now, using the fact that $R^{(j)}\indep Y \mid R^{(j+1)}$, we can rewrite the expression above as
    \begin{align*}
        &= \frac{1}{n}\Bigg(\Vbeh\left[ \Ebeh\left[\frac{\ptar(R^{(j)})}{\pbeh(R^{(j)})}\Bigg| R^{(j+1)} \right]\, \Ebeh[Y | R^{(j+1)} ] \right] + \Ebeh\left[ \Vbeh\left[\frac{\ptar(R^{(j)})}{\pbeh(R^{(j)})}\, Y \Bigg| R^{(j+1)} \right] \right] \\
        &\qquad- \Vbeh\left[\frac{\ptar(R^{(j+1)})}{\pbeh(R^{(j+1)})}\, \Ebeh[Y\mid R^{(j+1)}]\right] - \Ebeh \left[\left(\frac{\ptar(R^{(j+1)})}{\pbeh(R^{(j+1)})}\right)^2\,\Vbeh[Y\mid R^{(j+1)}] \right] \Bigg)\\
        &= \frac{1}{n}\Bigg( \Ebeh\left[ \Vbeh\left[\frac{\ptar(R^{(j)})}{\pbeh(R^{(j)})}\, Y \Bigg| R^{(j+1)} \right] \right] - \Ebeh \left[\left(\frac{\ptar(R^{(j+1)})}{\pbeh(R^{(j+1)})}\right)^2\,\Vbeh[Y\mid R^{(j+1)}] \right]\Bigg).
    \end{align*}
    Moreover, again using the conditional independence $R^{(j)}\indep Y \mid R^{(j+1)}$, we can expand the first term in the expression above as follows:
    \begin{align*}
        \Ebeh\left[ \Vbeh\left[\frac{\ptar(R^{(j)})}{\pbeh(R^{(j)})}\, Y \Bigg| R^{(j+1)} \right] \right] &=
        \Ebeh\Bigg[ \Ebeh\left[\frac{\ptar^2(R^{(j)})}{\pbeh^2(R^{(j)})} \Bigg| R^{(j+1)} \right]\,\Ebeh[Y^2 | R^{(j+1)}] \\
        &\qquad- \left(\Ebeh\left[\frac{\ptar(R^{(j)})}{\pbeh(R^{(j)})} \Bigg| R^{(j+1)} \right] \Ebeh[Y | R^{(j+1)}] \right)^2 \Bigg]\\
        &\geq 
        \Ebeh\Bigg[ \left(\Ebeh\left[\frac{\ptar(R^{(j)})}{\pbeh(R^{(j)})} \Bigg| R^{(j+1)} \right]\right)^2 \,\Ebeh[Y^2 | R^{(j+1)}] \\
        &\qquad- \left(\frac{\ptar(R^{(j+1)})}{\pbeh(R^{(j+1)})} \Ebeh[Y | R^{(j+1)}] \right)^2 \Bigg]\\
        &= \Ebeh\Bigg[ \left(\frac{\ptar(R^{(j+1)})}{\pbeh(R^{(j+1)})}\right)^2 \, \Vbeh[Y\mid R^{(j+1)}] \Bigg].
    \end{align*}
    Here, to get the inequality above, we use the fact that $\E[X^2] \geq (\E[X])^2$. Putting this together, we get that $\Vbeh[\hat{\theta}^{(j)}_{\textup{G-MIPS}}] - \Vbeh[\hat{\theta}^{(j+1)}_{\textup{G-MIPS}}] \geq 0$.

    Moreover, the result $\Vbeh[\hat{\theta}^{(2)}_{\textup{G-MIPS}}] \geq \Vbeh[\thetamr]$ follows straightforwardly from above by defining $R^{(3)} \coloneqq Y$. Then, the embeddings satisfy the causal structure 
    \[
    R^{(0)} \rightarrow R^{(1)} \rightarrow R^{(2)}  \rightarrow R^{(3)} \rightarrow Y.
    \]
    Using the result above, we know that $\Vbeh[\hat{\theta}^{(2)}_{\textup{G-MIPS}}] \geq \Vbeh[\hat{\theta}^{(3)}_{\textup{G-MIPS}}]$. But now it is straightforward to see that $\hat{\theta}^{(3)}_{\textup{G-MIPS}} = \thetamr$, and the result follows.
\end{proof}

\paragraph{Intuition}
Here, $R^{(j+1)}$ can be thought of as the embedding obtained by `filtering out' relevant information about $Y$ from $R^{(j)}$. As such, the amount of `redundant' information regarding the outcome $Y$ decreases successively along the sequence $R^{(0)} (\coloneqq (X, A)), R^{(1)}, R^{(2)}$. As a result, the G-MIPS estimators which only consider the shift in the marginal distributions of $R^{(j)}$ due to policy shift become increasingly efficient with decreasing variance as $j$ increases. Define the representation $R^{(3)} \coloneqq Y$, then the corresponding G-MIPS estimator reduces to the MR estimator, i.e., $\hat{\theta}^{(3)}_{\textup{G-MIPS}} = \thetamr$. Moreover, this estimator has minimum variance among all the G-MIPS estimators $\{\hat{\theta}^{(j)}_{\textup{G-MIPS}}\}_{0\leq j\leq k}$, as the representation $R^{(3)}$ contains precisely the least amount of information necessary to obtain the outcome $Y$. In other words, $Y$ itself serves as the `best embedding' of covariate-action pair $R^{(0)}$ which contains all relevant information regarding $Y$. We verify this empirically in Appendix \ref{subsec:mips-empirical} by reproducing the experimental setup in \cite{saito2022off} along with the MR baseline. Additionally, the MR estimator does not rely on assumptions like \ref{assum:indep-general} for unbiasedness. 

In addition to this, solving the regression problem in Eq. \eqref{eq:embedding-ratio-estimation} will typically be more difficult when $R$ is higher dimensional (as is likely to be the case for many choices of embeddings $R$), leading to high bias. In contrast, for MR the embedding $R=Y$ is one dimensional and therefore the regression problem is significantly easier to solve and yields lower bias. Our empirical results in Appendix \ref{app:experiments} confirm this.

\subsection{Doubly robust G-MIPS estimators}
Consider the setup for the G-MIPS estimator shown in Figure \ref{fig:embedding_single}. In this case, we can derive a doubly robust extension of the G-MIPS estimator, denoted as GM-DR, which uses an estimate of the conditional mean $\tilde{\mu}(r) \approx \E[Y\mid R=r]$ as a control variate to decrease the variance of G-MIPS estimator. This can be explicitly written as follows:
\begin{align}
\thetagmdr \coloneqq \frac{1}{n} \sum_{i=1}^n \frac{\ptar(r_i)}{\pbeh(r_i)}\,(y_i - \tilde{\mu}(r_i)) + \tilde{\eta}(\tar). \label{eq:gmips-dr}    
\end{align}
where $\tilde{\eta}(\tar) = \frac{1}{n} \sum_{i=1}^n \sum_{r' \in \mathcal{R}} \tilde{\mu}(r') \, \ptar(r' \mid x_i)$ is the analogue of the direct method. Here, $\mathcal{R}$ denotes the space of the possible of the representations $R$\footnote{the $\sum_{r' \in \mathcal{R}}$ can be replaced with $\int_{r' \in \mathcal{R}} \mathrm{d}r'$ when $\mathcal{R}$ is continuous}. Moreover, given the density $p(r \mid x, a)$, we can compute $\ptar(r\mid x)$ using
\[
\ptar(r\mid x) = \sum_{a' \in \Aspace} p(r \mid x, a')\,\tar(a'\mid x).
\]
It is straightforward to extend ideas from \cite{dudik2014doubly} to show that estimator $\thetagmdr$ is doubly robust in that it will yield accurate value estimates if either the importance weights $\frac{\ptar(r)}{\pbeh(r)}$ or the outcome model $\tilde{\mu}(r)$ is well estimated. 

\paragraph{There is no analogous DR extension of the MR estimator}
A consequence of considering the embedding $R=Y$ (as in MR) is that in this case we do not have an analogous doubly robust extension as above. To see why this is the case, note that when $R=Y$, we get that $\tilde{\mu}(r) = \E[Y\mid R=r] = \E[Y\mid Y=y] = y$. If we substitute this $\tilde{\mu}(r)$ in \eqref{eq:gmips-dr}, we are simply left with $\tilde{\eta}(\tar)$ on the right hand side (as the first term cancels out). This means that the resulting estimator does not retain the doubly robust nature as we no longer obtain an accurate estimate if either the outcome model or the importance ratios are well estimated.
\section{Application to causal inference}\label{app:causal-inference}
In this section, we investigate the application of the MR estimator for the estimation of average treatment effect (ATE). In this setting, we suppose that $\Aspace = \{0, 1\}$, and the goal is to estimate ATE defined as follows:
\[
\ate \coloneqq \E[Y(1)-Y(0)]
\]
Here, we use the potential outcomes notation \citep{robins1986new} to denote the outcome under a deterministic policy $\tar(a'\mid x) = \mathbbm{1}(a'=a)$ as $Y(a)$. 

Specifically, the IPW estimator applied to ATE estimation yields:
\[
\ateipw = \frac{1}{n} \sum_{i=1}^n \rho_{\ate}(a_i, x_i) \times y_i,
\]
where 
\[
\rho_{\ate}(a, x) \coloneqq \frac{\mathbbm{1}(a=1) - \mathbbm{1}(a=0)}{\beh (a|x)}.
\]
Similarly, the MR estimator can be written as
\[
\atemr = \frac{1}{n}\sum_{i=1}^n w_{\ate}(y_i)\times y_i, 
\]
where
\[
w_{\ate}(y) = \frac{p_{\pi^{(1)}}(y) - p_{\pi^{(0)}}(y)}{\pbeh(y)},
\] 
and $\pi^{(a)}(a'\mid x) \coloneqq \mathbbm{1}(a'=a)$ for $a\in \{0,1\}$.

Again, using the fact that $w_{\ate}(Y) \eqas \E[\rho_{\ate}(A, X)\mid Y]$, we can obtain $w_{\ate}$ by minimising a simple mean-squared loss:
\begin{align*}
    w_{\ate} =\arg \min_{f} \Ebeh \Big[\frac{\mathbbm{1}(A=1)- \mathbbm{1}(A=0)}{\beh (A|X)}-f(Y)\Big]^2.
\end{align*}
\begin{proposition}[Variance comparison with IPW ATE estimator]\label{prop:ate_variance}
When the weights $\rho_{\ate}(a, x)$ and $w_{\ate}(y)$ are known exactly, we have that $\V[\atemr] \leq \V[\ateipw]$. Specifically,
\begin{align*}
    \V[\ateipw] - \V[\atemr] = \frac{1}{n}\E\left[ \V\left[\rho_{\ate}(A, X) | Y \right]\,Y^2 \right] \geq 0.
\end{align*}
\end{proposition}
\begin{proof}[Proof of Proposition \ref{prop:ate_variance}] We have
\begin{align}
    \V[\ateipw] - \V[\atemr] &= \frac{1}{n}\left( \V[\rho_{\ate}(A, X)\, Y] - \V[w_{\ate}(Y)\,Y] \right). \label{eq:variance_ate_ipw_minus_mr}
\end{align}
Using the tower law of variance, we get that
\begin{align*}
    \V[\rho_{\ate}(A, X)\, Y] 
    &= \V[\E[\rho_{\ate}(A, X)\,  Y\mid Y]] + \E[\V[\rho_{\ate}(A, X)\, Y\mid Y]]\\
    &= \V[\E[\rho_{\ate}(A, X)\mid Y]\,  Y] + \E[\V[\rho_{\ate}(A, X)\mid Y]\,Y^2]\\
    &= \V[w_{\ate}(Y)\,Y] + \E[\V[\rho_{\ate}(A, X)\mid Y]\,Y^2].
\end{align*}
Putting this together with \eqref{eq:variance_ate_ipw_minus_mr} we obtain,
\begin{align*}
    \V[\ateipw] - \V[\atemr] &= \frac{1}{n} \E[\V[\rho_{\ate}(A, X)\mid Y]\,Y^2],
\end{align*}
which straightforwardly leads to the result.
\end{proof}

Given the above definitions, the IPW estimator for $\E[Y(a)]$ would only consider datapoints with $A=a$, as it weights the samples using the policy ratios $\mathbbm{1}(A=a)/\beh(A|X)$ which are only non-zero when $A=a$. 
This is however not the case with the MR estimator, as it uses the weights $\ptar(Y)/\pbeh(Y)$ which are not necessarily zero for $A\neq a$. Therefore, MR uses all evaluation datapoints $\D$ when estimating $\E[Y(a)]$. The MR estimator therefore leads to a more efficient use of evaluation data in this example. 

Likewise, the doubly robust (DR) estimator applied to ATE estimation yields,
\begin{align*}
    \atedr \coloneqq \frac{1}{n} \sum_{i=1}^n \rho_{\ate}(a_i, x_i)\,\left(y_i - \hat{\mu}(a_i, x_i)\right) + \frac{1}{n} \sum_{i=1}^n \left( \hat{\mu}(1, x_i)-  \hat{\mu}(0, x_i)\right),
\end{align*}
where $\hat{\mu}(a, x)\approx \E[Y\mid X=x, A=a]$. 
Like in classical off-policy evaluation, DR yields an accurate estimator of ATE when either the weights $\rho_{\ate}(a, x)$ or the outcome model i.e., $\hat{\mu}(a, x) = \E[Y\mid X=x, A=a]$, are well estimated.
However, despite this doubly robust nature of the estimator, we can show that the variance of the DR estimator may be higher than that of the MR estimator in many cases. The following result formalises this variance comparison between the DR and MR estimators, and is analogous to the result in Proposition \ref{prop:var_dr} derived for classical off-policy evaluation. 
\begin{proposition}[Variance comparison with DR ATE estimator]\label{prop:ate_var_dr}
    When the weights $\rho_{\ate}(a, x)$ and $w_{\ate}(y)$ are known exactly,
    \begin{align*}
    &\V[\atedr] - \V[\atemr]\\
    &\qquad \qquad \qquad \geq \frac{1}{n}\E \left[ \V\left[ \rho_\ate(A, X)\,Y \mid Y \right] -  \V\left[ \rho_\ate(A, X)\mu(A, X) \mid X \right] \right],
\end{align*}
where $\mu(A, X) \coloneqq \E[Y\mid X, A]$.
\end{proposition}

\begin{proof}[Proof of Proposition \ref{prop:ate_var_dr}]
Using the law of total variance, we get that
\begin{align*}
    n\,\V[\atedr] &= \V[\rho_\ate(A, X)\,(Y -\hat{\mu}(A, X)) + (\hat{\mu}(1, X) - \hat{\mu}(0, X))]\\
    &= \V[ \E[\rho_\ate(A, X)\,(Y -\hat{\mu}(A, X)) + (\hat{\mu}(1, X) - \hat{\mu}(0, X))\mid X, A]] \\
    &\qquad+ \E[\V[\rho_\ate(A, X)\,(Y -\hat{\mu}(A, X)) + (\hat{\mu}(1, X) - \hat{\mu}(0, X))\mid X, A]]\\
    &= \V[ \rho_\ate(A, X)\,(\mu(A, X) -\hat{\mu}(A, X)) + (\hat{\mu}(1, X) - \hat{\mu}(0, X))]\\
    &\qquad+ \E[\rho^2_\ate(A, X)\V[Y\mid X, A]].
\end{align*}
    Again, using the law of total variance we can rewrite the first term on the RHS above as,
    \begin{align*}
        &\V[ \rho_\ate(A, X)\,(\mu(A, X) -\hat{\mu}(A, X)) + (\hat{\mu}(1, X) - \hat{\mu}(0, X))]\\
        &\quad= \V[\E[ \rho_\ate(A, X)\,(\mu(A, X) -\hat{\mu}(A, X)) + (\hat{\mu}(1, X) - \hat{\mu}(0, X))\mid  X]] \\
        &\qquad+ \E[\V[ \rho_\ate(A, X)\,(\mu(A, X) -\hat{\mu}(A, X)) + (\hat{\mu}(1, X) - \hat{\mu}(0, X))\mid  X]] \\
        &\quad\geq  \V[\E[ \rho_\ate(A, X)\,(\mu(A, X) -\hat{\mu}(A, X)) + (\hat{\mu}(1, X) - \hat{\mu}(0, X))\mid  X]]\\
        &\quad=  \V[\E[ \rho_\ate(A, X)\,(\mu(A, X) -\hat{\mu}(A, X)) + \rho_\ate(A, X)\,\hat{\mu}(A, X)\mid  X]]\\
        &\quad=  \V[\E[ \rho_\ate(A, X)\,\mu(A, X)\mid  X]],
    \end{align*}
    where, in the second last step above we use the fact that 
    \[
    \E[\rho_\ate(A, X)\,\hat{\mu}(A, X)\mid  X] = \hat{\mu}(1, X) - \hat{\mu}(0, X).
    \]

    Putting this together, we get that
    \begin{align*}
        n\,\V[\atedr] \geq \V[\E[ \rho_\ate(A, X)\,\mu(A, X)\mid  X]] + \E[\rho^2_\ate(A, X)\V[Y\mid X, A]].
    \end{align*}
    Therefore, 
    \begin{align*}
        &n\,(\V[\atedr] - \V[\atemr]) \\
        &\quad\geq \V[\E[ \rho_\ate(A, X)\,\mu(A, X)\mid  X]] + \E[\rho^2_\ate(A, X)\V[Y\mid X, A]] - \V[w_\ate(Y)\,Y]\\
        &\quad= \V[\E[ \rho_\ate(A, X)\,\mu(A, X)\mid  X]] + \E[\V[\rho_\ate(A, X)\,Y\mid X, A]] - \V[w_\ate(Y)\,Y]\\
        &\quad= \V[\E[ \rho_\ate(A, X)\,\mu(A, X)\mid  X]] + \V[\rho_\ate(A, X)\, Y] - \V[\E[\rho_\ate(A, X)\,Y\mid X, A]]\\
        &\qquad- \V[w_\ate(Y)\,Y]\\
        &\quad= \V[\E[ \rho_\ate(A, X)\,\mu(A, X)\mid  X]] + \V[\E[\rho_\ate(A, X)\mid Y]\, Y] + \E[\V[\rho_\ate(A, X)\mid Y]\,Y^2]\\
        &\qquad- \V[\E[\rho_\ate(A, X)\,Y\mid X, A]]- \V[w_\ate(Y)\,Y]\\
        &\quad= \V[\E[ \rho_\ate(A, X)\,\mu(A, X)\mid  X]] + \V[w_\ate(Y)\, Y] + \E[\V[\rho_\ate(A, X)\mid Y]\,Y^2]\\
        &\qquad- \V[\E[\rho_\ate(A, X)\,Y\mid X, A]]- \V[w_\ate(Y)\,Y]\\
        &\quad= \V[\E[ \rho_\ate(A, X)\,\mu(A, X)\mid  X]]- \V[\E[\rho_\ate(A, X)\,Y\mid X, A]] + \E[\V[\rho_\ate(A, X)\mid Y]\,Y^2]\\
        &\quad= \V[\rho_\ate(A, X)\,\mu(A, X)] - \E[\V[\rho_\ate(A, X)\,\mu(A, X)\mid  X]] - \V[\rho_\ate(A, X)\,\mu(A, X)]\\
        &\qquad+ \E[\V[\rho_\ate(A, X)\mid Y]\,Y^2]\\
        &\quad= \E \left[ \V\left[ \rho_\ate(A, X) \mid Y \right]\, Y^2 -  \V\left[ \rho_\ate(A, X)\mu(A, X) \mid X \right] \right].
    \end{align*}
\end{proof}

Proposition \ref{prop:ate_var_dr} shows that if $\V\left[Y\, \rho_\ate(A, X) \mid Y \right]$ is greater than the conditional variance $\V\left[ \rho_\ate(A, X)\mu(A, X) \mid X \right]$ on average, the variance of the MR estimator will be less than that of the DR estimator. Intuitively, this is likely to happen when the dimension of context space $\Xspace$ is high because in this case, the conditional variance over $X$ and $A$, $\V\left[Y\, \rho_\ate(A, X) \mid Y \right]$ is likely to be greater than the conditional variance over $A$, $\V\left[ \rho_\ate(A, X)\mu(A, X) \mid X \right]$.

\section{Experimental Results}\label{app:experiments}
In this section, we provide additional experimental details for the results presented in the main text. We also include extensive experimental results to provide further empirical evidence in favour of the MR estimator. 

\paragraph{Computational details}
We ran our experiments on Intel(R) Xeon(R) CPU E5-2690 v4 @ 2.60GHz with 8GB RAM
per core. We were able to use 150 CPUs in parallel to iterate over different configurations and seeds.
However, we would like to note that for each run our algorithms only requires 1 CPU and at most 30 minutes to run as our neural networks are relatively small. Throughout our experiments, whenever the outcome $Y$ was continuous, we used a fully connected neural network with three hidden layers with 512, 256 and 32 nodes respectively (and ReLU activation function) to estimate the weights $\hat{w}(y)$. On the other hand, when the outcome is discrete we can directly estimate $\hat{w}(y) \approx \E[\hat{\rho}(A, X)\mid Y=y]$ by calculating the sample mean of $\hat{\rho}(A, X)$ on samples with $Y=y$. Additionally, for each configuration of parameters in our experiments, we ran experiments for 10 different seeds (in \{0, 1, \ldots, 9\}).

\subsection{Alternative methodology of estimating MR}
In addition to the OPE baselines like IPW, DM and DR estimators considered in the main text, we also include empirically investigate an alternative methodology of estimating MR.
Below we describe this methodology, denoted as `MR (alt)', in greater detail:
\subsubsection{MR (alt)}\label{sec:alt-estimation-method}
Recall our definition of MR estimator:
\[
\thetamr \coloneqq \frac{1}{n} \sum_{i=1}^n w(y_i)\,y_i.
\]
In the main text, we propose estimating the weights $w(y)$ first and using this to estimate $\thetamr$ using the above expression. Alternatively, we can estimate $h(y) \coloneqq y\,w(y)$ using 
\begin{align*}
    h = \arg\min_{f} \, \Ebeh \left[ \Bigg(Y\,\frac{\tar(A|X)}{\beh (A|X)}-f(Y)\Bigg)^2\right].
\end{align*}
Subsequently, the MR estimator can be written as:
\[
\thetamr = \frac{1}{n}\sum_{i=1}^n h(y_i).
\]
We refer to this alternative methodology as `MR-alt' and compare it empirically against the original methodology (which we simply refer to as `MR'). 
In general, it is difficult to say which of the two methods will perform better. Intuitively speaking, in cases where the behaviour of the quantity $Y\,\frac{\tar(A|X)}{\beh (A|X)}$ with varying $Y$ is `smoother' than that of $\frac{\tar(A|X)}{\beh (A|X)}$, the alternative method is expected to perform better. Our empirical results in the next sections show that the relative performance of the two methods varies for different data generating mechanisms.

\subsection{Synthetic data experiments}\label{subsec:mips-empirical}
Here, we include additional experimental details for the synthetic data experiments presented in Section \ref{sec:exp-synth} for completeness. For this experiment, we use the same setup as the synthetic data experiment in \cite{saito2022off}, reproduced by reusing their code with minor modifications.

\paragraph{Setup}
Here, we sample the $d$-dimensional context vectors $x$ from a standard normal distribution. 
The setup used also includes $3$-dimensional categorical action embeddings $E \in \mathcal{E}$, which are sampled from the following conditional distribution given action $A=a$,
\[
p(e\mid a) = \prod_{k=1}^{3}\frac{\exp{(\alpha_{a, e_{k}})}}{\sum_{e'\in \mathcal{E}_k} \exp{(\alpha_{a, e'})}},
\]
which is independent of the context $X$. $\{\alpha_{a, e_k}\}$ is a set of parameters sampled independently from the standard normal distribution. Each dimension of $\mathcal{E}$ has a cardinality of $10$, i.e., $\mathcal{E}_k = \{1, 2, \dots, 10\}$.

\paragraph{Reward function}
The expected reward is then defined as:
\[
q(x, e) = \sum_{k=1}^{3} \eta_k \cdot (x^T\, M\, x_{e_k} + \theta_x^T\, x + \theta_e^T\, x_{e_k}),
\]
where $M$, $\theta_x$ and $\theta_e$ are parameter matrices or vectors sampled from a uniform distribution with range $[-1, 1]$. $x_{e_k}$ is a context vector corresponding to the $k$-th dimension of the action embedding, which is unobserved to the estimators. $\eta_k$ specifies the importance of the $k$-th dimension of the action embedding, sampled from Dirichlet distribution so that $\sum_{k=1}^{3} \eta_k = 1$. 

\paragraph{Behaviour and target policies}
The behaviour policy $\beh$ is defined by applying the softmax function to $q(x, a) = \E[q(X, E)\mid A=a, X=x]$ as 
\[
\beh(a\mid x) = \frac{\exp{(-q(x, a))}}{\sum_{a'\in \Aspace}\exp{(-q(x, a'))}}.
\]

For the target policy, we define the class of parametric policies,
\[
\pi^{\alpha^\ast}(a | x) = \alpha^\ast\,\ind(a = \arg\max_{a'\in \Aspace} q(x, a')) + \frac{1-\alpha^\ast}{|\Aspace|},
\]
where $\alpha^\ast \in [0, 1]$ controls the shift between the behaviour and target policies. As shown in the main text, as $\alpha^\ast \rightarrow 1$, the shift between behaviour and target policies increases.

\paragraph{Baselines}
In the main text, we compare MR with DM, IPW, DR and MIPS estimators. In addition to these baselines, here we also consider Switch-DR \citep{wang2017optimal} and DR with Optimistic Shrinkage (DRos) \citep{su2020doubly}.
Following \cite{saito2022off}, we use the random forest \citep{breiman2001machine} along with 2-fold cross-fitting \citep{newey2018cross} to obtain $\hat{q}(x, e)$ for DR and DM methods.
To estimate $\pbeh(a\mid x, e)$ for MIPS estimator, we use logistic regression. 
We also include the results for MR estimated using the alternative methodology described in Section \ref{sec:alt-estimation-method}. We refer to this as `MR (alt)'.

\paragraph{Estimation of behaviour policy $\hatbeh$ and marginal ratio $\hat{w}(y)$}
We do not assume that the true behaviour policy $\beh$ is known, and therefore estimate $\hatbeh$ using the available training data.
For the MR estimator, we estimate the behaviour policy using a random forest classifier trained on 50\% of the training data and use the rest of the training data to estimate the marginal ratios $\hat{w}(y)$ using multi-layer perceptrons (MLP). Moreover, for a fair comparison we use a different behaviour policy estimate $\hatbeh$ for all other baselines which is trained on the entire training data.

\begin{figure}[h!]
    \centering
	\begin{subfigure}{0.8\textwidth}
	    \centering
	    \includegraphics[width=1\textwidth]{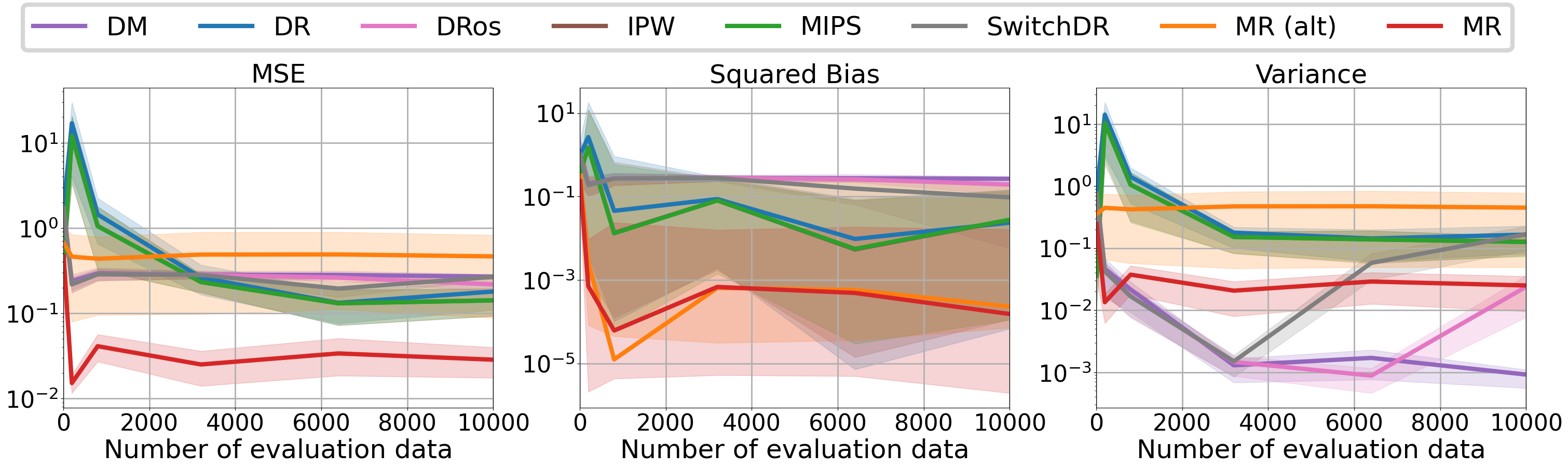}
	    \subcaption{$d=1000$, $n_{a}=100$, $\alpha^\ast = 0.8$}
	    \label{subfig:d-1000-na-250-neval-mips}
	\end{subfigure}\\
	\begin{subfigure}{0.8\textwidth} 
	    \centering
	    \includegraphics[width=1\textwidth]{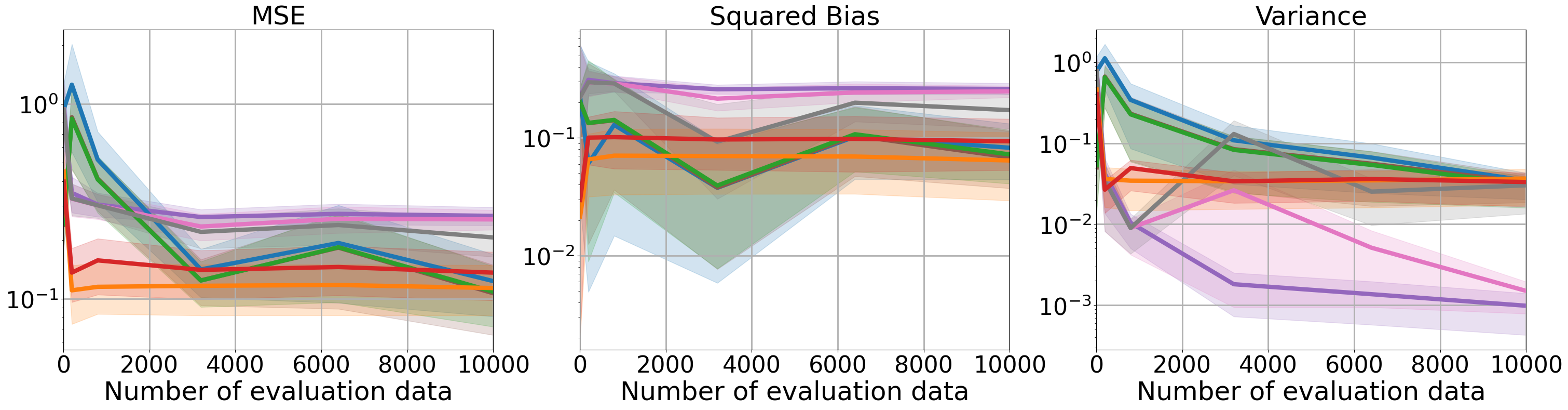}
	    \subcaption{$d=5000$, $n_{a}=250$, $\alpha^\ast = 0.8$}
	    \label{subfig:d-5000-na-250-neval-08-mips}
	\end{subfigure}\\
	\begin{subfigure}{0.8\textwidth} 
	    \centering
	    \includegraphics[width=1\textwidth]{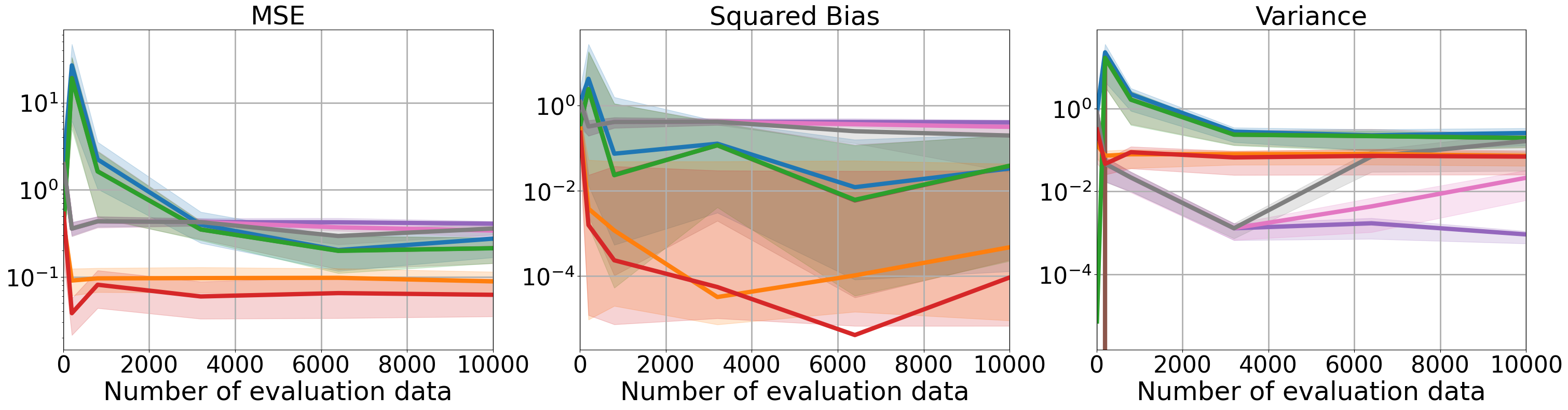}
	    \subcaption{$d=5000$, $n_{a}=250$, $\alpha^\ast = 1.0$}
	    \label{subfig:d-5000-na-250-neval-1-mips}
	\end{subfigure}
    \caption{MSE with varying size of evaluation dataset $n$ for different choices of parameters.}
    \label{fig:mse-vs-neval-mips}
\end{figure}

\begin{figure}[h!]
    \centering
	\begin{subfigure}{0.8\textwidth}
	    \centering
	    \includegraphics[width=1\textwidth]{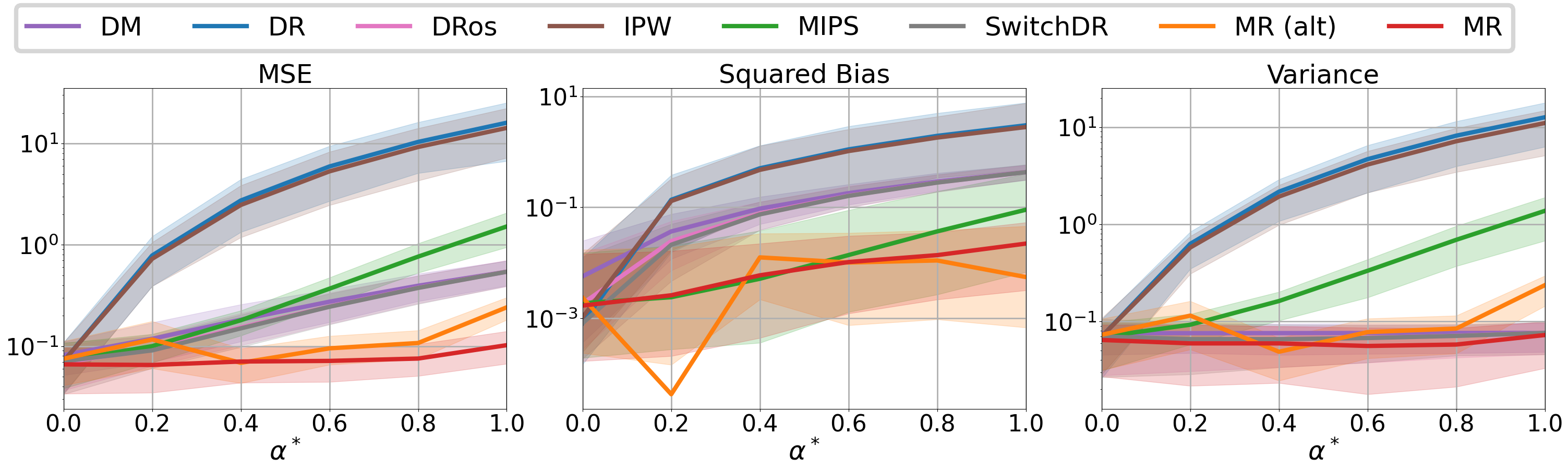}
	    \subcaption{$d=100$, $n_{a}=100$, $n = 100$}
	    \label{subfig:d-100-na-100-neval-100-alphatar-mips}
	\end{subfigure}\\
	\begin{subfigure}{0.8\textwidth} 
	    \centering
	    \includegraphics[width=1\textwidth]{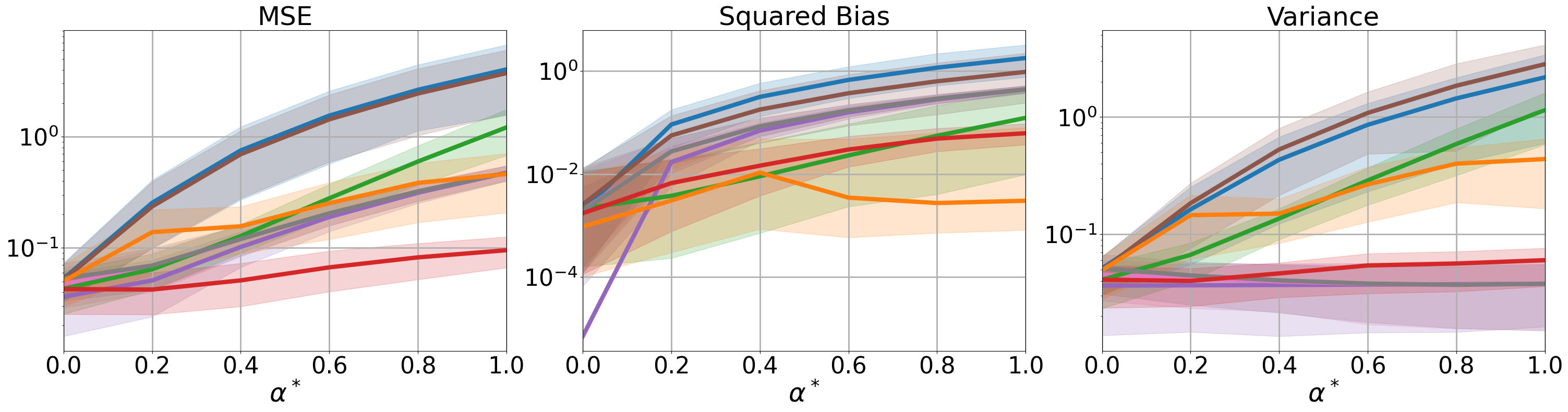}
	    \subcaption{$d=100$, $n_{a}=250$, $n = 100$}
	    \label{subfig:d-100-na-250-neval-100-alphatar-mips}
	\end{subfigure}\\
	\begin{subfigure}{0.8\textwidth} 
	    \centering
	    \includegraphics[width=1\textwidth]{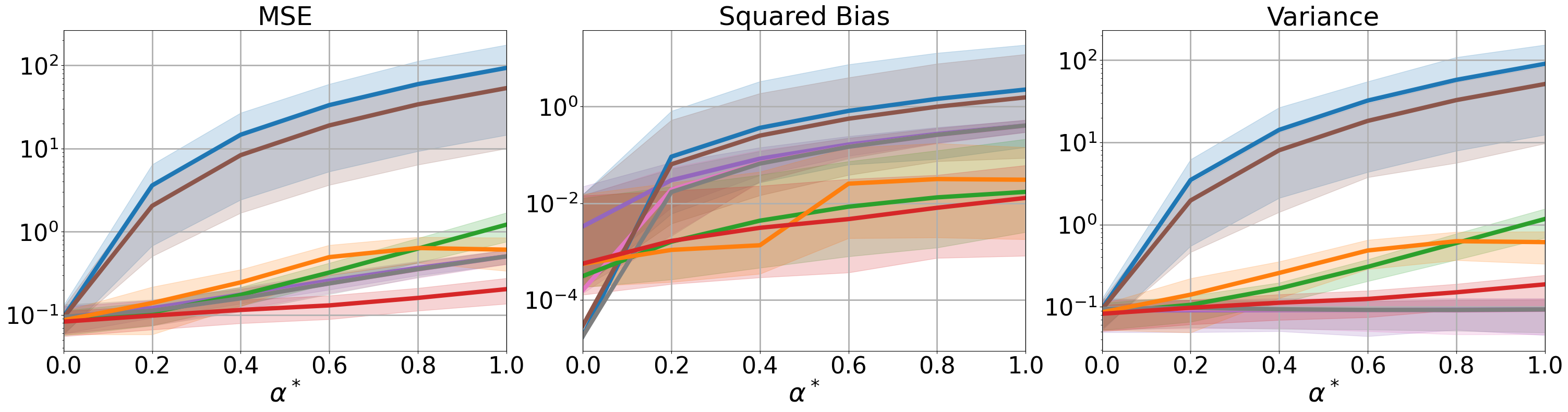}
	    \subcaption{$d=1000$, $n_{a}=250$, $n = 100$}
	    \label{subfig:d-1000-na-250-neval-100-alphatar-mips}
	\end{subfigure}
    \caption{MSE with varying $\alpha^\ast$ for different choices of parameters.}
    \label{fig:mse-vs-alphatar-mips}
\end{figure}

\subsubsection{Results}
For this experiment, the results are computed over 10 different sets of logged data replicated with different seeds, and in Figures \ref{fig:mse-vs-neval-mips} - \ref{fig:mse-vs-nac-mips} we use a total of $m=5000$ training data. 

\paragraph{Varying size of evaluation data $n$}
Figure \ref{fig:mse-vs-neval-mips} shows that MR outperforms the other baselines, in terms of MSE and squared bias, when the number of evaluation data $n\leq 1000$. Additionally, we observe that in this experiment, MR estimated using our original methods (`MR'), yields better results than the alternative method of estimating MR (`MR (alt)'). Moreover, while the variance of DM is lower than that of MR, the DM method has a high bias and consequently a high MSE. We note that while the difference between MSE and variance of MIPS and MR estimators decreases with increasing evaluation data size, MR still outperforms MIPS in terms of both MSE and variance.

\paragraph{Varying $\alpha^\ast$}
Figure \ref{fig:mse-vs-alphatar-mips} shows the results with increasing policy shift. It can be seen that overall MR methods achieve the smallest MSE with increasing policy shift. Moreover, the difference between MSE and variance of MR and IPW/DR methods increases with increasing policy shift, showing that MR performs especially better than these baselines when the difference between behaviour and target policies is large. Similarly, we observe in Figure \ref{fig:mse-vs-alphatar-mips} that as the shift between the behaviour and target policy increases with increasing $\alpha^\ast$, so does the difference between the MSE and variance of MR and the MIPS estimators. This shows that generally MR outperforms MIPS estimator in terms of variance and MSE, and that MR performs especially better than MIPS as the difference between behaviour and target policies increases.

\paragraph{Varying $d$ and $n_a$}
Figures \ref{fig:mse-vs-d-mips} and \ref{fig:mse-vs-nac-mips} show that MR outperforms the other baselines as the context dimensions and/or number of actions increase. In fact, these figures show that MR is significantly robust to increasing dimensions of action and context spaces, whereas baselines like IPW and DR perform poorly in large action spaces.

\paragraph{Varying $m$}
Figure \ref{fig:mse-vs-ntr-mips} shows the results with increasing number of training data $m$. We again observe that the MR methods `MR' and `MR (alt)' outperforms the other baselines in terms of the MSE and squared bias even when the number of training data is low. Moreover, the variance of both the MR estimators continues to improve with increasing number of training data.

In this experiment, we observe that overall `MR (alt)' performs worse than the original MR estimator (`MR' in the figures). However, as we observe in Appendix \ref{sec:app-additional-results}, this does not happen consistently across all experiments, which suggests that the comparative performance of the two MR methods depends on the data generating mechanism.

\begin{figure}[h!]
    \centering
	\begin{subfigure}{0.8\textwidth}
	    \centering
	    \includegraphics[width=1\textwidth]{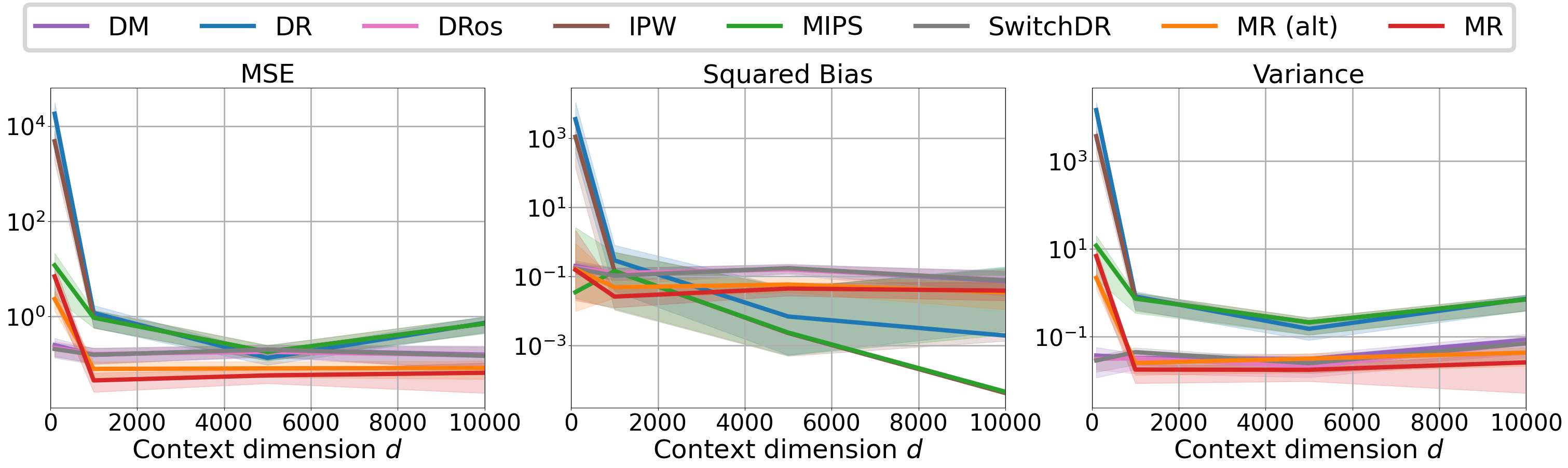}
	    \subcaption{$n_{a}=20$, $n = 200$, $\alpha^\ast = 0.8$}
	    \label{subfig:na-20-neval-200-alphatar-0-8-d-mips}
	\end{subfigure}\\
	\begin{subfigure}{0.8\textwidth} 
	    \centering
	    \includegraphics[width=1\textwidth]{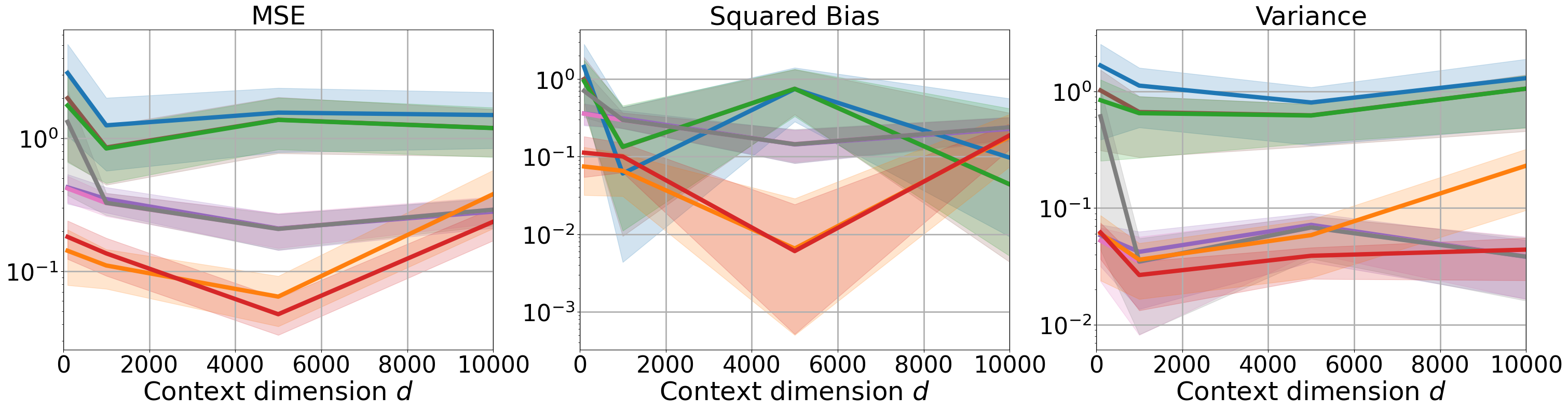}
	    \subcaption{$n_{a}=100$, $n = 200$, $\alpha^\ast = 0.8$}
	    \label{subfig:na-100-neval-200-alphatar-0-8-d-mips}
	\end{subfigure}\\
	\begin{subfigure}{0.8\textwidth} 
	    \centering
	    \includegraphics[width=1\textwidth]{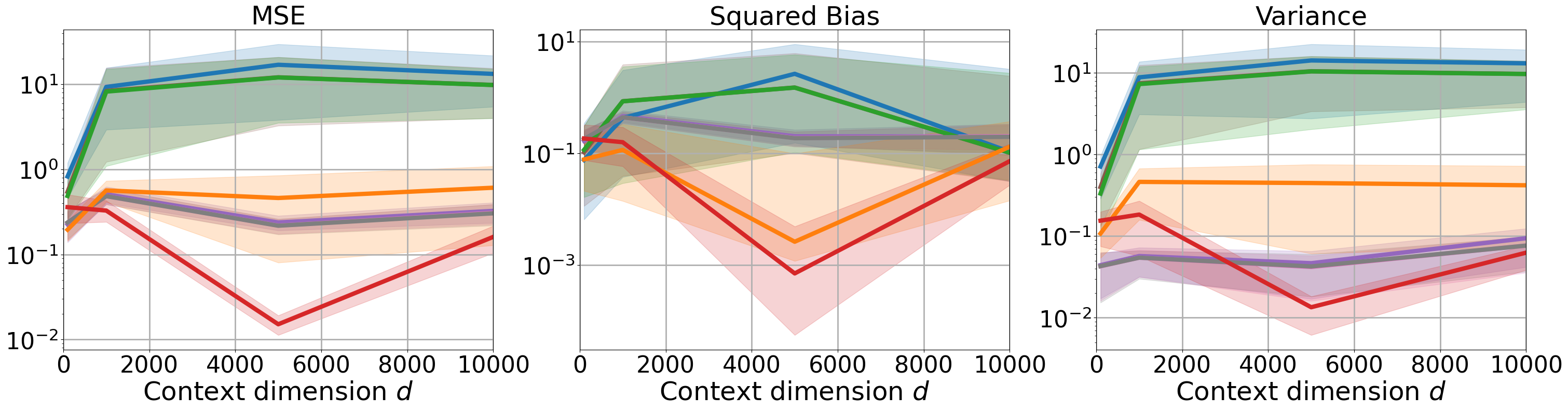}
	    \subcaption{$n_{a}=250$, $n = 200$, $\alpha^\ast = 0.8$}
	    \label{subfig:na-250-neval-200-alphatar-0-8-d-mips}
	\end{subfigure}
    \caption{MSE with varying context dimensions $d$ for different choices of parameters.}
    \label{fig:mse-vs-d-mips}
\end{figure}

\begin{figure}[h!]
    \centering
	\begin{subfigure}{0.8\textwidth}
	    \centering
	    \includegraphics[width=1\textwidth]{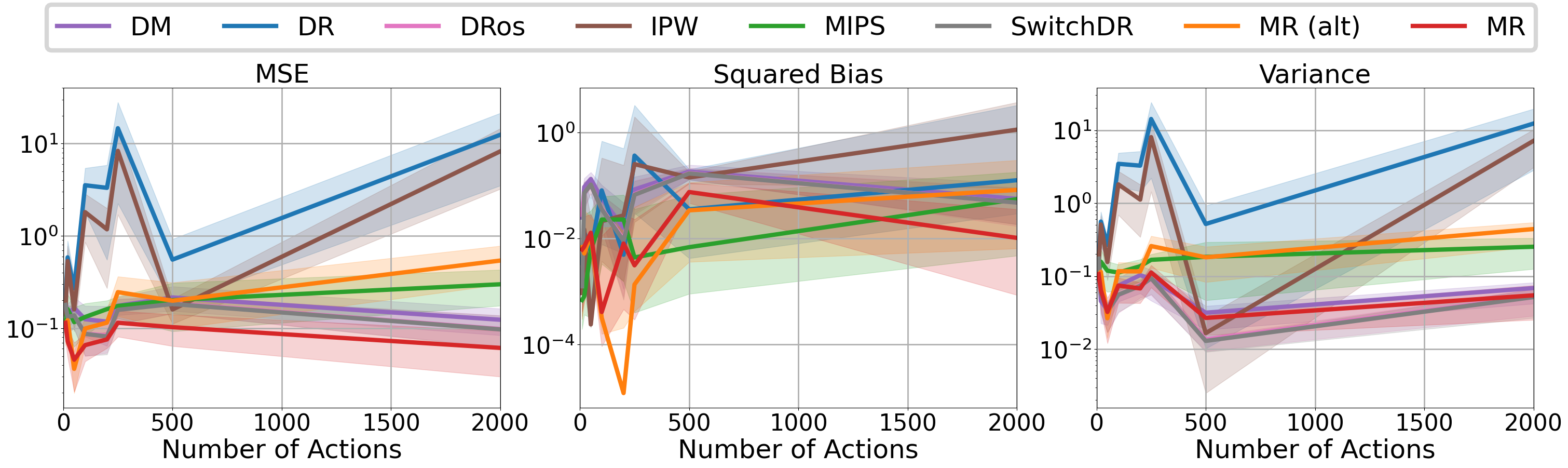}
	    \subcaption{$d=1000$, $n = 100$, $\alpha^\ast = 0.4$}
	    \label{subfig:d-1000-neval-100-alphatar-0-4-nac-mips}
	\end{subfigure}\\
	\begin{subfigure}{0.8\textwidth} 
	    \centering
	    \includegraphics[width=1\textwidth]{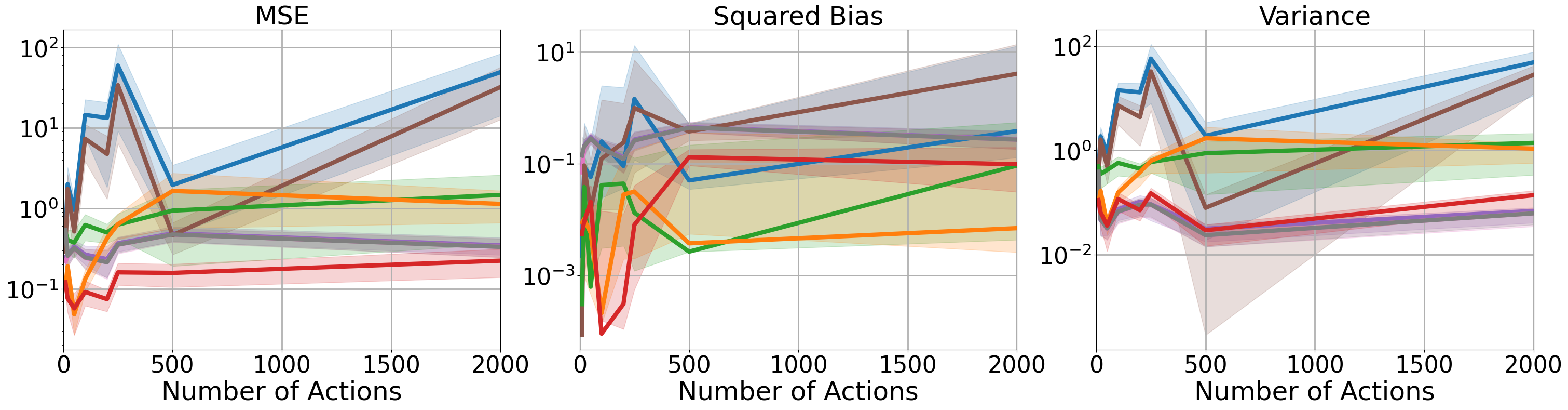}
	    \subcaption{$d=1000$, $n = 100$, $\alpha^\ast = 0.8$}
	    \label{subfig:d-1000-neval-100-alphatar-0-8-nac-mips}
	\end{subfigure}\\
	\begin{subfigure}{0.8\textwidth} 
	    \centering
	    \includegraphics[width=1\textwidth]{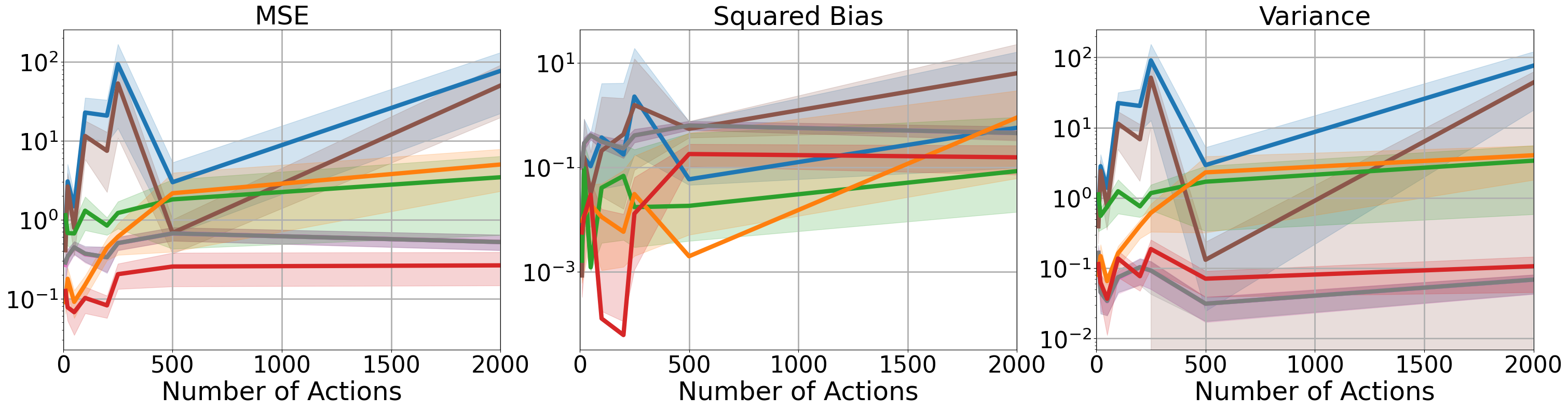}
	    \subcaption{$d=1000$, $n = 100$, $\alpha^\ast = 1.0$}
	    \label{subfig:d-1000-neval-100-alphatar-1-0-nac-mips}
	\end{subfigure}
    \caption{MSE with varying number of actions $n_a$ for different choices of parameters.}
    \label{fig:mse-vs-nac-mips}
\end{figure}

\begin{sidewaystable}[ht]
    \centering
        \caption{Mean-squared error results with 2 standard errors for synthetic data setup considered in Section \ref{sec:exp-synth} with $d=5000$, $n_a = 50$, $\alpha^\ast = 0.8$. We use a fixed budget of datapoints (denoted by $N$) for each baseline and in the case of MR we use $m=2000$ of the available datapoints to estimate $\hat{w}(y)$ and the rest of data to evaluate the MR estimator (i.e. $n = N-2000$ for MR). In contrast, for IPW and MIPS since the importance ratios are already known, we use all of the $N$ datapoints for evaluation of the off-policy value (i.e. $n=N$ for IPW and MIPS).}
    \label{tab:known_ratios}
    \begin{tiny}
    \begin{tabular}{l|llllll}
\toprule
& $N$ & 2800 & 3200 & 6400 & 10000 & 12000 \\
\midrule
\multirow{2}{*}{
\begin{tiny}
\textbf{GT weights $\rho(a, x)$ and estimated reward model $\hat{\mu}(a, x)$}
\end{tiny}}
& DM & 0.137$\pm$0.028 & 0.099$\pm$0.012 & 0.103$\pm$0.012 & 0.093$\pm$0.010 & 0.089$\pm$0.010 \\
\multirow{2}{*}{
\begin{tiny}
($m=2000$ used for training $\hat{\mu}(a, x)$ and $n=N-2000$ 
\end{tiny}}
& DR & 0.227$\pm$0.065 & 0.068$\pm$0.035 & 0.068$\pm$0.022 & \textbf{0.024$\pm$0.011} & 0.045$\pm$0.015 \\
\multirow{2}{*}{
\begin{tiny}
used for evaluation)
\end{tiny}} & DRos & 0.128$\pm$0.027 & 0.072$\pm$0.011 & 0.049$\pm$0.014 & 0.063$\pm$0.014 & 0.051$\pm$0.016 \\
& SwitchDR & 0.128$\pm$0.027 & 0.059$\pm$0.014 & 0.052$\pm$0.013 & 0.061$\pm$0.015 & 0.056$\pm$0.016 \\
\hline
\\
\multirow{2}{*}{
\begin{tiny}
\textbf{GT weights} (all of $N$ datapoints are used for evaluation)
\end{tiny}}
& IPW & 0.237$\pm$0.062 & 0.066$\pm$0.036 & 0.067$\pm$0.021 & 0.025$\pm$0.011 & \textbf{0.044$\pm$0.014} \\
& MIPS & 0.236$\pm$0.062 & 0.065$\pm$0.035 & 0.067$\pm$0.021 & 0.025$\pm$0.011 & \textbf{0.044$\pm$0.014} \\
\hline
\multirow{3}{*}{
\begin{tiny}
\textbf{Estimated weights $\hat{w}(y)$}
\end{tiny}}
\multirow{3}{*}{
\begin{tiny}
($m=2000$ used for training
\end{tiny}} \\
\multirow{3}{*}{
\begin{tiny}
and $n=N-2000$ used for evaluation)
\end{tiny}} 
& MR (Ours) & \textbf{0.045$\pm$0.015} & \textbf{0.042$\pm$0.014} & \textbf{0.048$\pm$0.020} & 0.049$\pm$0.020 & 0.047$\pm$0.016
\\
\\
\bottomrule
\end{tabular}
    \end{tiny}
\end{sidewaystable}
\subsubsection{Known policy ratios $\rho(a, x)$}
Our previous setting of unknown importance policy ratios $\rho(a, x)$ captures a wide variety of real-world applications, ranging from health care to autonomous driving. In addition, to demonstrate the utility of MR in settings with known $\rho(a, x), p(e\mid a, x)$ and unknown $w(y)$ (for our proposed method, MR), we have conducted additional experiments. Here, we use a fixed budget of datapoints (denoted by $N$) for each baseline and for MR we allocate $m=2000$ of the available datapoints to estimate $\hat{w}(y)$ and use the remaining for evaluating the MR estimator (i.e., $n=N-2000$ for MR). In contrast, for IPW and MIPS (since the importance ratios are already known), we use all of the $N$ datapoints to evaluate the off-policy value (i.e. $n=N$ for IPW and MIPS).

The results included in Table \ref{tab:known_ratios} show that MR achieves the smallest MSE among the baselines for $N\leq 6400$. However, we observe that the MSE of IPW, DR and MIPS (with true importance weights) falls below that of MR (with estimated weights $\hat{w}$) when the data size $N$ is large enough (i.e., $N\geq 10,000$). This is to be expected since IPW, DR and MIPS are unbiased (i.e., use ground truth importance ratios $\rho(a, x)$) whereas MR uses estimated weights $\hat{w}(y)$ (and hence may be biased). MR still performs the best when $N\leq 6400$.

\subsection{Experiments on classification datasets}\label{subsec:additional-experiments-classification}
Here, we conduct experiments on four classification datasets, OptDigits, PenDigits, SatImage and Letter datasets from the UCI repository \citep{dua2019uci}, the Digits dataset from scikit-learn library, as well as the Mnist \citep{deng2012mnist} and CIFAR-100 datasets \citep{krizhevsky2009learning}.

\paragraph{Setup}
Following previous works \citep{dudik2014doubly, kallus2021optimal, mehrdad2018more,wang2017optimal}, the classification datasets are transformed to contextual bandit feedback data. The classification dataset comprises $\{x_i, a^\gt_i\}_{i=1}^{n_0}$, where $x_i\in \Xspace$ are feature vectors and $a^\gt_i\in \Aspace$ are the ground-truth labels. In the contextual bandits setup, the feature vectors $x_i$ are considered to be the contexts, whereas the actions correspond to the possible class of labels. We split the dataset into training and testing datasets of sizes $m$ and $n$ respectively. We present the results for a range of different values of $m$ and $n$.

\paragraph{Reward function}
Let $X$ be a context with ground truth label $A^\gt$, we define the reward for action $A$ as:
\[
Y \coloneqq \ind(A = A^\gt).
\]

\paragraph{Behaviour and target policies}
Using the $m$ training datapoints, we first train a classifier $f: \Xspace\rightarrow \mathbb{R}^{|\Aspace|}$ which takes as input the feature vectors $x_i$ and outputs a vector of softmax probabilities over labels, i.e. the $a$-th component of the vector $f(x)$, denoted as $(f(x))_{a}$ corresponds to the estimated probability $\p(A^{\gt} = a \mid X=x)$.

Next, we use $f$ to define the ground truth behaviour policy, 
\[
\beh(a\mid x) = (f(x))_{a}.
\]
For the target policies, we use $f$ to define a parametric class of target policies using a trained classifier $f: \Xspace\rightarrow \mathbb{R}^{|\Aspace|}$. 
\[
\pi^{\alpha^\ast}(a\mid x) =\alpha^\ast\cdot \ind(a= \arg\max_{a' \in \Aspace} (f(x))_{a'}) +  \frac{1-\alpha^\ast}{|\Aspace|},
\]
where $\alpha^\ast \in [0, 1]$. A value of $\alpha^\ast$ close to 1 leads to a near-deterministic and well-performing policy. As $\alpha^\ast$ decreases, the policy gets increasingly worse and `noisy'. In this experiment, we consider target policies $\tar = \pi^{\alpha^\ast}$ for $\alpha^\ast \in \{0.0, 0.2, 0.4, \dots, 1.0\}$. 

Using the behaviour policy defined above, we generate the contextual bandits data described with training and evaluation datasets of sizes $m$ and $n$ respectively.

\paragraph{Estimation of behaviour policy $\hatbeh$ and marginal ratio $\hat{w}(y)$}
We do not assume that the behaviour policy $\beh$ is known, and therefore estimate it using training data. To estimate the behaviour policy $\hatbeh$, we train a random forest classifier using the training data. This estimate of behaviour policy is used for all the baselines in our experiment. 
Since the reward is binary, we can estimate the marginal ratios $\hat{w}(y) = \Ebeh[\hat{\rho}(A, X)\mid Y=y]$ by directly estimating the sample mean of $\hat{\rho}(A, X)$ for datapoints with $Y=y$. We re-use the $m$ training datapoints to estimate this sample mean. 

\paragraph{Baselines}
We compare our estimator with Direct Method (DM), IPW and DR estimators. 
In addition, we also consider Switch-DR \citep{wang2017optimal} and DR with Optimistic Shrinkage (DRos) \citep{su2020doubly}.
To estimate $\hat{q}(x, a)$ for DM and DR estimators, we use random forest classifiers (since reward $Y$ is binary). Moreover, because of the binary nature of $Y$, the alternative method of estimating MR yields the same estimator as the original method, therefore we do not consider the two separately here. Additionally, in this experiment, we do not include MIPS (or G-MIPS) baseline, as there is no natural informative embedding $E$ of the action $A$. 

\subsubsection{Results}
For this experiment, we compute the results over 10 different sets of logged data replicated with different seeds.
Figures \ref{fig:optdigits} - \ref{fig:cifar100} show the results corresponding to each baseline for the different datasets. It can be seen that across all datasets, the MR achieves the smallest MSE with increasing evaluation data size $n$. Moreover, across all datasets, MR attains the minimum MSE with relatively small number of evaluation data ($n\leq 100$).

Unlike the experiments in Section \ref{sec:exp-synth}, we observe that the KL-divergence between target and behaviour policy decreases as $\alpha^\ast$ increases (see Figure \ref{fig:kl_div_multiclass}). 
Therefore, as $\alpha^\ast$ increases the shift between target and behaviour policies decreases.
Figures \ref{fig:optdigits} - \ref{fig:digits} show that as $\alpha^\ast$ increases, 
the difference between the MSE, squared bias and variance of MR and the other baselines decreases. This confirms our findings from earlier experiments that MR performs especially better than the other baselines when the difference between behaviour and target policies is large.

Moreover, the figures also include results with increasing number of training data $m$. It can be seen that MR out-performs the baselines even when the number of training data $m$ is small ($m = 100$). Moreover, the relative advantage of MR improves with increasing $m$.

\begin{figure}[t]
    \centering
    \includegraphics[width=0.3\textwidth]{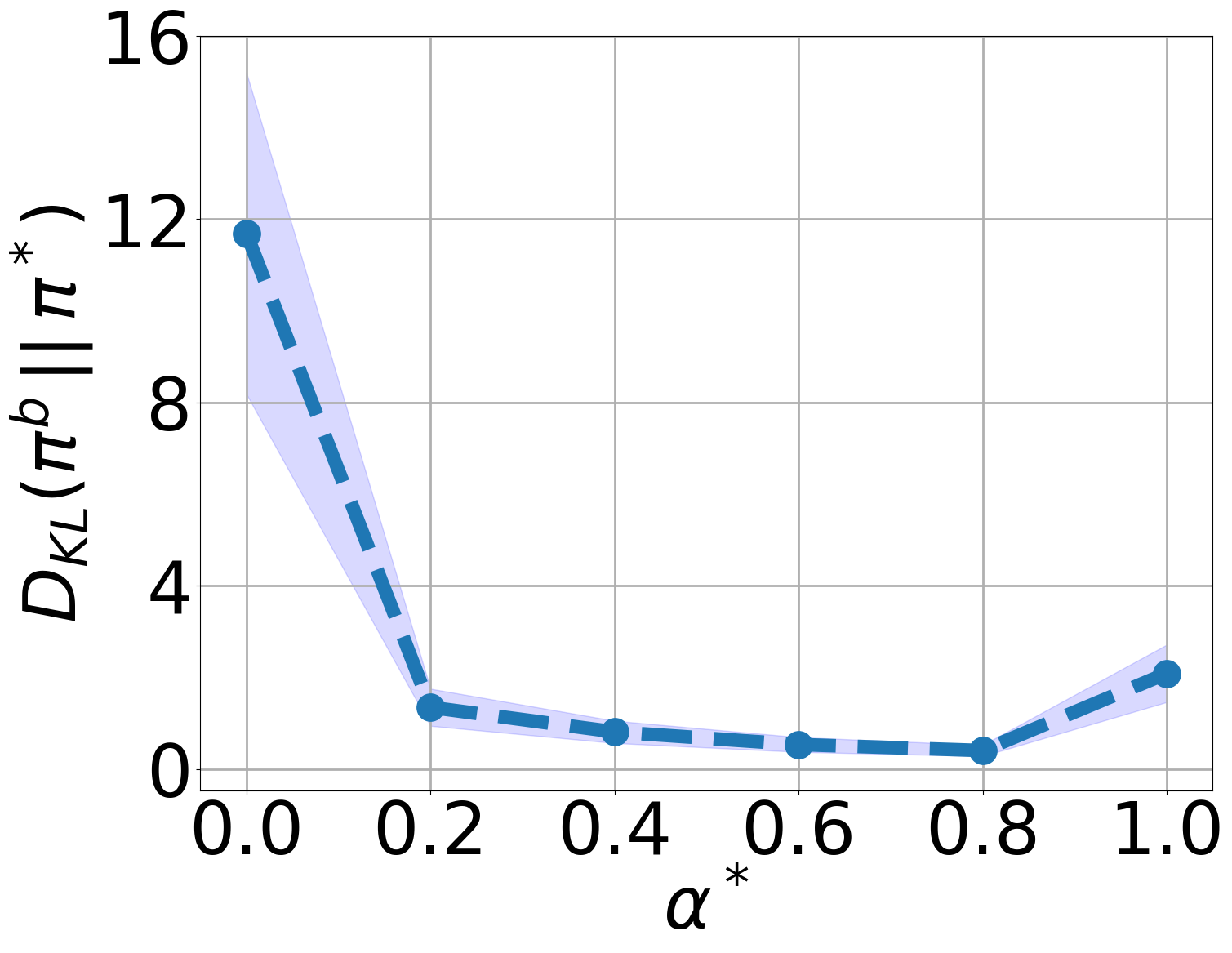}
    \caption{KL divergence $D_{\textup{KL}}(\beh \, || \, \tar)$ with increasing $\alpha^\ast$ for the classification data experiments. Here, we only include the results for a specific choice of parameters for the Letter dataset. We observe similar results for other datasets and parameter choices.}
    \label{fig:kl_div_multiclass}
\end{figure}

\begin{figure}[ht]
    \centering
	\begin{subfigure}{0.8\textwidth}
	    \centering
	    \includegraphics[width=1\textwidth]{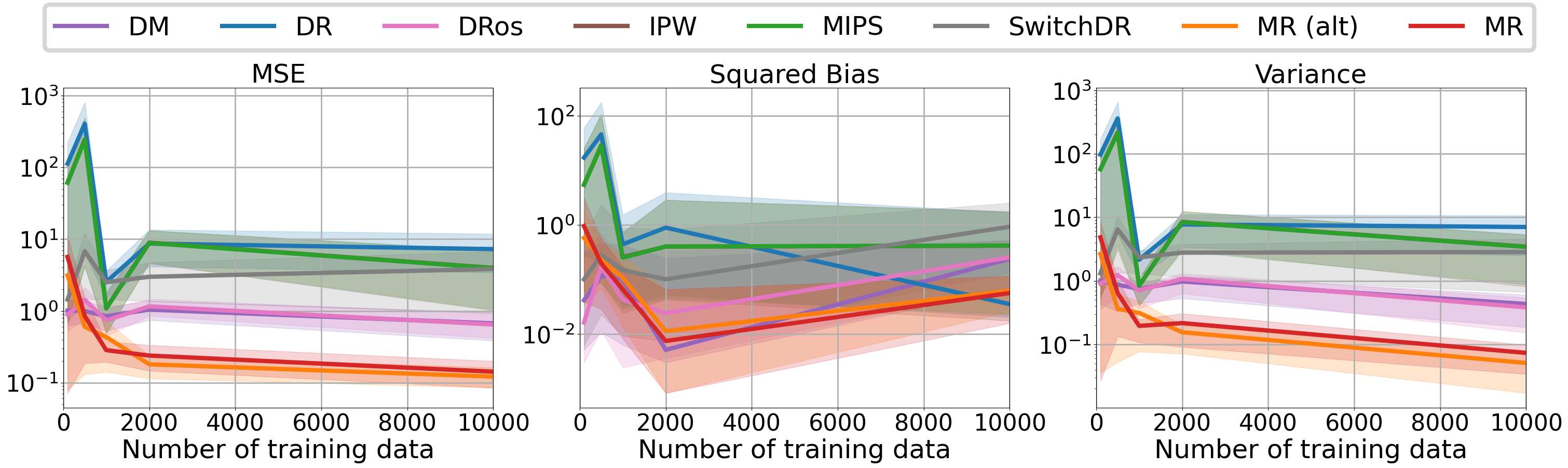}
	    \subcaption{$d=1000$, $n = 10$, $\alpha^\ast = 0.8$}
	    \label{subfig:d-1000-neval-10-alphatar-0-8-ntr-mips}
	\end{subfigure}\\
	\begin{subfigure}{0.8\textwidth} 
	    \centering
	    \includegraphics[width=1\textwidth]{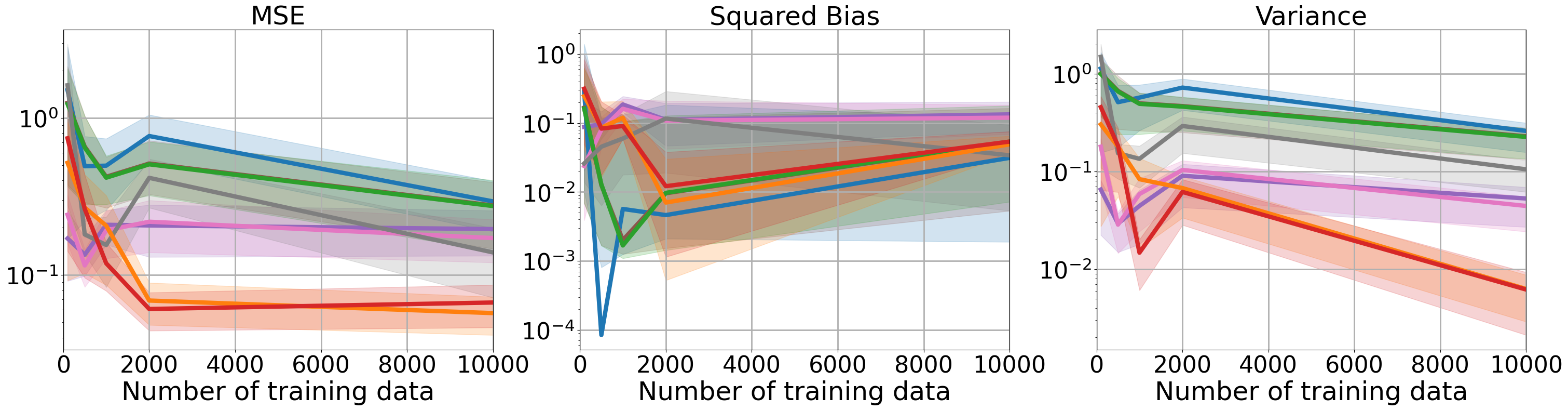}
	    \subcaption{$d=1000$, $n = 200$, $\alpha^\ast = 0.8$}
	    \label{subfig:d-1000-neval-200-alphatar-0-8-ntr-mips}
	\end{subfigure}\\
	\begin{subfigure}{0.8\textwidth} 
	    \centering
	    \includegraphics[width=1\textwidth]{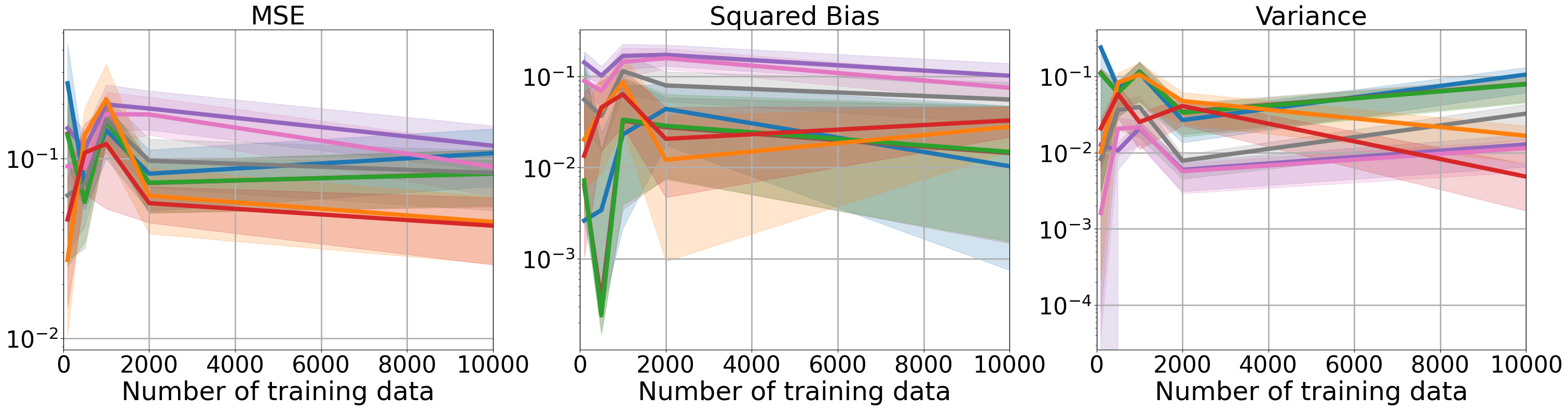}
	    \subcaption{$d=1000$, $n = 800$, $\alpha^\ast = 0.8$}
	    \label{subfig:d-1000-neval-800-alphatar-0-8-ntr-mips}
	\end{subfigure}
    \caption{MSE with varying number of training data $m$ for different choices of parameters.}
    \label{fig:mse-vs-ntr-mips}
\end{figure}

\begin{figure}[h!]
    \centering
	\begin{subfigure}{0.8\textwidth}
	    \centering
	    \includegraphics[width=1\textwidth]{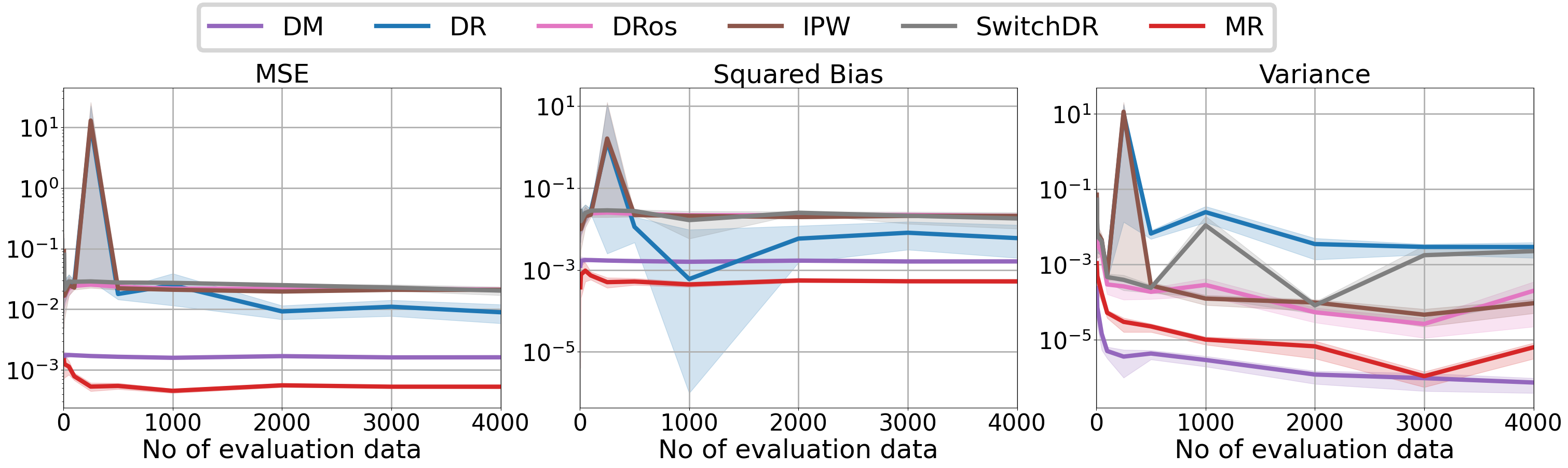}
	    \subcaption{Results with varying $n$ for $\alpha^\ast = 0.2$ and $m=1000$}
	    \label{subfig:opt-neval}
	\end{subfigure}\\
	\begin{subfigure}{0.8\textwidth} 
	    \centering
	    \includegraphics[width=1\textwidth]{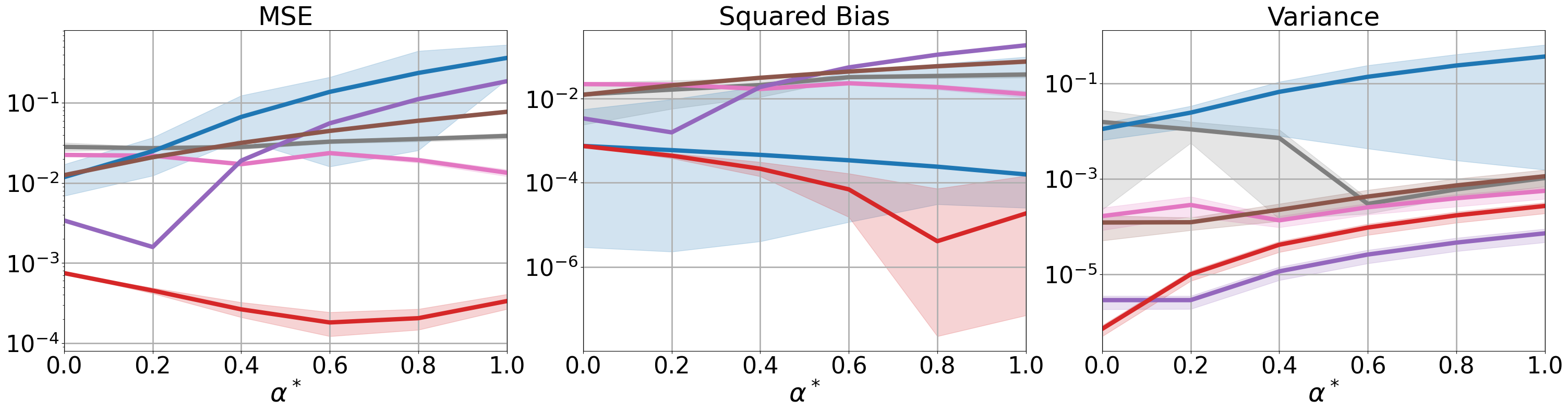}
	    \subcaption{Results with varying $\alpha^\ast$ for $m = n = 1000$}
	    \label{subfig:opt-ae}
	\end{subfigure}\\
    \begin{subfigure}{0.8\textwidth} 
	    \centering
	    \includegraphics[width=1\textwidth]{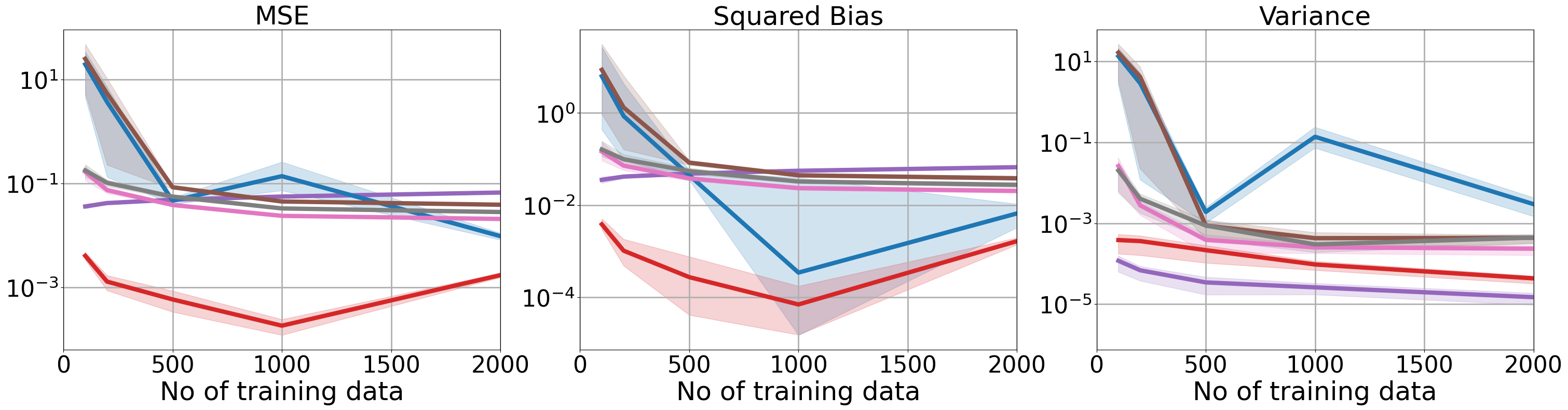}
	    \subcaption{Results with varying $m$ for $n = 1000$ and $\alpha^\ast = 0.6$}
	    \label{subfig:opt-ntr}
	\end{subfigure}\\
    \caption{Results for OptDigits dataset}
    \label{fig:optdigits}
\end{figure}

\begin{figure}[h!]
    \centering
	\begin{subfigure}{0.8\textwidth}
	    \centering
	    \includegraphics[width=1\textwidth]{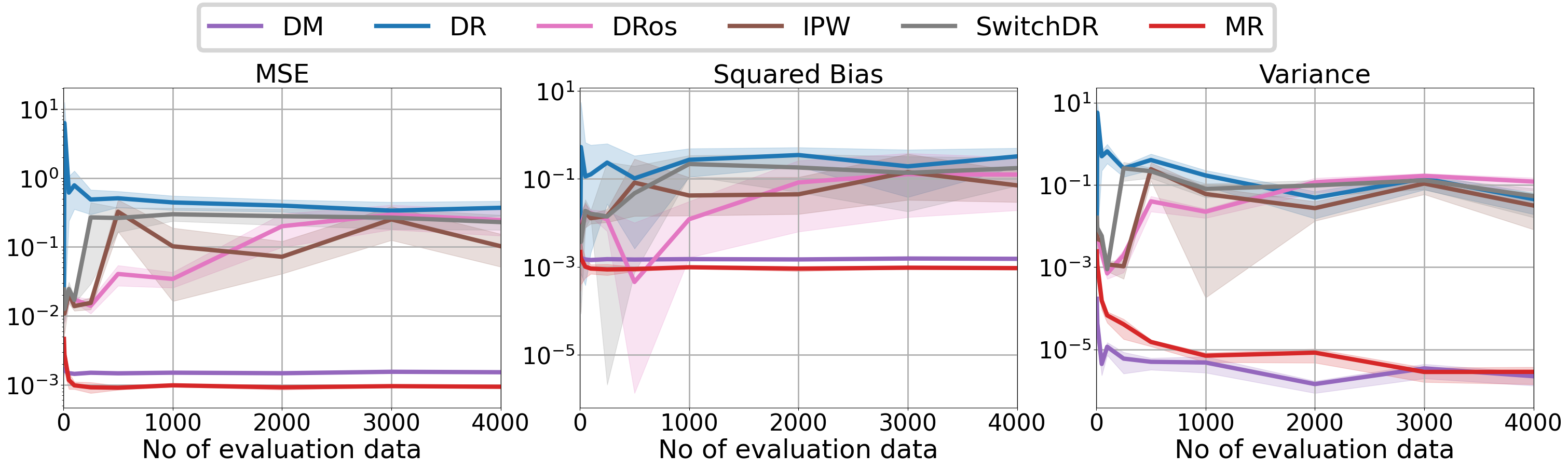}
	    \subcaption{Results with varying $n$ for $\alpha^\ast = 0.2$ and $m=1000$}
	    \label{subfig:pen-neval}
	\end{subfigure}\\
	\begin{subfigure}{0.8\textwidth} 
	    \centering
	    \includegraphics[width=1\textwidth]{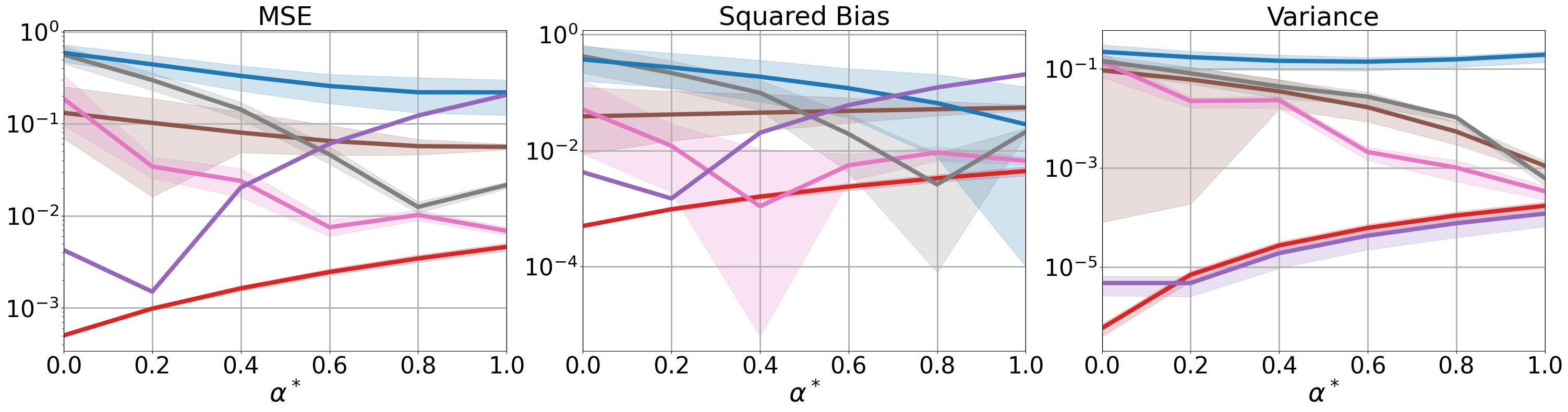}
	    \subcaption{Results with varying $\alpha^\ast$ for $m = n = 1000$}
	    \label{subfig:pen-ae}
	\end{subfigure}\\
        \begin{subfigure}{0.8\textwidth} 
	    \centering
	    \includegraphics[width=1\textwidth]{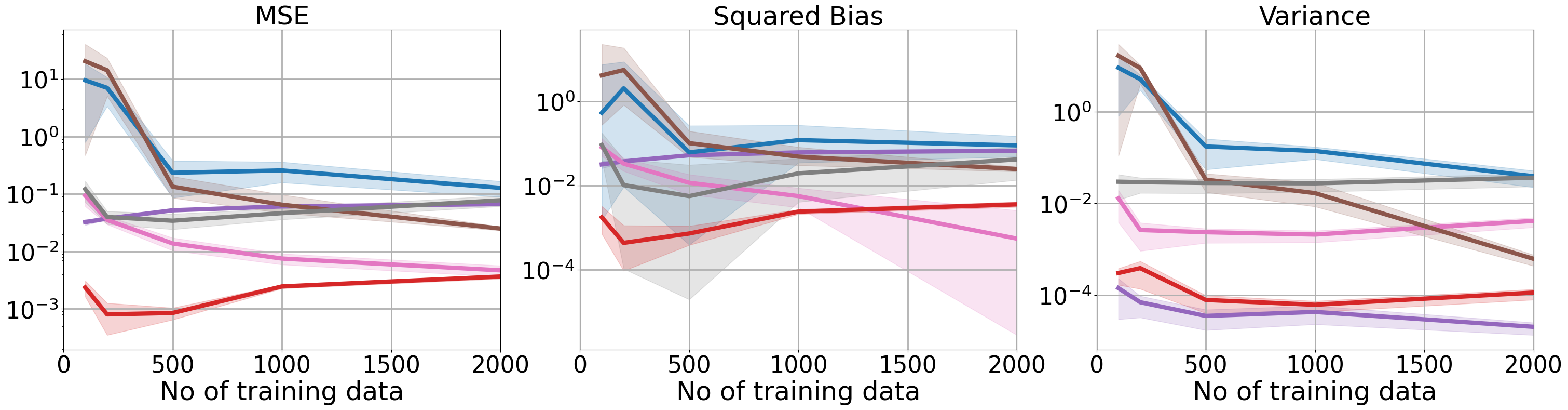}
	    \subcaption{Results with varying $m$ for $\alpha^\ast=0.6$ and $n = 1000$}
	    \label{subfig:pen-tr}
	\end{subfigure}
    \caption{Results for PenDigits dataset}
    \label{fig:pendigits}
\end{figure}

\begin{figure}[h!]
    \centering
	\begin{subfigure}{0.8\textwidth}
	    \centering
	    \includegraphics[width=1\textwidth]{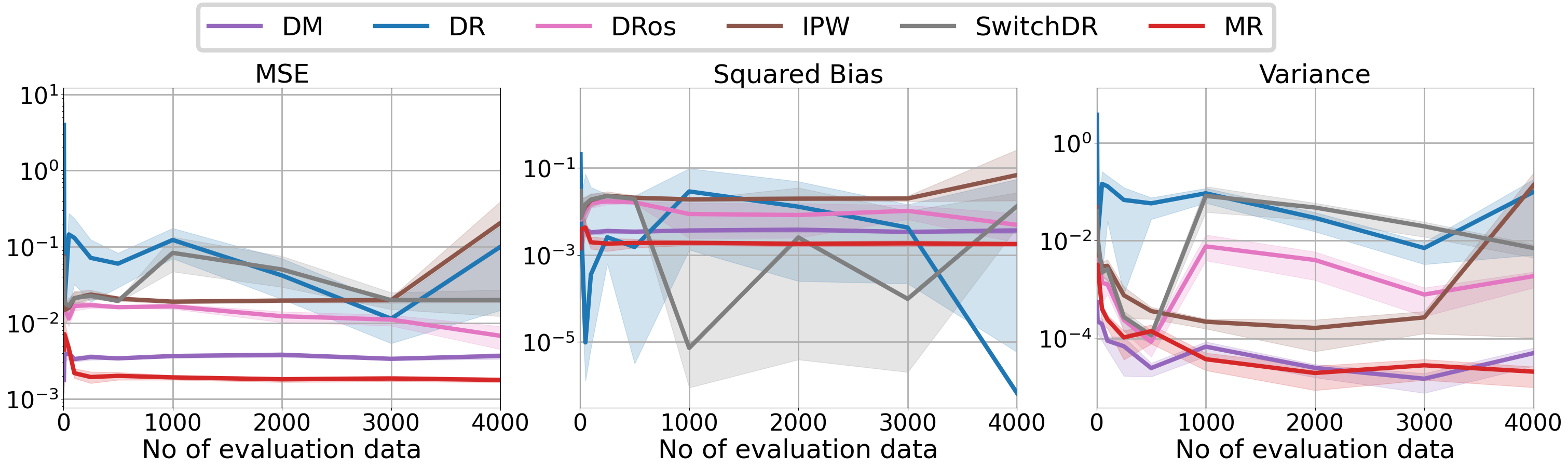}
	    \subcaption{Results with varying $n$ for $\alpha^\ast = 0.2$ and $m=1000$}
	    \label{subfig:sat-neval}
	\end{subfigure}\\
	\begin{subfigure}{0.8\textwidth} 
	    \centering
	    \includegraphics[width=1\textwidth]{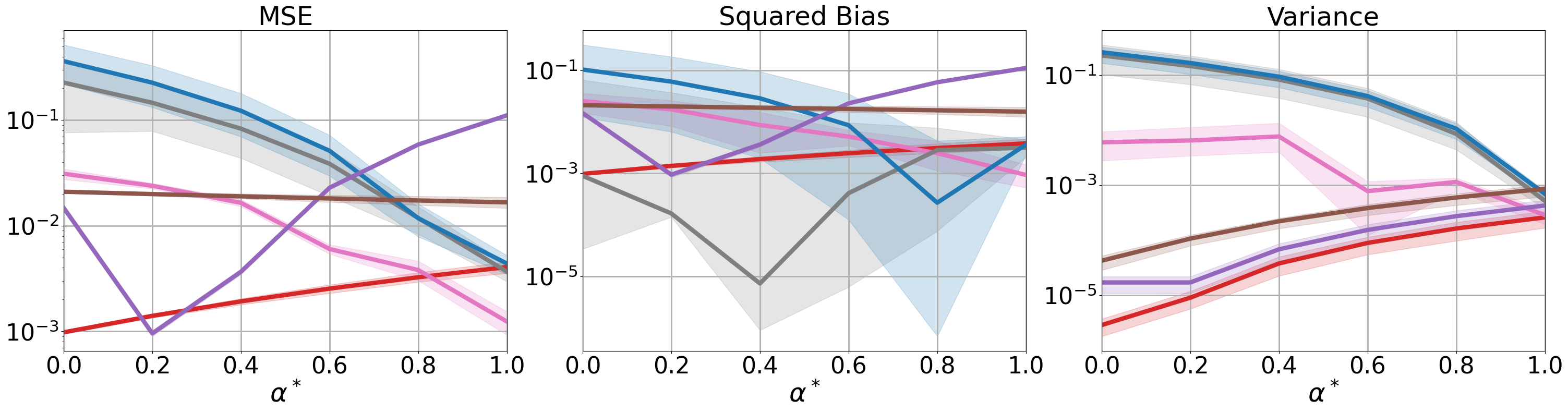}
	    \subcaption{Results with varying $\alpha^\ast$ for $n = 1000$}
	    \label{subfig:sat-ae}
	\end{subfigure}\\
 	\begin{subfigure}{0.8\textwidth} 
	    \centering
	    \includegraphics[width=1\textwidth]{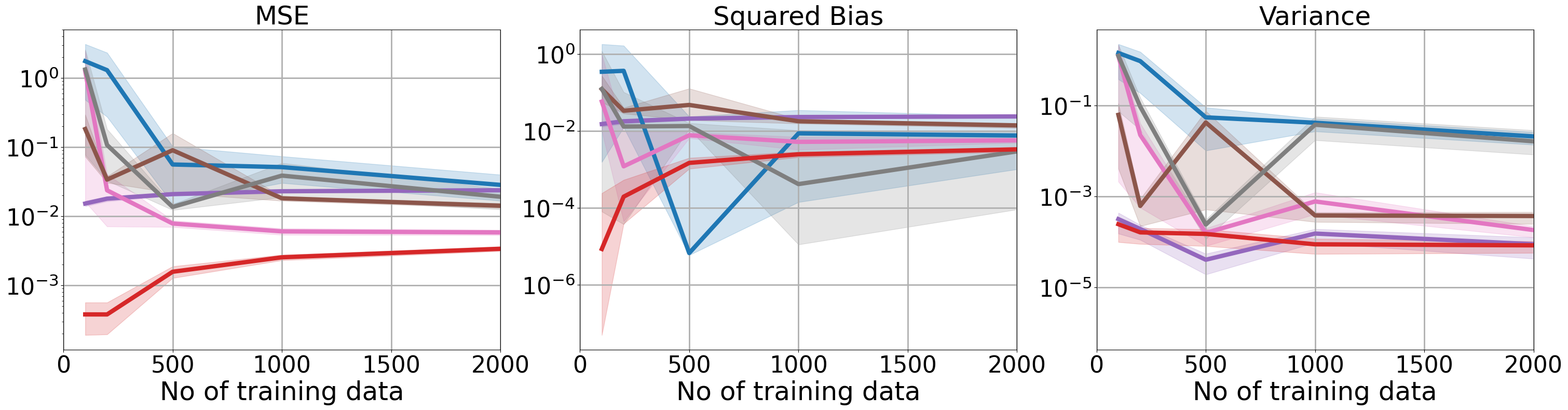}
	    \subcaption{Results with varying $m$ for $\alpha^\ast=0.6$ and $n = 1000$}
	    \label{subfig:sat-tr}
	\end{subfigure}
    \caption{Results for SatImage dataset}
    \label{fig:satimage}
\end{figure}

\begin{figure}[h!]
    \centering
	\begin{subfigure}{0.8\textwidth}
	    \centering
	    \includegraphics[width=1\textwidth]{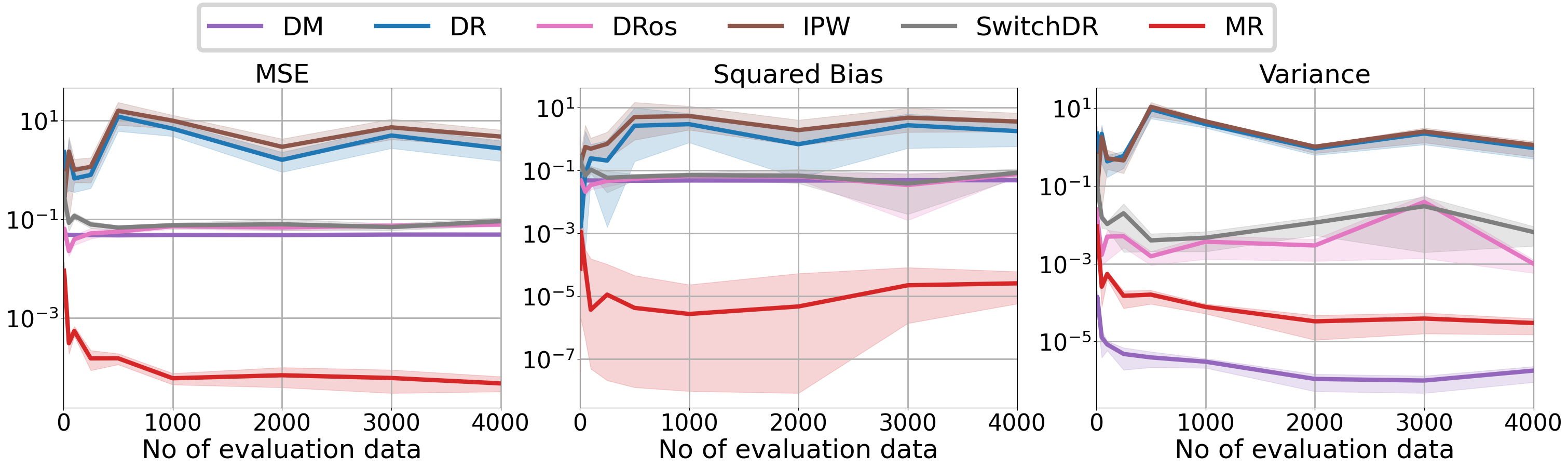}
	    \subcaption{Results with varying $n$ for $\alpha^\ast = 0.2$ and $m=1000$}
	    \label{subfig:letter-neval}
	\end{subfigure}\\
	\begin{subfigure}{0.8\textwidth} 
	    \centering
	    \includegraphics[width=1\textwidth]{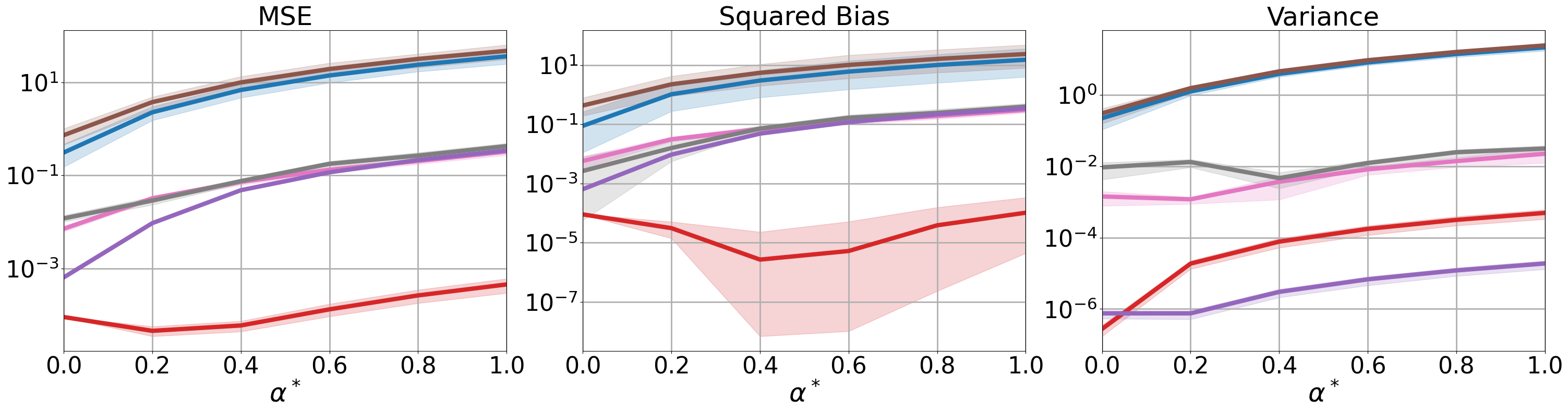}
	    \subcaption{Results with varying $\alpha^\ast$ for $m = n = 1000$}
	    \label{subfig:letter-ae}
	\end{subfigure}\\
        \begin{subfigure}{0.8\textwidth} 
	    \centering
	    \includegraphics[width=1\textwidth]{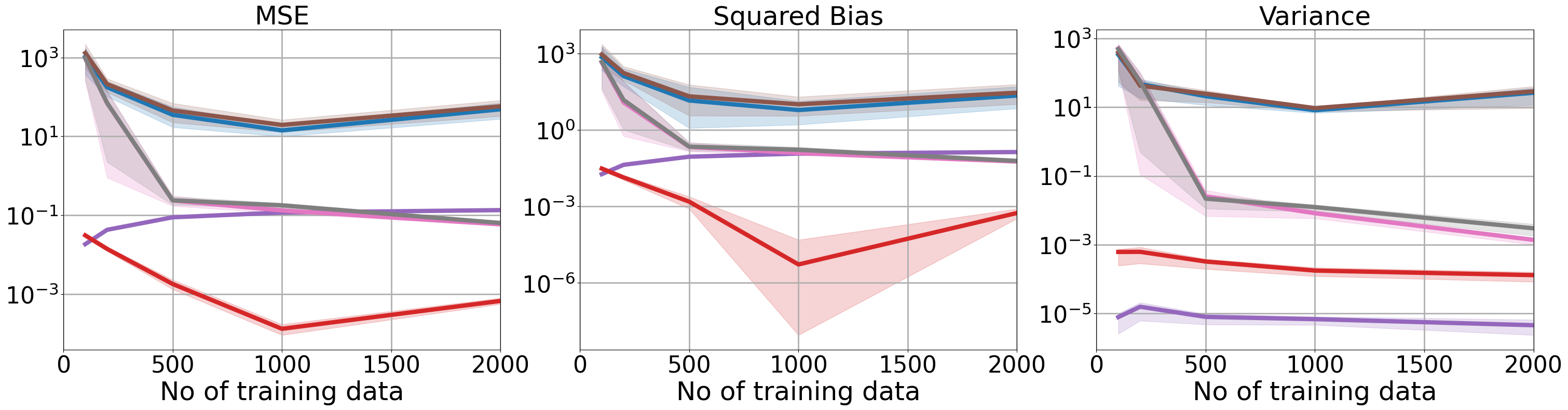}
	    \subcaption{Results with varying $m$ for $\alpha^\ast=0.6$ and $n = 1000$}
	    \label{subfig:letter-tr}
	\end{subfigure}
    \caption{Results for Letter dataset}
    \label{fig:letter}
\end{figure}

\begin{figure}[h!]
    \centering
	\begin{subfigure}{0.8\textwidth}
	    \centering
	    \includegraphics[width=1\textwidth]{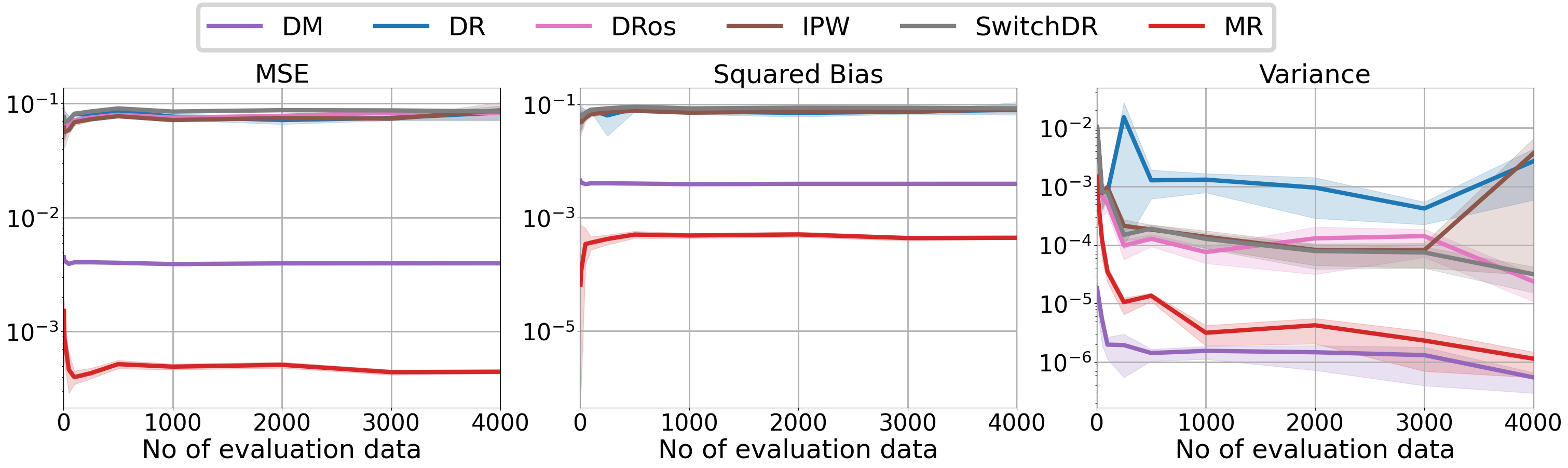}
	    \subcaption{Results with varying $n$ for $\alpha^\ast = 0.2$ and $m=1000$}
	    \label{subfig:mnist-neval}
	\end{subfigure}\\
	\begin{subfigure}{0.8\textwidth} 
	    \centering
	    \includegraphics[width=1\textwidth]{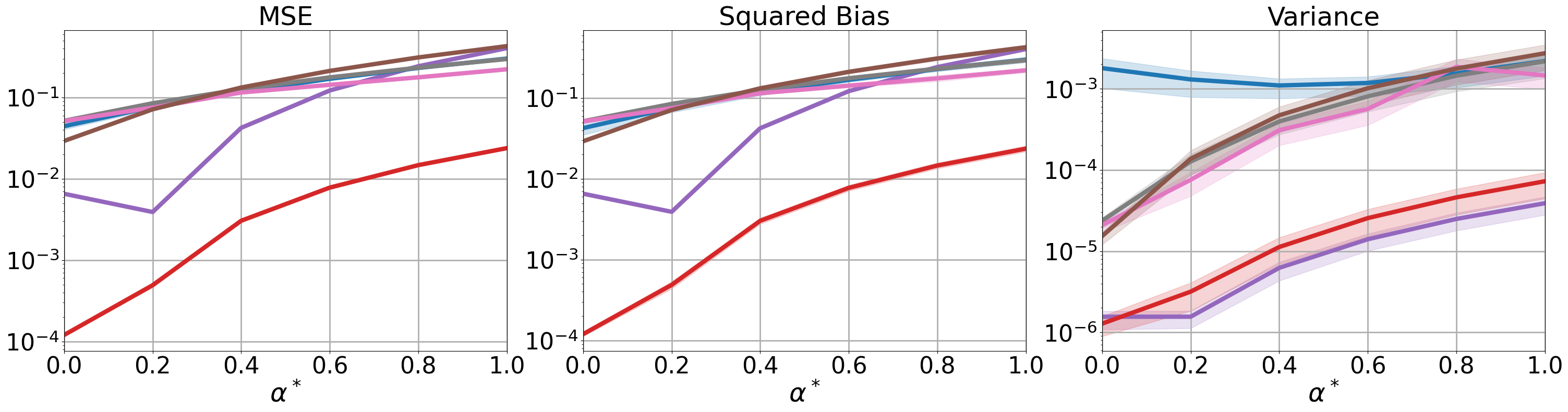}
	    \subcaption{Results with varying $\alpha^\ast$ for $m= n = 1000$}
	    \label{subfig:mnist-ae}
	\end{subfigure}\\
 	\begin{subfigure}{0.8\textwidth} 
	    \centering
	    \includegraphics[width=1\textwidth]{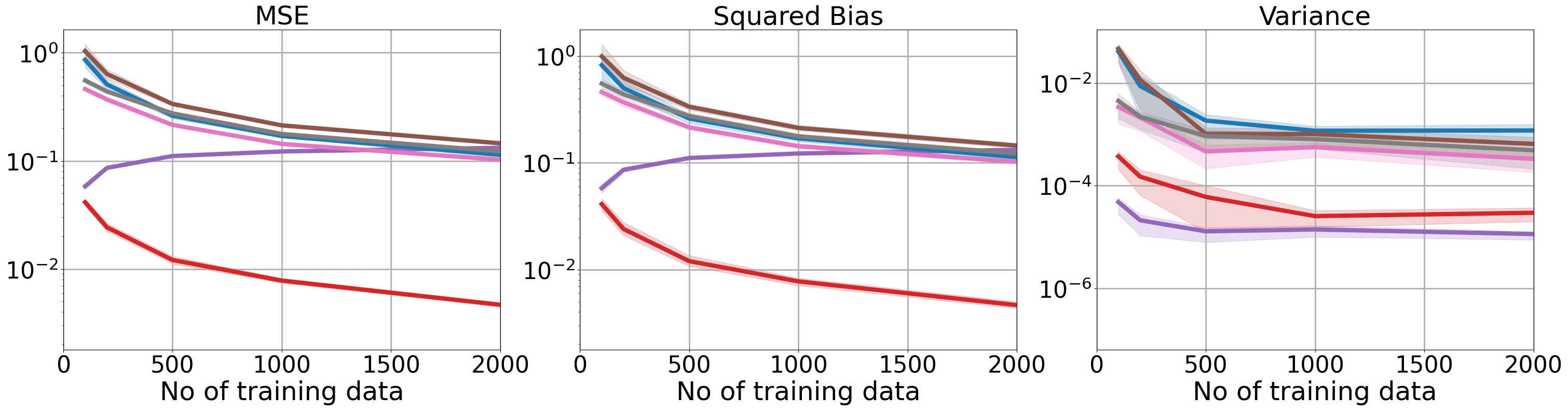}
	    \subcaption{Results with varying $m$ for $\alpha^\ast=0.6$ and $n = 1000$}
	    \label{subfig:mnist-tr}
	\end{subfigure}
    \caption{Results for Mnist dataset}
    \label{fig:mnist}
\end{figure}

\begin{figure}[h!]
    \centering
	\begin{subfigure}{0.8\textwidth}
	    \centering
	    \includegraphics[width=1\textwidth]{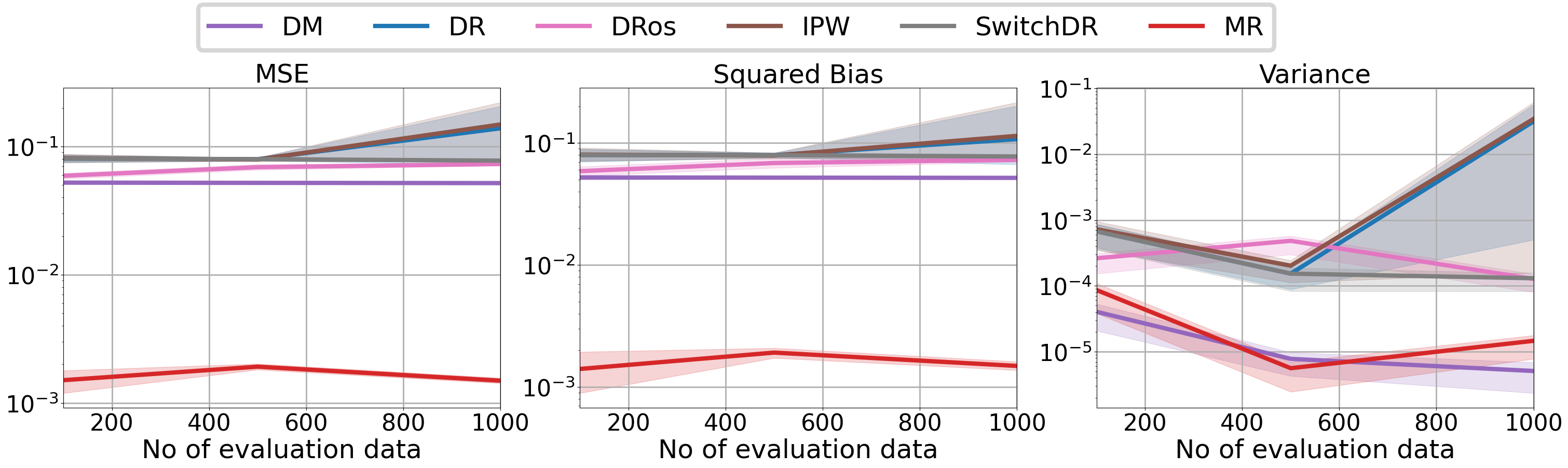}
	    \subcaption{Results with varying $n$ for $\alpha^\ast = 0.2$ and $m=500$}
	    \label{subfig:digits-neval}
	\end{subfigure}\\
	\begin{subfigure}{0.8\textwidth} 
	    \centering
	    \includegraphics[width=1\textwidth]{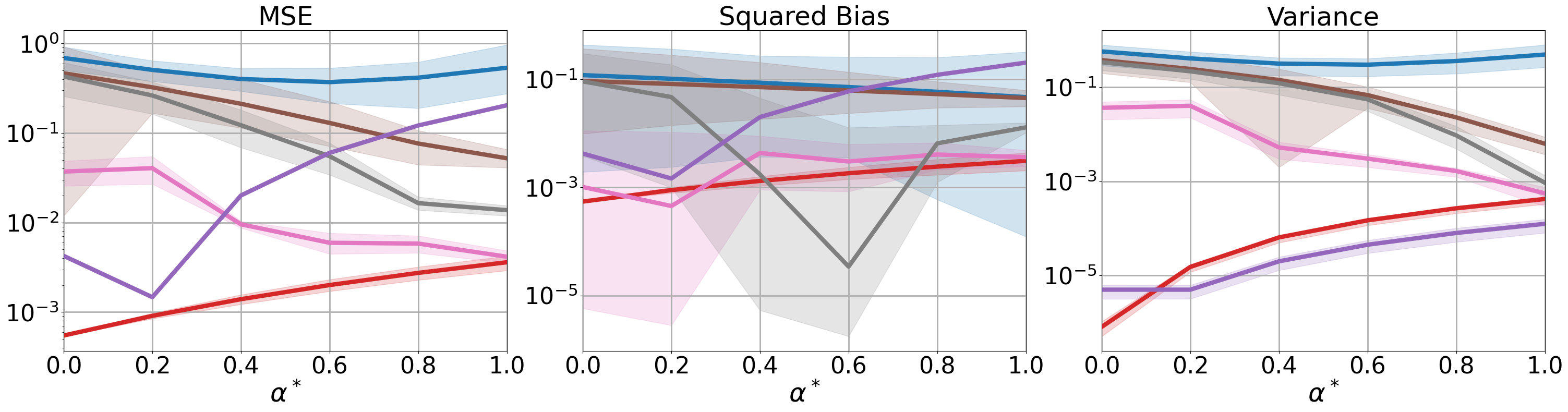}
	    \subcaption{Results with varying $\alpha^\ast$ for $n = 500$ and $m=1000$}
	    \label{subfig:digits-ae}
	\end{subfigure}\\
 	\begin{subfigure}{0.8\textwidth} 
	    \centering
	    \includegraphics[width=1\textwidth]{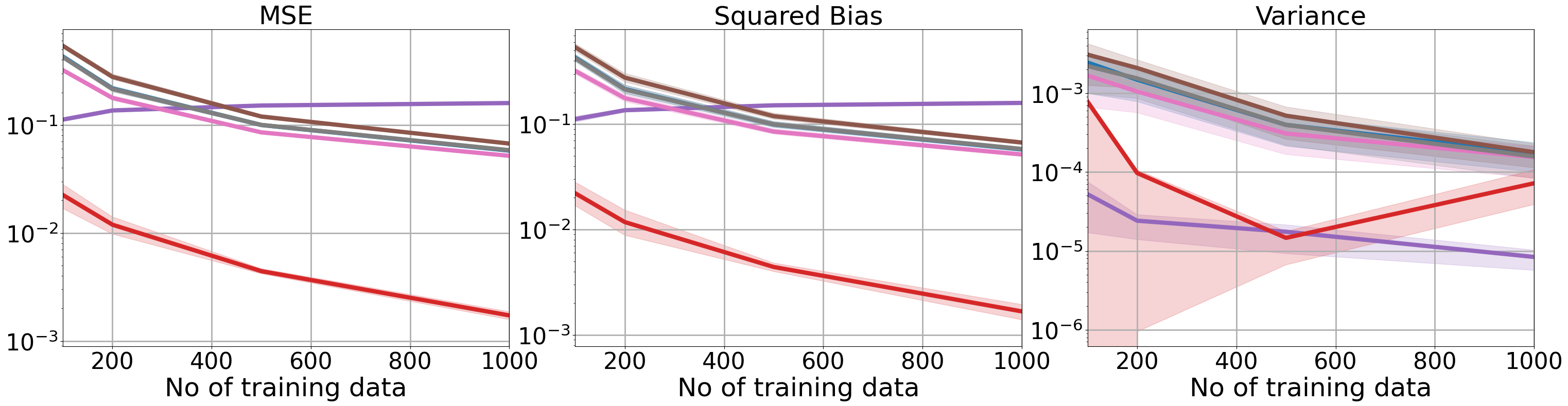}
	    \subcaption{Results with varying $m$ for $\alpha^\ast=0.6$ and $n = 500$}
	    \label{subfig:digits-tr}
	\end{subfigure}
    \caption{Results for Digits dataset. Note that compared to other datasets we consider smaller maximum dataset sizes $m,n$ here as the total number of available datapoints was 1797.}
    \label{fig:digits}
\end{figure}

\begin{figure}[h!]
    \centering
	\begin{subfigure}{0.8\textwidth}
	    \centering
	    \includegraphics[width=1\textwidth]{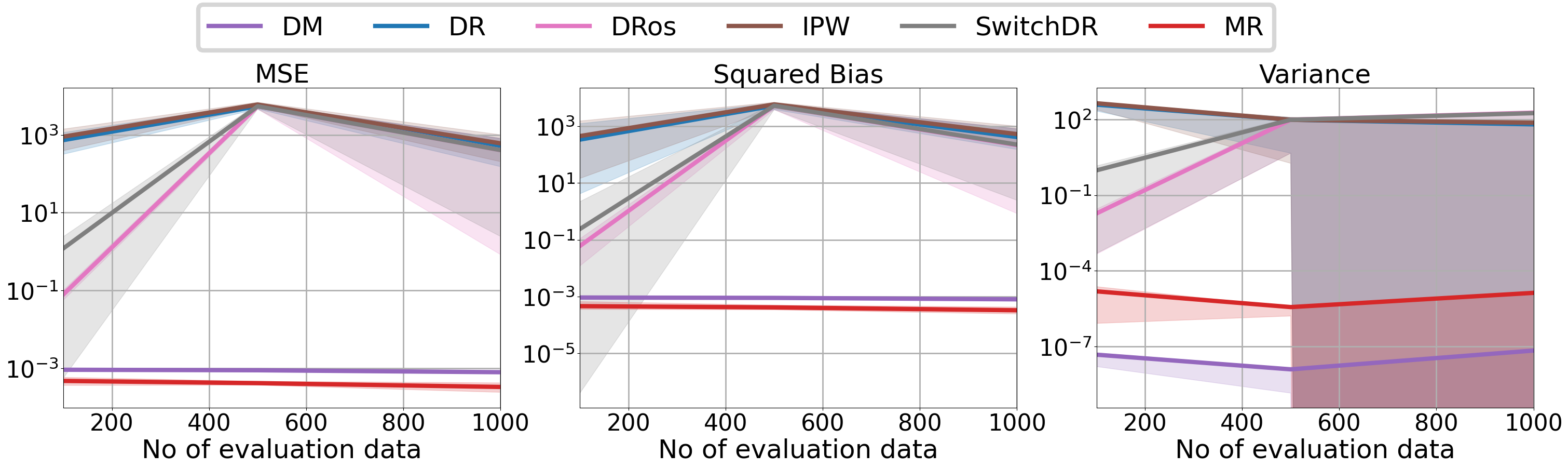}
	    \subcaption{Results with varying $n$ for $\alpha^\ast = 0.4$ and $m=2000$}
	    \label{subfig:cifar100-neval}
	\end{subfigure}\\
	\begin{subfigure}{0.8\textwidth} 
	    \centering
	    \includegraphics[width=1\textwidth]{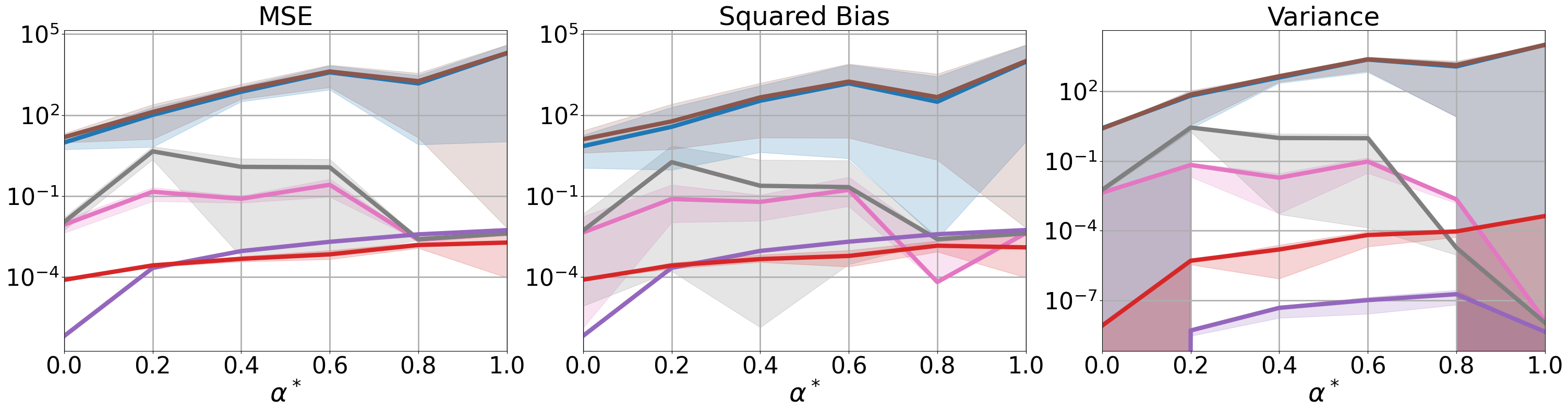}
	    \subcaption{Results with varying $\alpha^\ast$ for $n = 100$ and $m=2000$}
	    \label{subfig:cifar-ae}
	\end{subfigure}\\
 	\begin{subfigure}{0.8\textwidth} 
	    \centering
	    \includegraphics[width=1\textwidth]{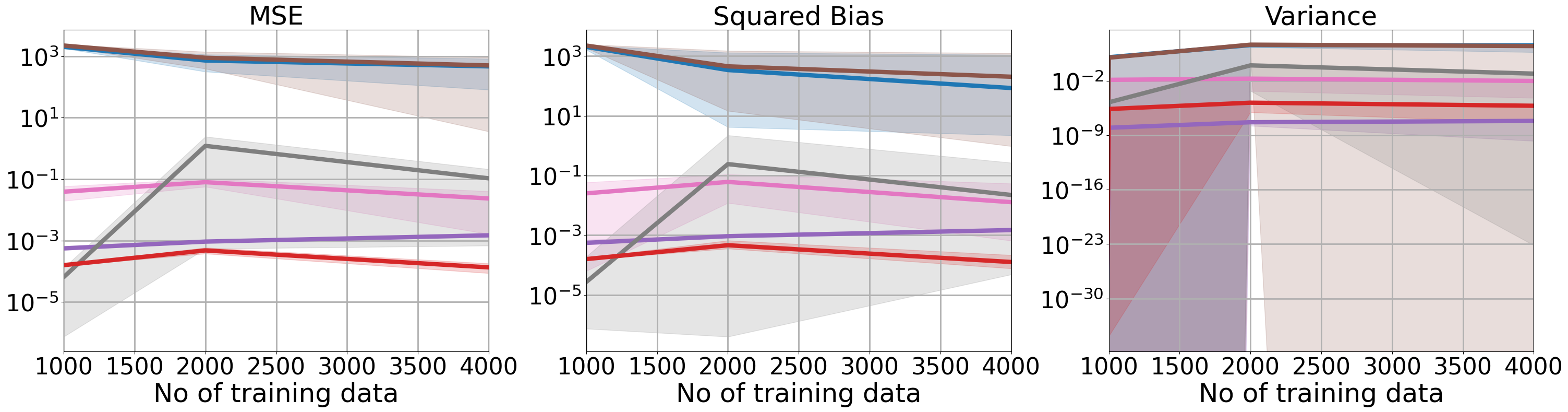}
	    \subcaption{Results with varying $m$ for $\alpha^\ast=0.4$ and $n = 100$}
	    \label{subfig:cifar100-tr}
	\end{subfigure}
    \caption{Results for CIFAR-100 dataset.}
    \label{fig:cifar100}
\end{figure}

\subsection{Application to Average Treatment Effect (ATE) estimation}\label{app:ate-empirical}
In this subsection, we provide additional details for our experiment applying MR to the problem of ATE estimation presented in the main text. We begin by describing the dataset being used in this experiment.

\paragraph{Twins dataset}
We use the Twins dataset as studied by \cite{louizos2017causal}, which comprises data from twin births in the USA between 1989-1991. The treatment $a=1$ corresponds to being born the heavier twin and the outcome $Y$ corresponds to the mortality of each of the twins in their first year of life. Since the data includes records for both twins, their outcomes would be considered as the two potential outcomes. Specifically, $Y(1)$ corresponds to the mortality of the heavier twin (and likewise for $Y(0)$). Closely following the methodology of \cite{louizos2017causal}, we only chose twins which are the same sex and weigh less than 2kgs. This provides us with a dataset of 11984 pairs of twins. 

The mortality rate for the lighter twin is 18.9\% and for the heavier twin is 16.4\%, leading to the ATE value being $\theta_\ate = -2.5\%$. For each twin-pair we obtained 46 covariates relating to the parents, the pregnancy and birth. 

\paragraph{Treatment assignment}
To simulate an observational study, we selectively hide one of the two twins by defining the treatment variable $A$ which depends on the feature \emph{GESTAT10}. This feature, which takes integer values from 0 to 9, is obtained by grouping the number of gestation weeks prior to birth into 10 groups.
Then we sample actions $A$ as follows, 
\[
A \mid X \sim \textup{Bern}(Z/10),
\]
where $Z$ is \emph{GESTAT10}, and $X$ are all the 46 features corresponding to a twin pair (including \emph{GESTAT10}). 

Using the treatment assignments defined above, we generate the observational data by selectively hiding one of the two twins from each pair. Next, we randomly split this dataset into training and evaluation datasets of sizes $m$ and $n$ respectively. In this experiment, we consider $m=5000$ training datapoints. 

\paragraph{Baselines}
Recall that ATE estimation can be formulated as the difference between off-policy values of deterministic policies $\pi^{(1)} \coloneqq \ind(A=1)$ and $\pi^{(0)} \coloneqq \ind(A=0)$. Therefore, any OPE estimator can be applied to ATE estimation. In this experiment, we compare our estimator against the baselines considered in our OPE experiments in Section \ref{subsec:additional-experiments-classification}. This includes the Direct Method (DM), IPW and DR estimators as well as Switch-DR \citep{wang2017optimal} and DR with Optimistic Shrinkage (DRos) \citep{su2020doubly}. To estimate $\hat{q}(x, a)$ for DM and DR estimators, we use multi-layer perceptrons (MLP) trained on the $m$ training datapoints. Additionally, we estimate the behaviour policy $\hatbeh$ using random forest classifier trained on the full training dataset. 

Since the outcome in this experiment is binary, we estimate the weights $w(y) = \Ebeh[\hat{\rho}(A, X)\mid Y=y]$ directly by estimating the sample mean of $\hat{\rho}(A, X)$ for datapoints with $Y=y$. This means that the alternative method of estimating MR yields the same value as the default method. We therefore do not consider these estimators separately. Additionally, since there is no natural embedding $R$ of the covariate-action space which satisfies the conditional dependence Assumption \ref{assum:indep-general}, we do not consider the G-MIPS (or MIPS) estimator either.   

\paragraph{Performance metric}
For our evaluation, we consider the absolute error in ATE estimation, $\epsilon_\ate$, defined as:
\[
\epsilon_\ate \coloneqq | \hat{\theta}^{(n)}_\ate - \theta_\ate |.
\]
Here, $\hat{\theta}^{(n)}_\ate$ denotes the value of the ATE estimated using $n$ evaluation datapoints. For example, for the IPW estimator, the $\hat{\theta}^{(n)}_\ate$ can be written as:
\[
\hat{\theta}^{(n)}_\ate = \ateipw = \frac{1}{n} \sum_{i=1}^n \left(\frac{\ind(a_i=1)-\ind(a_i=0)}{\hatbeh(a_i\mid x_i)}\right)\, y_i.
\]

All results for this experiment are provided in the main text.

\subsection{Additional synthetic data experiments} \label{sec:app-additional-results}
\begin{figure}[ht]
     \centering
     \begin{subfigure}[b]{0.8\textwidth}
         \centering
         \includegraphics[width=\textwidth]{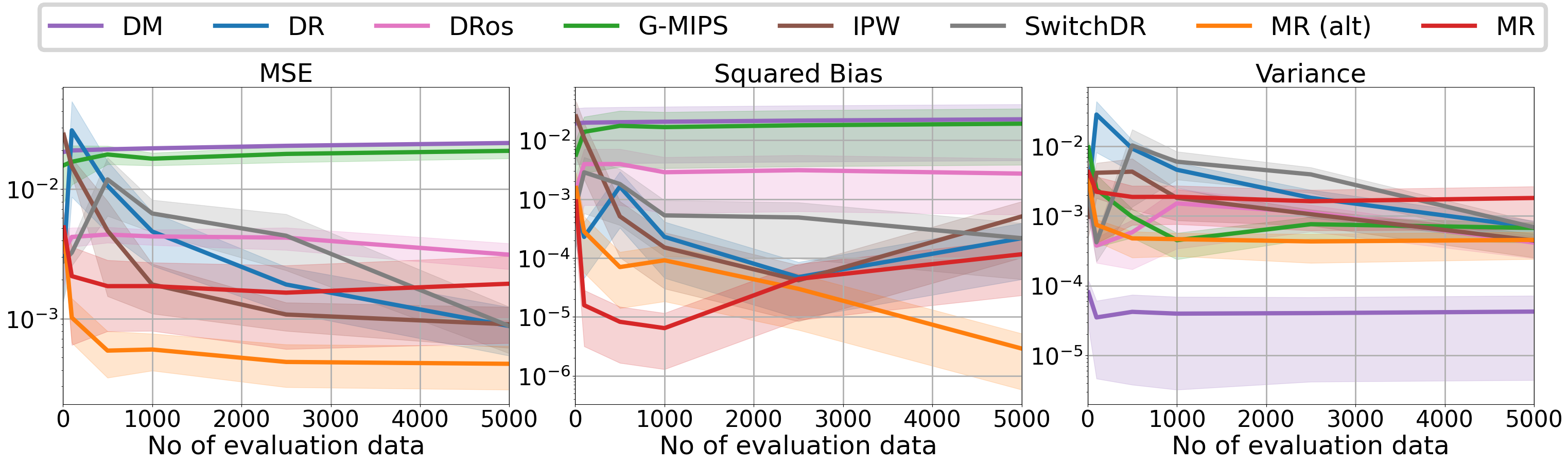}
         \caption{$d=1000$, $n_{a}=100$, $\alpha^\ast = 0.4$.}
         \label{fig:mse-vs-neval-conf2a}
     \end{subfigure}\\
     \begin{subfigure}[b]{0.8\textwidth}
         \centering
         \includegraphics[width=\textwidth]{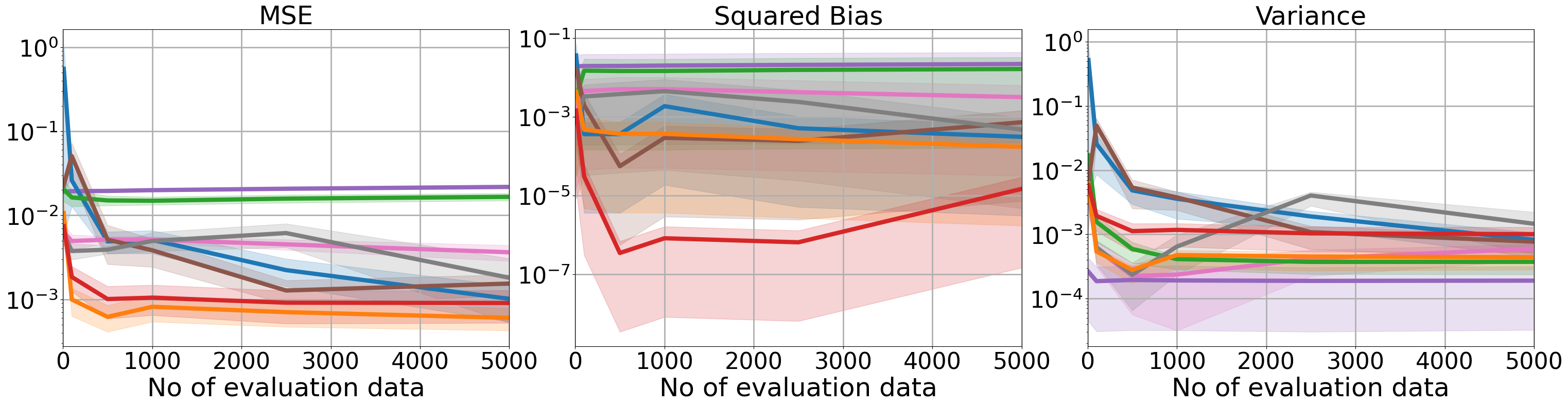}
         \caption{$d=10000$, $n_{a}=100$, $\alpha^\ast = 0.4$.}
         \label{fig:mse-vs-neval-conf2b}
     \end{subfigure}
     \caption{Results with varying size of evaluation dataset $n$.}
     \label{fig:mse-vs-neval-conf2}
 \end{figure}

 \begin{figure}[ht]
     \centering
    \begin{subfigure}[b]{0.8\textwidth}
         \centering
         \includegraphics[width=\textwidth]{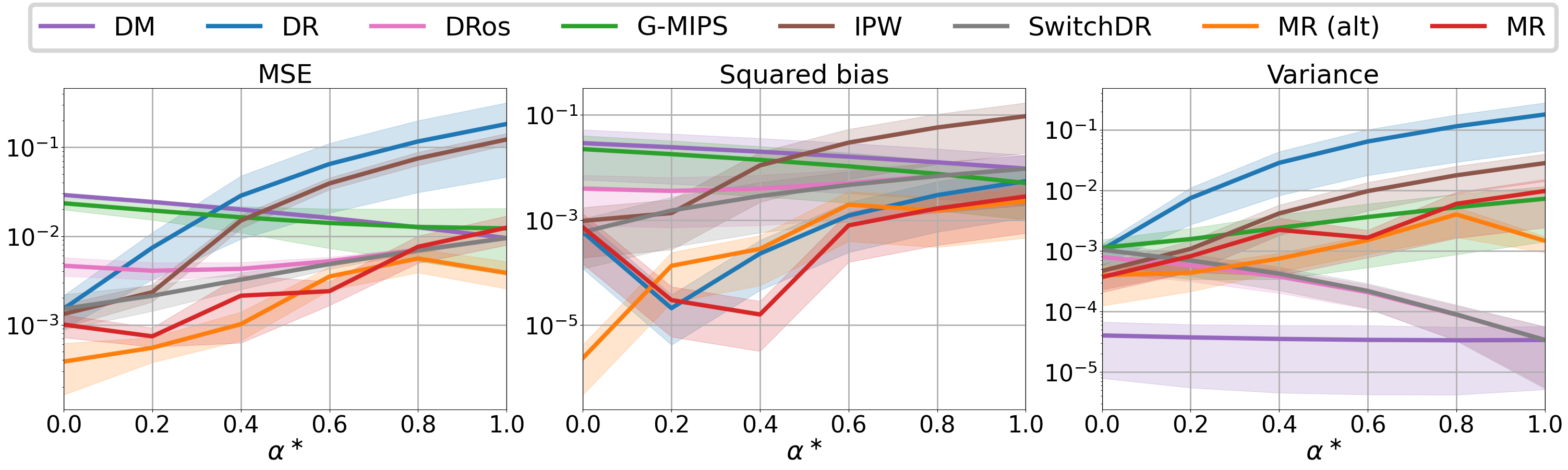}
         \caption{$d=1000$, $n_{a}=100$, $n = 100$.}
         \label{fig:mse-vs-betatar-conf2a}
     \end{subfigure}\\
     \begin{subfigure}[b]{0.8\textwidth}
         \centering
         \includegraphics[width=\textwidth]{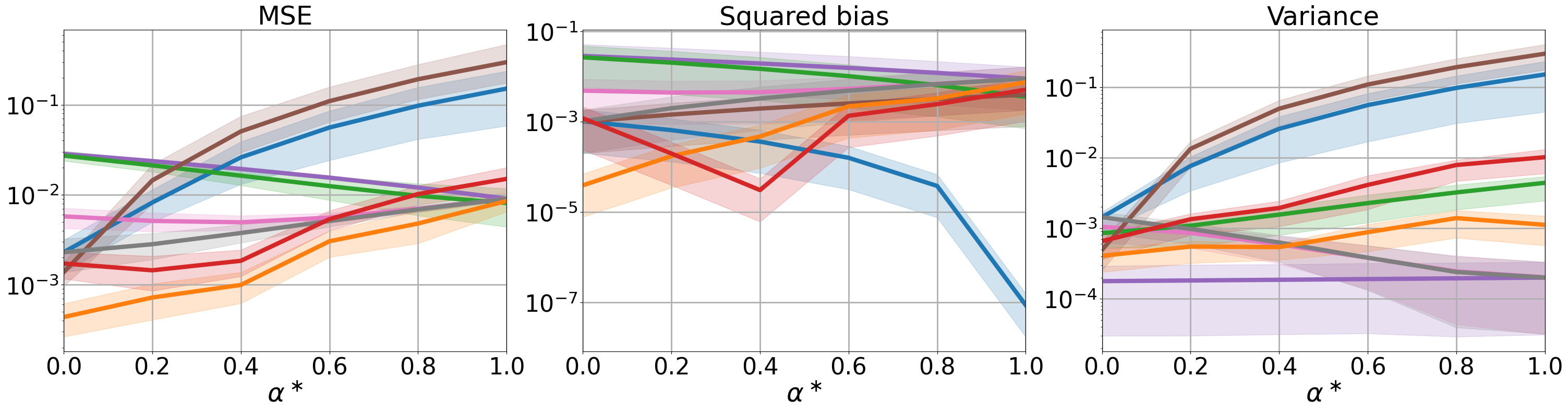}
         \caption{$d=10000$, $n_{a}=100$, $n = 100$.}
         \label{fig:mse-vs-betatar-conf2b}
     \end{subfigure}
     \caption{Results with varying $\alpha^\ast$.}
     \label{fig:mse-vs-betatar-conf2}
 \end{figure}

 \begin{figure}[ht]
     \centering
    \begin{subfigure}[b]{0.8\textwidth}
         \centering
         \includegraphics[width=\textwidth]{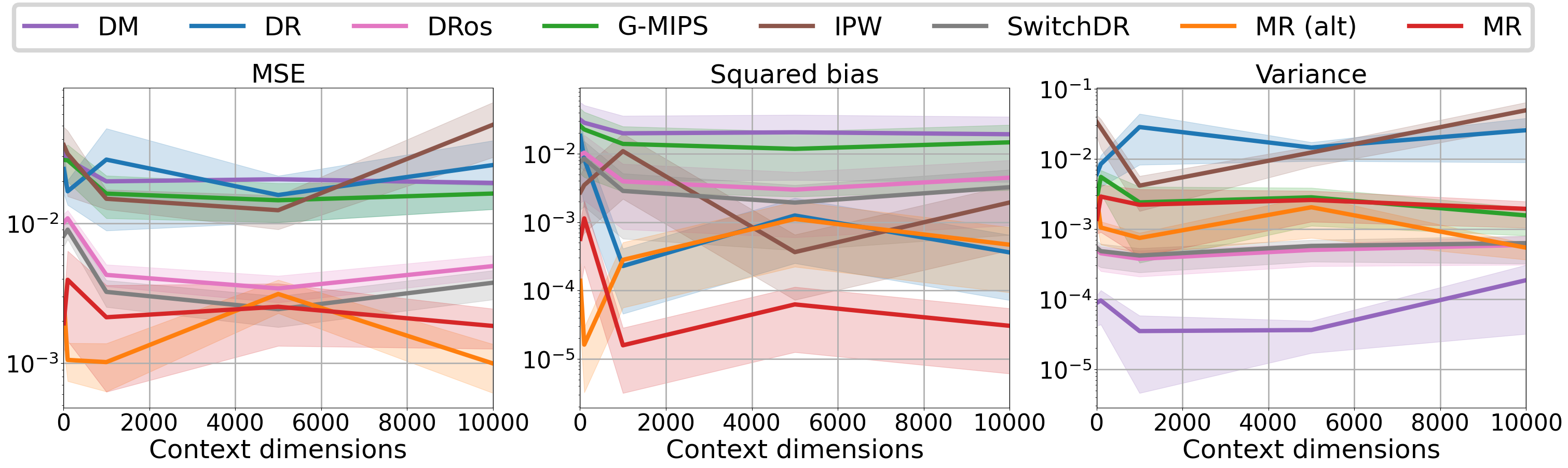}
         \caption{$n_{a}=100$, $n = 100$, $\alpha^\ast = 0.4$.}
         \label{fig:mse-vs-d-conf2a}
     \end{subfigure}\\
     \begin{subfigure}[b]{0.8\textwidth}
         \centering
         \includegraphics[width=\textwidth]{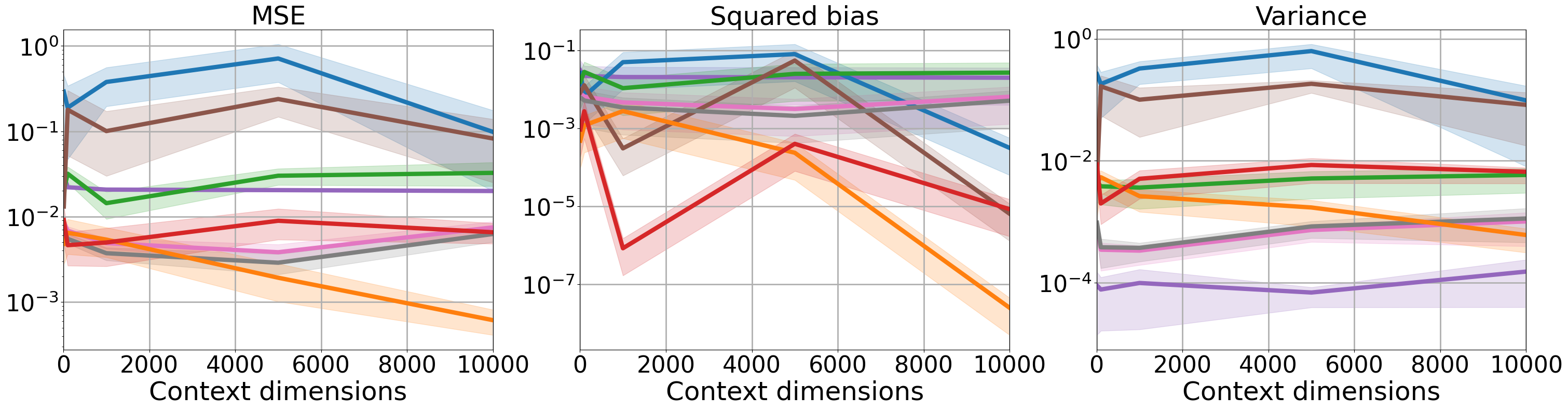}
         \caption{$n_{a}=500$, $n = 100$, $\alpha^\ast = 0.4$.}
         \label{fig:mse-vs-d-conf2b}
     \end{subfigure}
     \caption{Results with varying context dimensions $d$.}
     \label{fig:mse-vs-d-conf2}
 \end{figure}

 \begin{figure}[ht]
     \centering
    \begin{subfigure}[b]{0.8\textwidth}
         \centering
         \includegraphics[width=\textwidth]{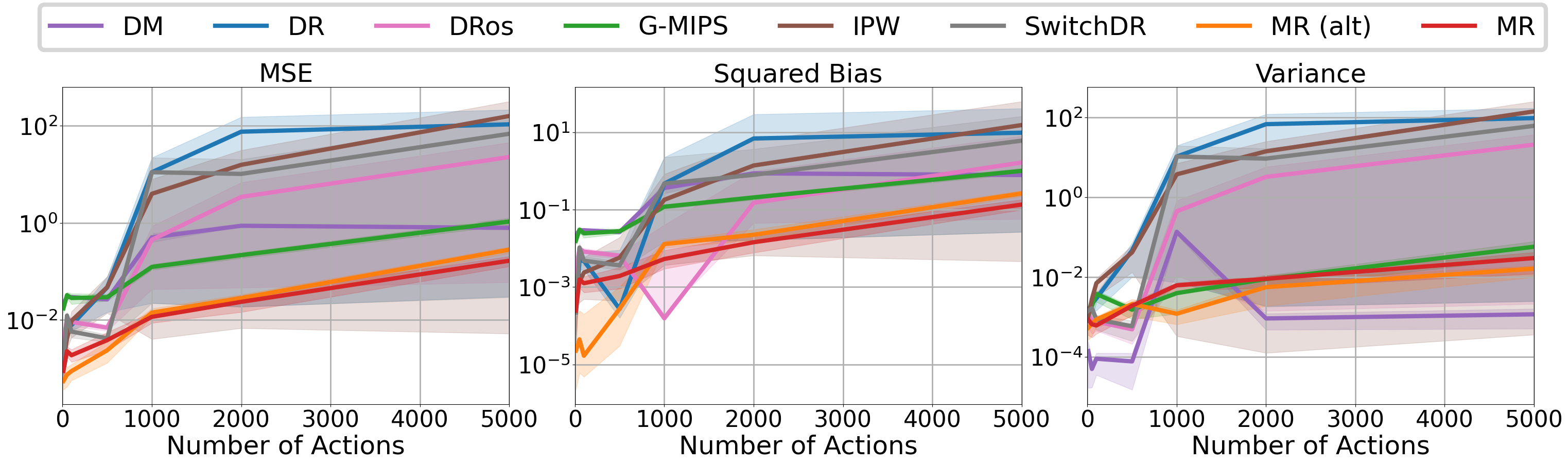}
         \caption{$d=100$, $n = 100$, $\alpha^\ast = 0.2$.}
         \label{fig:mse-vs-nac-conf2a}
     \end{subfigure}\\
     \begin{subfigure}[b]{0.8\textwidth}
         \centering
         \includegraphics[width=\textwidth]{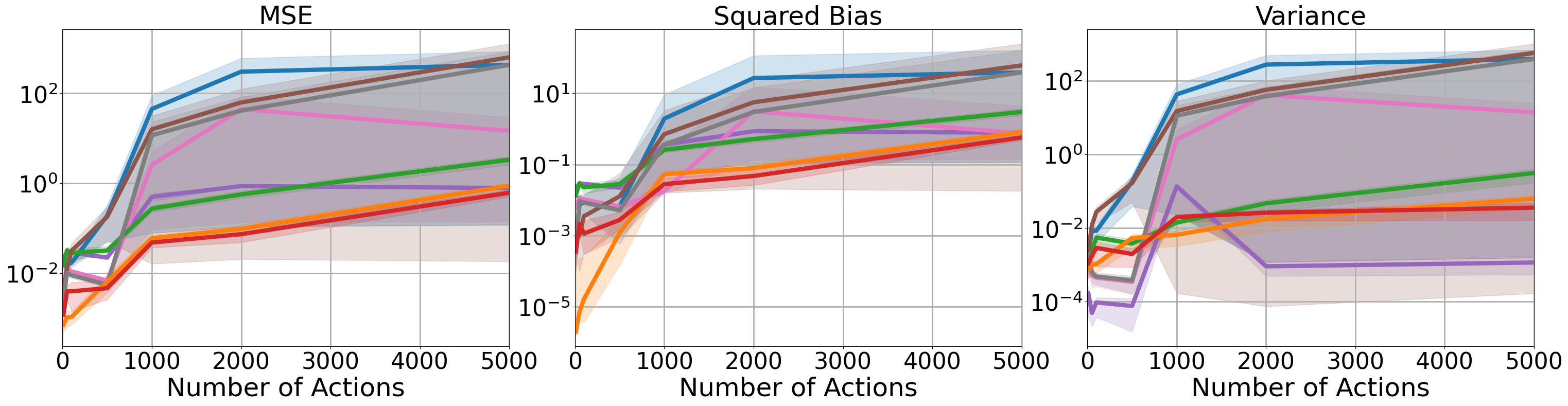}
         \caption{$d=100$, $n = 100$, $\alpha^\ast = 0.4$.}
         \label{fig:mse-vs-nac-conf2b}
     \end{subfigure}
     \caption{Results with varying number of actions $n_{a}$.}
     \label{fig:mse-vs-nac-conf2}
 \end{figure}

In addition to the synthetic data experiments provided in Section \ref{sec:exp-synth}, we also consider an additional synthetic data setup to obtain further empirical evidence in favour of the MR estimator, and also compare it against the generalised version of the MIPS estimator (described as G-MIPS in Appendix \ref{app:gmips}).
Here, we use a similar setup to \cite{saito2022off} (albeit without action embeddings $E$) where the $d$-dimensional context vectors $x$ are sampled from a standard normal distribution. Likewise, the action space is finite and comprises of $n_a$ actions, i.e.\ $\Aspace = \{0, \dots, n_a-1\}$, with $n_a$ taking a range of different values. The reward function is defined as follows:

 \paragraph{Reward function}
The expected reward $q(x, a)\coloneqq\E[Y\mid x, a]$ for these experiments is defined as follows:
\[
    q(x, a) = \sin \left(a \cdot ||x||_2 \right). 
\]
The reward $Y$ is obtained by adding a normal noise random variable to $q(x, a)$
\[
Y = q(X, A) + \epsilon, 
\]
where $\epsilon \sim \mathcal{N}(0, 0.01)$. Here, it can be seen that conditional on $R=(||X||_2, A)$, the reward $Y$ does not depend on $(X, A)$, i.e., the embedding $R$ satisfies the conditional independence assumption $Y \indep (X, A) \mid R$. 

\paragraph{Behaviour and target policies}
We first define a behaviour policy by applying softmax function to $q(x, a)$ as
\[
\beh(a\mid x) = \frac{\exp{(q(x, a))}}{\sum_{a' \in \Aspace} \exp{(q(x, a'))}}.
\]
Just like in Section \ref{sec:exp-synth}, to investigate the effect of increasing policy shift, we define a class of policies,
\[
\pi^{\alpha^\ast}(a | x) = \alpha^\ast\,\ind(a = \arg\max_{a'\in \Aspace} q(x, a')) + \frac{1-\alpha^\ast}{|\Aspace|} \quad \textup{where} \quad q(x, a) \coloneqq \E[Y\mid X=x, A=a],
\]
where $\alpha^\ast \in [0, 1]$ allows us to control the shift between $\beh$ and $\tar$. Again, the shift between $\beh$ and $\tar$ increases as $\alpha^\ast \rightarrow 1$. Using the ground truth behaviour policy $\beh$, we generate a dataset which is split into training and evaluation datasets of sizes $m$ and $n$ respectively. 

In Figures \ref{fig:mse-vs-neval-conf2} - \ref{fig:mse-vs-nac-conf2}, we present the results for this experimental setup for different choices of paramater configurations. 

\paragraph{Estimation of behaviour policy $\hatbeh$ and marginal ratio $\hat{w}(y)$}
For the MR estimator, we estimate the behaviour policy using a random forest classifier trained on 50\% of the training data and use the rest of the training data to estimate the marginal ratios $\hat{w}(y)$ using multi-layer perceptrons (MLP). Moreover, for a fair comparison we use a different behaviour policy estimate $\hatbeh$ for all other baselines which is trained on the entire training data. 

\paragraph{Additional Baselines}
In addition to the baselines considered in the main text (Section \ref{sec:exp-synth}), we also consider Switch-DR \citep{wang2017optimal} and DR with Optimistic Shrinkage (DRos) \citep{su2020doubly}. In addition, we also include the results for MR estimated using the alternative method (`MR (alt)') outlined in Section \ref{sec:alt-estimation-method}. For the G-MIPS estimator (defined in Appendix \ref{app:gmips}) considered here, we use $R = (a, ||x||_2)$\footnote{It is easy to see that in our setup, the embedding $R = (a, ||x||_2)$ satisfies the conditional independence assumption $Y \indep (X, A) \mid R$ needed for G-MIPS estimator to be unbiased}. 
To estimate $\hat{q}(x, a)$ for DM and DR estimators, we use multi-layer perceptrons (MLPs).

\subsubsection{Results}
For this experiment, the results are computed over 10 different sets of logged data replicated with different seeds, and in Figures \ref{fig:mse-vs-neval-conf2} - \ref{fig:mse-vs-nac-conf2} we use a total of $m=5000$ training data. 

\paragraph{Varying $n$}
Figure \ref{fig:mse-vs-neval-conf2} shows that MR outperforms the other baselines, in terms of MSE and squared bias, when the number of evaluation data $n\leq 1000$. Additionally, we observe that in this experiment, MR esitmated using alternative methods, MR (alt), yields better results than the original method of estimating MR. Moreover, while the variance of DM is lower than that of MR, the DM method has a high bias and consequently a high MSE.

\paragraph{Varying $\alpha^\ast$}
Figure \ref{fig:mse-vs-betatar-conf2} shows the results with increasing policy shift. It can be seen that overall MR methods achieve the smallest MSE with increasing policy shift. Moreover, the difference between MSE and variance of MR and IPW/DR methods increases with increasing policy shift, showing that MR performs especially better than these baselines when the difference between behaviour and target policies is large.

\paragraph{Varying $d$ and $n_a$}
Figures \ref{fig:mse-vs-d-conf2} and \ref{fig:mse-vs-nac-conf2} show that MR outperforms the other baselines as the context dimensions and/or number of actions increase. In fact, Figure \ref{fig:mse-vs-nac-conf2} shows that MR is significantly robust to increasing action space, whereas baselines like IPW and DR perform poorly in large action spaces.

\paragraph{Varying $m$}
Figure \ref{fig:mse-vs-ntr-conf2} shows the results with increasing number of training data $m$. We again observe that the MR methods `MR' and `MR (alt)' outperforms the other baselines in terms of the MSE and squared bias even when the number of training data is low. Moreover, the variance of both the MR estimators continues to improve with increasing number of training data.

Unlike our experimental results in Section \ref{subsec:mips-empirical}, `MR (alt)' performs better than the original MR estimator overall. This shows that one of these two methods is not better than the other consistently in all cases, and their relative performance depends on the dataset under consideration. 

\begin{figure}[ht]
     \centering
    \begin{subfigure}[b]{0.8\textwidth}
         \centering
         \includegraphics[width=\textwidth]{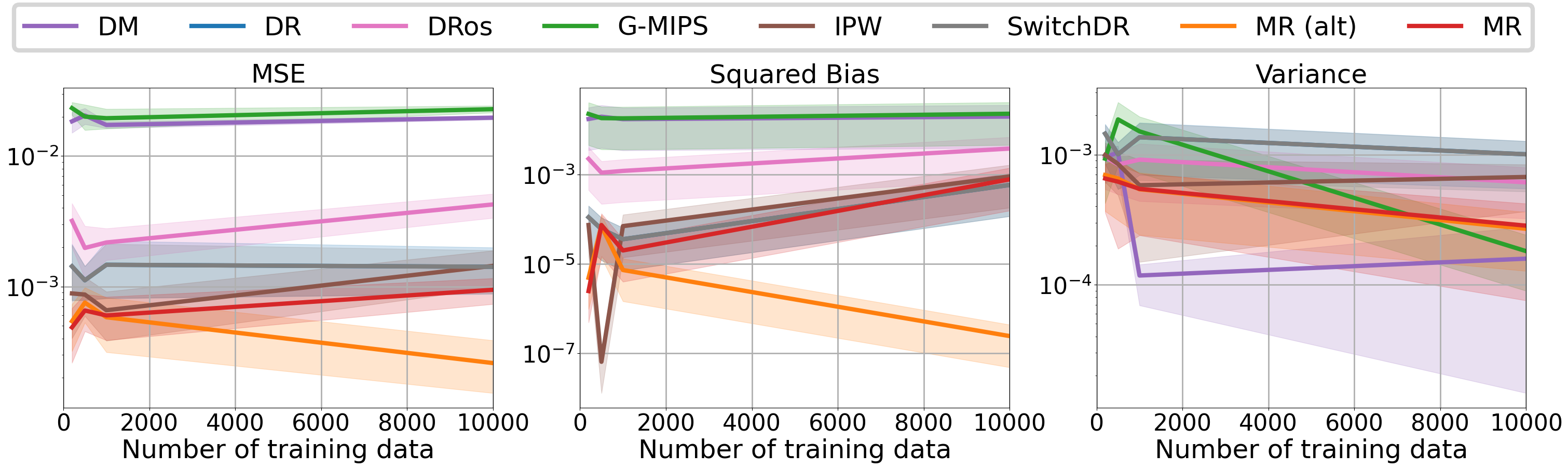}
         \caption{$d=5000$, $n = 100$, $n_a = 10$, $\alpha^\ast = 0.2$.}
         \label{fig:mse-vs-ntr-conf2a}
     \end{subfigure}\\
     \begin{subfigure}[b]{0.8\textwidth}
         \centering
         \includegraphics[width=\textwidth]{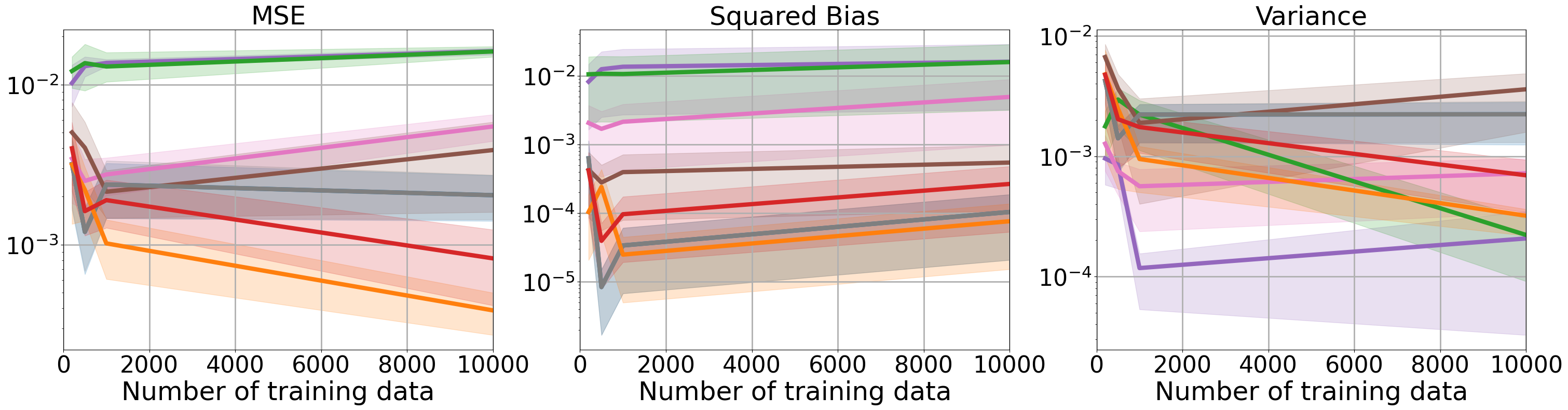}
         \caption{$d=5000$, $n = 100$, $n_a = 10$, $\alpha^\ast = 0.4$.}
         \label{fig:mse-vs-ntr-conf2b}
     \end{subfigure}
     \caption{Results with varying number of training data $m$.}
     \label{fig:mse-vs-ntr-conf2}
 \end{figure}

\subsection{Self-normalised MR estimator}
\begin{figure}[ht]
     \centering
    \begin{subfigure}[b]{0.75\textwidth}
         \centering
         \includegraphics[width=\textwidth]{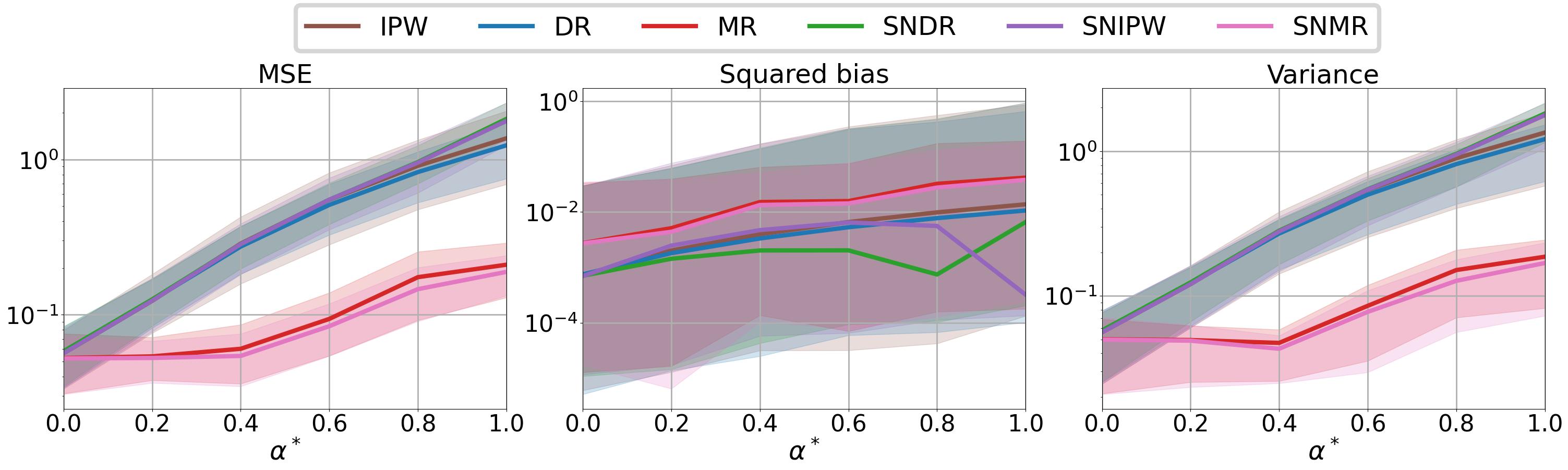}
         \caption{$d=10000$, $n = 200$, $n_a = 20$, $m = 5000$.}
         \label{fig:self-norma}
     \end{subfigure}\\
     \begin{subfigure}[b]{0.75\textwidth}
         \centering
         \includegraphics[width=\textwidth]{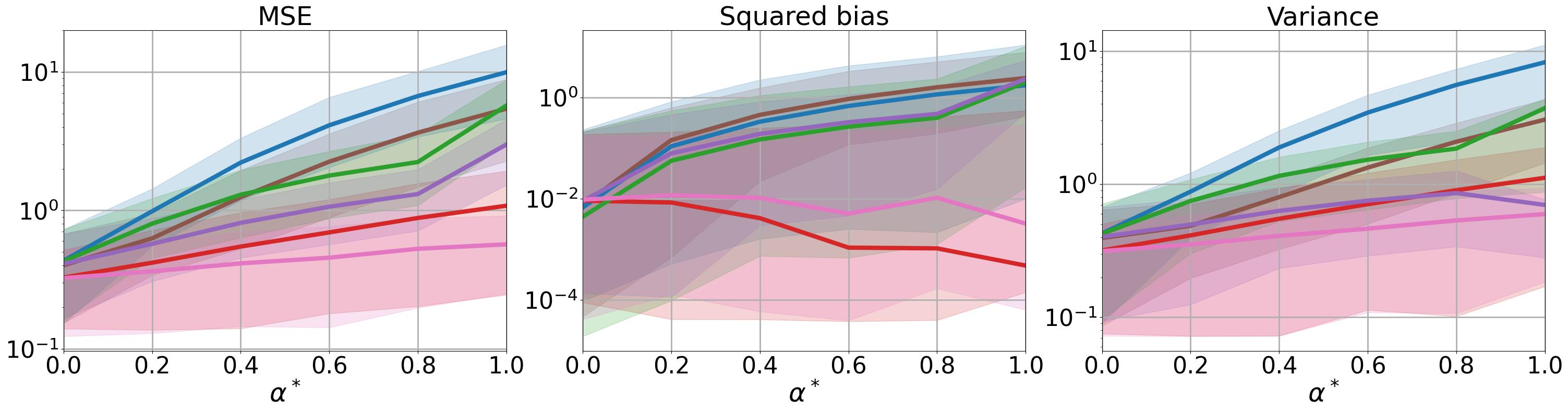}
         \caption{$d=5000$, $n = 200$, $n_a = 20$, $m=1000$.}
         \label{fig:self-normb}
     \end{subfigure}\\
     \begin{subfigure}[b]{0.75\textwidth}
         \centering
         \includegraphics[width=\textwidth]{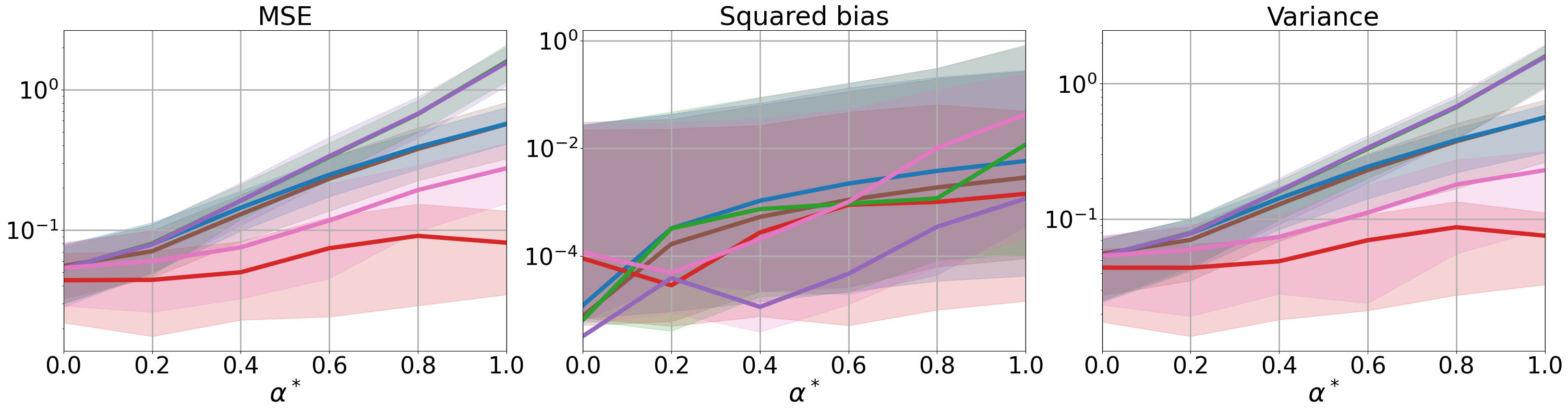}
         \caption{$d=10000$, $n = 200$, $n_a = 20$, $m=5000$.}
         \label{fig:self-normc}
     \end{subfigure}
     \caption{Results for self-normalised estimators with varying target policy shift $\alpha^\ast$ for synthetic data setup considered in Section \ref{sec:exp-synth}. Here, ``SN'' denotes self-normalised estimators.}
     \label{fig:self-norm}
 \end{figure}
Self-normalization trick has been used in practice to reduce the variance in off-policy estimators \citep{swaminathan2015the}. This technique is also applicable to the MR estimator, and leads to the self-normalized MR estimator (denoted as $\thetasnmr$) defined as follows:
\[
\thetasnmr \coloneqq \sum_{i=1}^n \frac{w(Y_i)}{\sum_{j=1}^n w(Y_j)}\,Y_i.
\]

We conducted experiments to investigate the effect of self-normalisation on the performance of the IPW, DR and MR estimators. Figure \ref{fig:self-norm} shows results for three different choices of parameter configurations. Overall, we observe that in all settings, the MR and self-normalised MR (SNMR) estimator outperform all other baselines including the self-normalised IPW and DR estimators (denoted as SNIPW and SNDR respectively). Moreover, in some settings, where the importance ratios achieve very high values, self-normalisation can reduce the variance and MSE of the corresponding estimator (for example, Figure \ref{fig:self-normb}). However, we also observe cases in which self-normalization does not significantly change the results (Figure \ref{fig:self-norma}), or may even slightly worsen the MSE of the estimators (Figure \ref{fig:self-normc}). 
        \chapter{\label{app:copp}Conformal Off-Policy Prediction in Contextual Bandits}

\minitoc

\section{Proofs}\label{sec:proofs}
\subsection{Proof of Proposition \ref{coverage_theorem}} 
This proof is a direct adaptation of \cite[Lemma 3]{tibshirani2020conformal}, and has only been included for the sake of completeness.

In this proof, we use the notion of \textit{weighted exchangeability} as defined in Section 3.2 of \cite{tibshirani2020conformal}.
\begin{definition}[Weighted exchangeability]\label{def:weighted_exch}
Random variables $V_1, \dots, V_n$ are said to be \textit{weighted exchangeable} with weight functions $w_1, \dots, w_n$, if the density $f$ of their joint distribution can be factorized as
\begin{align}
    f(v_1, \dots, v_n) = \prod_{i=1}^n w_i(v_i) g(v_1, \dots, v_n)
\end{align}
where $g$ is any function that does not depend on the ordering of its inputs, i.e. $g(v_{\sigma(1)}, \dots, v_{\sigma(n)}) = g(v_1, \dots, v_n)$ for any permutation $\sigma$ of $1, \dots, n$.
\end{definition}

\begin{lemma}\label{exchangeability_lemma}
Let $Z_i = (X_i, Y_i) \in \mathbb{R}^d \times \mathbb{R}$, $i=1,...,n+1$, be such that $\{(X_i, Y_i)\}_{i=1}^n \overset{\textup{i.i.d.}}{\sim}P^{\pi^b}_{X,Y}$ and $(X_{n+1}, Y_{n+1}) \sim P^{\pi^*}_{X,Y}$. Then $Z_1, \dots, Z_{n+1}$ are weighted exchangeable with weights $w_i \equiv 1$, $i\leq n$ and $w_{n+1}(X,Y) = \mathrm{d}P^{\pi^{*}}_{X,Y}/\mathrm{d}P^{\pi^{b}}_{X,Y}(X,Y)$.
\end{lemma}

\begin{proof}
The proof below is merely a verification that our proposed weights still retain the coverage guarantees and is mainly taken from \cite{tibshirani2020conformal}. Hence, we follow the same strategy as in \cite{tibshirani2020conformal}, with the exception that we have the weights as in Lemma \ref{exchangeability_lemma}, hence inducing a lot of simplifications. As in \cite{tibshirani2020conformal}, we assume for simplicity that $V_1, \dots, V_{n+1}$ are distinct almost surely, however the result holds in general case as well. We define $f$ as the joint distribution of the random variables $\{X_i, Y_i\}_{i=1}^{n+1}$. We also denote $E_z$ as the event of $\{Z_1, \dots, Z_{n+1}\}$ = $\{z_1, \dots, z_{n+1}\}$ and let $v_i = s(z_i) = s(x_i, y_i)$, then for each $i$:
\begin{align}
    \mathbbm{P}\{V_{n+1} = v_i| E_z\} = \mathbbm{P}\{Z_{n+1}=z_i|E_z\} = \frac{\sum_{\sigma:\sigma(n+1)=i}f(z_{\sigma(1)}, \dots, z_{\sigma(n+1)})}{\sum_{\sigma}f(z_{\sigma(1)}, \dots, z_{\sigma(n+1)})}
\end{align}

Now using the fact that $Z_1, \dots, Z_{n+1}$ are weighted exchangeable:

\begin{align}
    \frac{\sum_{\sigma:\sigma(n+1)=i}f(z_{\sigma(1)}, \dots, z_{\sigma(n+1)})}{\sum_{\sigma}f(z_{\sigma(1)}, \dots, z_{\sigma(n+1)})}&= \frac{\sum_{\sigma:\sigma(n+1)=i}\prod_{j=1}^{n+1}w_{j}(z_{\sigma(j)})g(z_{\sigma(1)}, \dots, z_{\sigma(n+1)})}{\sum_{\sigma}\prod_{j=1}^{n+1} w_{j}(z_{\sigma(j)})g(z_{\sigma(1)}, \dots, z_{\sigma(n+1)})}\label{eq:simply}\\ 
    &= \frac{w_{n+1}(z_i)g(z_{1}, \dots, z_{n+1})}{\sum_{j=1}^{n+1} w_{n+1}(z_{j})g(z_{1}, \dots, z_{n+1})}\nonumber\\
    &= p_i^w(z_{n+1}) \nonumber
\end{align}
where we recall that
\begin{align}
    p_i^{w}(x, y) \coloneqq \frac{w(X_i, Y_i)}{\sum_{j=1}^n w(X_j, Y_j) + w(x, y)}. \nonumber
\end{align}
We get simplifications in \eqref{eq:simply} due to the weights defined in Lemma \ref{exchangeability_lemma}, i.e. $w_i \equiv 1$ for $i \leq n$ and $w_{n+1}(x, y) = w(x, y) = \mathrm{d}P^{\pi^{*}}_{X,Y}/\mathrm{d}P^{\pi^{b}}_{X,Y}(x,y)$.
Next, just as in \cite{tibshirani2020conformal} we can view:
\begin{align}
    V_{n+1} = v_i| E_z \sim \sum_{i=1}^{n+1}p_i^w(z_{n+1}) \delta_{v_i}
\end{align}
which implies that:
$$
\mathbbm{P}\{V_{n+1} \leq \text{Quantile}_{\beta}(\sum_{i=1}^{n+1}p_i^w(z_{n+1}) \delta_{v_i}) | E_z\} \geq \beta.
$$
This is equivalent to 
$$
\mathbbm{P}\{V_{n+1} \leq \text{Quantile}_{\beta}(\sum_{i=1}^{n+1}p_i^w(Z_{n+1}) \delta_{v_i}) | E_z\} \geq \beta
$$
and, after marginalizing, one has
$$
\mathbbm{P}\{V_{n+1} \leq \text{Quantile}_{\beta}(\sum_{i=1}^{n+1}p_i^w(Z_{n+1}) \delta_{v_i})\} \geq \beta
$$
This is equivalent to the claim in Proposition \ref{coverage_theorem}.
\end{proof}

\subsection{Proof of Proposition \ref{prop2}}
The following proof is an adaptation of \cite[Proposition 1]{lei2020conformal} to our setting.

Before detailing the main proof, we introduce a preliminary result which will be used in the proof of Proposition \ref{prop2}.

\begin{lemma}\label{Aevent}
Let $\hat{w}(x,y)$ be an estimate of the weights $w(x,y) = \mathrm{d}P^{\pi^{*}}_{X,Y}/\mathrm{d}P^{\pi^{b}}_{X,Y}(x,y)$, and $$(\expb[\hat{w}(X,Y)^r])^{1/r} \leq M_r < \infty$$ for some $r \geq 2$. Let $(X_i, Y_i) \overset{\textup{i.i.d.}}{\sim} P^{\pi^b}_{X,Y}$ and $\mathcal{A}$ denote the event that 
\[
\sum_{i=1}^n \hat{w}(X_i, Y_i) \leq n/2.
\]
Then, 
\[
\mathbb{P}(\mathcal{A}) \leq \frac{c_1 M_r^2}{n}
\]
where $c_1$ is an absolute constant, and the probability is taken over $\{X_i, Y_i\}_{i=1}^n  \overset{\textup{i.i.d.}}{\sim} P^{\pi^b}_{X,Y}$.
\end{lemma}

\subsubsection*{Proof of Lemma \ref{Aevent}}
The condition $\expb[\hat{w}(X, Y)^r] < \infty \implies \behprob(\hat{w}(X, Y)< \infty) = 1$ and $\expb[\hat{w}(X, Y)]<\infty$. WLOG assume $\expb[\hat{w}(X, Y)] = 1$. Recall that $p_{i}^{\hat{w}}(x, y) \coloneqq \frac{\hat{w}(X_i, Y_i)}{\sum_{i=1}^n \hat{w}(X_i, Y_i) + \hat{w}(x,y)}$, and therefore, $p_i^{\hat{w}}(x,y)$ are invariant to weight scaling. Since $\expbcal[\hat{w}(X_i, Y_i)]^2 \leq M_r^2$ and $\expbcal(\hat{w}(X_i, Y_i)) = 1$, using Chebyshev's inequality %
\begin{align}
    \p\left( \sum_{i=1}^n  \hat{w}(X_i, Y_i) \leq n/2 \right) =& \p\left( \sum_{i=1}^n  (\hat{w}(X_i, Y_i)-1) \leq -n/2 \right) \nonumber\\
    \leq& \p\left( |\sum_{i=1}^n  (\hat{w}(X_i, Y_i) -1)| \geq n/2 \right) \nonumber \\
    \leq& \frac{4}{n^2}\E \Bigg[\left( \sum_{i=1}^n \hat{w}(X_i, Y_i) - \E[\hat{w}(X_i, Y_i)] \right)^2\Bigg] \nonumber \\
    =& \frac{4}{n^2} \left\{ n \E |\hat{w}(X_1, Y_1) - \E[\hat{w}(X_1, Y_1)] |^2\right\}  \label{holder1}\\ 
    \leq& \frac{16}{n^2}  n \E |\hat{w}(X_1, Y_1)|^2  \label{holder2} \\
    \leq& \frac{c_1 M_r^2}{n} \nonumber
\end{align}
where to get from (\ref{holder1}) to (\ref{holder2}) we use:
\begin{align}
    \E |\hat{w}(X_1, Y_1) - \E[\hat{w}(X_1, Y_1)] |^2 &\leq 2 \E \left[\hat{w}(X_1, Y_1)^2 + \E[\hat{w}(X_1, Y_1)]^2 \right] \nonumber\\
    &\leq 4\E[\hat{w}(X_1, Y_1)^2]. \nonumber
\end{align}
\qed

We can now prove Proposition \ref{prop2}.
\begin{proof}
The condition $\expb[\hat{w}(X, Y)^r] < \infty \implies \behprob(\hat{w}(X, Y)< \infty) = 1$ and $\expb[\hat{w}(X, Y)]<\infty$. WLOG assume $\expb[\hat{w}(X, Y)] = 1$. Let $\tilde{P}^{\pi^*}_{X, Y}$ be a probability measure with 
\[
    \mathrm{d}\tilde{P}^{\pi^*}_{X, Y}(x,y) \coloneqq \hat{w}(x,y) \mathrm{d}P^{\pi^b}_{X, Y}(x,y)
\]
and $(\tilde{X}, \tilde{Y}) \sim \tilde{P}^{\pi^*}_{X,Y}$ that is independent of the data. By H\"older's inequality, 
\begin{align}
    \expatt[\hat{w}(\tilde{X}, \tilde{Y})] =& \int_{\tilde{x}, \tilde{y}} \frac{\mathrm{d}\tilde{P}^{\pi^*}(\tilde{x}, \tilde{y})}{\mathrm{d}P^{\pi^b}(\tilde{x}, \tilde{y})}\mathrm{d}\tilde{P}^{\pi^*}(\tilde{x}, \tilde{y}) \nonumber \\
    =& \expb[\hat{w}(X, Y)^2] \leq M_r^2 < \infty \nonumber 
\end{align}
Note using Proposition \ref{coverage_theorem} with $(\tilde{X}, \tilde{Y})$ denoting $(X_{n+1}, Y_{n+1})$ for simplicity 
\begin{align}
    &\mathbb{P}(\tilde{Y} \in \hat{C}(\tilde{X}, \tilde{Y})) \nonumber \\
    &\quad= \exptt \left[\mathbb{P}\left(s(\tilde{X}, \tilde{Y}) \leq \text{Quantile}_{1-\alpha}\left(\sum_{i=1}^n p_i^{\hat{w}}(\tilde{X}, \tilde{Y})\delta_{V_i} + p_{n+1}^{\hat{w}}(\tilde{X}, \tilde{Y})\delta_\infty \right)\mid \mathcal{E}(\tilde{V})\right)\right] \label{eq4}
\end{align}
where $\mathcal{E}(\tilde{V})$ denotes the unordered set of $V_1, \dots, V_{n+1}$. Marginalising over $\{(X_i, Y_i)\}_{i=1}^n$, we obtain
\begin{align}
    (\ref{eq4}) \leq \E\left(1-\alpha + \max_{i \in [n+1]} p_i^{\hat{w}}(\tilde{X}, \tilde{Y}) \right) \label{eq5}
\end{align}
where the expectation is over $\{(X_i, Y_i)\}_{i=1}^n \overset{\textup{i.i.d.}}{\sim} P^{\pi^b}_{X, Y}$ and $(\tilde{X}, \tilde{Y}) \sim \tilde{P}^{\pi^*}_{X, Y}$. Let $\mathcal{A}$ denote the event that 
\[
\sum_{i=1}^n \hat{w}(X_i, Y_i) \leq n/2.
\]
using Lemma \ref{Aevent} and $\E[\hat{w}(\tilde{X}, \tilde{Y})] \leq M_r^2 $, we get that
\begin{align}
    \E\left[\max_{i \in [n+1]} p_i^{\hat{w}}(\tilde{X}, \tilde{Y})\right] =& \E\left[ \frac{\max \{\hat{w}(\tilde{X}, \tilde{Y}), \max_i \hat{w}(X_i, Y_i) \}}{\hat{w}(\tilde{X}, \tilde{Y}) + \sum_{i=1}^n \hat{w}(X_i, Y_i) } \right] \nonumber \\
    \leq& \E\left[ \frac{\max \{\hat{w}(\tilde{X}, \tilde{Y}), \max_i \hat{w}(X_i, Y_i) \}}{\hat{w}(\tilde{X}, \tilde{Y}) + \sum_{i=1}^n \hat{w}(X_i, Y_i) } \mathbbm{1}_{\mathcal{A}^C} \right] + \mathbb{P}(\mathcal{A}) \nonumber \\
    \leq& \E\left[ \frac{2\max \{\hat{w}(\tilde{X}, \tilde{Y}), \max_i \hat{w}(X_i, Y_i) \}}{n} \mathbbm{1}_{\mathcal{A}^C} \right] + \frac{c_1 M_r^2}{n} \nonumber \\
    \leq& \frac{2}{n} \left( \E[\hat{w}(\tilde{X}, \tilde{Y})] + \E \max_i \hat{w}(X_i, Y_i) 
    \right) + \frac{c_1 M_r^2}{n} \nonumber \\
    \leq& \frac{2}{n}\left( \E[\hat{w}(\tilde{X}, \tilde{Y})] + \left(\sum_{i=1}^n \E[\hat{w}(X_i, Y_i)^r]\right)^{1/r}\right) + \frac{c_1 M_r^2}{n} \nonumber \\
    \leq& \frac{2}{n}\left(M_r^2 + n^{1/r}M_r \right) + \frac{c_1 M_r^2}{n}.  \nonumber
\end{align} 
This implies that 
\[
\ttar(Y \in \hat{C}(X)) \leq 1-\alpha + cn^{1/r-1}
\]
for some constant $c$ that only depends on $M_r$ and $r$.
Note that 
\begin{align}
    | \ttar(Y \in \hat{C}(X)) - \tarprob(Y \in \hat{C}(X)) | \leq d_{\textup{TV}}(\tilde{P}^{\pi^*}, P^{\pi^*})  \label{eq6}
\end{align}
where $d_{\textup{TV}}$ is the total variation norm which satisfies
\begin{align}
    d_{\textup{TV}}(\tilde{P}^{\pi^*}, P^{\pi^*}) =& \frac{1}{2} \int | \hat{w}(x,y)\mathrm{d}P^{\pi^b}(x,y) - \mathrm{d}P^{\pi^*}(x,y) | \nonumber \\
    =& \frac{1}{2} \int | \hat{w}(x,y)\mathrm{d}P^{\pi^b}(x,y) - w(x,y) \mathrm{d}P^{\pi^b}(x,y) | \nonumber\\
    =& \frac{1}{2} \expb[|\hat{w}(X,Y) - w(X,Y) |] = \Delta_w. \label{eq7}
\end{align}
Putting together (\ref{eq6}) and (\ref{eq7}), we get
\begin{align}
    \tarprob(Y \in \hat{C}(X)) \leq 1-\alpha + \Delta_w + cn^{1/r-1}.
\end{align}
For the lower bound, using Proposition \ref{coverage_theorem} we get that 
\begin{align}
    \p_{(\tilde{X}, \tilde{Y}) \sim \tilde{P}^{\pi^*}_{X,Y}}(\tilde{Y} \in \hat{C}(\tilde{X}, \tilde{Y})) =& \mathbb{P}\left(s(\tilde{X}, \tilde{Y}) \leq \text{Quantile}_{1-\alpha}\left(\sum_{i=1}^n p_i^{\hat{w}}(\tilde{X}, \tilde{Y})\delta_{V_i} + p_{n+1}^{\hat{w}}(\tilde{X}, \tilde{Y})\delta_\infty \right)\right) \nonumber\\
    \geq 1-\alpha.
\end{align}
Using (\ref{eq6}) we thus obtain 
\begin{align}
    \tarprob(Y \in \hat{C}(X)) \geq& \ttar(Y \in \hat{C}(X)) - d_{TV}(\tilde{P}^{\pi^*}, P^{\pi^*}) \nonumber \\
    \geq& 1-\alpha - \Delta_w.
\end{align}
\end{proof}

\subsection{Proof of Proposition \ref{conditional-res}}
For notational convenience, we suppress the subscripts $m$ and $n$ in $\hat{q}, \hat{w}, \hat{C}$. Moreover, we use $\hat{w}_i$ to denote $\hat{w}(X_i, Y_i)$ and $\eta(x, y)$ to denote $\textup{Quantile}_{1-\alpha}(\sum_{i=1}^n \hat{p}_i(x, y)\delta_{V_i} + \hat{p}_{n+1}(x, y) \delta_{\infty})$.

\begin{proof}

We use $(\tilde{X}, \tilde{Y}) \sim P^{\pi^*}_{X,Y}$ in place of $(X_{n+1}, Y_{n+1})$ and let $\epsilon < r/2$. By the definition of $\hat{C}(\tilde{X})$, we directly have
\begin{align}
    \p(\tilde{Y} \in \hat{C}(\tilde{X}) \mid \tilde{X}) =& \p(s(\tilde{X}, \tilde{Y}) \leq \eta(\tilde{X}, \tilde{Y}) \mid \tilde{X}) \nonumber \\
    \geq& \p(s^*(\tilde{X}, \tilde{Y}) \leq \eta(\tilde{X}, \tilde{Y}) - H(\tilde{X}) \mid \tilde{X}) \label{dr-e1}
\end{align}
where $ s^*(\tilde{X}, \tilde{Y}) \coloneqq \max \{\tilde{Y} - q_{\alpha_{hi}}(\tilde{X}), q_{\alpha_{lo}}(\tilde{X}) - \tilde{Y} \}$ and the probability is taken over $\{(X_i, Y_i)\}_{i=1}^n\overset{\textup{i.i.d.}}{\sim} P^{\pi^b}_{X, Y}$ and $\tilde{Y} \sim P^{\pi^*}_{Y|X=\tilde{X}}$. We then get 
\begin{align}
    \eqref{dr-e1} \geq& \p(s^*(\tilde{X}, \tilde{Y}) \leq -\epsilon - H(\tilde{X}) \mid \tilde{X}) - \p(\eta(\tilde{X}, \tilde{Y}) < -\epsilon \mid \tilde{X}) \nonumber \\
    \geq& \p(s^*(\tilde{X}, \tilde{Y}) \leq -\epsilon - H(\tilde{X}) \mid \tilde{X}) \left(\mathbbm{1}(H(\tilde{X}) \leq \epsilon) + \mathbbm{1}(H(\tilde{X}) > \epsilon)\right) - \p(\eta(\tilde{X}, \tilde{Y}) < -\epsilon \mid \tilde{X}) \label{dr-eq2} \\
    \geq& \left(\p(s^*(\tilde{X}, \tilde{Y}) \leq 0 \mid \tilde{X}) - b_2 \{ \epsilon + H(\tilde{X})\}\right)\mathbbm{1}(H(\tilde{X}) \leq \epsilon)\nonumber \\ 
    & + \p(s^*(\tilde{X}, \tilde{Y}) \leq -\epsilon - H(\tilde{X}) \mid \tilde{X})\mathbbm{1}(H(\tilde{X}) > \epsilon) - \p(\eta(\tilde{X}, \tilde{Y}) < -\epsilon \mid \tilde{X})   \label{dr-eq3} \\
    \geq& \p(s^*(\tilde{X}, \tilde{Y}) \leq 0 \mid \tilde{X})\mathbbm{1}(H(\tilde{X}) \leq \epsilon) - b_2 \{ \epsilon + H(\tilde{X})\mathbbm{1}(H(\tilde{X}) \leq \epsilon)\} \nonumber \\ 
    &+ \left(\p(s^*(\tilde{X}, \tilde{Y}) \leq 0 \mid \tilde{X}) - \p(s^*(\tilde{X}, \tilde{Y}) \in (-\epsilon - H(\tilde{X}), 0))\right)\mathbbm{1}(H(\tilde{X}) > \epsilon) \nonumber\\
    &- \p(\eta(\tilde{X}, \tilde{Y}) < -\epsilon \mid \tilde{X})\nonumber\\
    \geq& \p(s^*(\tilde{X}, \tilde{Y})\leq 0 \mid \tilde{X} ) - b_2 \{ \epsilon + H(\tilde{X}) \mathbbm{1}(H(\tilde{X}) \leq \epsilon ) \} - \mathbbm{1}(H(\tilde{X}) > \epsilon)\nonumber\\
    &- \p(\eta(\tilde{X}, \tilde{Y}) < -\epsilon \mid \tilde{X}) \label{dr-eq4-add}
\end{align}
where, to get from \eqref{dr-eq2} to \eqref{dr-eq3}, we use the condition $2\epsilon < r$ and Assumption 2

\begin{align}
    \eqref{dr-eq4-add} \geq& \p(s^*(\tilde{X}, \tilde{Y})\leq 0 \mid \tilde{X} ) -  b_2 \{ \epsilon + H(\tilde{X})\} - \mathbbm{1}(H(\tilde{X}) > \epsilon) - \p(\eta(\tilde{X}, \tilde{Y}) < -\epsilon \mid \tilde{X}) \label{dr-eq4} \\
    =& 1 - \alpha -  b_2 \{ \epsilon + H(\tilde{X})\} - \mathbbm{1}(H(\tilde{X}) > \epsilon) - \p(\eta(\tilde{X}, \tilde{Y}) < -\epsilon \mid \tilde{X}). \label{dr-eq5}
\end{align}
Next, we derive an upper bound on $\p(\eta(\tilde{X}, \tilde{Y}) < -\epsilon \mid \tilde{X})$. Let $G$ denote the CDF of the random distribution $\sum_{i=1}^n \hat{p}_i(x, y)\delta_{V_i} + \hat{p}_{n+1}(x, y) \delta_{\infty}$. Then, $\eta(\tilde{X}, \tilde{Y}) < -\epsilon$ implies $G(-\epsilon) \geq 1-\alpha$ and thus $\p(\eta(\tilde{X}, \tilde{Y}) < -\epsilon \mid \tilde{X}) \leq \p(G(-\epsilon) \geq 1-\alpha \mid \tilde{X})$ a.s.
Moreover, we have
\begin{align}
    \p(G(-\epsilon) \geq 1 - \alpha \mid \tilde{X}) =& \p \left( \frac{\sum_{i=1}^n \hat{w}_i \mathbbm{1}(V_i \leq - \epsilon) }{\sum_{i=1}^n \hat{w}_i + \hat{w}(\tilde{X}, \tilde{Y}) } \geq 1 - \alpha \mid \tilde{X} \right) \nonumber \\
    \leq& \p \left( \frac{\sum_{i=1}^n \hat{w}_i \mathbbm{1}(V_i \leq - \epsilon) }{\sum_{i=1}^n \hat{w}_i} \geq 1 - \alpha \mid \tilde{X} \right) \label{dr-eq6} \\ 
    =& \p \left( \frac{\sum_{i=1}^n \hat{w}_i \mathbbm{1}(V_i \leq - \epsilon) }{\sum_{i=1}^n \hat{w}_i} \geq 1 - \alpha \right) \label{dr-eq1}
\end{align}
where, to get from \eqref{dr-eq6} to \eqref{dr-eq1} we use the independence of $\{(X_i, Y_i)\}_{i=1}^n$ and $\tilde{X}$.
Now we observe that
\begin{align}
    \frac{\sum_{i=1}^n \hat{w}_i \mathbbm{1}(V_i \leq - \epsilon) }{n} =& \frac{\sum_{i=1}^n (\hat{w}_i - w_i) \mathbbm{1}(V_i \leq - \epsilon) }{n} + \frac{\sum_{i=1}^n w_i \mathbbm{1}(V_i \leq - \epsilon) }{n}. \nonumber
\end{align}
As $n\rightarrow \infty$, the strong law of large numbers yields
\begin{align}
    \left| \frac{\sum_{i=1}^n (\hat{w}_i - w_i) \mathbbm{1}(V_i \leq - \epsilon) }{n}\right| &\overset{a.s.}{\longrightarrow} \left|\expb\left[ (\hat{w}(X, Y) - w(X,Y))\mathbbm{1}(s(X,Y) \leq -\epsilon)\right]\right| \nonumber \\
    &\leq \expb\left[ |\hat{w}(X, Y) - w(X,Y)|\mathbbm{1}(s(X,Y) \leq -\epsilon)\right] \nonumber \\
    &\leq \expb\left[ |\hat{w}(X, Y) - w(X,Y)|\right] \overset{m \rightarrow \infty}{\longrightarrow} 0
\end{align}
from Assumption 1 and
\begin{align}
    \frac{\sum_{i=1}^n w_i \mathbbm{1}(V_i \leq - \epsilon) }{n} \overset{a.s.}{\longrightarrow} \expb[w(X,Y) \mathbbm{1}(s(X,Y) \leq -\epsilon)] = \tarprob(s(X,Y) \leq -\epsilon).
\end{align}
Using the triangle inequality,
\begin{align}
    \tarprob(s(X,Y) \leq -\epsilon) &\leq \tarprob(s^*(X,Y) \leq -\epsilon/2) + \p(H(X) \geq \epsilon/2) \label{dr-eq7} \\
    &\leq \tarprob(s^*(X,Y) \leq 0) - \epsilon b_1/2 + 2^k\mathbb{E}[H^k(X)]/\epsilon^k \label{dr-eq8} \\
    &= 1- \alpha - \epsilon b_1/2 + 2^k\mathbb{E}[H^k(X)]/\epsilon^k \overset{m \rightarrow \infty}{\longrightarrow} 1- \alpha - \epsilon b_1/2. \nonumber
\end{align}
To get from \eqref{dr-eq7} to \eqref{dr-eq8}, we use Assumption 2 and Markov's inequality. Similarly, we have

\begin{align}
    \frac{\sum_{i=1}^n \hat{w}_i}{n} =& \frac{\sum_{i=1}^n (\hat{w}_i - w_i)}{n} + \frac{\sum_{i=1}^n w_i }{n} \nonumber
\end{align}
so, as $n \rightarrow \infty$, 
\begin{align}
    \left| \frac{\sum_{i=1}^n (\hat{w}_i - w_i) }{n}\right| &\overset{a.s.}{\longrightarrow} \left|\expb\left[ (\hat{w}(X, Y) - w(X,Y))\right]\right| \nonumber\\
    &\leq \expb\left[ |\hat{w}(X, Y) - w(X,Y)|\right] \overset{m \rightarrow \infty}{\longrightarrow} 0,
\end{align}
and
\begin{align}
    \frac{\sum_{i=1}^n w_i }{n} \overset{a.s.}{\longrightarrow} \expb[w(X,Y)] = 1.
\end{align}

Putting this all together using the continuous mapping theorem, we get that, almost surely,

\begin{align}
\lim_{m \rightarrow \infty} \lim_{n \rightarrow \infty} \frac{\sum_{i=1}^n \hat{w}_i \mathbbm{1}(V_i \leq - \epsilon) }{\sum_{i=1}^n \hat{w}_i} = \lim_{m \rightarrow \infty} \lim_{n \rightarrow \infty} \frac{\sum_{i=1}^n \hat{w}_i \mathbbm{1}(V_i \leq - \epsilon)/n }{\sum_{i=1}^n \hat{w}_i/n} = 1- \alpha - \epsilon b_1/2.
\end{align}
Since convergence almost surely implies convergence in probability, we have
\begin{align}
    \lim_{m \rightarrow \infty} \lim_{n \rightarrow \infty} \p \left( \frac{\sum_{i=1}^n \hat{w}_i \mathbbm{1}(V_i \leq - \epsilon) }{\sum_{i=1}^n \hat{w}_i} \geq 1 - \alpha \right) = 0.
\end{align}
This implies that, for any $\epsilon > 0$, $\lim_{m \rightarrow \infty} \lim_{n \rightarrow \infty} \p(\eta(\tilde{X}, \tilde{Y}) < -\epsilon \mid \tilde{X}) = 0$ almost surely.

Using Markov's inequality and Assumption 3
\begin{align}
    \p(H(X) > \epsilon) \leq \mathbb{E}[H^k(X)]/\epsilon^k \overset{m \rightarrow \infty}{\longrightarrow} 0.
\end{align}
So as $m\rightarrow \infty$, $H(X)\overset{\mathcal{P}}{\rightarrow} 0$. Similarly, $\mathbbm{1}(H(X) > \epsilon) \overset{\mathcal{P}}{\rightarrow} 0$ as $m\rightarrow \infty$.

Recall (using \ref{dr-eq5}) that, for any $\epsilon \in (0, r/2)$, almost surely,
\begin{align}
    \p(\tilde{Y} \in \hat{C}(\tilde{X}) \mid \tilde{X}) - (1-\alpha -b_2 \epsilon) \geq - b_2 H(\tilde{X}) - \mathbbm{1}(H(\tilde{X}) > \epsilon) - \p(\eta(\tilde{X}, \tilde{Y}) < -\epsilon \mid \tilde{X}).
\end{align}
For given $t > 0$, pick $\epsilon < \min(r/2, t/2b_2)$. Then,
\begin{align}
    \p(\tilde{Y} \in \hat{C}(\tilde{X}) \mid \tilde{X}) - (1-\alpha -t/2) \geq - b_2 H(\tilde{X}) - \mathbbm{1}(H(\tilde{X}) > \epsilon) - \p(\eta(\tilde{X}, \tilde{Y}) < -\epsilon \mid \tilde{X}). \label{dr-eq10}
\end{align}
Each term on the right hand side of \eqref{dr-eq10} converges in probability to 0 as $m, n \rightarrow \infty$, and therefore using continuous mapping theorem  
$$ 
b_2 H(\tilde{X}) + \mathbbm{1}(H(\tilde{X}) > \epsilon) + \p(\eta(\tilde{X}, \tilde{Y}) < -\epsilon \mid \tilde{X}) \overset{\mathcal{P}}{\rightarrow} 0.
$$
This implies
\begin{align}
    &\p(\p(\tilde{Y} \in \hat{C}(\tilde{X}) \mid \tilde{X}) \leq 1-\alpha -t) \nonumber\\
    &\quad \leq \p(b_2 H(\tilde{X}) + \mathbbm{1}(H(\tilde{X}) > \epsilon) + \p(\eta(\tilde{X}, \tilde{Y}) < -\epsilon \mid \tilde{X}) \geq t/2) \rightarrow 0. \nonumber
\end{align}
Therefore, 
\begin{align}
    \lim_{m \rightarrow \infty} \lim_{n \rightarrow \infty} \p(\p(\tilde{Y} \in \hat{C}(\tilde{X}) \mid \tilde{X}) \leq 1-\alpha -t) = 0.
\end{align}
\end{proof}

\newpage

\section{Conformal Off-Policy Prediction (COPP)}

\subsection{Further comments on the differences between \cite{lei2020conformal} and COPP}\label{sec:comp_lc}
In this subsection, we elaborate on the differences between our work and \cite{lei2020conformal}.

Firstly, \cite{lei2020conformal} consider a setup in which the distribution of $X$ is shifted, and construct intervals on the outcome under a specific (deterministic) action, i.e.\ $Y(a)$. In contrast, we consider a setup in which the distribution of $Y|X$ is shifted due to a change in the policy which is non trivial, and construct bounds on the outcome under this new policy (which could be stochastic). Additionally, since the theory in our methodology relies on the ratio of the joint distribution $P_{X,Y}$, our framework can be straightforwardly extended to the case where both, the conditional $P_{Y|X}$ and the covariate distribution $P_X$ shift.

Secondly, as already mentioned in section \ref{sec:related_work}, \cite{lei2020conformal} can only be applied to the case where we have a deterministic target policy and a discrete action space, whereas COPP generalizes to the stochastic policy and continuous action space. This limitation of \cite{lei2020conformal} can be partially addressed by employing the ``\textit{union method}'' as described in the main text, which consists of constructing CP intervals for each action separately before taking the union of the intervals. However, we showed in our experiments that this leads to overly conservative intervals i.e. coverage above the required $1-\alpha$ in Table \ref{tab:coverage_toy}. This is because the predictive interval does not depend on the target policy, since every action is treated identically when taking the union. This approach is moreover unsuitable for continuous action spaces, whereas COPP applies without modification.

Thirdly, as stated in in section \ref{sec:related_work}, even in the case when we only consider deterministic target policies, there is an important methodological difference between COPP and \cite{lei2020conformal}. \cite{lei2020conformal} construct the intervals on $Y(a)$ by only using calibration data with $A=a$ (see eq. 3.4 in \cite{lei2020conformal}). In contrast, it can be shown that COPP uses the entire calibration data when constructing intervals on $Y(a)$. This is a consequence of integrating out the actions in the weights $w(x, y)$ (sec \ref{sec:weights}). This empirically leads to smaller variance in coverage compared to \cite{lei2020conformal} as evidenced by the experimental results in \ref{subsec:comp_lc}.

Finally, in our paper we are \emph{not} interested in a linear combination of the $Y(a)$ as in \cite{lei2020conformal}, who consider the linear combination of the form $Y(1)-Y(0)$. Instead, as described in section \ref{sec:problem_setup}, we are interested in the outcome $Y$ under the new target policy $\pi^*$ (sometimes denoted as $Y(\pi^*)$ in the literature), which cannot be expressed as a linear combination of $Y(a)$. As a result, there does not appear to be a straightforward application of \cite[Section 4.3]{lei2020conformal} to our setup which relies on the linear combination assumption to be applicable.

\begin{figure*}[t]
    \centering
    \includegraphics[width=0.45\textwidth, height=0.3\textwidth]{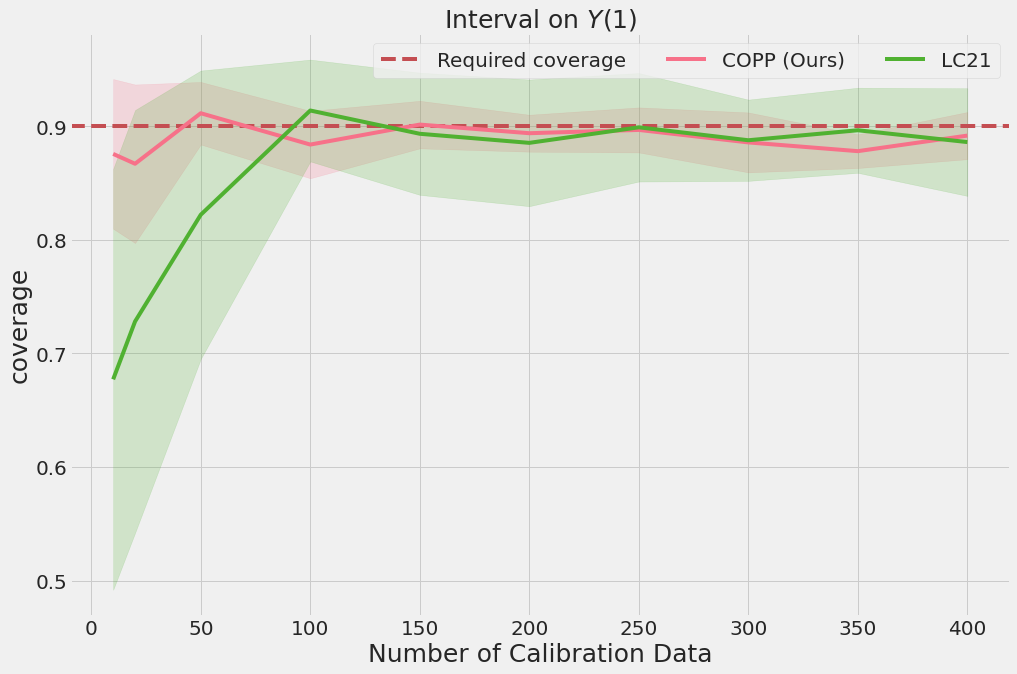}
    \includegraphics[width=0.45\textwidth, height=0.3\textwidth]{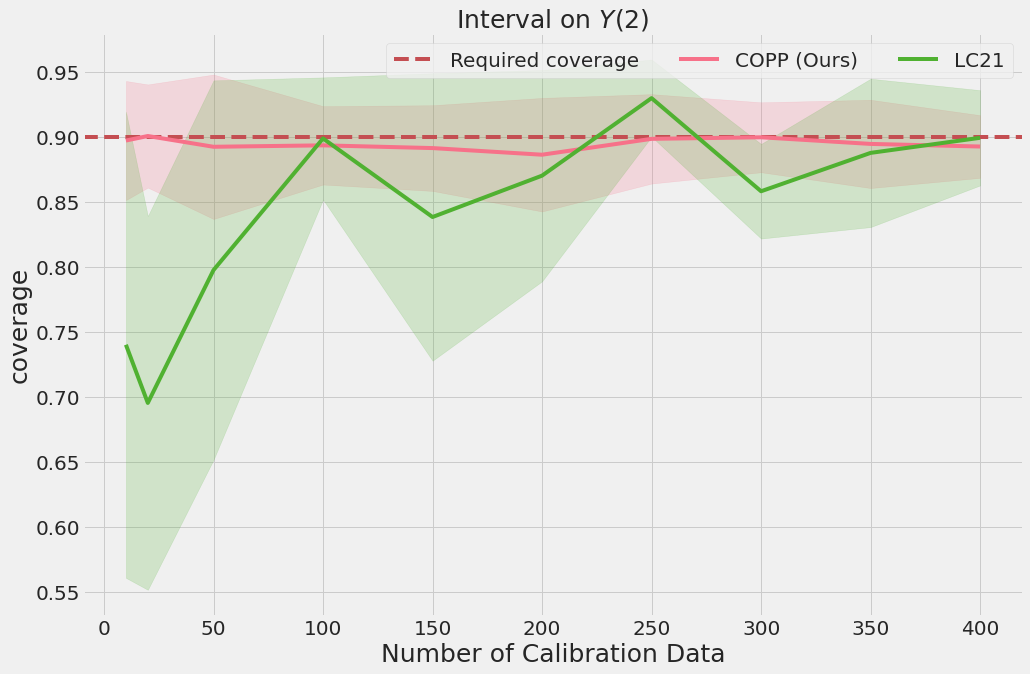}\\
    \includegraphics[width=0.45\textwidth, height=0.3\textwidth]{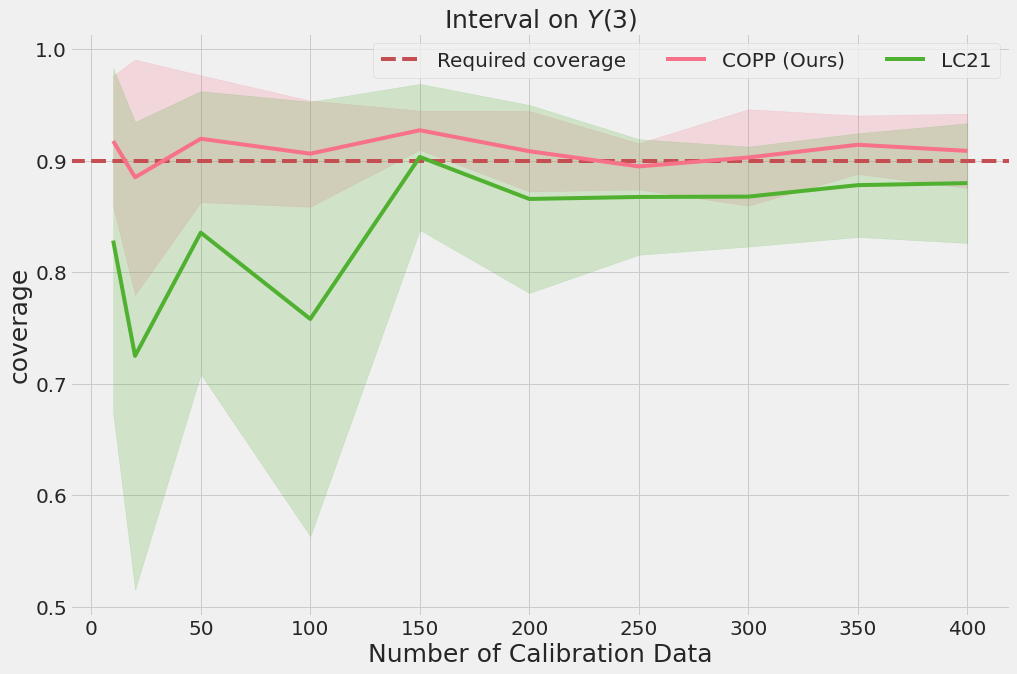}
    \includegraphics[width=0.45\textwidth, height=0.3\textwidth]{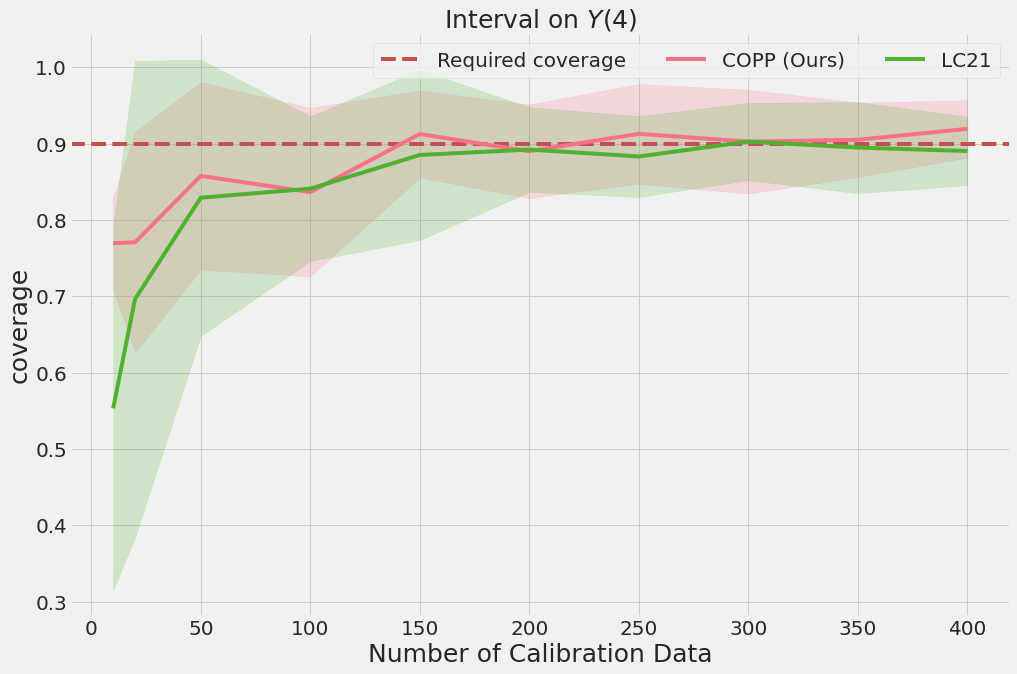}
    \caption{Results for synthetic data experiment with $\pi^b = \pi_{0.3}$ and deterministic target policies.}
    \label{fig:comp_lc}
\end{figure*}
\subsection{Comparison with \cite{lei2020conformal} on deterministic target policies.} \label{subsec:comp_lc}
In order to further clarify the distinction between COPP and \cite{lei2020conformal}, we conducted additional experiments when the target policy is deterministic i.e. $\pi^*(A|X) = \mathbbm{1}\{A=a\}$. In the main text we modified \cite{lei2020conformal} to our setting of stochastic policies by constructing the conformal intervals through the union of the CP sets across the actions. Here we aim to apply COPP to the setting of \cite{lei2020conformal}, i.e. deterministic target policy.

As mentioned in in the main text, given that we are integrating out the action in Eq. \ref{weight-est}, we are essentially able to use the full dataset when constructing the CP intervals. To see this explicitly, consider the case where $Y \mid X, A$ is a normal random variable (as in our toy experiment). In this case, it can be straightforwardly shown that the weights $w(x_i, y_i)$ will be non-zero, and therefore, when constructing the COPP intervals using \eqref{score-dist-pshift}, we are able to use all the calibration datapoints.

This is contrary to \cite{lei2020conformal}, who only consider calibration data with $A=a$, when constructing the CP intervals for $Y(a)$. Below, we use the same experimental setup as our toy experiment in section \ref{sec:exp_toy} (see section \ref{sec:toy_experiments_descrip} for more details) with the difference here that we now consider deterministic target policies. In figure \ref{fig:comp_lc} we plot the coverage for given deterministic target policies against the number of calibration datapoints. In this figure, we refer to the methodology of \cite{lei2020conformal} as \emph{LC21}. Here, we use the behavioural policy $\pi_{0.3}$ and a deterministic target policy which takes a single fixed action $a \in \{1, 2, 3, 4\}$ at test time. In the title of each subfigure, $Y(a)$ corresponds to the outcome for the target policy $\pi^*(A=a\mid X) = \mathbbm{1}(A=a)$.

\textbf{Results}: We first note in Figure \ref{fig:comp_lc} that the coverage of COPP intervals has a lower variance than \cite{lei2020conformal}. Given that COPP is able to use all the data when constructing the CP intervals, as opposed to \cite{lei2020conformal} which only uses a subset, our bounds have lower variance while also attaining the coverage guarantees. We observe this difference particularly in the case when we have little calibration data. Given that \cite{lei2020conformal} have to split the data into $4$ different splits (we have 4 different actions), the calibration data for each action is relatively small, whereas we are able to use the whole dataset to construct our CP intervals.

\subsection{Motivation of using stochastic policies for bandits}
One of the key difference between our method and that of \cite{lei2020conformal} is that our method can be applied to the setting where the target policy is stochastic. In many settings, deterministic target policies might not be applicable such as in the settings of recommendation systems or RL where exploration is needed \citep{swaminathan2016off, su2020doubly}. For example, COPP can be used to compare different recommendation systems given some logged data. We explore this application in our MSR experiments where the target policies correspond to different recommendation systems which are, by default, stochastic. Other applications which also make use of stochastic policies bandit problems can be found in \cite{su2020doubly, farajtabar2018more}.

\subsection{COPP for Group-balanced coverage}\label{sec:grp-bal}
As \cite{conf-bates} point out, we may want predictive intervals that have same error rates across multiple different groups. Using our example of a recommendation system, we may want the predictive intervals to have same coverage across male and female users. 

Formally, this problem can be expressed as follows. Let $\Omega = \{\Omega_1, \cdots, \Omega_k \}$ be subsets of $\mathcal{X} \times \mathcal{Y}$ with $\tarprob((X,Y) \in \Omega_j) > 0$ for $j\in \{1, \dots, k\}$. We would like to construct  predictive intervals $\hat{C}_n^\Omega$ which satisfy 
\begin{align}
    \tarprob(Y \in \hat{C}_n^\Omega(X) \mid (X, Y) \in  \Omega_j) \geq 1-\alpha \hspace{0.2cm} \textup{for all $j\in \{1, \dots, k\}$.} \nonumber
\end{align}
CP offers us the ability to construct such intervals $\hat{C}_n^\Omega$, by simply running algorithm \ref{cp_covariate_shift} (main text) on each group separately. This has been visualized in figure \ref{fig:grps}. 

\begin{figure}[!htp]
    \centering
    \includegraphics[height=0.25\textwidth]{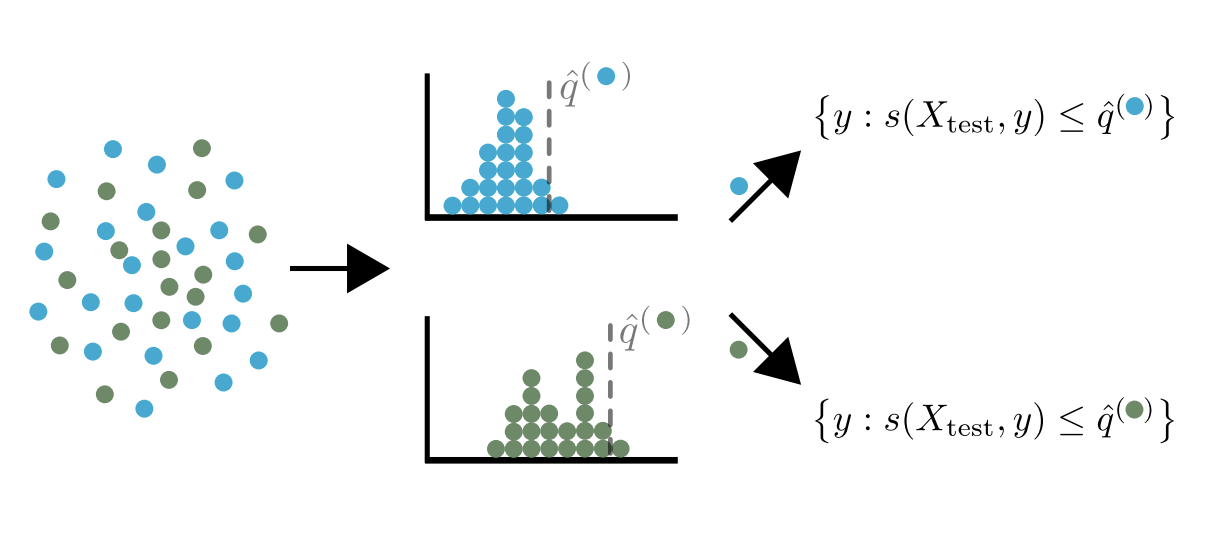}
    \caption{Figure taken from \cite{conf-bates}. To achieve group-balanced coverage, we simply run conformal prediction separately on each group.}
    \label{fig:grps}
\end{figure}
Formally, this procedure can be described as follows. We group scores into different groups according to each subset.
\begin{align}
    \{(X_i^{\Omega_j}, Y_i^{\Omega_j})\}_{i = 1}^{n_j} &\coloneqq \{(X_i, Y_i): (X_i, Y_i) \in \Omega_j\}_{i = 1}^{n} \hspace{0.2cm} \textup{and,} \nonumber \\
    V_i^{\Omega_j} &\coloneqq (X_i^{\Omega_j}, Y_i^{\Omega_j}) \nonumber
\end{align}
Then, within each subset, we calculate the conformal quantile, 
\begin{align}
    \eta^{\Omega_j}(x, y) \coloneqq \textup{Quantile}_{1-\alpha}( \hat{F}^{\Omega_j}_n(x, y)) \nonumber
\end{align}
where,
\begin{align}
    \hat{F}^{\Omega_j}_n(x, y) &\coloneqq \sum_{i=1}^{n_j} p_i^{\Omega_j} (x, y)\delta_{V_i^{\Omega_j}} + p_{n+1}^{\Omega_j}(x,y)\delta_\infty  \hspace{0.2cm} \textup{where,} \nonumber \\
    p_i^{\Omega_j} (x, y) &\coloneqq \frac{w(X^{\Omega_j}_i, Y^{\Omega_j}_i)}{\sum_{i=1}^{n_j} w(X^{\Omega_j}_i, Y^{\Omega_j}_i) + w(x,y)} \nonumber\\
    p_{n+1}^{\Omega_j} (x, y) &\coloneqq \frac{w(x,y)}{\sum_{i=1}^{n_j} w(X^{\Omega_j}_i, Y^{\Omega_j}_i) + w(x,y)} \nonumber
\end{align}
Next, we construct the set $\hat{C}_n^{\Omega}$ as follows:
\begin{align}
    \hat{C}_n^{\Omega}(x^{test}) &\coloneqq \bigcup_{j=1}^k \hat{C}_n^{\Omega_j}(x^{test}) \hspace{0.2cm} \textup{where,} \nonumber \\
    \hat{C}_n^{\Omega_j}(x^{test}) &\coloneqq \{y:  (x^{test}, y) \in \Omega_j \textup{ and } s(x^{test},y) \leq \eta^{\Omega_j}(x^{test}, y)  \}. \nonumber \\
\end{align}
\begin{proposition}[Coverage guarantee for class-balanced conformal prediction]\label{prop:grp-balanced-cp}
Let $\Omega = \{\Omega_1, \cdots, \Omega_k \}$ be subsets of $\mathcal{X} \times \mathcal{Y}$ with $\tarprob((X,Y) \in \Omega_j) > 0$ for $j\in \{1, \dots, k\}$. Then, the set $\hat{C}_n^{\Omega}$ defined above satisfies the coverage guarantee 
\begin{align}
    \tarprob(Y \in \hat{C}_n^\Omega(X) \mid (X, Y) \in  \Omega_j) \geq 1-\alpha \hspace{0.2cm} \textup{for all $j\in \{1, \dots, k\}$.} \nonumber
\end{align}
\end{proposition}

\paragraph{Proof of Proposition \ref{prop:grp-balanced-cp}}

\begin{align}
    &\tarprob(Y \in \hat{C}_n^\Omega(X) \mid (X, Y) \in  \Omega_j) \nonumber\\
    &\geq \tarprob(Y \in \hat{C}_n^{\Omega_j}(X) \mid (X, Y) \in  \Omega_j) \nonumber \\
    &\geq \tarprob( (X, Y) \in \Omega_j: s(X,Y) \leq  \eta^{\Omega_j}(X, Y) \mid (X, Y) \in \Omega_j) \label{eq:prop-gp-bal}  
\end{align}
Define the measure $P^{j}_{X, Y}$ by restricting $P^{\pi^*}_{X, Y}$ to $\Omega_j$, i.e.
\[
P^j_{X, Y}(x, y) \propto P^{\pi^*}_{X, Y}(x, y) \mathbbm{1}((x, y) \in \Omega_j)
\]
Then, \eqref{eq:prop-gp-bal} can be written as
\begin{align}
    \eqref{eq:prop-gp-bal} =& \p_{(X, Y) \sim P^{j}_{X, Y}}(s(X,Y) \leq  \eta^{\Omega_j}(X, Y)) \label{eq1:prop-gp-bal}
\end{align}
Moreover, for $(x, y) \in \Omega_j$ we have
\begin{align}
    w(x, y) = \frac{P^{\pi^*}_{X, Y}(x, y)}{P^{\pi^b}_{X, Y}(x, y)} \propto \frac{P^{j}_{X, Y}(x, y)}{P^{\pi^b}_{X, Y}(x, y)} \nonumber
\end{align}
Since $p_i^{\Omega_j}(x, y)$ is invariant to scaling of weights $w(x, y)$, replacing the weights by $\tilde{w}(x, y) = \frac{P^{j}_{X, Y}(x, y)}{P^{\pi^b}_{X, Y}(x, y)}$ keeps the conformal sets unchanged. 

Therefore, using Proposition \ref{coverage_theorem}, the conformal sets constructed will provide coverage guarantees under the measure $P^{j}_{X, Y}$, i.e.
\begin{align}
    \p_{(X, Y) \sim P^{j}_{X, Y}}(s(X,Y) \leq  \eta^{\Omega_j}(X, Y)) \geq 1-\alpha \nonumber
\end{align}
Using \eqref{eq1:prop-gp-bal}, we get that
\begin{align}
    \tarprob(Y \in \hat{C}_n^\Omega(X) \mid (X, Y) \in  \Omega_j) \geq \p_{(X, Y) \sim P^{j}_{X, Y}}(s(X,Y) \leq  \eta^{\Omega_j}(X, Y)) \geq 1-\alpha \nonumber
\end{align}
\qed

\subsubsection{COPP for class-balanced coverage}
\begin{algorithm}
\SetAlgoLined
\textbf{Inputs:} Observational data $\mathcal{D}_{obs}=\{X_i, A_i, Y_i\}_{i=1}^{n_{obs}}$, conf. level $\alpha$, a score function $s(x,y)\in\mathbb{R}$, new data point $x^{test}$, target policy $\pi^*$ \;
\textbf{Output:} $\hat{C}^{\mathcal{Y}}_n(x^{test})$ with coverage guarantee \eqref{label_cond}\;
Split $\mathcal{D}_{obs}$ into training data ($\mathcal{D}_{tr}$) and calibration data ($\mathcal{D}_{cal}$) of sizes $m$ and $n$ respectively\;
Use $\mathcal{D}_{tr}$ to estimate weights $\hat{w}(\cdot, \cdot)$\;
\For{$y \in\mathcal{Y}$}{
Let $\{X_j^y, Y_j^y\}_{j=1}^{n_y}$ be the following subset of calibration data: $\{(X_i, Y_i): Y_i = y\}$\;
Let $V_j^y \coloneqq s(X_j^y, Y_j^y)$, for $j = 1, \dots, n_y$\;
Define $\hat{F}_{n}^{x, y} = \sum_{i=1}^{n_y} p_i^w(x, y)\delta_{V^y_i} + p_{n+1}^w(x,y)\delta_\infty$\;
where, $p_{i}^w(x, y) \coloneqq \frac{w(X^y_i, Y^y_i)}{\sum_{i=1}^{n_y} w(X^y_i, Y^y_i) + w(x,y)}$, $p_{n+1}^w(x, y) \hspace{-0.1cm} \coloneqq \hspace{-0.1cm} \frac{w(x,y)}{\sum_{i=1}^{n_y} w(X^y_i, Y^y_i) + w(x,y)}$\;
$\eta(x, y) \coloneqq \text{Quantile}_{1-\alpha}( \hat{F}_{n}^{x, y})$
}
Define $\hat{C}^{\mathcal{Y}}_n(x^{test}) \coloneqq  \{y: s(x^{test},y) \leq \eta(x^{test}, y) \} $\;
\textbf{Return} $\hat{C}^{\mathcal{Y}}_n(x^{test})$
\caption{COPP for class-balanced coverage}
  \label{cp_label_conditioned}
\end{algorithm}
In the case when $Y$ is discrete, we construct predictive sets, $\hat{C}^{\mathcal{Y}}_n(x)$, which offer label conditioned coverage guarantees using the methodology described above,
\begin{align}
    \tarprob(Y \in \hat{C}^{\mathcal{Y}}_n(X) \mid Y = y) \geq 1- \alpha, \hspace{0.2cm} \textup{for all $y\in \mathcal{Y}$} \label{label_cond}
\end{align}
This is a strictly stronger guarantee than marginal coverage, i.e. $\tarprob(Y \in \hat{C}_n(X)) \geq 1- \alpha$. To understand what \eqref{label_cond} means, consider our running example of recommendation systems, where the outcome $Y$ is whether the recommendation is relevant (0) or not (1) to the user. Then, Eq. \eqref{label_cond} ensures that out of the users who received irrelevant recommendations, the predictive sets contain `not relevant' (1) at least $100\cdot(1-\alpha)\%$ of the times. This can be thought of as controlling the false negative rate of irrelevant recommendations at $100\cdot\alpha\%$. The same is true for users who receive relevant recommendations. This is particularly useful when data is imbalanced, for example, when the majority of the users in observational receive relevant recommendations. 
\subsection{Weights estimation $\hat{w}(x, y)$}\label{sec:weights_estimation_app}
\subsubsection{Consistent estimation of the weights does not imply consistent estimation of $\hat{P}(y| x, a)$}
In Proposition \ref{coverage_theorem}, we assume to have consistent estimator of $w(x, y)$ which begs the following question: In general, does a consistent estimate of $w(x, y)$ imply that we also obtain a consistent estimate of $P(y|x, a)$? In particular, one could then just use the estimate of $\hat{P}(y|x, a)$ to construct the predictive interval. However, we answer the above question with the negative by supplying a counter-example.

\paragraph{Counter-example}
Let $X \in [1, + \infty), a\in \mathbbm{R} \textup{ s.t. } |a| < K$ for $K \in \mathbbm{R}_{>0}$.

Let $Y|X, a \sim \mathcal{N}((KX^2+a)^{0.5}, (KX^2-a))$.

We have $\mathbbm{E}[Y^2|X, a] = Var(Y|X, a) + \mathbbm{E}[Y|X, a]^2 = KX^2 +a + KX^2 -a = 2KX^2$ (independent of $a$)

Next let 
\begin{align}
  \hat{P}(y|x, a) \coloneqq \frac{y^2P(y|x, a)}{2Kx^2}. \label{def:pyxhat} 
\end{align}
Recall that 
\begin{align}
    w(x, y) = \frac{\int P(y|x, a) \pi^*(a|x) \mathrm{d}a}{\int P(y| x, a) \pi^b(a|x) \mathrm{d}a}
\end{align}
Using the above definition of $\hat{P}(y|x, a)$ we have:
\begin{align}
    \hat{w}(x, y) &= \frac{\int \hat{P}(y|x, a) \pi^*(a|x) \mathrm{d}a}{\int \hat{P}(y| x, a) \pi^b(a|x) \mathrm{d}a} \nonumber\\ 
     &= \frac{\int P(y|x, a) \frac{Y^2}{2KX^2} \pi^*(a|x) \mathrm{d}a}{\int P(y| x, a) \frac{Y^2}{2KX^2} \pi^b(a|x) \mathrm{d}a} \nonumber\\ 
    &= w(x, y).\nonumber
\end{align}
Hence, $w(x, y) \equiv \hat{w}(x, y) \centernot\implies \hat{P}(y|x, a) \equiv P(y|x, a)$.
\qed 

More generally, if there exists a function $\Phi: \mathcal{X}\times \mathcal{Y} \rightarrow \mathbb{R}_{\geq0}$ such that 
\begin{enumerate}
    \item $\Phi(x, y)$ is not constant in $y$
    \item $0<\E[\Phi(X,Y) \mid X, A]< \infty$, and does not depend on $A$
\end{enumerate}
Then, we can define $\tilde{P}(y|x, a) \coloneqq P(y|x, a) \Phi(x, y)/\E[\Phi(X,Y) \mid X, A]$, and the weights computed using $\tilde{P}(y|x, a)$ will be the equal to $w(x, y)$ even though $\tilde{P}(y|x, a) \ne P(y|x, a)$.
\subsubsection{Alternative ways to estimate $\hat{w}(x, y)$ without estimating $\hat{P}(y| x, a)$}\label{sec:alternate_weights_est}
In this section, we show how we could estimate $w(x, y)$ without having to estimate $\hat{P}(y|x, a)$. One way to obtain an estimate $\hat{w}(x, y)$ is by taking a closer look at the definition of $w(x, y)$ and rewriting the ratio.
\begin{align}
    w(x,y)&=\frac{P_{X, Y}^{\pi^*}(x,y)}{P_{X, Y}^{\pi^b}(x,y)} \nonumber\\
    &=\int \frac{P_{X,A, Y}^{\pi^*}(x,a,y)}{P_{X,A, Y}^{\pi^b}(x,a,y)}P_{A|X,Y}^{\pi^b}(a|x,y)\mathrm{d}a \nonumber \\
    &=\int \frac{\pi^{\ast}(a|x)}{\pi^b (a|x)}P_{A|X,Y}^{\pi^b}(a|x,y)\mathrm{d}a \nonumber\\
    &= \E_{A \sim P^{\pi^b}_{A\mid X=x,Y=y}}\Big[ \frac{\pi^{\ast}(A|x)}{\pi^b (A|x)} \Big]. \label{eq:ratioidentity}
\end{align}

\begin{lemma}\label{prop:weights-est}
Let $w(x,y)=\frac{P_{X, Y}^{\pi^*}(x,y)}{P_{X, Y}^{\pi^b}(x,y)}$, then
\begin{align}
    w(x,y) = \arg\min_{f} \mathbb{E}_{X,A,Y \sim P^{\pi^b}_{X,A,Y}} \Big[\Big|\Big|\frac{\pi^{\ast}(A|X)}{\pi^b (A|X)}-f(X,Y)\Big|\Big|^2\Big]. \label{eq:weights-obj}
\end{align}
\end{lemma}
\paragraph{Proof of Lemma \ref{prop:weights-est}}
This follows directly from the identity (\ref{eq:ratioidentity}). We prove it here for sake of completeness.
\begin{align}
    &\mathbb{E}_{X,A,Y \sim P^{\pi^b}_{X,A,Y}} \Big[\Big|\Big|\frac{\pi^{\ast}(A|X)}{\pi^b (A|X)}-f(X,Y)\Big|\Big|^2\Big] \nonumber \\
    &= \mathbb{E}_{X,Y \sim P^{\pi^b}_{X,Y}} \Big[\E_{A \sim P^{\pi^b}_{A\mid X,Y}} \Big|\Big|\frac{\pi^{\ast}(A|X)}{\pi^b (A|X)}-f(X,Y)\Big|\Big|^2\Big] \nonumber \\
    &= \mathbb{E}_{X,Y \sim P^{\pi^b}_{X,Y}} \Big[\textup{Var}_{A \sim P^{\pi^b}_{A\mid X,Y}}\Big[ \frac{\pi^{\ast}(A|X)}{\pi^b (A|X)} \Big] + \left(\E_{A \sim P^{\pi^b}_{A\mid X,Y}}\Big[ \frac{\pi^{\ast}(A|X)}{\pi^b (A|X)} \Big] - f(X,Y) \right)^2 \Big].
     \label{eq:w-reg}
\end{align}
Where, \eqref{eq:w-reg} is minimized if $f(x, y) = \E_{A \sim P^{\pi^b}_{A\mid X=x,Y=y}}\Big[ \frac{\pi^{\ast}(A|x)}{\pi^b (A|x)} \Big] = w(x,y)$.
\qed

Using Lemma \ref{prop:weights-est}, we can thus approximate $w(x,y)$ by minimizing the loss
\begin{align}
    \hat{w}(x, y) =\arg \min_{f_\theta} \mathbb{E}_{X,A,Y \sim P^{\pi^b}_{X,A,Y}} \Big[\Big|\Big|\frac{\pi^{\ast}(A|X)}{\pi^b (A|X)}-f_\theta(X,Y)\Big|\Big|^2\Big] \label{eq:weights-loss}
\end{align}
Hence we see that the ratio estimation problem can be rewritten as a regression problem where $f_\theta(x,y)$ is for example a neural network. This allows one to estimate directly, without the need for estimating $\hat{P}(y\mid x, a)$ first.

\newpage
\section{Estimation of the quantiles of the target distribution}\label{sec:estimating_target_quantiles}
As mentioned in Section \ref{sec:cond_cov}, we present here a way to estimate the quantiles of the target distribution $P_{X,Y}^{\pi^*}$ consistently when the ground truth weight function $w(x, y)$ is known. As we are interested in the quantiles, we will be using the pinball loss to train our model $\hat{f}_{\theta}$ defined by
\begin{align}
    L_{\alpha}(\theta, x, y) = \begin{cases}
         \alpha (\hat{f}_{\theta}(x)-y) \qquad &\text{ if } (\hat{f}_{\theta}(x)-y) > 0, \\
         (1-\alpha) (y -\hat{f}_{\theta}(x)) \qquad &\text{ if } (\hat{f}_{\theta}(x)-y) < 0.
    \end{cases} \nonumber
\end{align}
Then we have the following objective to optimize:
\begin{align}
\expt[L_\alpha(\theta, X, Y)] &= \int_{X,Y} L_\alpha(\theta, x, y) P_{X,Y}^{\pi^*}(\mathrm{d}x,\mathrm{d}y) \nonumber \\
&= \int_{X,Y} L_\alpha(\theta, x, y) \frac{\mathrm{d}P_{X, Y}^{\pi^*}(x, y)}{\mathrm{d}P_{X,Y}^{\pi^b}(x, y)} P_{X,Y}^{\pi^b}(\mathrm{d}x,\mathrm{d} y) \nonumber \\
&= \int_{X,Y} L_\alpha(\theta, x, y) w(x, y) P_{X,Y}^{\pi^b}(\mathrm{d}x,\mathrm{d}y)\nonumber\\
&= \expb[L_\alpha(\theta, X, Y) w(X, Y)]. \nonumber
\end{align}
The above holds true if the true weight function is known. However in the case where we only have a consistent estimator of $w(x, y)$, it remains to be proven that the above objective will also yield a consistent estimator of the quantiles under $\pi^*$. We leave this for future work to prove as we are simply providing a possible avenue to relax the assumptions in Proposition \ref{sec:cond_cov}.

\newpage
\section{Experiments}\label{sec:exps_app}
The code for our experiments is available at \url{https://anonymous.4open.science/r/COPP-75F5} and we ran all our experiments on Intel(R) Xeon(R) CPU E5-2690 v4 @ 2.60GHz with 8GB RAM per core. We were able to use 100 CPUs in parallel to iterate over different configurations and seeds. However, we would like to note that our algorithms only requires 1 CPU and at most 10 mins to run, as our networks are relatively small.
\subsection{Toy Experiment}\label{sec:toy_experiments_descrip}
\subsubsection{Synthetic data experiments setup}
\paragraph{Model.}
The observational data distribution is defined as follows:
\begin{align}
    & X_i \overset{\textup{i.i.d.}}{\sim} \mathcal{N}(0,9) \nonumber \\
    & A_i \mid x_i \sim \pi^b(\cdot \mid x_i) \hspace{0.2cm} \textup{where $\pi^b$ has been defined below} \nonumber \\
    & Y_i \mid x_i, a_i \sim \mathcal{N}(a_i * x_i, 1) \nonumber
\end{align}

\paragraph{Behaviour and Target Policies.}
We define a family of policies $\pi_\epsilon(a \mid x)$ as follows:
\begin{align}
&\pi_\epsilon(a|x) \coloneqq
     \begin{cases}
          \epsilon\mathbbm{1}(a \in \{1,2,3\}) + (1-3\epsilon)\mathbbm{1}(a=4) &  \textup{if } |x|\in (3, \infty)\\
          \epsilon\mathbbm{1}(a \in \{1,2,4\}) + (1-3\epsilon)\mathbbm{1}(a=3) & \textup{if } |x|\in (2, 3]\\
          \epsilon\mathbbm{1}(a \in \{1,3,4\}) + (1-3\epsilon)\mathbbm{1}(a=2) & \textup{if } |x|\in (1, 2]\\
          \epsilon\mathbbm{1}(a \in \{2,3,4\}) + (1-3\epsilon)\mathbbm{1}(a=1) & \textup{if } |x|\in [0, 1]\\
          \end{cases} \nonumber
\end{align}
We use the parameter $\epsilon \in (0,1/3)$ to control the policy shift between target and behaviour policies. For the behaviour policy $\pi^b$, we use $\epsilon^b = 0.3$, and for target policies $\pi^*$, we use $\epsilon^* \in \{0.1, 0.2, 0.3\}$. Here we use $m=1000$ training datapoints.

\paragraph{Neural Network Architectures}
\begin{itemize}
    \item To approximate the behaviour policy $\pi^b$, we use a neural network with 2 hidden layers and 16 nodes in each hidden layer, and ReLU activation function.
    \item To approximate $P(y|x, a)$, we use $\mathcal{N}(\mu(x, a), \sigma(x, a))$, where $\mu$ and $\sigma$ are neural networks with one-hidden layer, 32 nodes in the hidden layer, and ReLU activation function.
    \item For the score function, we train the quantiles $\hat{q}_{\alpha/2}$ and $\hat{q}_{1 - \alpha/2}$ using quantile regression, each of which are modelled using neural networks with one-hidden layer, 32 nodes in the hidden layer, and ReLU activation functions.
\end{itemize}

\paragraph{Results: Coverage as a function of increase calibration data}\label{app:N-cal_exp_toy}
As mentioned in the main text, we have also performed experiments to investigate how much calibration data is needed for COPP as well as other methods to converge to the required $90\%$ coverage. In the below figure \ref{fig:Toy_GT} we have plotted the coverage as a function of $n$ calibration data points. Our proposed method is converging much faster to the required coverage compared to the competing methods.

\begin{figure*}[htp!]
    \centering
    \includegraphics[width=0.45\textwidth, height=0.3\textwidth]{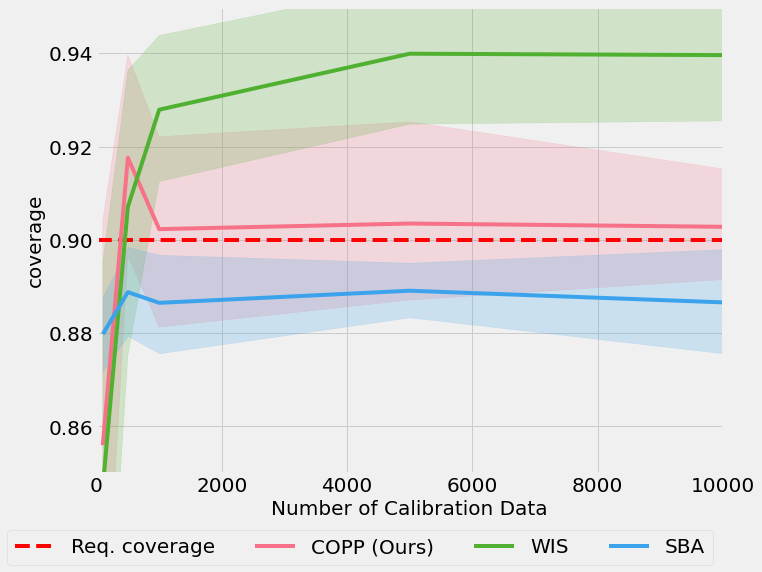}
    \includegraphics[width=0.45\textwidth, height=0.3\textwidth]{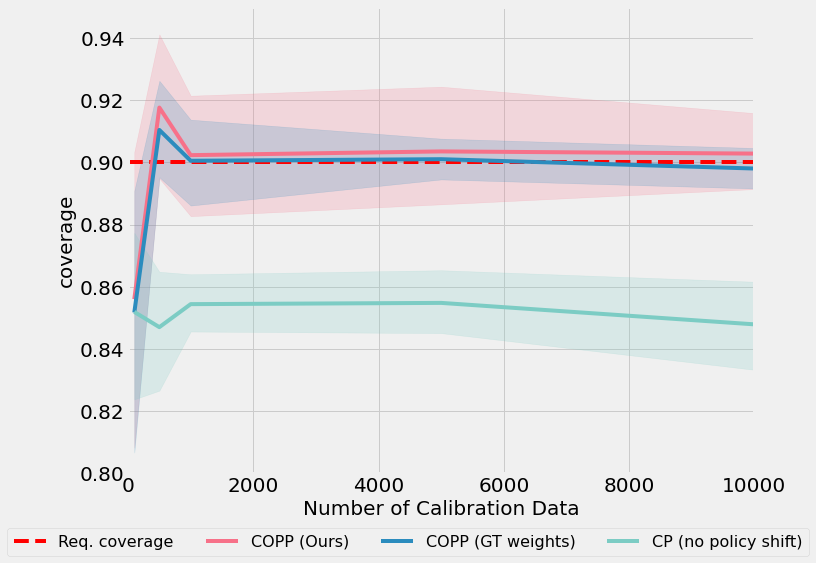}
    \caption{Results for synthetic data experiment with $\pi^b = \pi_{0.3}$ and the target policy is $\pi^* = \pi_{0.1}$. \textbf{Left:} our proposed method is able to converge to the required coverage rather quickly compared to the competing methods. \textbf{Right:} here we see that our method is on par with using the GT weights. Due to estimation error, COPP with estimated weights has slightly higher variance in terms of coverage}
    \label{fig:Toy_GT}
\end{figure*}

\paragraph{Additional experimental baseline using weighted quantile regression.}
In order to add an additional baseline that is also covariate dependent, we have added some experiments using the weighted quantile regression (WQR) as described in Sec. \ref{sec:estimating_target_quantiles} on our toy experiments from Sec. \ref{sec:exp} in the main text. Below in Table \ref{tab:coverage_toy_app} and Table \ref{tab:length_toy_app} we see the complete coverage table with the respective interval lengths. Note also that WQR does not seem to perform well as it does not have any statistical guarantees and heavily relies on good estimation of the ratio. We have added these experiments here in the appendix for completeness and did not add it in the main text as the results were not comparable to other baselines.

\begin{table}[t]
      \centering
      \caption{Mean Coverage as a function of policy shift with 2 standard errors over 10 runs. We have added weighted quantile regression (WQR) for completeness and note that it does not seem to perform well.}\label{tab:coverage_toy_app}
      \resizebox{0.7\columnwidth}{!}{%
        \begin{tabular}{lccc}
\toprule
Coverage &  $\Delta_{\epsilon}=0.0$ &  $\Delta_{\epsilon}=0.1$ &  $\Delta_{\epsilon}=0.2$ \\
\midrule
COPP (Ours)            &                    \textbf{0.90 $\pm$ 0.01}&                    \textbf{0.90 $\pm$ 0.01}&                    \textbf{0.91 $\pm$ 0.01}\\
WIS                  &                    \textbf{0.89 $\pm$ 0.01}&                     \textbf{0.91 $\pm$ 0.02}&                     0.94 $\pm$ 0.02\\
SBA                  &                     \textbf{0.90 $\pm$ 0.01}&                     0.88 $\pm$ 0.01&                     0.87 $\pm$ 0.01\\
\midrule
\midrule
COPP (GT weights Ours)      &                     \textbf{0.90 $\pm$ 0.01}&                     \textbf{0.90 $\pm$ 0.01}&                     \textbf{0.90 $\pm$ 0.01}\\
CP (no policy shift) &                     \textbf{0.90 $\pm$ 0.01}&                     0.87 $\pm$ 0.01&                     0.85 $\pm$ 0.01\\
CP (union) &                      0.96 $\pm$ 0.01 &         0.96 $\pm$ 0.01 &         0.96 $\pm$ 0.01 \\
\red{WQR}          &         \red{0.82 $\pm$ 0.04} &         \red{0.76 $\pm$ 0.03} &          \red{0.70 $\pm$ 0.03} \\
\bottomrule
\end{tabular}
}
\end{table}

\begin{table}[h!]
      \centering
      \caption{Mean Interval Length as a function of policy shift with 2 standard errors over 10 runs. We have added weighted quantile regression (WQR) for completeness and note that it does not seem to perform well.}\label{tab:length_toy_app}
      \resizebox{0.7\columnwidth}{!}{%
        \begin{tabular}{lccc}
\toprule
Interval Lengths &  $\Delta_{\epsilon}=0.0$ &  $\Delta_{\epsilon}=0.1$ &  $\Delta_{\epsilon}=0.2$ \\
\midrule
COPP (Ours)           &                     9.08 $\pm$ 0.10&                     9.48 $\pm$ 0.22&                     9.97 $\pm$ 0.38\\
WIS                  &                    \red{24.14 $\pm$ 0.30}&               \red{32.96 $\pm$ 1.80}&             \red{43.12 $\pm$ 3.49}\\
SBA                  &                     8.78 $\pm$ 0.12&                     8.94 $\pm$ 0.10&                     8.33 $\pm$ 0.09\\
\midrule
\midrule
COPP (GT weights Ours)      &                     8.91 $\pm$ 0.09&                     9.25 $\pm$ 0.12&                     9.59 $\pm$ 0.20\\
CP (no policy shift) &                     9.00 $\pm$ 0.10&                     9.00 $\pm$ 0.10&                     9.00 $\pm$ 0.10\\
CP (union) &                     10.66 $\pm$ 0.18 &         11.04 $\pm$ 0.2 &         11.4 $\pm$ 0.26 \\
\red{WQR}         &         \red{8.55 $\pm$ 0.50} &         \red{8.61 $\pm$ 0.52} &          \red{8.70 $\pm$ 0.55} \\
\bottomrule
\end{tabular}
}
\end{table}

\newpage
\subsubsection{Experiments with continuous action space}\label{subsec:cts_act}
As mentioned in the main text and also in Sec. \ref{sec:comp_lc}, our proposed method, contrary to the work of \cite{lei2020conformal} is able to also handle continuous action space. Given that we are integrating out the actions when computing the weights in Eq. \ref{weight-est} our method trivially extends to the continuous action space, whereas \cite{lei2020conformal} is only applicable for discrete action spaces, as they compute conformal intervals conditioned on a given action.

\paragraph{Model.}
The observational data distribution is defined as follows:
\begin{align}
    & X_i \overset{\textup{i.i.d.}}{\sim} \mathcal{N}(0,4) \nonumber \\
    & A_i \mid x_i \sim \mathcal{N}(x_i/4, 1) \hspace{0.2cm} \nonumber \\
    & Y_i \mid x_i, a_i \sim \mathcal{N}(a_i + x_i, 1) \nonumber
\end{align}
\paragraph{Target Policies.}
We define a family of policies $\pi_\epsilon(a \mid x)$ as follows:
\begin{align}
    \pi_\epsilon(a \mid x) = \mathcal{N}(x/4 + \epsilon, 1). \label{tar_pols}
\end{align}
In our experiments, for the target policy $\pi^*$, we use $\pi^* = \pi_{\epsilon^*}$ for $\epsilon^* \in \{0, 0.5, 1, 1.5, 2, 2.5\}$.

\paragraph{Results.}
Table \ref{tab:cov_cts} shows the coverages of different methods as the policy shift $\epsilon^*$ increases. The behaviour policy $\pi^b = \pi_{0}$ is fixed and we use $n=5000$ calibration datapoints and $m=1000$ training points, across 10 runs. Table \ref{tab:cov_cts} shows, how COPP stays very close to the required coverage of $90\%$ across all target policies with $\epsilon^* \leq 2.0$, compared to WIS and SBA. Both, WIS intervals and SBA intervals suffer from under-coverage i.e. below the required coverage. These results again support our hypothesis from Sec. \ref{sec:weights}, which stated that COPP is less sensitive to estimation errors of $\hat{P}(y|x, a)$ compared to directly using $\hat{P}(y|x, a)$ for the intervals i.e. SBA.

Next, Table \ref{tab:len_cts} shows the mean interval lengths and even though WIS intervals are under-covered, the average interval length is huge compared to COPP. Additionally, for $\epsilon^* \in \{0, 0.5, 1, 1.5\}$, COPP with estimated weights produces results which are close to COPP intervals with ground truth weights. This shows that when the behaviour and target policies have reasonable overlap, the effect of weights estimation error on COPP results is limited. However, as $\epsilon^*$ increases to $2.0$ and $2.5$, the overlap between behaviour and target policies becomes low. We empirically note that this leads to high weights estimation error and consequently under-coverage in COPP with estimated weights. In contrast, COPP with ground truth weights still achieves required coverage, even though it becomes conservative when the overlap is low. Figure \ref{fig:pols_cts_acs} visualises how the overlap between target and behaviour policies decreases with increasing $\epsilon^*$. It can be seen that $\epsilon^* \in \{2, 2.5\}$ leads to very low overlap between the behaviour and target data.

\begin{table}[htp!]
\begin{center}
\caption{Mean Coverage as a function of policy shift with 2 standard errors over 10 runs.}
\label{tab:cov_cts}
\resizebox{0.9\columnwidth}{!}{
\begin{tabular}{lllllll}
\toprule
Coverage &              $\epsilon^*=0.0$ &              $\epsilon^*=0.5$ &              $\epsilon^*=1.0$ &              $\epsilon^*=1.5$ &              $\epsilon^*=2.0$ &              $\epsilon^*=2.5$ \\
\midrule
COPP (Ours)                   &   \textbf{0.90 $\pm$ 0.01} &  \textbf{0.91 $\pm$ 0.01} &  0.92 $\pm$ 0.01 &  \textbf{0.91 $\pm$ 0.01} &  \textbf{0.89 $\pm$ 0.02} &  0.85 $\pm$ 0.02 \\
WIS                           &  0.87 $\pm$ 0.01 &  0.87 $\pm$ 0.01 &  0.87 $\pm$ 0.01 &  0.87 $\pm$ 0.02 &  \textbf{0.89 $\pm$ 0.02} &  0.83 $\pm$ 0.02 \\
SBA                           &  0.86 $\pm$ 0.01 &  0.86 $\pm$ 0.01 &  0.86 $\pm$ 0.01 &  0.86 $\pm$ 0.01 &  \textbf{0.89 $\pm$ 0.02} &  0.83 $\pm$ 0.02 \\
\midrule
\midrule
COPP (GT Weights Ours)             &   \textbf{0.90 $\pm$ 0.01} &  \textbf{0.91 $\pm$ 0.01} &  \textbf{0.91 $\pm$ 0.01} &   \textbf{0.90 $\pm$ 0.01} &  0.96 $\pm$ 0.02 &   0.93 $\pm$ 0.02 \\
CP (no policy shift) &   \textbf{0.90 $\pm$ 0.01} &  0.88 $\pm$ 0.01 &  0.82 $\pm$ 0.01 &  0.73 $\pm$ 0.01 &   0.60 $\pm$ 0.01 &  0.46 $\pm$ 0.01\\
\bottomrule
\end{tabular}
}
\end{center}
\end{table}

\begin{table}[h!]
\begin{center}
\caption{Mean Interval Length as a function of policy shift with 2 standard errors over 10 runs.}
\label{tab:len_cts}
\resizebox{0.9\columnwidth}{!}{
\begin{tabular}{lllllll}
\toprule
Interval Lengths &              $\epsilon^*=0.0$ &              $\epsilon^*=0.5$ &              $\epsilon^*=1.0$ &              $\epsilon^*=1.5$  & $\epsilon^*=2.0$ & $\epsilon^*=2.5$\\
\midrule
COPP (Ours)                   &  4.75 $\pm$ 0.04 &  5.08 $\pm$ 0.09 &  5.89 $\pm$ 0.14 &  6.92 $\pm$ 0.18 &  7.82 $\pm$ 0.41 &  8.45 $\pm$ 0.44\\
WIS                           &   9.55 $\pm$ 0.1 &  9.56 $\pm$ 0.12 &  9.56 $\pm$ 0.27 &  9.44 $\pm$ 0.38 &   9.40 $\pm$ 0.59 &  9.08 $\pm$ 0.64 \\
SBA                           &  4.38 $\pm$ 0.03 &  4.37 $\pm$ 0.03 &  4.36 $\pm$ 0.04 &  4.34 $\pm$ 0.07 &   4.31 $\pm$ 0.1 &  4.28 $\pm$ 0.14 \\
\midrule
\midrule
COPP (GT Weights Ours)             &  4.73 $\pm$ 0.05 &  5.07 $\pm$ 0.09 &  5.87 $\pm$ 0.14 &  6.82 $\pm$ 0.13 &  7.57 $\pm$ 0.19 &  8.07 $\pm$ 0.22 \\
CP (no policy shift) &   4.70 $\pm$ 0.05 &   4.70 $\pm$ 0.05 &   4.70 $\pm$ 0.05 &   4.70 $\pm$ 0.05 &   4.70 $\pm$ 0.05 &   4.70 $\pm$ 0.05 \\
\bottomrule
\end{tabular}
}
\end{center}
\end{table}

\begin{figure*}[t]
    \centering
    \includegraphics[width=0.45\textwidth, height=0.3\textwidth]{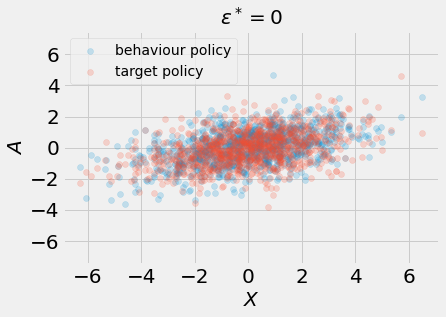}
    \includegraphics[width=0.45\textwidth, height=0.3\textwidth]{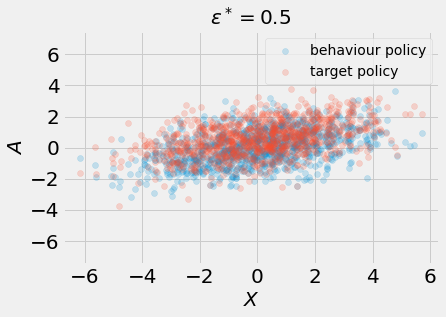}\\
    \includegraphics[width=0.45\textwidth, height=0.3\textwidth]{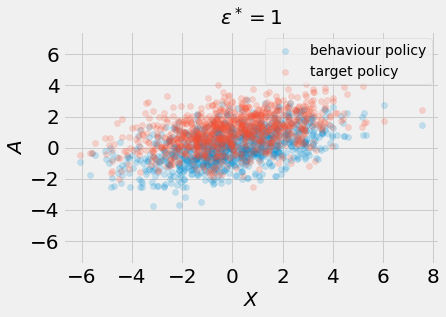}
    \includegraphics[width=0.45\textwidth, height=0.3\textwidth]{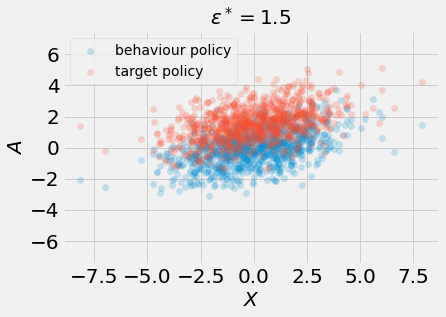}\\
    \includegraphics[width=0.45\textwidth, height=0.3\textwidth]{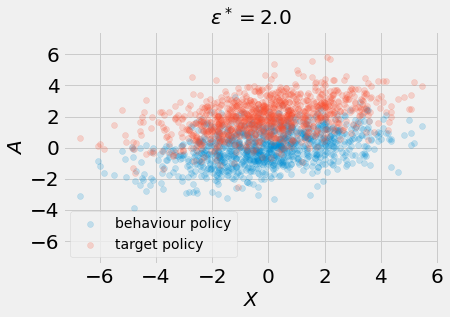}
    \includegraphics[width=0.45\textwidth, height=0.3\textwidth]{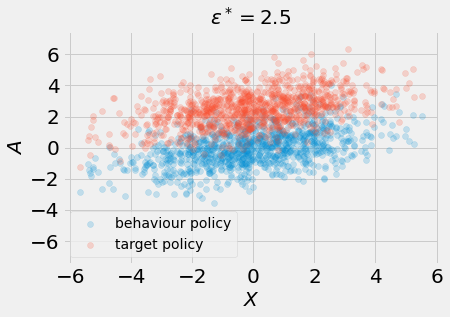}
    \caption{Plots of $A$ against $X$, where $X \sim \mathcal{N}(0, 4)$ and $A\mid X$ is sampled from behaviour and target policies. Here, target policies are defined in \eqref{tar_pols} for $\epsilon^* \in \{0, 0.5, 1, 1.5, 2, 2.5\}$.}
    \label{fig:pols_cts_acs}
\end{figure*}

\newpage

\subsection{Experiments on Microsoft Ranking Dataset}\label{sec:MSR_experiments_decrip}

\paragraph{Dataset details.}
The dataset contains relevance scores for websites recommended to different users, and comprises of $30,000$ user-website pairs. For a user $i$ and website $j$, the data contains a $136$-dimensional feature vector $u_i^j$, which consists of user $i$'s attributes corresponding to website $j$, such as length of stay or number of clicks on the website. Furthermore, for each user-website pair, the dataset also contains a relevance score, i.e. how relevant the website was to the user.

First, given a user $i$ we sample (with replacement) $5$ websites,  $\{u_i^j\}_{j=1}^5$, from the data. Next, we reformulate this into a contextual bandit where $A \in \{1,2,3,4,5\}$ corresponds to the website we recommend to a user. For a user $i$, we define $X$ by combining the $5$ feature vectors corresponding to the user, i.e. $X \in \mathbb{R}^{5 \times 136}$, where $x_i = (u^1_{i},u^2_{i},u^3_{i},u^4_{i}, u^5_{i})$. In addition, $Y \in\{0,1,2,3,4\}$ corresponds to the relevance score for the $A$'th website, i.e. the recommended website. The goal is to construct prediction sets that are guaranteed to contain the true relevance score with a probability of $90\%$. Here we use $m=5000$ training data points.

\paragraph{Behaviour and Target Policies.}
We first train a Neural Network (NN) classifier model mapping each 136-dimensional feature vector to the
softmax scores for each relevance score class, $\hat{f}_\theta:\mathcal{U} \rightarrow [0,1]^5$. We use this trained model $\hat{f}_\theta$ to define a family of policies such that we pick the most relevant website as predicted by $\hat{f}_\theta$ with probability $\epsilon$ and the rest uniformly with probability $(1-\epsilon)/4$. Formally, this has been expressed as follows. We use $\hat{f}^{\textup{label}}_\theta$ to denote the relevance class predicted by $\hat{f}_\theta$, i.e. $\hat{f}^{\textup{label}}_\theta(u) \coloneqq \argmax_i\{\hat{f}_\theta(u)_i\}$. 

Then,
\begin{align}
    \pi_\epsilon (a\mid X=(u^1, u^2,u^3,u^4,u^5)) \coloneqq& \epsilon \mathbbm{1}(a = \argmax_j\{ \hat{f}^{\textup{label}}_\theta(u^j) \}) \nonumber \\
    &+ (1-\epsilon)/4 \mathbbm{1}(a \neq \argmax_j\{ \hat{f}^{\textup{label}}_\theta(u^j) \}) \nonumber
\end{align}

\paragraph{Estimation of ratios, $\hat{w}(X, Y)$.}
To estimate the $\hat{P}(y \mid x, a)$ we use the trained model $\hat{f}_\theta$ as follows:
\[
\hat{P}(y \mid x = (u^1, u^2,u^3,u^4,u^5), a) = \hat{f}_\theta(u^a)_y
\]
where $\hat{f}_\theta(u^a)_y$ corresponds to the softmax prediction of $u^a$ for label $y$ under the model $\hat{f}_\theta$. To estimate the behaviour policy $\hat{\pi}^b$, we train a classifier model $\mathcal{X} \rightarrow \mathcal{A}$ using a neural network. We use \eqref{weight-est} to estimate the weights $\hat{w}(x, y)$.

\paragraph{Neural Network Architectures}
\begin{itemize}
    \item To approximate the behaviour policy, we use a neural network with 2 hidden layers and 25 nodes in each hidden layer, ReLU activations and softmax output.
    \item To approximate $\hat{f}_{\theta}$, we use a neural network with 2 hidden layers with 64 nodes each and ReLU activations.
\end{itemize}

\paragraph{Results: Coverage as a function of increase calibration data.}\label{app:N-cal_exp_msr}
As mentioned in the main text, we have also performed experiments to investigate how much calibration data is needed for COPP as well as other methods to converge to the required $90\%$ coverage. In the below plot we have plotted the coverage as a function of $n$ calibration data points. We observe that our proposed method is converging much faster to the required coverage compared to the competing methods.

\begin{figure*}[htp!]
    \centering
    \includegraphics[width=0.5\textwidth]{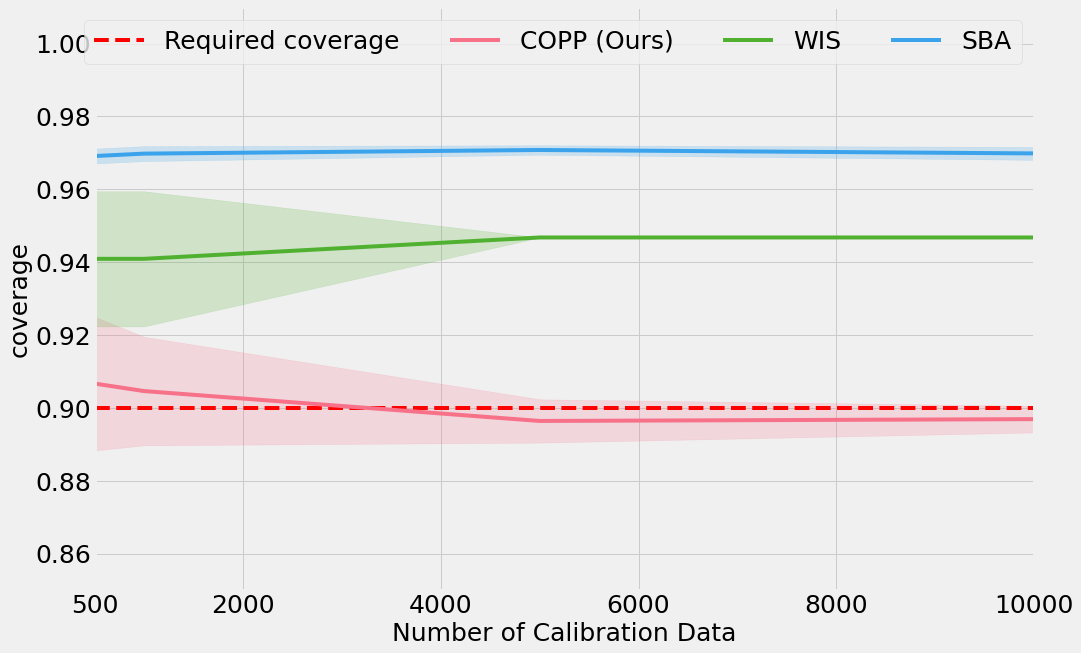} %
    \label{fig:ncal-msr}
    \caption{Results of Microsoft Ranking Dataset experiment with behaviour policy $\pi^b = \pi_{0.5}$ and the target policy is $\pi^* = \pi_{0.2}$. Our proposed method is able to converge to the required coverage rather quickly compared to the competing methods}
    \label{fig:msr}
\end{figure*}
\newpage
\subsubsection{Results: COPP for Class-balanced coverage}\label{sec:results_class_bal_coverage}

Table \ref{tab:label-cond} shows the coverages of COPP predictive sets ($\hat{C}_n$ with marginal coverage guarantee constructed using algorithm \ref{cp_covariate_shift}) and COPP intervals with label conditioned coverage ($\hat{C}^{\mathcal{Y}}_n$ satisfying \eqref{label_cond} constructed using algorithm \ref{cp_label_conditioned}). Extensions of WIS and SBA to the conditional case are not straightforward and hence have not been included. For $\hat{C}_n$, while the overall coverage is very close to the required coverage of $90\%$, we see that there is under-coverage for $Y = 0,1,2,3$. This can be explained by the data imbalance -- the number of test data points with $Y = 0,1,2,3$ is significantly lower than $Y=4$. 

This under-coverage problem disappears in $\hat{C}^{\mathcal{Y}}_n$. Instead, in cases where number of data points is small, ($Y = 0,1,2,3$), the predictive sets $\hat{C}^{\mathcal{Y}}_n$ are conservative (i.e. have coverage $> 90\%$). As a result, the overall coverage increases to 0.941. This is a price to be paid for label conditioned coverage -- the overall coverage may increase, however, being conservative in safety-critical settings is better than being overly optimistic.

\begin{table}[t]
\begin{center}
\caption{Coverages for COPP with and without label conditioned coverage, $\hat{C}^{\mathcal{Y}}_n$ and $\hat{C}_n$ respectively. Overall coverage refers to marginal coverage while $Y=y$ refers to coverage conditioned on $Y=y$. Here $n_{test}$ corresponds to the number of test data points ($\sim P^{\pi^*}$).}
\label{tab:label-cond}
\resizebox{0.5\columnwidth}{!}{
\begin{tabular}{lccccr}
\toprule
   & $n_{test}$  & $\hat{C}_n$ Cov      & $\hat{C}^{\mathcal{Y}}_n$ Cov      \\
\midrule
Overall & 5000  & 0.896 $\pm$ 0.005  & 0.941 $\pm$ 0.003 \\
$Y=0$ & 266  & \red{0.700 $\pm$ 0.020}  & 1.000 $\pm$ 0.000  \\
$Y=1$ & 293  & \red{0.526  $\pm$ 0.019} & 1.000 $\pm$ 0.000  \\
$Y=2$ & 228  & \red{0.772 $\pm$  0.018} & 0.990 $\pm$ 0.029 \\
$Y=3$ & 320  & \red{0.852 $\pm$  0.015} & 0.964 $\pm$ 0.035 \\
$Y=4$ & 3893 & 0.950 $\pm$ 0.006 & 0.928 $\pm$ 0.003 \\
\bottomrule
\end{tabular}
}
\end{center}
\end{table}

\subsection{UCI Dataset experiments}\label{sec:UCI}
Following \cite{risk-assessment, doubly-robust, adaptive-ope} we apply COPP on UCI classification datasets. We can pose classification as contextual bandits by defining the covariates $\mathcal{X}$ as the features, the action space $\mathcal{A} =\mathcal{K}$, where $\mathcal{K}$ is the set of labels, and the outcomes are binary, i.e. $\mathcal{Y}= \{0,1\}$, defined by $Y \mid X, A = \mathbbm{1}(X \textup{ belongs to class }A)$. Here we use $m=1000$ training data points.

\paragraph{Behaviour and Target Policies.}
First we train a neural network classifier mapping each covariate to the softmax scores for each class, $\hat{f}_\theta: \mathcal{X} \rightarrow [0,1]^{|\mathcal{K}|}$. We use this trained model $\hat{f}_\theta$ to define a family of policies such that we pick the most likely label as predicted by $\hat{f}_\theta$ with probability $\epsilon$ and the rest uniformly with probability. Formally, this can be expressed as follows:
\begin{align}
    &\pi_\epsilon (a\mid x) \coloneqq \epsilon \mathbbm{1}(a = \argmax_{k \in \mathcal{K}}\{ \hat{f}_\theta(x)_{k} \}) + (1-\epsilon)/(|\mathcal{K}|-1)\mathbbm{1}(a \neq \argmax_{k \in \mathcal{K}}\{ \hat{f}_\theta(x)_{k} \}) \nonumber
\end{align}
Like other experiments, we use $\epsilon$ to control the shift between behaviour and target policies. For $\pi^b$, we use $\epsilon^b = 0.5$ and for $\epsilon^* \in \{0.05, 0.3, 0.4, 0.5, 0.6, 0.7, 0.95\}$. Using this behaviour policy $\pi^b$, we generate an observational dataset $\mathcal{D}_{obs} = \{x_i, a_i, y_i\}_{i=1}^{n_{obs}}$ which is then split into training $\mathcal{D}_{tr}$ and calibration datasets $\mathcal{D}_{cal}$, of sizes $m$ and $n$ respectively.

\paragraph{Estimation of ratios, $\hat{w}(X, Y)$.}
To estimate the $\hat{P}(y \mid x, a)$ we use the trained model $\hat{f}_\theta$ as follows:
\[
\hat{P}(Y = 1 \mid x, a) = \hat{f}_\theta(x)_a
\]
where $\hat{f}_\theta(x)_a$ corresponds the softmax prediction of $x$ for label $a$ under the model $\hat{f}_\theta$. To estimate the behaviour policy $\hat{\pi}^b$, we train a classifier model $\mathcal{X} \rightarrow \mathcal{A}$ using a neural network. We use \eqref{weight-est} in main text to estimate weights $\hat{w}(x, y)$.

\paragraph{Score.} We define $\hat{P}^{\pi^b}(y \mid x) = \sum_{i \in \mathcal{K}} \hat{\pi}^b(A = i|x) \hat{P}(y|x, A = i)$. Using similar formulation as in \cite{conf-bates}, we define the score as 
\[
s(x, y) = \sum_{y' = 0, 1} \hat{P}^{\pi^b}(y' \mid x) \mathbbm{1}(\hat{P}^{\pi^b}(y' \mid x) \geq \hat{P}^{\pi^b}(y \mid x))
\]

\paragraph{Neural Network Architectures}
\begin{itemize}
    \item To approximate the behaviour policy, we use a neural network with 2 hidden layers and 64 nodes in each hidden layer, ReLU activations and softmax output.
    \item To approximate $\hat{f}_{\theta}$, we use a neural network with 2 hidden layers with 64 nodes each and ReLU activations.
\end{itemize}

\paragraph{Results.} Tables \ref{tab:yeast}-\ref{tab:satimage} show the coverages across varying target policies for different classification datasets. The behaviour policy $\pi^b = \pi_{0.5}$ is fixed and we use $n=5000$ calibration datapoints, across 10 runs with $m=5000$ training data. The tables show that COPP is able to provide the required coverage of 90\% across all target policies. Moreover, compared to COPP, SBA and WIS are overly conservative. WIS estimates are not adaptive w.r.t. $X$, and as a result, the predictive sets produced are uninformative (i.e. contain all outcomes) in these experiments where the outcome is binary. 

We have also included a comparison of COPP using estimated behaviour policy with COPP using GT behaviour policy. The latter provides more accurate coverage, and using estimated behaviour policy provides slightly over-covered predictive sets comparatively in most cases. This can be explained by policy estimation error. Additionally, we observe that using standard CP leads to predictive sets which are not adaptive to policy shift. As a result, the standard CP predictive sets get overly conservative (optimistic) as $\Delta_\epsilon$ becomes more negative (positive).

\newpage
\begin{table}[h!]
\begin{center}
\caption{Yeast dataset results}
\label{tab:yeast}
\resizebox{0.8\columnwidth}{!}{
\begin{tabular}{llllllll}
\toprule
& $\Delta_{\epsilon}=-0.45$ & $\Delta_{\epsilon}=-0.2$ & $\Delta_{\epsilon}=-0.1$ & $\Delta_{\epsilon}=0.0$ & $\Delta_{\epsilon}=0.1$ & $\Delta_{\epsilon}=0.2$ & $\Delta_{\epsilon}=0.45$ \\
\midrule
COPP (Ours)            &              0.92$\pm$0.00 &             0.92$\pm$0.00 &             0.92$\pm$0.00 &            0.92$\pm$0.00 &            0.92$\pm$0.00 &            0.92$\pm$0.00 &             0.91$\pm$0.00 \\
WIS                  &             0.99$\pm$0.01 &              1.00$\pm$0.00 &              1.00$\pm$0.00 &             1.00$\pm$0.00 &             1.00$\pm$0.00 &             1.00$\pm$0.00 &              1.00$\pm$0.00 \\
SBA                  &              0.98$\pm$0.00 &              1.00$\pm$0.00 &              1.00$\pm$0.00 &             1.00$\pm$0.00 &             1.00$\pm$0.00 &             1.00$\pm$0.00 &              1.00$\pm$0.00 \\
\midrule
\midrule
COPP (GT behav policy) &              0.91$\pm$0.00 &             0.91$\pm$0.00 &              0.90$\pm$0.00 &             0.90$\pm$0.00 &             0.90$\pm$0.00 &             0.90$\pm$0.00 &              0.90$\pm$0.00 \\
CP (no policy shift) &              0.97$\pm$0.00 &             0.93$\pm$0.00 &             0.92$\pm$0.00 &             0.90$\pm$0.00 &            0.89$\pm$0.00 &            0.87$\pm$0.00 &             0.83$\pm$0.00 \\
\bottomrule
\end{tabular}
}
\end{center}
\end{table}

\begin{table}[h!]
\begin{center}
\begin{small}
\begin{sc}
\caption{Ecoli dataset results}
\label{tab:ecoli}
\resizebox{0.8\columnwidth}{!}{
\begin{tabular}{llllllll}
\toprule
 & $\Delta_{\epsilon}=-0.45$ & $\Delta_{\epsilon}=-0.2$ & $\Delta_{\epsilon}=-0.1$ & $\Delta_{\epsilon}=0.0$ & $\Delta_{\epsilon}=0.1$ & $\Delta_{\epsilon}=0.2$ & $\Delta_{\epsilon}=0.45$ \\

\midrule
COPP (Ours)            &              0.92$\pm$0.00 &             0.91$\pm$0.00 &             0.91$\pm$0.00 &             0.90$\pm$0.00 &             0.90$\pm$0.00 &             0.90$\pm$0.00 &              0.90$\pm$0.00 \\
WIS                  &               1.00$\pm$0.00 &              1.00$\pm$0.00 &              1.00$\pm$0.00 &             1.00$\pm$0.00 &             1.00$\pm$0.00 &             1.00$\pm$0.00 &              1.00$\pm$0.00 \\
SBA                  &               1.00$\pm$0.00 &              1.00$\pm$0.00 &              1.00$\pm$0.00 &             1.00$\pm$0.00 &             1.00$\pm$0.00 &             1.00$\pm$0.00 &              1.00$\pm$0.00 \\
\midrule
\midrule
COPP (GT behav policy) &              0.91$\pm$0.00 &              0.90$\pm$0.00 &              0.90$\pm$0.00 &             0.90$\pm$0.00 &             0.90$\pm$0.00 &             0.90$\pm$0.00 &             0.90$\pm$0.01 \\
CP (no policy shift) &              0.92$\pm$0.00 &             0.91$\pm$0.00 &             0.91$\pm$0.00 &             0.90$\pm$0.00 &             0.90$\pm$0.00 &            0.89$\pm$0.00 &             0.88$\pm$0.00 \\
\bottomrule
\end{tabular}
}
\end{sc}
\end{small}
\end{center}
\end{table}

\begin{table}[h!]
\begin{center}
\begin{small}
\begin{sc}
\caption{Letter dataset results}
\label{tab:letter}
\resizebox{0.8\columnwidth}{!}{
\begin{tabular}{llllllll}
\toprule
 & $\Delta_{\epsilon}=-0.45$ & $\Delta_{\epsilon}=-0.2$ & $\Delta_{\epsilon}=-0.1$ & $\Delta_{\epsilon}=0.0$ & $\Delta_{\epsilon}=0.1$ & $\Delta_{\epsilon}=0.2$ & $\Delta_{\epsilon}=0.45$ \\

\midrule
COPP (Ours)            &              0.95$\pm$0.00 &             0.93$\pm$0.00 &             0.93$\pm$0.00 &            0.92$\pm$0.00 &            0.92$\pm$0.00 &            0.92$\pm$0.00 &             0.91$\pm$0.00 \\
WIS                  &               1.00$\pm$0.00 &              1.00$\pm$0.00 &              1.00$\pm$0.00 &             1.00$\pm$0.00 &             1.00$\pm$0.00 &             1.00$\pm$0.00 &              1.00$\pm$0.00 \\
SBA                  &              0.97$\pm$0.00 &              1.00$\pm$0.00 &              1.00$\pm$0.00 &             1.00$\pm$0.00 &             1.00$\pm$0.00 &             1.00$\pm$0.00 &              1.00$\pm$0.00 \\
\midrule
\midrule
COPP (GT behav policy) &              0.92$\pm$0.00 &             0.91$\pm$0.00 &             0.91$\pm$0.00 &             0.90$\pm$0.00 &            0.89$\pm$0.00 &            0.89$\pm$0.00 &             0.88$\pm$0.00 \\
CP (no policy shift) &              0.99$\pm$0.00 &             0.94$\pm$0.00 &             0.92$\pm$0.00 &             0.90$\pm$0.00 &            0.88$\pm$0.00 &            0.86$\pm$0.00 &             0.81$\pm$0.00 \\
\bottomrule
\end{tabular}
}
\end{sc}
\end{small}
\end{center}
\end{table}

\begin{table}[h!]
\begin{center}
\begin{small}
\begin{sc}
\caption{Optdigits dataset results}
\label{tab:optdigits}
\resizebox{0.8\columnwidth}{!}{
\begin{tabular}{llllllll}
\toprule
 & $\Delta_{\epsilon}=-0.45$ & $\Delta_{\epsilon}=-0.2$ & $\Delta_{\epsilon}=-0.1$ & $\Delta_{\epsilon}=0.0$ & $\Delta_{\epsilon}=0.1$ & $\Delta_{\epsilon}=0.2$ & $\Delta_{\epsilon}=0.45$ \\
\midrule
COPP (Ours)            &              0.93$\pm$0.00 &             0.93$\pm$0.00 &             0.93$\pm$0.00 &            0.93$\pm$0.00 &            0.93$\pm$0.00 &            0.93$\pm$0.00 &             0.93$\pm$0.00 \\
WIS                  &             0.99$\pm$0.01 &              1.00$\pm$0.00 &              1.00$\pm$0.00 &             1.00$\pm$0.00 &             1.00$\pm$0.00 &             1.00$\pm$0.00 &              1.00$\pm$0.00 \\
SBA                  &              0.97$\pm$0.00 &              1.00$\pm$0.00 &              1.00$\pm$0.00 &             1.00$\pm$0.00 &             1.00$\pm$0.00 &             1.00$\pm$0.00 &             0.99$\pm$0.00 \\
\midrule
\midrule
COPP (GT behav policy) &              0.91$\pm$0.00 &              0.90$\pm$0.00 &              0.90$\pm$0.00 &             0.90$\pm$0.00 &             0.90$\pm$0.00 &            0.89$\pm$0.00 &             0.89$\pm$0.00 \\
CP (no policy shift) &              0.97$\pm$0.00 &             0.93$\pm$0.00 &             0.91$\pm$0.00 &             0.90$\pm$0.00 &            0.88$\pm$0.00 &            0.87$\pm$0.00 &             0.83$\pm$0.00 \\
\bottomrule
\end{tabular}
}
\end{sc}
\end{small}
\end{center}
\end{table}

\begin{table}[h!]
\begin{center}
\begin{small}
\begin{sc}
\caption{Pendigits dataset results}
\label{tab:pendigits}
\resizebox{0.8\columnwidth}{!}{
\begin{tabular}{llllllll}
\toprule
 & $\Delta_{\epsilon}=-0.45$ & $\Delta_{\epsilon}=-0.2$ & $\Delta_{\epsilon}=-0.1$ & $\Delta_{\epsilon}=0.0$ & $\Delta_{\epsilon}=0.1$ & $\Delta_{\epsilon}=0.2$ & $\Delta_{\epsilon}=0.45$ \\
\midrule
COPP (Ours)            &              0.92$\pm$0.00 &             0.92$\pm$0.00 &             0.92$\pm$0.00 &            0.92$\pm$0.00 &            0.92$\pm$0.00 &            0.92$\pm$0.00 &             0.91$\pm$0.00 \\
WIS                  &               1.00$\pm$0.00 &              1.00$\pm$0.00 &              1.00$\pm$0.00 &             1.00$\pm$0.00 &             1.00$\pm$0.00 &             1.00$\pm$0.00 &              1.00$\pm$0.00 \\
SBA                  &              0.97$\pm$0.00 &              1.00$\pm$0.00 &              1.00$\pm$0.00 &             1.00$\pm$0.00 &             1.00$\pm$0.00 &             1.00$\pm$0.00 &             0.99$\pm$0.00 \\
\midrule
\midrule
COPP (GT behav policy) &              0.91$\pm$0.00 &              0.90$\pm$0.00 &              0.90$\pm$0.00 &             0.90$\pm$0.00 &             0.90$\pm$0.00 &            0.89$\pm$0.00 &             0.89$\pm$0.00 \\
CP (no policy shift) &              0.99$\pm$0.00 &             0.94$\pm$0.00 &             0.92$\pm$0.00 &             0.90$\pm$0.00 &            0.88$\pm$0.00 &            0.86$\pm$0.00 &             0.81$\pm$0.00 \\
\bottomrule
\end{tabular}
}
\end{sc}
\end{small}
\end{center}
\end{table}

\begin{table}[h!]
\begin{center}
\begin{small}
\begin{sc}
\caption{Satimage dataset results}
\label{tab:satimage}
\resizebox{0.8\columnwidth}{!}{
\begin{tabular}{llllllll}
\toprule
 & $\Delta_{\epsilon}=-0.45$ & $\Delta_{\epsilon}=-0.2$ & $\Delta_{\epsilon}=-0.1$ & $\Delta_{\epsilon}=0.0$ & $\Delta_{\epsilon}=0.1$ & $\Delta_{\epsilon}=0.2$ & $\Delta_{\epsilon}=0.45$ \\

\midrule
COPP (Ours)            &              0.92$\pm$0.00 &             0.91$\pm$0.00 &             0.91$\pm$0.00 &            0.91$\pm$0.00 &            0.91$\pm$0.00 &            0.91$\pm$0.00 &             0.91$\pm$0.00 \\
WIS                  &               1.00$\pm$0.00 &              1.00$\pm$0.00 &              1.00$\pm$0.00 &             1.00$\pm$0.00 &             1.00$\pm$0.00 &             1.00$\pm$0.00 &              1.00$\pm$0.00 \\
SBA                  &              0.98$\pm$0.00 &              1.00$\pm$0.00 &              1.00$\pm$0.00 &             1.00$\pm$0.00 &             1.00$\pm$0.00 &             1.00$\pm$0.00 &             0.99$\pm$0.00 \\
\midrule
\midrule
COPP (GT behav policy) &               0.90$\pm$0.00 &              0.90$\pm$0.00 &              0.90$\pm$0.00 &             0.90$\pm$0.00 &             0.90$\pm$0.00 &             0.90$\pm$0.00 &             0.89$\pm$0.00 \\
CP (no policy shift) &              0.97$\pm$0.00 &             0.93$\pm$0.00 &             0.92$\pm$0.00 &             0.90$\pm$0.00 &            0.88$\pm$0.00 &            0.87$\pm$0.00 &             0.83$\pm$0.00 \\
\bottomrule
\end{tabular}
}
\end{sc}
\end{small}
\end{center}
\end{table}

\clearpage

\section{How the miscoverage depends on $\hat{P}(y\mid x, a)$}
\begin{proposition}
Let
\begin{align*}
    \tilde{w}(x, y) \coloneqq \frac{\int \hat{P}(y\mid x, a)\pi^*(a\mid x)\mathrm{d}a}{\int \hat{P}(y\mid x, a)\pi^b(a\mid x)\mathrm{d}a}.
\end{align*}
Assume that
$\hat{P}(y\mid x, a)/P(y\mid x, a) \in [1/\Gamma, \Gamma]$ for some $\Gamma \geq 1$.
Then, $$\Delta_w \coloneqq \tfrac{1}{2}\expb \mid \tilde{w}(X,Y) - w(X,Y)\mid \leq \Gamma^2 - 1.$$
\end{proposition}
\begin{proof}
In this proof, we investigate the error of the weights as a function of the error in $\hat{P}(y\mid x, a)$. Therefore, to isolate this effect we ignore the Monte Carlo error, and assume known behavioural policy $\pi^b$.

Under the assumption above, we have that

\begin{align*}
    \frac{1/ \Gamma \int P(y\mid x, a)\pi^*(a\mid x)\mathrm{d}a}{\Gamma \int P(y\mid x, a)\pi^b(a\mid x)\mathrm{d}a} \leq &\tilde{w}(x, y) \leq \frac{\Gamma \int P(y\mid x, a)\pi^*(a\mid x)\mathrm{d}a}{1/\Gamma \int P(y\mid x, a)\pi^b(a\mid x)\mathrm{d}a}.\\
    \implies \frac{1}{\Gamma^2} w(x, y) \leq &\tilde{w}(x, y) \leq \Gamma^2 w(x, y)
\end{align*}
This means that, 

\begin{align*}
    \left(\frac{1}{\Gamma^2}-1 \right) w(x, y) \leq &\tilde{w}(x, y) - w(x, y) \leq (\Gamma^2 - 1) w(x, y)
\end{align*}
So, 
\begin{align*}
    \mid \tilde{w}(x, y) - w(x, y)\mid \leq (\Gamma^2 - 1) w(x, y)
\end{align*}
And therefore, 
\begin{align*}
    \expb \mid \tilde{w}(X,Y) - w(X,Y)\mid \leq (\Gamma^2 - 1) \expb[w(X, Y)] = \Gamma^2 - 1
\end{align*}
\end{proof}

        \chapter{\label{app:causal}Causal Falsification of Digital Twins}

\minitoc

\section{Notation} \label{sec:notation}

\begin{tabularx}{\linewidth}{l X}
$z_{\tx:\tx'}$ & The sequence of elements $(z_\tx, \ldots, z_{\tx'})$ (or the empty sequence when $\tx > \tx'$) \\
$\mathcal{Z}_{\tx:\tx'}$ (where each $\mathcal{Z}_{i}$ is a set) & The cartesian product $\mathcal{Z}_{\tx} \times \cdots \times \mathcal{Z}_{\tx'}$ (or the empty set when $\tx > \tx'$) \\
$Z_{\tx:\tx'}(\ax_{1:\tx'})$ & The sequence of potential outcomes $Z_\tx(\ax_{1:\tx}), \ldots, Z_{\tx'}(\ax_{1:\tx'})$ (or the empty sequence when $\tx > \tx'$) \\
$\Law[Z]$ & The distribution of the random variable $Z$ \\
$\Law[Z \mid M]$ & The conditional distribution of $Z$ given $M$, where $M$ is either an event or a random variable \\
$Z \eqas Z'$ & The random variables $Z$ and $Z'$ are almost surely equal, i.e.\ $\Prob(Z = Z') = 1$ \\
$Z \ci Z'$ & The random variables $Z$ and $Z'$ are independent \\
$Z \ci Z' \mid Z''$ & The random variables $Z$ and $Z'$ are conditionally independent given the random variable $Z''$ \\
$\ind(E)$ & Indicator function of some event $E$ %
\end{tabularx}

\section{Proof of Proposition \ref{prop:interventional-correctness-alternative-characterisation} (unconditional form of interventional correctness)} \label{sec:unconditional-interventional-correctness-proof}

\begin{proof}
Fix any choice of $\ax_{1:\T} \in \Aspace_{1:\T}$.
By disintegrating $\Law[\X_0, \Xt_{1:\T}(\X_0, \ax_{1:\T})]$ and $\Law[\X_{0:\T}(\ax_{1:\T})]$ along their common $\Xspace_0$-marginal (which is namely $\Law[\X_0]$), it holds that
\begin{equation} \label{eq:interventional-correctness-alternative}
    \Law[\X_0, \Xt_{1:\T}(\X_0, \ax_{1:\T})] = \Law[\X_{0:\T}(\ax_{1:\T})]
\end{equation}
if and only if
\begin{equation} \label{eq:joint-interventional-correctness-proof-intermediate-step}
    \Law[\Xt_{1:\T}(\X_0, \ax_{1:\T}) \mid \X_0 = \xx_0] = \Law[\X_{1:\T}(\ax_{1:\T}) \mid \X_0 = \xx_0]
\end{equation}
for $\Law[\X_0]$-almost all $\xx_0 \in \Xspace_0$.
But now, our definition of $\Xt_{1:\T}(\xx_0, \ax_{1:\T})$ in terms of $\twinfunction_\tx$ and $\twinnoise_{1:\tx}$ means we can write $\Xt_{1:\T}(\X_0, \ax_{1:\T}) = \boldsymbol{\twinfunction}(\X_0, \ax_{1:\T}, \twinnoise_{1:\T})$,
where
\[
    \boldsymbol{\twinfunction}(\xx_0, \ax_{1:\T}, \ux_{1:\T}) \coloneqq (\twinfunction_1(\xx_0, \ax_1, \ux_1), \ldots, \twinfunction_\T(\xx_0, \ax_{1:\T}, \ux_{1:\T})).
\]
For all $\xx_0 \in \Xspace_0$ and measurable $\B_{1:\T} \subseteq \Xspace_{1:\T}$, we then have
\begin{align*}
    \Law[\Xt_{1:\T}(\xx_0, \ax_{1:\T})](\B_{1:\T}) &= \E[\ind(\boldsymbol{\twinfunction}(\xx_0, \ax_{1:\T}, \twinnoise_{1:\T}) \in \B_{1:\T})] \\
    &= \int \ind(\boldsymbol{\twinfunction}(\xx_0, \ax_{1:\T}, \ux_{1:\T}) \in \B_{1:\T}) \, \Law[\twinnoise_{1:\T}](\dee \ux_{1:\T}).
\end{align*}
It is standard to show that the right-hand side is a Markov kernel in $\xx_0$ and $\B_{1:\T}$.
Moreover, for any measurable $\B_0 \subseteq \Xspace_0$, we have
\begin{align*}
    &\int_{\B_0} \Law[\Xt_{1:\T}(\xx_0, \ax_{1:\T})](\B_{1:\T}) \, \Law[\X_0](\dee \xx_0) \\
        &\qquad= \int_{\B_0} \left[\int \ind(\boldsymbol{\twinfunction}(\xx_0, \ax_{1:\T}, \ux_{1:\T}) \in \B_{1:\T}) \, \Law[\twinnoise_{1:\T}](\dee \ux_{1:\T}) \right] \, \Law[\X_0](\dee \xx_0) \\
        &\qquad= \int \ind(\xx_0 \in \B_0, \boldsymbol{\twinfunction}(\xx_0, \ax_{1:\T}, \ux_{1:\T}) \in \B_{1:\T}) \, \Law[\X_0, \twinnoise_{1:\T}](\dee \xx_0, \dee \ux_{1:\T}) \\
        &\qquad= \Law[\X_0, \Xt_{1:\T}(\X_0, \ax_{1:\T})](\B_{0:\T}),
\end{align*}
where the second step follows because $\X_0 \ci \twinnoise_{1:\T}$.
It therefore follows that $(\xx_0, \B_{1:\T}) \mapsto \Law[\Xt_{1:\T}(\xx_0, \ax_{1:\T})](\B_{1:\T})$ is a regular conditional distribution of $\Xt_{1:\T}(\X_0, \ax_{1:\T})$ given $\X_0$, i.e.\
\[
    \Law[\Xt_{1:\T}(\xx_0, \ax_{1:\T})] = \Law[\Xt_{1:\T}(\X_0, \ax_{1:\T}) \mid \X_0 = \xx_0] \qquad \text{for $\Law[\X_0]$-almost all $\xx_0 \in \Xspace_0$.}
\]
Substituting this into \eqref{eq:joint-interventional-correctness-proof-intermediate-step}, we see that \eqref{eq:interventional-correctness-alternative} holds if and only if
\[
    \Law[\Xt_{1:\T}(\xx_0, \ax_{1:\T})] = \Law[\X_{1:\T}(\ax_{1:\T}) \mid \X_0 = \xx_0]
\]
for $\Law[\X_0]$-almost all $\xx_0 \in \Xspace_0$.
The result now follows since $\ax_{1:\T}$ was arbitrary.
\end{proof}

\section{Online prediction} \label{sec:online-prediction}

\subsection{Correctness in the online setting}

A distinguishing feature of many digital twins is their ability to integrate real-time information obtained from sensors in their environment \citep{barricelli2019survey}.
It is therefore relevant to consider a setting in which a twin is used repeatedly to make a sequence of predictions over time, each time taking all previous information into account.
One way to formalize this is to instantiate our model for the twin at each timestep.
For example, we could represent the predictions made by the twin at $\tx = 0$ after observing initial covariates $\xx_0$ as potential outcomes $(\Xt^1_{1:\T}(\xx_0, \ax_{1:\T}) : \ax_{1:\T} \in \Aspace_{1:\T})$, similar to what we did in the main text.
We could then represent the predictions made by the twin after some action $\ax_1$ is taken and an additional observation $\xx_1$ is made via potential outcomes $(\Xt^2_{2:\T}(\xx_{0:1}, \ax_{1:\T}) : \ax_{2:\T} \in \Aspace_{2:\T})$.
More generally, for $\tx \in \{1, \ldots, \T\}$, we could introduce potential outcomes $(\Xt^{\tx}_{\tx:\T}(\xx_{0:\tx-1}, \ax_{1:\T}) : \ax_{\tx:\T} \in \Aspace_{\tx:\T})$ to represent the predictions that the twin would make at time $\tx$ after the observations $\xx_{0:\tx-1}$ are made and the actions $\ax_{1:\tx-1}$ are taken.

This extended model requires a new definition of correctness than our Definition \ref{eq:interventional-correctness} from the main text.
A natural approach is to say that the twin is correct in this new setting if
\begin{equation}
    \Law[\Xt_{\tx:\T}^\tx(\xx_{0:\tx-1}, \ax_{1:\T})]
        = \Law[\X_{\tx:\T}(\ax_{1:\T}) \mid \X_{0:\tx-1}(\ax_{1:\tx-1}) = \xx_{0:\tx-1}] \label{eq:online-interventional-correctness-def}
\end{equation}
for all $\tx \in \{1, \ldots, \T\}$, $\ax_{1:\T} \in \Aspace_{1:\T}$, and $\Law[\X_{0:\tx-1}(\ax_{1:\tx-1})]$-almost all $\xx_{0:\tx-1} \in \Xspace_{0:\tx-1}$.
A twin with this property would at each step be able to accurately simulate the future in light of previous information, use this to choose a next action to take, observe the result of doing so, and then repeat.
It is possible to show that \eqref{eq:online-interventional-correctness-def} holds if and only if we have
\begin{align*}
    \Law[\Xt^1_{1:\T}(\xx_0, \ax_{1:\T})] &= \Law[\X_{1:\T}(\ax_{1:\T}) \mid \X_{0}=\xx_0] \\
    \Law[\Xt^\tx_{\tx:\T}(\xx_{0:\tx-1}, \ax_{1:\T})]
        &= \Law[\Xt^1_{\tx:\T}(\xx_0, \ax_{1:\T}) \mid \Xt_{1:\tx-1}^1(\xx_0, \ax_{1:\tx-1}) = \xx_{1:\tx-1}]
\end{align*}
for all $\tx \in \{1, \ldots, \T\}$, $\ax_{1:\T} \in \Aspace_{1:\T}$, $\Law[\X_0]$-almost all $\xx_0 \in \Xspace_0$, and $\Law[\Xt_{1:\tx-1}^1(\xx_0, \ax_{1:\tx-1})]$-almost all $\xx_{1:\tx-1} \in \Xspace_{1:\tx-1}$.
The first condition here says that $\Xt^1_{1:\T}(\xx_0, \ax_{1:\T})$ must be interventionally correct in the sense of Definition \ref{eq:interventional-correctness} from the main text.
The second condition says that the predictions made by the twin across different timesteps must be internally consistent with each other insofar as their conditional distributions must align.
This holds automatically in many circumstances, such as if the predictions of the twin are obtained from a Bayesian model (for example), and otherwise could be checked numerically given the ability to run simulations from the twin, without the need to obtain data or refer to the real-world process in any way.
As such, the problem of assessing the correctness of the twin in this new sense primarily reduces to the problem of assessing the correctness of $\Xt^1_{1:\T}(\xx_0, \ax_{1:\T})$ in the sense of Definition \ref{eq:interventional-correctness} in the main text, which motivates our focus on that condition.

\subsection{Alternative notions of online correctness} \label{sec:online-correctness-alternative-notion-supp}

An important and interesting subtlety arises in this context that is worth noting.
In general it does not follow that a twin correct in the sense of \eqref{eq:online-interventional-correctness-def} satisfies
\begin{equation}
    \Law[\Xt_{\tx:\T}^\tx(\xx_{0:\tx-1}, \ax_{1:\T})] \
        = \Law[\X_{\tx:\T}(\ax_{1:\T}) \mid \X_{0:\tx-1}(\ax_{1:\tx-1}) = \xx_{0:\tx-1}, \A_{1:\tx-1} = \ax_{1:\tx-1}] \label{eq:online-interventional-correctness-alt}
\end{equation}
for all $\ax_{1:\T} \in \Aspace_{1:\T}$, and $\Law[\X_{0:\tx-1}(\ax_{1:\tx-1}) \mid \A_{1:\tx-1} = \ax_{1:\tx-1}]$-almost all $\xx_{0:\tx-1} \in \Xspace_{0:\tx-1}$, 
since in general it does not hold that 
\begin{multline*}
    \Law[\X_{\tx:\T}(\ax_{1:\T}) \mid \X_{0:\tx-1}(\ax_{1:\tx-1}) = \xx_{0:\tx-1}]
        = \Law[\X_{\tx:\T}(\ax_{1:\T}) \mid \X_{0:\tx-1}(\ax_{1:\tx-1}) = \xx_{0:\tx-1}, \A_{1:\tx-1} = \ax_{1:\tx-1}].
\end{multline*}
for all $\ax_{1:\T} \in \Aspace_{1:\T}$ and $\Law[\X_{0:\tx-1}(\ax_{1:\tx-1}) \mid \A_{1:\tx-1} = \ax_{1:\tx-1}]$-almost all $\xx_{0:\tx-1} \in \Xspace_{0:\tx-1}$ unless the actions $\A_{1:\tx-1}$ are unconfounded.
(Here as usual $\A_{1:\T}$ denotes the actions of a behavioural agent; see Section \ref{sec:data-driven-twin-assessment} of the main text.)
In other words, a twin that is correct in the sense of \eqref{eq:online-interventional-correctness-def} will make accurate predictions at time $\tx$ when every action taken before time $\tx$ was unconfounded (as occurs for example when the twin is directly in control of the decision-making process), but in general not when certain taken actions before time $\tx$ were chosen by a behavioural agent with access to more context than is available to the twin (as may occur for example when the twin is used as a decision-support tool).
However, should it be desirable, our framework could be extended to encompass the alternative condition in \eqref{eq:online-interventional-correctness-alt} by relabelling the observed history $(\X_{0:\tx-1}(\A_{1:\tx-1}), \A_{1:\tx-1})$ as $\X_0$, and then assessing the correctness of the potential outcomes $\Xt_{\tx:\T}^\tx(\xx_{0:\tx-1}, \ax_{1:\T})$ in the sense of Definition \ref{eq:interventional-correctness} from the main text.

Overall, the ``right'' notion of correctness in this online setting is to some extent a design choice.
We believe our causal approach to twin assessment provides a useful framework for formulating and reasoning about these possibilities, and consider the investigation of assessment strategies for additional usage regimes to be an interesting direction for future work.

\section{Proof of Theorem \ref{prop:nonidentifiability} (interventional distributions are not identifiable)} \label{sec:non-identifiability-result-proof-supp}

It is well-known in the causal inference literature that the interventional behaviour of the real-world process cannot be uniquely identified from observational data.
For completeness, we now provide a self-contained proof of this result in our notation.
Our statement here is lengthier than Theorem \ref{prop:nonidentifiability} in the main text in order to clarify what is meant by ``uniquely identified'': intuitively, the idea is that there always exist distinct families of potential outcomes whose interventional behaviours differ and yet give rise to the same observational data.
\begin{theorem}
    Suppose we have $\ax_{1:\T} \in \Aspace_{1:\T}$ such that $\Prob(\A_{1:\T} \neq \ax_{1:\T}) > 0$.
    Then there exist potential outcomes $(\tilde{\X}_{0:\T}(\ax_{1:\T}') : \ax_{1:\T}' \in \Aspace_{1:\T})$ such that %
    \begin{equation} \label{eq:nonidentifiability-proof-almost-sure-equality}
        (\tilde{\X}_{0:\T}(\A_{1:\T}), \A_{1:\T}) \eqas (\X_{0:\T}(\A_{1:\T}), \A_{1:\T}).
    \end{equation}
    but for which $\Law[\tilde{\X}_{0:\T}(\ax_{1:\tx})] \neq \Law[\X_{0:\T}(\ax_{1:\tx})]$.
\end{theorem}

\begin{proof}
    Our assumption that $\Prob(\A_{1:\T} \neq \ax_{1:\T}) > 0$ means there must exist some $\tx \in \{1, \ldots, \T\}$ such that $\Prob(\A_{1:\tx} \neq \ax_{1:\tx}) > 0$.
    Since $\Xspace_\tx = \R^{\Xspacedim_\tx}$, we may also choose some $\xx_\tx \in \Xspace_\tx$ with $\Prob(\X_\tx(\ax_{1:\tx}) = \xx_\tx \mid \A_{1:\tx} \neq \ax_{1:\tx}) \neq 1$.
    Then, for each $\sx \in \{0, \ldots, \T\}$ and $\ax_{1:\sx}' \in \Aspace_{1:\sx}$, define
    \[
         \tilde{\X}_{\sx}(\ax_{1:\sx}') \coloneqq \begin{cases}
            \ind(\A_{1:\tx} = \ax_{1:\tx}) \, \X_{\tx}(\ax_{1:\tx}) + \ind(\A_{1:\tx} \neq \ax_{1:\tx}) \, \xx_\tx & \text{if $\sx = \tx$ and $\ax_{1:\sx}' = \ax_{1:\tx}$} \\
            \X_{\sx}(\ax_{1:\sx}') & \text{otherwise},
         \end{cases}
    \]
    It is then easily checked that \eqref{eq:nonidentifiability-proof-almost-sure-equality} holds, but
    \begin{align*}
        \Law[\tilde{\X}_{\tx}(\ax_{1:\tx})] &= \Law[\tilde{\X}_{\tx}(\ax_{1:\tx}) \mid \A_{1:\tx} = \ax_{1:\tx}] \, \Prob(\A_{1:\tx} = \ax_{1:\tx}) + \Law[\tilde{\X}_{\tx}(\ax_{1:\tx}) \mid \A_{1:\tx} \neq \ax_{1:\tx}] \, \Prob(\A_{1:\tx} \neq \ax_{1:\tx}) \\
        &= \Law[X_{\tx}(\ax_{1:\tx}) \mid \A_{1:\tx} = \ax_{1:\tx}] \, \Prob(\A_{1:\tx} = \ax_{1:\tx}) + \mathrm{Dirac}(\xx_\tx) \, \Prob(\A_{1:\tx} \neq \ax_{1:\tx}) \\
        &\neq \Law[X_{\tx}(\ax_{1:\tx}) \mid \A_{1:\tx} = \ax_{1:\tx}] \, \Prob(\A_{1:\tx} = \ax_{1:\tx}) + \Law[\X_{\tx}(\ax_{1:\tx}) \mid \A_{1:\tx} \neq \ax_{1:\tx}] \, \Prob(\A_{1:\tx} \neq \ax_{1:\tx}) \\
        &= \Law[X_{\tx}(\ax_{1:\tx})],
    \end{align*}
    from which the result follows.
\end{proof}

\section{Deterministic potential outcomes are unconfounded} \label{eq:deterministic-potential-outcomes-are-unconfounded}

In this section we expand on our earlier claim that, if the real-world process is deterministic, then the observational data is unconfounded.
We first make this claim precise.
By ``deterministic'', we mean that there exist measurable functions $\gx_\tx$ for $\tx \in \{1, \ldots, \T\}$ such that
\begin{equation} \label{eq:potential-outcomes-are-deterministic}
    \X_\tx(\ax_{1:\tx}) \eqas \gx_\tx(\X_{0:\tx-1}(\ax_{1:\tx-1}), \ax_{1:\tx}) \qquad \text{for all $\tx \in \{1, \ldots, \T\}$ and $\ax_{1:\tx} \in \Aspace_{1:\tx}$.}
\end{equation}
By ``unconfounded'', we mean that the \emph{sequential randomisation assumption (SRA)} introduced by Robins \citep{robins1986new} holds, i.e.\
\begin{equation} \label{eq:actions-are-unconfounded}
    (\X_{\sx}(\ax_{1:\sx}) : \sx \in \{1, \ldots, \T\}, \ax_{1:\sx} \in \Aspace_{1:\sx}) \ci \A_\tx \mid \X_{0:\tx-1}(\A_{1:\tx-1}), \A_{1:\tx-1} \qquad \text{for all $\tx \in \{1, \ldots, \T\}$},
\end{equation}
where $\ci$ denotes conditional independence.
Intuitively, this says that, apart from the historical observations $(\X_{0:\tx-1}(\A_{1:\tx-1}), \A_{1:\tx-1})$, any additional factors that influence the agent's choice of action $\A_\tx$ are independent of the behaviour of the real-world process.
The SRA provides a standard formulation of the notion of unconfoundedness in longitudinal settings such as ours (see \cite[Chapter 5]{tsiatis2019dynamic} for a review).

It is now a standard exercise to show that \eqref{eq:potential-outcomes-are-deterministic} implies \eqref{eq:actions-are-unconfounded}.
We include a proof below for completeness.
Key to this is the following straightforward Lemma.

\begin{lemma}\label{lem:determinism_conditional_independence}
Suppose $U$ and $V$ are random variables such that, for some measurable function $g$, it holds that $U \eqas g(V)$.
Then, for any other random variable $W$, we have
\[
    U \ci W \mid V.
\]
\end{lemma}

\begin{proof}
By standard properties of conditional expectations, for any measurable sets $S_1$ and $S_2$, we have almost surely
\begin{align*}
    \Prob(U \in S_1, W \in S_2 \mid V) &= \E[\ind(g(V) \in S_1) \, \ind(W \in S_2) \mid V] \\
    &= \ind(g(V) \in S_1) \, \E[\ind(W \in S_2) \mid V] \\
    &= \E[\ind(U \in S_1) \mid V] \, \Prob(W \in S_2 \mid V) \\
    &= \Prob(U \in S_1 \mid V) \, \Prob(W \in S_2 \mid V),
\end{align*}
which gives the result.
\end{proof}

It is now easy to see that \eqref{eq:potential-outcomes-are-deterministic} implies \eqref{eq:actions-are-unconfounded}.
Indeed, by recursive substitution, it is straightforward to show that there exist measurable functions $\tilde{g}_\tx$ for $\tx \in \{1, \ldots, \T\}$ such that
\[
    \X_\tx(\ax_{1:\tx}) \eqas \tilde{g}_\tx(\X_{0}, \ax_{1:\tx}) \qquad \text{for all $\tx \in \{1, \ldots, \T\}$ and $\ax_{1:\tx} \in \Aspace_{1:\tx}$},
\]
and so
\[
    (\X_{\sx}(\ax_{1:\sx}) : \sx \in \{1, \ldots, \T\}, \ax_{1:\sx} \in \Aspace_{1:\sx})
        = (\tilde{g}_\tx(\X_{0}, \ax_{1:\sx}) : \sx \in \{1, \ldots, \T\}, \ax_{1:\sx} \in \Aspace_{1:\sx}).
\]
The right-hand side is now seen to be a measurable function of $\X_0$ and hence certainly of $(\X_{0:\tx-1}(\A_{1:\tx-1}), \A_{1:\tx-1})$, so that the result follows by Lemma \ref{lem:determinism_conditional_independence}.

\section{Motivating toy example} \label{sec:motivating-example}

We provide here a toy scenario that illustrates intuitively the pitfalls that may arise when assessing twins using observational data without properly accounting for causal considerations (including unmeasured confounding in particular).
Suppose a digital twin has been designed for a particular make of car, e.g.\ to facilitate autonomous driving \citep{allamaa2022sim2real}.
The twin simulates how quantities such as the velocity and fuel consumption of the car respond as certain inputs are applied to it, such as braking, acceleration, steering, etc.
We wish to assess the accuracy of this twin using a dataset obtained from a fleet of the same make.
The braking performance of these vehicles is significantly affected by the age of their brake pads: if these are fairly new, then an aggressive braking strategy will stop the car, while if these are old, then the same aggressive strategy will send the car into a skid that will reduce braking efficacy.
Brake pad age is not recorded in the data we have obtained, but \emph{was} known to the drivers who operated these vehicles (e.g.\ perhaps they were aware of how recently their car was serviced), and so the drivers of cars with old brake pads tended to avoid braking aggressively out of safety concerns.

A naive approach to twin assessment in this situation would directly compare the outputs of the twin with the data and conclude the twin is accurate if these match closely.
However, in this scenario, the data contains a spurious relationship between braking strategy and the performance of the car: since aggressive braking is only observed for cars with new brake pads, the data appears to show that aggressive braking is effective at stopping the car, while in fact this is not the case for cars with older brake pads.
As such, the naive assessment approach would yield misleading information about the twin: a twin that captures only the behaviour of cars with newer brake pads would appear to be correct, while a twin that captures the full range of possibilities (i.e.\ regardless of brake pad age) would deviate from the observational data and appear therefore less accurate.
Figure \ref{fig:syn_ex} illustrates this pictorially under a toy model for this scenario.

\begin{figure}
    \centering
    \includegraphics[height=5cm]{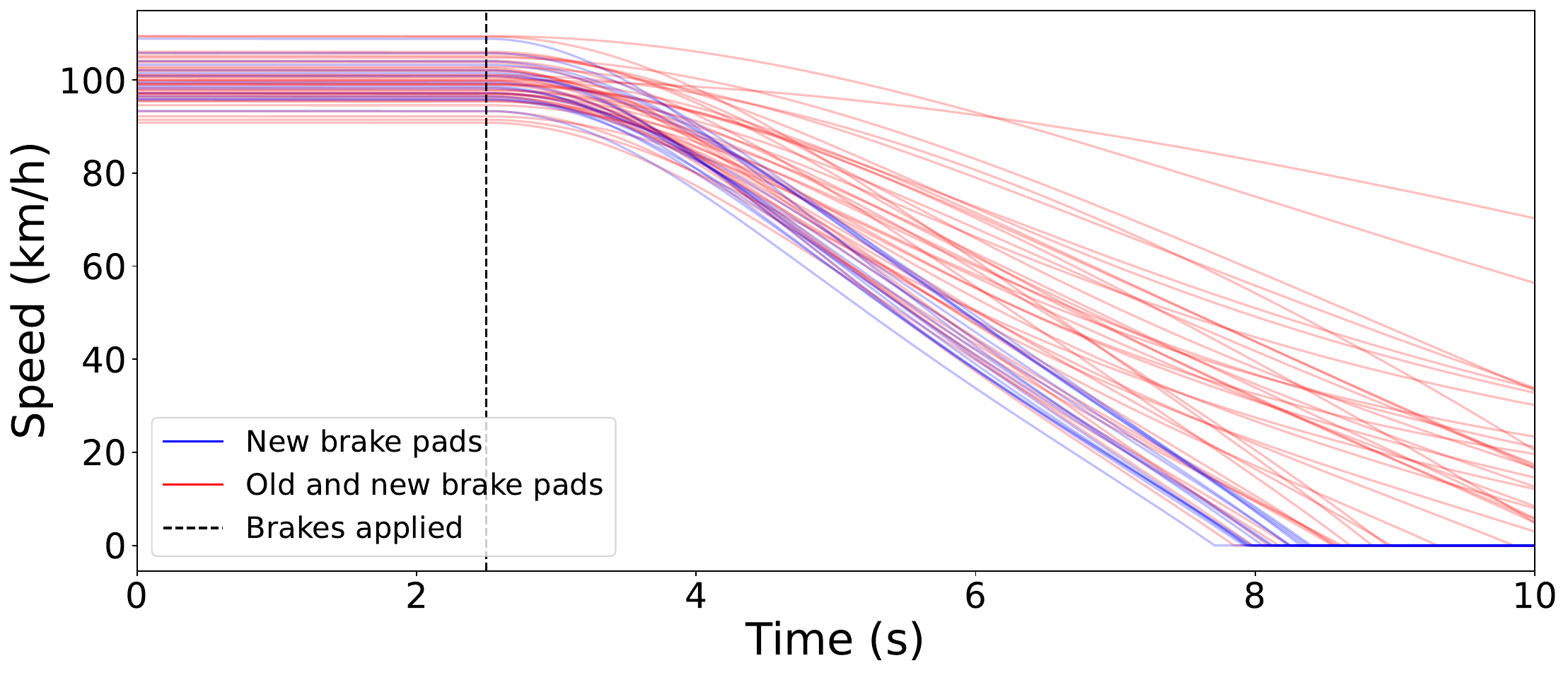}

    \caption{The discrepancy between observational data and interventional behavior. The data only show the effect of aggressive braking on cars with new brake pads (blue). This differs from what \emph{would} be observed if aggressive braking were applied to the entire fleet of cars, encompassing both those with old and new brake pads (red).}

    \label{fig:syn_ex}
\end{figure}

In the causal inference literature, any unmeasured quantity (e.g.\ brake pad age) that affects both some choice of action taken in the data (e.g.\ aggressive braking) and the resulting observation (e.g.\ speed) is referred to as an \emph{unmeasured confounder}.
In general, whenever an unmeasured confounder is present, a potential discrepancy arises between how the real-world process was observed to behave in the dataset and how it \emph{would} behave under certain interventions.
An obvious approach towards mitigating this possibility is to measure additional quantities that may affect the outcome of interest.
For example, if brake pad age were included in the data in the scenario above, then it would be possible to adjust for its effect on braking performance. 
However, in many cases, gathering additional data may be costly or impractical.
Moreover, even if this strategy is pursued, it is rarely possible to rule out the possibility of unmeasured confounding altogether, especially for complicated real-world problems \citep{tsiatis2019dynamic}. 
For example, in the scenario above, it is very conceivable that some other factor such as weather conditions could play a similar confounding role as brake pad age, and so would need also to be included in the data, and so on.
Analogous scenarios are also easily forthcoming for other application domains such as medicine and economics \citep{manski1995identification,tsiatis2019dynamic,hernan2020causal}.
As such, rather than attempting to sidestep the issue of unmeasured confounding, we instead propose a methodology for assessing twins using data that is robust to its presence. %

\section{Causal bounds} \label{sec:causal-bounds-proofs}

\subsection{Proof of Theorem \ref{thm:causal-bounds}}

\begin{proof}
We prove the lower bound; the upper bound is analogous.
It is easily checked that 
\begin{multline} \label{eq:causal-bounds-proof-first-step}
    \E[\Y(\ax_{1:\tx}) \mid \X_{0:\tx}(\ax_{1:\tx}) \in \B_{0:\tx}] \\
        = \E[\Y(\ax_{1:\tx}) \mid \X_{0:\tx}(\ax_{1:\tx}) \in \B_{0:\tx}, \A_{1:\tx} = \ax_{1:\tx}] \, \Prob(\A_{1:\tx} = \ax_{1:\tx} \mid \X_{0:\tx}(\ax_{1:\tx}) \in \B_{0:\tx}) \\
            + \E[\Y(\ax_{1:\tx}) \mid \X_{0:\tx}(\ax_{1:\tx}) \in \B_{0:\tx}, \A_{1:\tx} \neq \ax_{1:\tx}] \, \Prob(\A_{1:\tx} \neq \ax_{1:\tx} \mid \X_{0:\tx}(\ax_{1:\tx}) \in \B_{0:\tx}).
\end{multline}
If $\Prob(\A_{1:\tx} = \ax_{1:\tx} \mid \X_{0:\tx}(\ax_{1:\tx}) \in \B_{0:\tx}) > 0$, then
\begin{align*}
    \E[\Y(\ax_{1:\tx}) \mid \X_{0:\tx}(\ax_{1:\tx}) \in \B_{0:\tx}, \A_{1:\tx} = \ax_{1:\tx}]
        &= \E[\Y(\A_{1:\tx}) \mid \X_{0:\tx}(\A_{1:\tx}) \in \B_{0:\tx}, \A_{1:\tx} = \ax_{1:\tx}] \\
        &= \E[\Y(\A_{1:\N}) \mid \X_{0:\N}(\A_{1:\N}) \in \B_{0:\N}, \A_{1:\tx} = \ax_{1:\tx}],
\end{align*}
where the second step follows because $\Prob(\N = \tx \mid \A_{1:\tx} = \ax_{1:\tx}) = 1$.
Similarly, if $\Prob(\A_{1:\tx} \neq \ax_{1:\tx} \mid \X_{0:\tx}(\ax_{1:\tx}) \in \B_{0:\tx}) > 0$, then \eqref{eq:Y-boundedness-assumption} implies
\[
    \E[\Y(\ax_{1:\tx}) \mid \X_{0:\tx}(\ax_{1:\tx}) \in \B_{0:\tx}, \A_{1:\tx} \neq \ax_{1:\tx}]
        \geq \ylo.
\]
Substituting these results into \eqref{eq:causal-bounds-proof-first-step}, we obtain
\begin{multline} \label{eq:causal-bounds-proof-second-step}
    \E[\Y(\ax_{1:\tx}) \mid \X_{0:\tx}(\ax_{1:\tx}) \in \B_{0:\tx}] \\
    \geq \E[\Y(\A_{1:\tx}) \mid \X_{0:\N}(\A_{1:\N}) \in \B_{0:\N}, \A_{1:\tx} = \ax_{1:\tx}] \, \Prob(\A_{1:\tx} = \ax_{1:\tx} \mid \X_{0:\tx}(\ax_{1:\tx}) \in \B_{0:\tx}) \\
            + \ylo \, \Prob(\A_{1:\tx} \neq \ax_{1:\tx} \mid \X_{0:\tx}(\ax_{1:\tx}) \in \B_{0:\tx}).
\end{multline}
Now observe that the right-hand side of \eqref{eq:causal-bounds-proof-second-step} is a convex combination with mixture weights $\Prob(\A_{1:\tx} = \ax_{1:\tx} \mid \X_{0:\tx}(\ax_{1:\tx}) \in \B_{0:\tx})$ and $\Prob(\A_{1:\tx} \neq \ax_{1:\tx} \mid \X_{0:\tx}(\ax_{1:\tx}) \in \B_{0:\tx})$.
We can bound
\begin{align}
    \Prob(\A_{1:\tx} = \ax_{1:\tx} \mid \X_{0:\tx}(\ax_{1:\tx}) \in \B_{0:\tx})
        &= \frac{\Prob(\X_{0:\tx}(\ax_{1:\tx}) \in \B_{0:\tx}, \A_{1:\tx} = \ax_{1:\tx})}{\Prob(\X_{0:\tx}(\ax_{1:\tx}) \in \B_{0:\tx})} \notag \\
        &\geq \frac{\Prob(\X_{0:\tx}(\ax_{1:\tx}) \in \B_{0:\tx}, \A_{1:\tx} = \ax_{1:\tx})}{\Prob(\X_{0:\N}(\ax_{1:\N}) \in \B_{0:\N})} \notag \\
        &= \frac{\Prob(\X_{0:\N}(\A_{1:\N}) \in \B_{0:\N}, \A_{1:\tx} = \ax_{1:\tx})}{\Prob(\X_{0:\N}(\A_{1:\N}) \in \B_{0:\N})} \notag \\
        &= \Prob(\A_{1:\tx} = \ax_{1:\tx} \mid \X_{0:\N}(\A_{1:\N}) \in \B_{0:\N}), %
\end{align}
where the inequality holds because $\tx \geq \N$ almost surely, and the second equality holds because the definition of $\N$ means
\[
    \X_{0:\N}(\ax_{1:\N}) \eqas \X_{0:\N}(\A_{1:\N}).
\]
As such, we can bound the convex combination in \eqref{eq:causal-bounds-proof-second-step} from below by replacing its mixture weights with $\Prob(\A_{1:\tx} = \ax_{1:\tx} \mid \X_{0:\N}(\A_{1:\N}) \in \B_{0:\N})$ and $\Prob(\A_{1:\tx} \neq \ax_{1:\tx} \mid \X_{0:\N}(\ax_{1:\N}) \in \B_{0:\N})$, which shifts weight from the $\E[\Y(\A_{1:\tx}) \mid \A_{1:\tx} = \ax_{1:\tx}, \X_{0:\N}(\A_{1:\N}) \in \B_{0:\N}]$ term onto the $\ylo$ term.
This yields
\begin{align*}
    &\E[\Y(\ax_{1:\tx}) \mid \X_{0:\tx}(\ax_{1:\tx}) \in \B_{0:\tx}] \\
        &\qquad\qquad\geq \E[\Y(\A_{1:\tx}) \mid \X_{0:\N}(\A_{1:\N}) \in \B_{0:\N}, \A_{1:\tx} = \ax_{1:\tx}] \, \Prob(\A_{1:\tx} = \ax_{1:\tx} \mid \X_{0:\N}(\ax_{1:\N}) \in \B_{0:\N}) \\
        &\qquad\qquad\qquad+ \ylo \, \Prob(\A_{1:\tx} \neq \ax_{1:\tx} \mid \X_{0:\N}(\ax_{1:\N}) \in \B_{0:\N}) \\
        &\qquad\qquad= \E[\Y(\A_{1:\tx}) \, \ind(\A_{1:\tx} = \ax_{1:\tx}) + \ylo \, \ind(\A_{1:\tx} \neq \ax_{1:\tx}) \mid \X_{0:\N}(\A_{1:\N}) \in \B_{0:\N}] \\
        &\qquad\qquad= \E[\Ylo \mid \X_{0:\N}(\A_{1:\N}) \in \B_{0:\N}].
\end{align*}
\end{proof}

\subsection{Proof of Proposition \ref{prop:our-bounds-vs-manskis}}

\begin{proof}
From the definition of $\Yup$, we have straightforwardly
\begin{multline*}
    \Qup = \E[\Y(\A_{1:\tx}) \mid \X_{0:\N}(\A_{1:\N}) \in \B_{0:\N}, \A_{1:\tx} = \ax_{1:\tx}] \, \Prob(\A_{1:\tx} = \ax_{1:\tx} \mid \X_{0:\N}(\A_{1:\N}) \in \B_{0:\N}) \\
            +  \yup \, \Prob(\A_{1:\tx} \neq \ax_{1:\tx} \mid \X_{0:\N}(\A_{1:\N}) \in \B_{0:\N}).
\end{multline*}
A similar expression holds for $\Qlo$.
Subtracting these two expressions yields
\[
    \Qup - \Qlo = (\yup - \ylo) \, (1 - \Prob(\A_{1:\tx}=\ax_{1:\tx} \mid \X_{0:\N}(\A_{1:\N}) \in \B_{0:\N})). 
\]
Similar manipulations show that
\[
    \E[\Yup] - \E[\Ylo] = (\yup - \ylo) \, (1 - \Prob(\A_{1:\tx}=\ax_{1:\tx})),
\]
and the result now follows.

\end{proof}

\subsection{Proof of Proposition \ref{prop:sharpness-of-bounds} and discussion} \label{sec:sharpness-of-bounds-supplement}

\begin{proof}
    We consider the case of the lower bound; the case of the upper bound is analogous.
    Choose $\xx_{1:\T} \in \B_{1:\T}$ arbitrarily.
    (Certainly some choice is always possible, since each $\B_{\sx}$ has positive measure and is therefore nonempty.)
    Define
    \begin{align*}
        \tilde{\X}_0 &\coloneqq \X_0 \\
        \tilde{\X}_{\sx}(\ax_{1:\sx}') &\coloneqq \ind(\A_{1:\sx} = \ax_{1:\sx}') \, \X_{\sx}(\ax_{1:\sx}') + \ind(\A_{1:\sx} \neq \ax_{1:\sx}') \, \xx_{\sx} \qquad \text{for each $\sx \in \{0, \ldots, \T\}$ and $\ax_{1:\sx}' \in \Aspace_{1:\sx}$},
    \end{align*}
    and similarly let
    \[
        \tilde{\Y}(\ax_{1:\tx}') = \ind(\A_{1:\tx} = \ax_{1:\tx}') \, \Y(\ax_{1:\tx}') + \ind(\A_{1:\tx} \neq \ax_{1:\tx}')\, \ylo \qquad \text{for all $\ax_{1:\tx}' \in \Aspace_{1:\tx}$}.
    \]
    It is easy to check that $(\tilde{\X}_{0:\T}(\A_{1:\T}), \tilde{\Y}(\A_{1:\tx}), \A_{1:\T}) \eqas (\X_{0:\T}(\A_{1:\T}), \Y(\A_{1:\tx}), \A_{1:\T})$.
    But now we have directly $\tilde{\Y}(\ax_{1:\tx}) = \Ylo$.
    Moreover, it is easily checked from the definition of $\N$ and $\tilde{\X}_{0:\tx}(\ax_{1:\tx})$ that
    \[
        \tilde{\X}_{0:\tx}(\ax_{1:\tx}) \eqas (\X_{0:\N}(\A_{1:\N}), \xx_{\N+1:\tx}),
    \]
    so that
    \begin{align*}
        \ind(\tilde{\X}_{0:\tx}(\ax_{1:\tx}) \in \B_{0:\tx})
        &\eqas \ind(\tilde{\X}_{0:\N}(\ax_{1:\N}) \in \B_{0:\N}, \xx_{\N+1:\tx} \in \B_{\N+1:\tx}) \\
        &\eqas \ind(\X_{0:\N}(\A_{1:\N}) \in \B_{0:\N})
    \end{align*}
    since each $\xx_{\sx} \in \B_\sx$.
    Consequently,
    \begin{align*}
        \E[\tilde{\Y}(\ax_{1:\tx}) \mid \tilde{\X}_{0:\tx}(\ax_{1:\tx}) \in \B_{0:\tx}]
        &= \E[\Ylo \mid \tilde{\X}_{0:\tx}(\ax_{1:\tx}) \in \B_{0:\tx}] \\
        &= \E[\Ylo \mid \X_{0:\N}(\A_{1:\N}) \in \B_{0:\N}],
    \end{align*}
    which gives the result.
\end{proof}

\subsection{Bounds on the conditional expectation given specific covariate values} \label{sec:impossibility-of-bounds-for-continuous-data}

Theorem \ref{thm:causal-bounds} provides a bound on $\E[\Y(\ax_{1:\tx}) \mid \X_{0:\tx}(\ax_{1:\tx}) \in \B_{0:\tx}]$, i.e.\ the conditional expectation given the \emph{event} $\{\X_{0:\tx}(\ax_{1:\tx}) \in \B_{0:\tx}\}$, which is assumed to have positive probability.
We consider here the prospect of obtaining bounds on $\E[\Y(\ax_{1:\tx}) \mid \X_{0:\tx}(\ax_{1:\tx})]$, i.e.\ the conditional expectation given the \emph{value} of $\X_{0:\tx}(\ax_{1:\tx})$.
For falsification purposes, this would provide a means for determining that twin is incorrect when it outputs specific values of $\Xt_{0:\tx}(\ax_{1:\tx})$, rather than just that it is incorrect on average across all runs that output values $\Xt_{0:\tx}(\ax_{1:\tx}) \in \B_{0:\tx}$.

When $\X_{0:\tx}(\ax_{1:\tx})$ is discrete, Theorem \ref{thm:causal-bounds} yields measurable functions $\lo{\gx}, \up{\gx} : \Xspace_{0:\tx} \to \R$ such that
\begin{equation} \label{eq:psi-lo-up-defining-property}
    \lo{\gx}(\X_{0:\tx}(\ax_{1:\tx}))
        \leq \E[\Y(\ax_{1:\tx}) \mid \X_{0:\tx}(\ax_{1:\tx})] \leq \up{\gx}(\X_{0:\tx}(\ax_{1:\tx})) \qquad \text{almost surely}.
\end{equation}
In particular, $\lo{\gx}(\xx_{0:\tx})$ is obtained as the value of $\E[\Ylo \mid \X_{0:\N}(\A_{1:\N}) \in \B_{0:\N}]$ for $\B_{0:\tx} \coloneqq \{\xx_{0:\tx}\}$, and similarly for $\up{\gx}(\xx_{0:\tx})$.
Moreover, since the constants $\ylo, \yup \in \R$ in Theorem \ref{thm:causal-bounds} were allowed to depend on $\B_{0:\tx}$, and hence here on each choice of $\xx_{0:\tx} \in \Xspace_{0:\tx}$, we may think of these now as measurable functions $\ylo, \yup : \Xspace_{0:\tx} \to \R$ satisfying
\begin{equation} \label{eq:y-boundedness-functional-assumption}
    \ylo(\X_{0:\tx}(\ax_{1:\tx})) \leq \Y(\ax_{1:\tx}) \leq \yup(\X_{0:\tx}(\ax_{1:\tx})) \qquad \text{almost surely}.
\end{equation}
In other words, when $\X_{0:\tx}(\ax_{1:\tx})$ is discrete, Theorem \ref{thm:causal-bounds} provides bounds on the conditional expectation of $\Y(\ax_{1:\tx})$ given the value of $\X_{0:\tx}(\ax_{1:\tx})$ whenever we have $\ylo$ and $\yup$ such that \eqref{eq:y-boundedness-functional-assumption} holds.

When $\Prob(\X_{1:\tx}(\ax_{1:\tx}) \in \B_{1:\tx}) > 0$, a fairly straightforward modification of the proof of Theorem \ref{thm:causal-bounds} yields bounds of the following form:
\begin{align}
    \E[\Ylo \mid \X_0, \X_{1:\N}(\A_{1:\N}) \in \B_{1:\N}] 
        &\leq \E[\Y(\ax_{1:\tx}) \mid \X_0, \X_{1:\tx}(\ax_{1:\tx}) \in \B_{1:\tx}] \notag \\
        &\qquad\qquad\leq \E[\Yup \mid \X_0, \X_{1:\N}(\A_{1:\N}) \in \B_{1:\N}] \qquad \text{almost surely}. \label{eq:bareinboim-style-bounds-with-continuous-initial-covariates}
\end{align}
In particular, this holds regardless of whether or not $\X_0$ is discrete.
In turn, if $\X_{1:\tx}(\ax_{1:\tx})$ is discrete, then by a similar argument as was given in the previous subsection, this yields almost sure bounds on $\E[\Y(\ax_{1:\tx}) \mid \X_{0:\tx}(\ax_{1:\tx})]$ of the form in \eqref{eq:psi-lo-up-defining-property}, provided \eqref{eq:y-boundedness-functional-assumption} holds.
Alternatively, by taking $\B_{1:\tx} \coloneqq \Xspace_{1:\tx}$, \eqref{eq:bareinboim-style-bounds-with-continuous-initial-covariates} yields bounds of the form
\[
    \E[\Ylo \mid \X_0] \leq \E[\Y(\ax_{1:\tx}) \mid \X_0] \leq \E[\Yup \mid \X_0].
\]
If the action sequence $\ax_{1:\tx}$ is thought of as a single choice of an action from the extended action space $\Aspace_{1:\tx}$, then this recovers the bounds originally proposed by \citet{manski}, which allowed conditioning on potentially continuous pre-treatment covariates corresponding to our $\X_0$.

\subsection{Proof of Theorem \ref{thm:no-causal-bounds-for-continuous-data} and discussion}

\begin{proof}
    Suppose we have a permissible $\lo{\gx}$.
    (The case of $\up{\gx}$ is analogous).
    Choose $\xx_{1:\T} \in \Xspace_{1:\T}$ arbitrarily, and define new potential outcomes
    \begin{align*}
        \tilde{\X}_0 &\coloneqq \X_0 \\
        \tilde{\X}_r(\ax_{1:r}') &\coloneqq %
            \ind(\A_{1:r} = \ax_{1:r}') \, \X_{r}(\ax_{1:r}') + \ind(\A_{1:r} \neq \ax_{1:r}') \, \xx_{r} \qquad \text{for $r \in \{1, \ldots, \T\}$ and $\ax_{1:r}' \in \Aspace_{1:r}$}.
    \end{align*}
    Similarly, define
    \begin{align*}
        \tilde{\Y}(\ax_{1:\tx}') &\coloneqq \ind(\A_{1:\tx} = \ax_{1:\tx}') \, \Y(\ax_{1:\tx}') + \ind(\A_{1:\tx} \neq \ax_{1:\tx}') \, \ylo(\tilde{\X}_{0:\tx}(\ax_{1:\tx}')) \qquad \text{for all $\ax_{1:\tx}' \in \Aspace_{1:\tx}$}.
    \end{align*}
    It immediately follows that
    \[
        (\tilde{\X}_{0:\T}(\A_{1:\T}), \tilde{\Y}(\A_{1:\tx}), \A_{1:\T}) \eqas (\X_{0:\T}(\A_{1:\T}), \Y(\A_{1:\tx}), \A_{1:\T}).
    \]
    Moreover, it is easily checked that
    \[
        \ylo(\tilde{\X}_{0:\tx}(\ax_{1:\tx})) \leq \tilde{\Y}(\ax_{1:\tx}) \leq \yup(\tilde{\X}_{0:\tx}(\ax_{1:\tx})) \qquad \text{almost surely}.
    \]
    As such, since $\lo{\gx}$ is permissible, we must have, almost surely,
    \begin{align}
        \lo{\gx}(\tilde{\X}_{0:\tx}(\ax_{1:\tx})) &\leq \E[\tilde{\Y}(\ax_{1:\tx}) \mid \tilde{\X}_{0:\tx}(\ax_{1:\tx})] \notag \\
         &= \begin{multlined}[t]
            \E[\tilde{\Y}(\A_{1:\tx}) \mid \tilde{\X}_{0:\tx}(\ax_{1:\tx}), \A_{1:\tx} = \ax_{1:\tx}] \, \Prob(\A_{1:\tx} = \ax_{1:\tx} \mid \tilde{\X}_{0:\tx}(\ax_{1:\tx})) \\
                + \underbrace{\E[\tilde{\Y}(\ax_{1:\tx}) \mid \tilde{\X}_{0:\tx}(\ax_{1:\tx}), \A_{1:\tx} \neq \ax_{1:\tx}]}_{=\ylo(\tilde{\X}_{0:\tx}(\ax_{1:\tx}))} \, \Prob(\A_{1:\tx} \neq \ax_{1:\tx} \mid \tilde{\X}_{0:\tx}(\ax_{1:\tx})).
        \end{multlined} \label{eq:no-continuous-bounds-proof-convex-combination}
    \end{align}
    Now, by our definition of $\tilde{\X}_{0:\tx}(\ax_{1:\tx})$, we have almost surely
    \begin{align*}
        \ind(\A_{1} \neq \ax_{1}) \, \Prob(\A_{1:\tx} = \ax_{1:\tx} \mid \tilde{\X}_{0:\tx}(\ax_{1:\tx}))
            &= \ind(\A_{1} \neq \ax_{1}, \tilde{\X}_{\sx}(\ax_{1:\sx}) = \xx_\sx) \, \Prob(\A_{1:\tx} = \ax_{1:\tx} \mid \tilde{\X}_{0:\tx}(\ax_{1:\tx})) \\
            &= \ind(\A_{1} \neq \ax_{1}) \, \E[\ind(\A_{1:\tx} = \ax_{1:\tx}, \tilde{\X}_{\sx}(\ax_{1:\sx}) = \xx_\sx) \mid \tilde{\X}_{0:\tx}(\ax_{1:\tx})] \\
            &= \ind(\A_{1} \neq \ax_{1}) \, \E[\ind(\A_{1:\tx} = \ax_{1:\tx}, \X_{\sx}(\A_{1:\sx}) = \xx_\sx) \mid \tilde{\X}_{0:\tx}(\ax_{1:\tx})] \\
            &= 0,
    \end{align*}
    where the last step follows by our assumption that $\Prob(\X_{\sx}(\A_{1:\sx}) = \xx_{\sx}) = 0$.
    Combining this with \eqref{eq:no-continuous-bounds-proof-convex-combination}, we get, almost surely,
    \begin{align}
        \ind(\A_{1} \neq \ax_{1}) \, \lo{\gx}(\X_{0}, \xx_{1:\tx}) &=  \ind(\A_{1} \neq \ax_{1}) \, \lo{\gx}(\tilde{\X}_{0:\tx}(\ax_{1:\tx})) \notag \\
            &\leq \ind(\A_{1} \neq \ax_{1}) \, \ylo(\tilde{\X}_{0:\tx}(\ax_{1:\tx})) \notag \\
            &= \ind(\A_{1} \neq \ax_{1}) \, \ylo(\X_{0}, \xx_{1:\tx}). \label{eq:no-continuous-bounds-intermediate-step}
    \end{align}
    Now let $\xx_0 \in \Xspace_0$ be the value such that $\Prob(\X_0 = \xx_0) = 1$.
    Using our assumption that $\Prob(\A_1 \neq \ax_1) > 0$ and the fact that $\xx_{1:\tx}$ was arbitrary, we obtain
    \[
        \lo{\gx}(\xx_{0:\tx}) \leq \ylo(\xx_{0:\tx}) \qquad \text{for all $\xx_{1:\tx} \in \Xspace_{1:\tx}$}.
    \]
    The result now follows.
\end{proof}

To gain intuition for the phenomenon underlying Theorem \ref{thm:no-causal-bounds-for-continuous-data}, consider a simplified model consisting of $\Xspace$-valued potential outcomes $(\X(\ax') : \ax \in \Aspace)$, $\R$-valued potential outcomes $(\Y(\ax') : \ax \in \Aspace)$, and an $\Aspace$-valued random variable $\A$ representing the choice of action.
(This constitutes a special case of our setup with $\T = 1$ and $\Xspace_0$ taken to be a singleton set.)
Suppose moreover that the following conditions hold:
\begin{align*}
    \Prob(\X(\A) = \xx) &= 0 \qquad \text{for all $\xx \in \Xspace$} \\
    \Prob(\A = \ax) &< 1.
\end{align*}
We then have
\begin{equation} \label{eq:no-continuous-bounds-toy-example}
    \E[\Y(\ax) \mid \X(\ax)]
        \eqas \E[\Y(\A) \mid \X(\A), \A = \ax] \, \Prob(\A = \ax \mid \X(\ax)) + \E[\Y(\ax) \mid \X(\ax), \A \neq \ax] \, \Prob(\A \neq \ax \mid \X(\ax)).
\end{equation}
But now, since the behaviour of $\X(\ax)$ is only observed on $\{\A = \ax\}$, for any given value of $\xx \in \Xspace$, we cannot rule out the possibility that
\[
    \X(\ax) = \ind(\A = \ax) \, \X(\A) + \ind(\A \neq \ax) \, \xx \qquad \text{almost surely}.
\]
In turn, since $\Prob(\A = \ax) > 0$, this would imply $\Prob(\X(\ax) = \xx) > 0$, and, since $\Prob(\X(\A) = \xx) = 0$, that $\Prob(\A = \ax \mid \X(\ax) = \xx) = 0$.
From \eqref{eq:no-continuous-bounds-toy-example}, this would yield
\[
    \E[\Y(\ax) \mid \X(\ax) = \xx] = \E[\Y(\ax) \mid \X(\ax) = \xx, \A \neq \ax].
\]
But now, since the behaviour of $\Y(\ax)$ is unobserved on $\{\A \neq \ax\}$, intuitively speaking, the observational distribution does not provide any information about the value of the right-hand side, and therefore about the behaviour of $\E[\Y(\ax) \mid \X(\ax)]$ more generally since $\xx \in \Xspace$ was arbitrary.

\section{Hypothesis testing methodology} 

\subsection{Validity of testing procedure} \label{sec:hyp-testing-supplement}

We show here that our procedure for testing $\Qt \geq \Qlo$ based on the one-sided confidence intervals $\qlo{\alpha}$ and $\qt{\alpha}$ has the correct probability of type I error, provided $\qlo{\alpha}$ and $\qt{\alpha}$ have the correct coverage probabilities.
In particular, the result below (which applies a standard union bound argument) shows that if $\Qt \geq \Qlo$, then our test rejects (i.e.\ $\qt{\alpha} < \qlo{\alpha}$) with probability at most $\alpha$.
An analogous result is easily proven for testing $\Qt \leq \Qup$ also, with $\qlo{\alpha}$ replaced by a one-sided upper $(1-\alpha/2)$-confidence interval for $\Qup$, and $\qt{\alpha}$ replaced by a one-sided lower $(1 - \alpha/2)$-confidence interval for $\Qt$.

\begin{proposition}
Suppose that for some $\alpha \in (0, 1)$ we have random variables $\qt{\alpha}$ and $\qlo{\alpha}$ satisfying
\begin{align}
    \Prob(\Qlo \geq \qlo{\alpha}) &\geq 1 - \frac{\alpha}{2} \label{eq:qlo-confidence-interval-guarantee} \\
    \Prob(\Qt \leq \qt{\alpha}) &\geq 1 - \frac{\alpha}{2} \label{eq:qt-confidence-interval-guarantee}.
\end{align}
If $\Qt \geq \Qlo$, then $\Prob(\qt{\alpha} < \qlo{\alpha}) \leq \alpha$.
\end{proposition}

\begin{proof}
If $\Qt \geq \Qlo$, then we have
\[
    \{\qt{\alpha} < \qlo{\alpha}\} \subseteq \{\Qt > \qt{\alpha}\} \cup \{\Qlo < \qlo{\alpha}\}.
\]
To see this, note that
\[
    (\{\Qt > \qt{\alpha}\} \cup \{\Qlo < \qlo{\alpha}\})^c
        = \{\Qt > \qt{\alpha}\}^c \cap \{\Qlo < \qlo{\alpha}\}^c
        = \{\Qt \leq \qt{\alpha}\} \cap \{\Qlo \geq \qlo{\alpha}\}
        \subseteq \{\qlo{\alpha} \leq \qt{\alpha}\}.
\]
As such,
\begin{align*}
    \Prob(\qt{\alpha} < \qlo{\alpha}) \leq \Prob(\{\Qt > \qt{\alpha}\} \cup \{\Qlo < \qlo{\alpha}\}) \leq \Prob(\Qt > \qt{\alpha}) + \Prob(\Qlo < \qlo{\alpha}) \leq \alpha/2 + \alpha/2 = \alpha.
\end{align*}
\end{proof}

\subsection{Unbiased sample mean estimates of $\Qlo$, $\Qt$, and $\Qup$} \label{sec:confidence-intervals-methodology-supplement}

We use our data to obtain one-sided confidence intervals $\qlo{\alpha}$ and $\qt{\alpha}$ satisfying \eqref{eq:qlo-confidence-interval-guarantee} and \eqref{eq:qt-confidence-interval-guarantee} as required by our procedure for testing $\Qt \geq \Qlo$.
We use an analogous procedure to obtain confidence intervals for testing $\Qt \leq \Qup$.
We tried two techniques for this: an exact method based on Hoeffding's inequality, and an approximate method based on bootstrapping.
Conceptually, both are based on obtaining unbiased sample mean estimates of $\Qlo$ and $\Qt$, which we describe now, before giving the particulars of each method in the next two subsections.

We begin with our sample mean estimator of $\Qlo$.
Recall that we assume access to a dataset $\D$ consisting of i.i.d.\ copies of observational trajectories of the form
\[
    \X_0, \A_1, \X_1(\A_1), \ldots, \A_\T, \X_\T(\A_{1:\T}).
\]
Let $\D(\ax_{1:\tx}, \B_{0:\tx})$ be the subset of trajectories in $\D$ for which $\X_{0:\N}(\A_{1:\N})\in\B_{0:\N}$.
Obtaining $\D(\ax_{1:\tx}, \B_{0:\tx})$ is possible since the only random quantity that $\N = \max\{0 \leq \sx \leq \tx \mid \A_{1:\sx} = \ax_{1:\sx}\}$ depends on is $\A_{1:\tx}$, which is included in the data.
We denote the cardinality of $\D(\ax_{1:\tx}, \B_{0:\tx})$ by $\nx \coloneqq \abs{\D(\ax_{1:\tx}, \B_{0:\tx})}$.
We then denote by $\Ylo^{(i)}$ for $i \in \{1, \ldots, \nx\}$ the corresponding values of
\begin{align*}
    \Ylo &= \ind(\A_{1:\tx} = \ax_{1:\tx}) \, \fx(\X_{0:\tx}(\A_{1:\tx})) + \ind(\A_{1:\tx} \neq \ax_{1:\tx}) \, \ylo
\end{align*}
obtained from each trajectory in $\D(\ax_{1:\tx}, \B_{0:\tx})$.
This is again possible since both terms only depends on the observational quantities $(\X_{0:\tx}(\A_{1:\tx}), \A_{1:\tx})$.
It is easily seen that the values of $\Ylo^{(i)}$ are i.i.d.\ and satisfy $\E[\Ylo^{(i)}] = \Qlo$.
As a result, the sample mean
\begin{equation} \label{eq:YClmean-definition-supplement}
    \YClmean \coloneqq \frac{1}{\nx} \sum_{i=1}^{\nx} \Ylo^{(i)} 
\end{equation}
is an unbiased estimator of $\Qlo$.

We obtain an unbiased sample mean estimate of $\Qt$ in a similar fashion as for $\Qlo$.
Recall that we assume access to a dataset $\Dt(\ax_{1:\tx})$ consisting of i.i.d.\ copies of
\[
    \X_0, \Xt_1(\X_0, \ax_1), \ldots, \Xt_t(\X_0, \ax_{1:\tx}).
\]
Let $\Dt(\ax_{1:\tx}, \B_{0:\tx})$ denote the subset of twin trajectories in $\Dt(\ax_{1:\tx})$ for which $(\X_0, \Xt_{\tx}(\X_0, \ax_{1:\tx})) \in \B_{0:\tx}$, and denote its cardinality by $\widehat{\nx} \coloneqq \abs{\Dt(\ax_{1:\tx}, \B_{0:\tx})}$.
Then denote by $\Yt^{(i)}$ for $i \in \{1 \ldots, \widehat{\nx}\}$ the corresponding values of
\[
    \Yt = \fx(\X_0, \Xt_{1:\tx}(\X_0, \ax_{1:\tx}))
\]
obtained from each trajectory in $\Dt(\ax_{1:\tx}, \B_{0:\tx})$.
It is easily seen that the values $\Yt^{(i)}$ are i.i.d.\ (since the entries of $\Dt(\ax_{1:\tx})$ are) and satisfy $\E[\Yt^{(i)}] = \Qt$.
As a result, the sample mean
\[
    \Ytmean \coloneqq \frac{1}{\widehat{\nx}} \sum_{i=1}^{\widehat{\nx}} \Yt^{(i)}
\]
is an unbiased estimator of $\Qt$.

\subsection{Exact confidence intervals via Hoeffding's inequality}

Recall that we assume in Section \ref{sec:hypotheses-from-causal-bounds-setup} that $\Y(\ax_{1:\tx})$ has the form $\Y(\ax_{1:\tx}) = \fx(\X_{0:\tx}(\ax_{1:\tx}))$, and that moreover
\begin{equation} \label{eq:f-boundedness-assumption-hoeffding-proof}
    \ylo \leq \fx(\xx_{0:\tx}) \leq \yup \qquad \text{for all $\xx_{0:\tx} \in \B_{0:\tx}$.}
\end{equation}
This means $\Yt^{(i)}$ is almost surely bounded in $[\ylo, \yup]$, and so $\Ytmean$ gives rise to one-sided confidence intervals via an application of Hoeffding's inequality.
The exact form of these confidence intervals is as follows:

\begin{proposition} \label{prop:hoeffding-confidence-bounds-supp}
If \eqref{eq:f-boundedness-assumption-hoeffding-proof} holds, then for each $\alpha \in (0, 1)$, letting
\[
    \CIlen \coloneqq (\yup - \ylo) \, \sqrt{\frac{1}{2 \nx} \, \log \frac{2}{\alpha}}  \qquad \text{and} \qquad \widehat{\CIlen} \coloneqq (\yup - \ylo) \, \sqrt{\frac{1}{2 \widehat{\nx}} \, \log \frac{2}{\alpha}},
\]
and similarly
\begin{align*}
    \qlo{\alpha} \coloneqq \YClmean - \CIlen  \qquad \text{and} \qquad
    \qt{\alpha} \coloneqq \Ytmean + \widehat{\CIlen},
\end{align*}
it follows that
\begin{align*}
    \Prob(\Qlo \geq \qlo{\alpha}) \geq 1 - \frac{\alpha}{2} \qquad \text{and} \qquad \Prob(\Qt \leq \qt{\alpha}) \geq 1 - \frac{\alpha}{2}.
\end{align*}
\end{proposition}

\begin{proof}
We only prove the result for $\qlo{\alpha}$; the other statement can be proved analogously.
Recall that $\YClmean$ is the empirical mean of i.i.d.\ samples $\Ylo^{(i)}$ for $i\in \{1, \ldots, \nx\}$ with $\E[\Ylo^{(i)}]=\Qlo$ (see \eqref{eq:YClmean-definition-supplement}).
Moreover, by \eqref{eq:f-boundedness-assumption-hoeffding-proof}, $\Ylo^{(i)}$ is almost surely bounded in $[\ylo, \yup]$.
Hoeffding's inequality then implies that
\begin{align*}
    \Prob(\YClmean - \Qlo > \CIlen) &\leq \exp\left(- \frac{2 \nx \CIlen^2}{(\yup - \ylo)^2 } \right).
\end{align*}
In turn, some basic manipulations yield
\begin{align*}
    \Prob(\Qlo \geq \qlo{\alpha}) &= 1 - \Prob(\Qlo < \YClmean - \CIlen) \\
    &\geq 1 - \exp\left(- \frac{2 \nx \CIlen^2}{(\yup - \ylo)^2 } \right) \\
    &= 1- \frac{\alpha}{2}.
\end{align*}
\end{proof}

\subsection{Approximate confidence intervals via bootstrapping} \label{subsec:bootstrapping}

While Hoeffding's inequality yields the probability guarantees in \eqref{eq:qlo-confidence-interval-guarantee} and \eqref{eq:qt-confidence-interval-guarantee} exactly, the confidence intervals obtained can be conservative.
Consequently, our testing procedure may have lower probability of falsifying certain twins that in fact do not satisfy the causal bounds.
To address this, we also consider an approximate approach based on bootstrapping that can produce tighter confidence intervals.
While other schemes are possible, bootstrapping provides a general-purpose approach that is straightforward to implement and works well in practice.

At a high level, our approach here is again to construct one-sided level $1 - \alpha/2$ confidence intervals via bootstrapping \citep{efron1979bootstrap} on $\Qlo$ and $\Qt$.
Many bootstrapping procedures for obtaining confidence intervals have been proposed in the literature \citep{tibshirani1993introduction,davison1997bootstrap,hesterberg2015what}. 
Our results reported below were obtained via the \emph{reverse percentile} bootstrap (see \cite{hesterberg2015what} for an overview).
(We also tried the \emph{percentile} bootstrap method, which obtained nearly indistinguishable results.)
In particular, this method takes
\[
    \qlo{\alpha} \coloneqq 2 \YClmean - \Delta 
    \qquad \qt{\alpha} \coloneqq 2 \widehat{\mu} - \widehat{\Delta},
\]
where $\Delta$ and $\widehat{\Delta}$ correspond to the approximate $1 - \alpha / 2$ and $\alpha / 2$ quantiles of the distributions of 
\[
    \frac{1}{\nx} \sum_{i=1}^{\nx} \Ylo^{(i^\ast)} 
    \qquad\text{and}\qquad
    \frac{1}{\widehat{\nx}} \sum_{i=1}^{\widehat{\nx}} \Yt^{(i^\ast)},
\]
where each $\Ylo^{(i^\ast)}$ and $\Y^{(i^\ast)}$ is obtained by sampling uniformly with replacement from among the values of $\Ylo^{(i)}$ and $\Y^{(i)}$.
In our case study, as is typically done in practice, we approximated $\Delta$ and $\widehat{\Delta}$ via Monte Carlo sampling.
It can be shown that the confidence intervals produced in this way obtain a coverage level that approaches the desired level of $1 - \alpha/2$ as $\nx$ and $\widehat{\nx}$ grow to infinity under mild assumptions \citep{hall1988theoretical}.

\section{Experimental Details} \label{sec:experiments-supplement}

\subsection{MIMIC preprocessing} \label{sec:mimic-preprocessing-supp}

For data extraction and preprocessing, we re-used the same procedure as \cite{ai-clinician} with minor modifications.
For completeness, we describe the pre-processing steps applied in \cite{ai-clinician} and subsequently outline our modifications to these.

Following \cite{ai-clinician}, we extracted adult patients fulfilling the sepsis-3 criteria \citep{sepsis-criteria}. 
Sepsis was defined as a suspected infection (as indicated by prescription of antibiotics and sampling of bodily fluids for microbiological culture) combined with evidence of organ dysfunction, defined by a SOFA score $\geq 2$ \citep{sepsis-criteria, seymour2016assessment}.  

Following \cite{ai-clinician}, we excluded patients for whom any of the following was true: their age was less than $18$ years old at the time of ICU admission; their mortality not documented; their IV fluid/vasopressors intake was not documented; their treatment was withdrawn.

We made the following modifications to the preprocessing code of \cite{ai-clinician} for our experiment.
First, instead of extracting physiological quantities (e.g.\ heart rate) every 4 hours, we extracted these every hour.
Additionally, we excluded patients with any missing hourly vitals during the first 4 hours of their ICU stay.

We then extracted a total of 19 quantities of interest listed in Table \ref{tab:mimic-features}.
Of these, 17 were physiological quantities associated with the patient, including static demographic quantities (e.g.\ age), patient vital signs (e.g.\ heart rate), and patient lab values (e.g.\ potassium blood concentration).
All of these were continuous values, apart from sex.
These were chosen as the subset of physiological quantities extracted from MIMIC by \cite{ai-clinician} that are also modelled by Pulse, and were used to define our observation spaces $\Xspace_\tx$ as described next.
The remaining 2 quantities (intravenous fluids and vasopressor doses) were chosen since they correspond to treatments that the patient received, and were used to define our action spaces $\Aspace_\tx$ as described below.

\subsection{Sample splitting}

Before proceeding further, we randomly selected 5\% of the extracted our trajectories (583 trajectories, denoted as $\D_0$) to use for preliminary tasks such as choosing the parameters of our hypotheses.
We reserved the remaining 95\% (11,094 trajectories, denoted as $\D$) for the actual testing.
By a standard sample splitting argument \citep{cox1975note}, the statistical guarantees of our testing procedure established above continue to apply even when our hypotheses are defined in this data-dependent way.

\begin{table}[t]%
\centering
\begin{footnotesize}
\begin{tabular}{lll}
Category &  Physiological quantity \\
\midrule
Demographic & Age \\
& Sex  \\
& Weight \\
\midrule
Vital Signs & Heart rate (HR)  \\
& Systolic blood pressure (SysBP)  \\
& Diastolic blood pressure (DiaBP)  \\
& Mean blood pressure (MeanBP) \\
& Respiratory Rate (RR) \\
& Skin Temperature (Temp) \\
\midrule
Lab Values & Potassium Blood Concentration (Potassium)  \\
& Sodium Blood Concentration (Sodium)  \\
& Chloride Blood Concentration (Chloride)  \\
& Glucose Blood Concentration (Glucose)  \\
& Calcium Blood Concentration (Calcium)  \\
& Bicarbonate Blood Concentration ($\textup{HCO}_3$)  \\
& Arterial $\textup{O}_2$ Pressure ($\textup{PaO}_2$)  \\
& Arterial $\textup{CO}_2$ Pressure ($\textup{PaCO}_2$) \\
\midrule
Treatments & Intravenous fluid (IV) dose \\
& Vasopressor dose \\
\bottomrule
\end{tabular}
\end{footnotesize}
\caption{Physiological quantities and treatments extracted from MIMIC}\label{tab:mimic-features}
\end{table}

\subsection{Observation spaces} \label{sec:observation-space-definition-supp}

Our $\Xspace_0$ consisted of the following features: age, sex, weight, heart rate, systolic blood pressure, diastolic blood pressure and respiration rate.
We chose $\Xspace_0$ in this way because, out of the 17 physiological quantities we extracted from MIMIC, these were the quantities that can be initialised to user-provided values before starting a simulation in the version of Pulse we considered (4.x).
(In contrast, Pulse initialises the other 10 features to default values.)
For the remaining observation spaces, we used the full collection of the 17 physiological quantities we extracted to define $\Xspace_1 = \cdots = \Xspace_4$.
We encoded all features in $\Xspace_\tx$ numerically, i.e.\ $\Xspace_0 = \R^7$, and $\Xspace_\tx = \R^{17}$ for $\tx \in \{1, 2, 3, 4\}$.

\subsection{Action spaces} \label{sec:action-space-definition-supp}

Following \cite{ai-clinician}, we constructed our action space using 2 features obtained from MIMIC, namely intravenous fluid (IV) and vasopressor doses.
To obtain discrete action spaces suitable for our framework, we used the same discretization procedure for these quantities as was used by \cite{ai-clinician}.
Specifically, we divided the hourly doses of intravenous fluids and vasopressors into 5 bins each, with the first bin corresponding to zero drug dosage, and the remaining 4 bins based on the quartiles of the non-zero drug dosages in our held-out observational dataset $\D_0$.
From this we obtained action spaces $\Aspace_1 = \cdots = \Aspace_{4}$ with $5 \times 5 = 25$ elements. 
Table \ref{tab:act_space} shows the dosage bins constructed in this way, as well as the frequency of each bin's occurrence in the observational data.

\begin{table}[t]%
    \centering
    \begin{footnotesize}
\begin{tabular}{ll|ccccc}
\multicolumn{1}{c}{} & \multicolumn{1}{c}{} & \multicolumn{5}{c}{Vasopressor dose ($\mu$g/kg/min)}\\
\multicolumn{1}{c}{} & \multicolumn{1}{c}{} &      0 &  0.0 - 0.061 &  0.061 - 0.15 &  0.15 - 0.313 &  $>0.313$ \\
\midrule
\multirow{5}{*}{IV dose (mL/h)} & 0        &  16659 &          329 &           256 &           152 &      145 \\
& 0 - 20   &   5840 &          428 &           351 &           244 &      145 \\
& 20 - 75  &   6330 &          297 &           378 &           383 &      309 \\
& 75 - 214 &   6232 &          176 &           175 &           197 &      273 \\
& $>214$    &   5283 &          347 &           488 &           544 &      747 \\
\end{tabular}
    \end{footnotesize}
\caption{Action space with frequency of occurrence in observational data} \label{tab:act_space}
\end{table}

\subsection{Hypothesis parameters} \label{sec:hypothesis-parameters-supplement}

We used our held-out observational dataset $\D_0$ to obtain a collection of hypothesis parameters $(\tx, \fx, \ax_{1:\tx}, \B_{0:\tx})$.
Specifically, for each physiological quantity of interest (e.g.\ heart rate) in the list of `Vital Signs' and `Lab Values' given in Table \ref{tab:mimic-features}, we did the following.
First, for each $\tx \in \{0, \ldots, 4\}$, we obtained 16 choices of $\B_{\tx}$ by discretizing the patient space $\Xspace_{\tx}$ into 16 subsets based on the values of certain features as follows: 2 bins corresponding to sex; 4 bins corresponding to the quartiles of the ages of patients in $\D_0$; 2 bins corresponding to whether or not the value of the chosen physiological quantity of interest at time $\tx$ was above or below its median value in $\D_0$.

Next, for each $\tx \in \{1, \ldots, 4\}$, $\ax_{1:\tx} \in \Aspace_{1:\tx}$, and sequence $\B_{0:\tx}$ with each $\B_{\tx'}$ as defined in the previous step, let $\D_0(\tx, \ax_{1:\tx}, \B_{0:\tx})$ denote the subset of $\D_0$ corresponding to $(\tx, \ax_{1:\tx}, \B_{0:\tx})$, i.e.\
\[
    \D_0(\tx, \ax_{1:\tx}, \B_{0:\tx}) \coloneqq \{\X_{0:\tx}(\A_{1:\tx}) \mid \text{$(\X_{0:\T}(\A_{1:\T}), \A_{1:\T}) \in \D_0$ with $\A_{1:\tx} = \ax_{1:\tx}$ and $\X_{0:\tx}(\A_{1:\T}) \in \B_{0:\tx}$}\}.
\]
We then selected the set of all triples $(\tx, \ax_{1:\tx}, \B_{0:\tx})$ such that $\D_0(\tx, \ax_{1:\tx}, \B_{0:\tx})$ contained at least one trajectory.
This meant the number of combinations of hypotheses parameters that we considered was limited to a tractable quantity, which had benefits both computationally, and also by ensuring that we did not sacrifice too much power when adjusting for multiple testing.

Finally, for each selected triple $(\tx, \ax_{1:\tx}, \B_{0:\tx})$, we chose a corresponding $\fx$ as follows.
First, we let $i \in \{1, \ldots, \Xspacedim_\tx\}$ denote the index of the physiological quantity of interest in $\Xspace_\tx = \R^{\Xspacedim_\tx}$.
We then set $\ylo, \yup$ to be the .2 and the .8 quantiles of the values in
\[
    \{(\X_\tx(\A_{1:\tx}))_i \mid \X_{0:\tx}(\A_{1:\tx}) \in \D_0(\tx, \ax_{1:\tx}, \B_{0:\tx})\}
\]
We then obtained $\fx : \Xspace_{0:\tx} \to \R$ as the function that extracts the physiological quantity of interest from $\Xspace_\tx$ and clips its value to between $\ylo$ and $\yup$, i.e.\ $\fx(\xx_{0:\tx}) \coloneqq \min(\max(\xx_{\tx})_{i}, \ylo), \yup)$.

Overall, accounting for all physiological quantities of interest, we obtained 721 distinct choices of $(\tx, \fx, \ax_{1:\tx}, \B_{0:\tx})$ in this way.
Figure \ref{fig:n_histograms} shows the amount of non-held out observational and twin data that we subsequently used for testing each hypothesis, i.e.\ the values of $n$ and $\widehat{n}$ as defined in Section \ref{sec:confidence-intervals-methodology-supplement} above.
(We describe how we generated our dataset of twin trajectories in Section \ref{sec:pulse-trajectories-supplement}.)

\subsection{Generating twin trajectories using the Pulse Physiology Engine}\label{sec:pulse-trajectories-supplement}

The Pulse Physiology Engine is an open source comprehensive human physiology simulator that has been used in medical education, research, and training. The core engine of Pulse is C++ based with APIs available in different languages, including python. Detailed documentation is available at: \href{https://pulse.kitware.com/}{pulse.kitware.com}.
Pulse allows users to initialize patient trajectories with given age, sex, weight, heart rate, systolic blood pressure, diastolic blood pressure and respiration rate and medical conditions such as sepsis, COPD, ARDS, etc. Once initialised, users have the ability to advance patient trajectories by a given time step (one hour in our case), and administer actions (e.g. administer a given dose of IV fluids or vasopressors).

In Algorithm \ref{algo:twin-data-generation} we describe how we generated the twin data to test the chosen hypotheses. Note that we sampled $\X_0$ without replacement as it ensures that each $\X_0$ is chosen at most once and consequently twin trajectories in $\Dt(\ax_{1:\tx})$ are i.i.d. 
Additionally, Algorithm \ref{algo:twin-data-generation} can be easily parallelised to improve efficiency.
Figure \ref{fig:n_histograms} shows histograms of the number of twin trajectories $\widehat{\nx}$ (as defined in Section \ref{sec:confidence-intervals-methodology-supplement} above) obtained in this way across all hypotheses.

\begin{algorithm}
\SetAlgoLined
\textbf{Inputs:} Action sequence $\ax_{1:\tx}$; Observational dataset $\D$.\\
\textbf{Output:} Twin data $\Dt(\ax_{1:\tx})$ of size $m$.\\
\For{$i = 1, \dots, m$}{
Sample $\X_0$ without replacement from $\D$\;
$\Xt_0 \leftarrow  \X_0$ i.e., initialize the Pulse trajectory with the information of $\X_0$\;
\For{$\tx' = 1, \dots, \tx$}{
Administer the median doses of IV fluids and vasopressors in action bin $\ax_{\tx'}$\;
\If{$\tx' \equiv 0$ \textup{(mod 3)}}{
Virtual patient in Pulse consumes nutrients and water, and urinates\;
}
Advance the twin trajectory by one hour\;
}
Add the trajectory $\Xt_{0:\tx}(\ax_{1:\tx})$ to $\Dt(\ax_{1:\tx})$\;
}
\textbf{Return} $\Dt(\ax_{1:\tx})$
\caption{Generating Twin data $\Dt(\ax_{1:\tx})$.}
\label{algo:twin-data-generation}
\end{algorithm}

\begin{figure}[t]
    \centering
    \includegraphics[height=21cm]{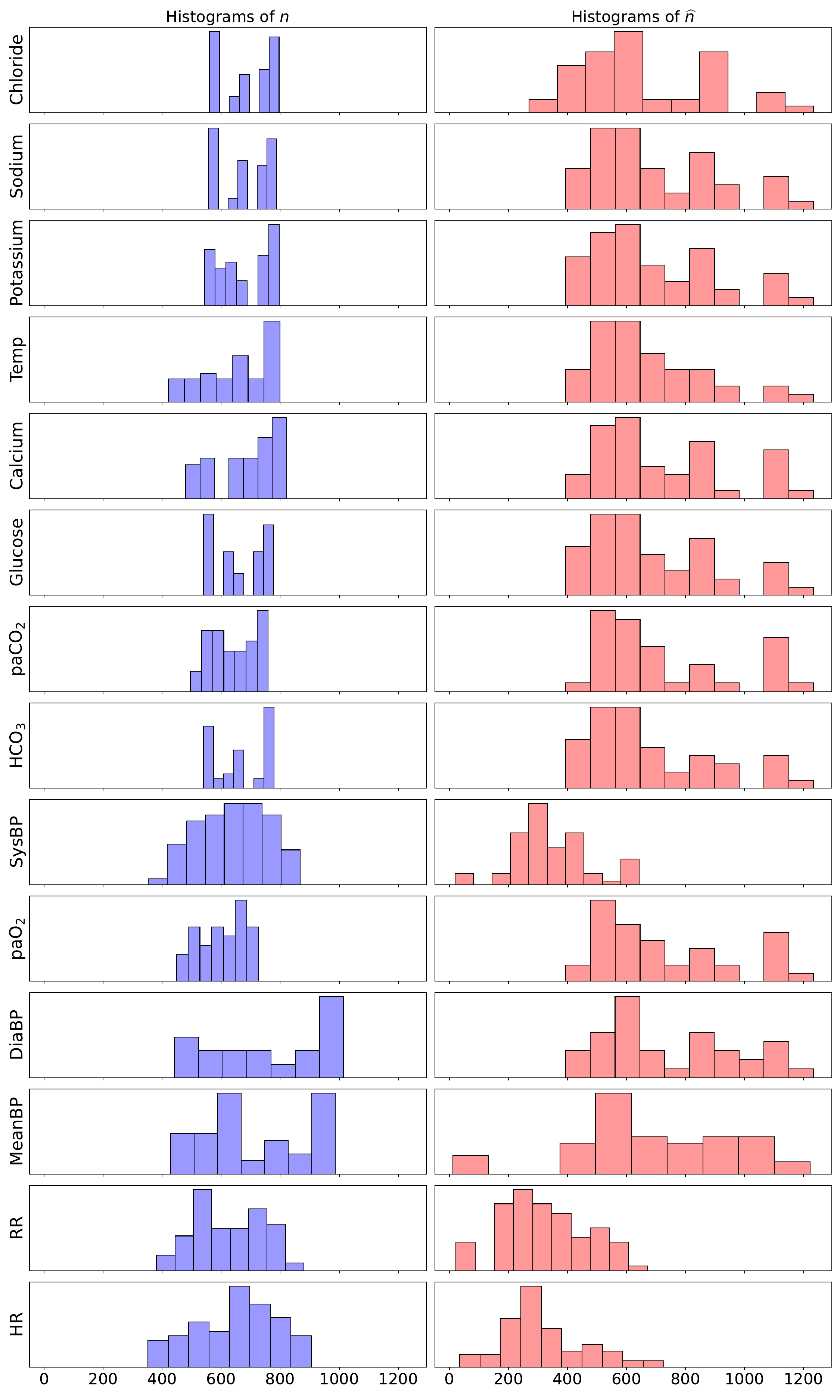}
    \caption{Histograms of $n$ and $\widehat{n}$ (as defined in Section \ref{sec:confidence-intervals-methodology-supplement}) across all hypothesis parameters corresponding to each physiological quantity of interest.}
    \label{fig:n_histograms}
\end{figure}

\begin{figure}[t]%
    \centering
    \includegraphics[height=8cm]{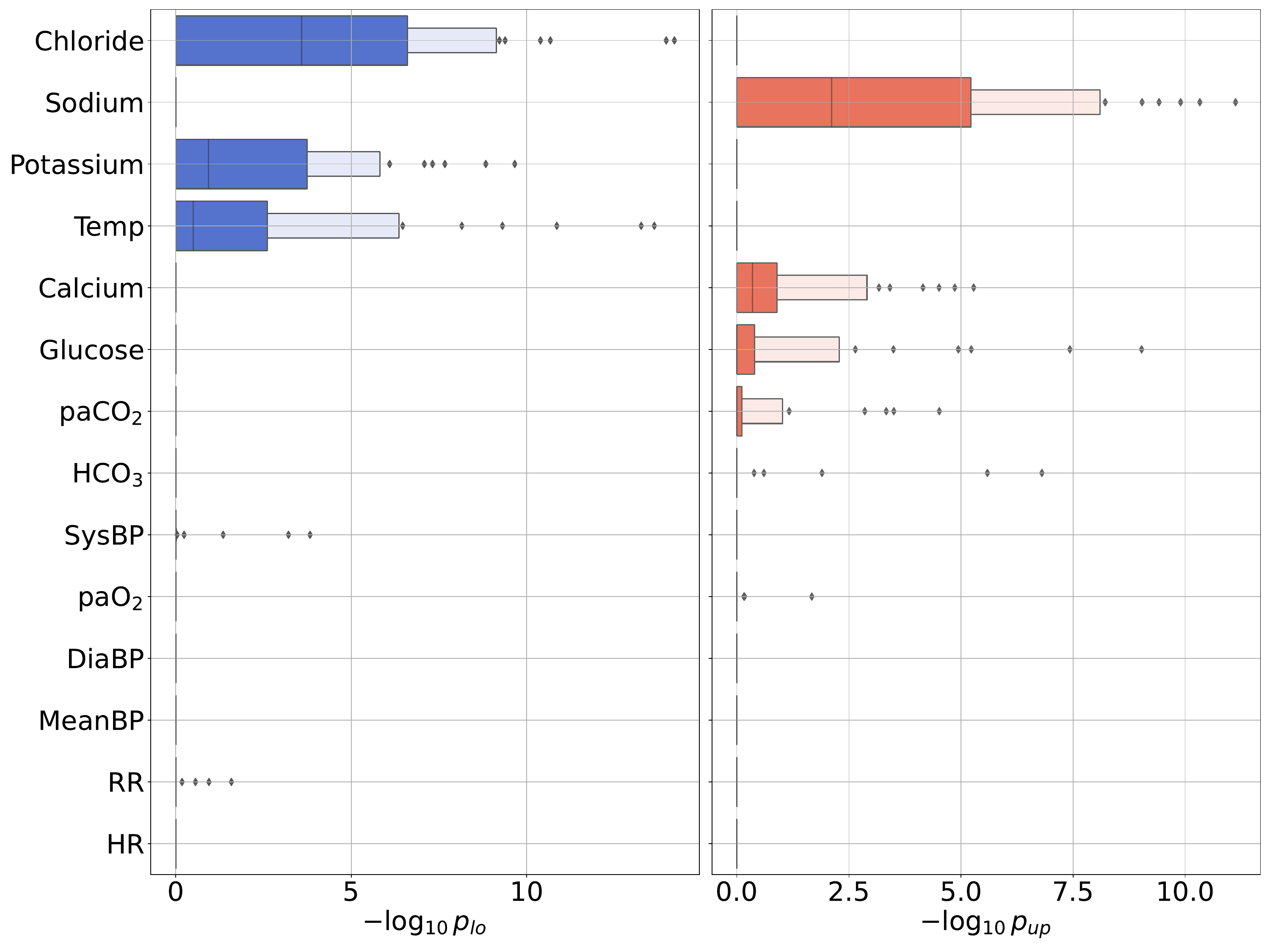}
    \caption{Boxenplots showing distributions of $-\log_{10}{p_\textup{lo}}$ and $-\log_{10}{p_\textup{up}}$ for different physiological quantities obtained via Hoeffding's inequality. Higher values indicate greater evidence in favour of rejection.}
    \label{fig:p_values_hoeff_complete}
\end{figure}

\begin{table}%
    \centering
\begin{footnotesize}
\begin{tabular}{l|cc|cc}
& \multicolumn{2}{c}{Ours} & \multicolumn{2}{c}{Manski} \\
                   Physiological quantity &  Rejs. &  Hyps. &  Rejs. &  Hyps. \\
\midrule
  Chloride Blood Concentration (Chloride) &            24 &            94 & 1 & 46 \\
      Sodium Blood Concentration (Sodium) &            21 &            94 & 9 & 46 \\
Potassium Blood Concentration (Potassium) &            13 &            94 & 0 & 46 \\
                  Skin Temperature (Temp) &            10 &            86 & 9 & 46 \\
    Calcium Blood Concentration (Calcium) &             5 &            88 & 0 & 46 \\
    Glucose Blood Concentration (Glucose) &             5 &            96 & 1 & 46 \\
      Arterial CO$_2$ Pressure (paCO$_2$) &             3 &            70 & 0 & 46 \\
Bicarbonate Blood Concentration (HCO$_3$) &             2 &            90 & 1 & 46 \\
       Systolic Arterial Pressure (SysBP) &             2 &           154 & 0 & 46 \\
        Arterial O$_2$ Pressure (paO$_2$) &             0 &            78 & 1 & 46 \\
                Arterial pH (Arterial\_pH) &             0 &            80 & 0 & 46 \\
      Diastolic Arterial Pressure (DiaBP) &             0 &            72 & 0 & 46 \\
          Mean Arterial Pressure (MeanBP) &             0 &            92 & 0 & 46 \\
                    Respiration Rate (RR) &             0 &           172 & 0 & 46 \\
                          Heart Rate (HR) &             0 &           162 & 0 & 46 \\
\bottomrule
\end{tabular}
\end{footnotesize}
    \caption{Total hypotheses (Hyps.) and rejections (Rejs.) per physiological quantity obtained using Hoeffding's inequality} \label{tab:hypotheses_hoeffding_full}
\end{table}

\begin{table}[t]
\centering
\begin{footnotesize}
\begin{tabular}{l|cc|cc}
& \multicolumn{2}{c}{Ours} & \multicolumn{2}{c}{Manski} \\
                   Physiological quantity &  Rejs. &  Hyps. &  Rejs. &  Hyps. \\
\midrule
  Chloride Blood Concentration (Chloride) &            47 &            94 & 1 & 46 \\
      Sodium Blood Concentration (Sodium) &            46 &            94 & 12 & 46 \\
Potassium Blood Concentration (Potassium) &            33 &            94 & 0 & 46 \\
                  Skin Temperature (Temp) &            43 &            86 & 13 & 46 \\
    Calcium Blood Concentration (Calcium) &            44 &            88 & 0 & 46 \\
    Glucose Blood Concentration (Glucose) &            19 &            96 & 0 & 46 \\
      Arterial CO$_2$ Pressure (paCO$_2$) &            13 &            70 & 0 & 46 \\
Bicarbonate Blood Concentration (HCO$_3$) &             8 &            90 & 0 & 46 \\
       Systolic Arterial Pressure (SysBP) &             8 &           154 & 0 & 46 \\
        Arterial O$_2$ Pressure (paO$_2$) &             4 &            78 & 1 & 46 \\
                Arterial pH (Arterial\_pH) &             0 &            80 & 0 & 46 \\
      Diastolic Arterial Pressure (DiaBP) &             0 &            72 & 0 & 46 \\
          Mean Arterial Pressure (MeanBP) &             3 &            92 & 0 & 46 \\
                    Respiration Rate (RR) &            12 &           172 & 0 & 46 \\
                          Heart Rate (HR) &             1 &           162 & 0 & 46 \\
\bottomrule
\end{tabular}
\end{footnotesize}
\caption{Total hypotheses (Hyps.) and rejections (Rejs.) per physiological quantity obtained using the reverse percentile bootstrap}  \label{tab:hypotheses_rev_percentile}
\end{table}

\subsection{Bootstrapping details} \label{sec:boostrapping-details-supplement}

In addition to Hoeffding's inequality, we also used reverse percentile bootstrap method (see e.g.\ \cite{hesterberg2015what}) to obtain our confidence intervals on $\Qlo$ and $\Qup$ as described in Section \ref{sec:confidence-intervals-methodology-supplement}.
We used 100 bootstrap samples for each confidence interval.
To avoid bootstrapping on small numbers of data points, we did not reject any hypothesis where either the number of observational trajectories $n$ or twin trajectories $\widehat{n}$ was less than 100, and returned a $p$-value of 1 in each such case.

Table \ref{tab:hypotheses_rev_percentile} shows the number of rejected hypotheses for each physiological quantity using this approach.
We observed a similar trend as in our results obtained using Hoeffding's inequality (Table \ref{tab:hypotheses_hoeffding_full}).
For example, we obtained high number of rejections for Sodium, Chloride and Potassium blood concentrations but few rejections for Arterial Pressure and Heart Rate.
Overall, bootstrapping increased the number of rejected hypotheses by a factor of roughly 3.3 compared with Hoeffding's inequality (281 vs.\ 85 rejections in total).
Like we described for Hoeffding's inequality in the main text, we also ran this analysis with each hypothesis obtained using the unconditional bounds of \cite{manski}, and again obtained substantially fewer rejections compared with our approach based on Theorem \ref{thm:causal-bounds}.

\begin{figure}[t]
    \centering
        \begin{subfigure}[b]{0.26\textwidth}
    \includegraphics[height=3.7cm]{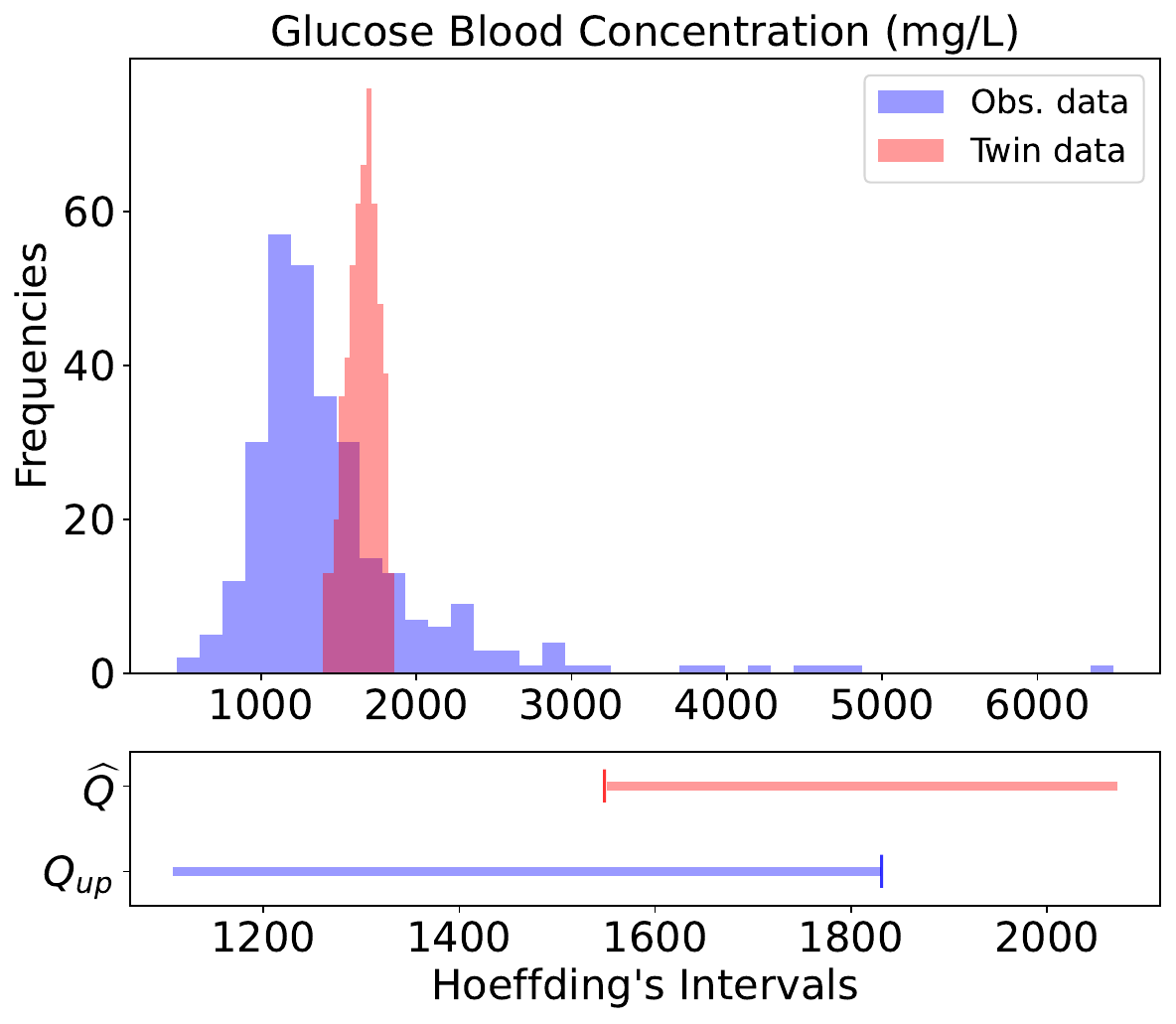}
    \subcaption{Not rejected}
    \label{fig:glucosea-supp}
    \end{subfigure}\hspace{1cm}%
    \begin{subfigure}[b]{0.26\textwidth}
    \includegraphics[height=3.7cm]{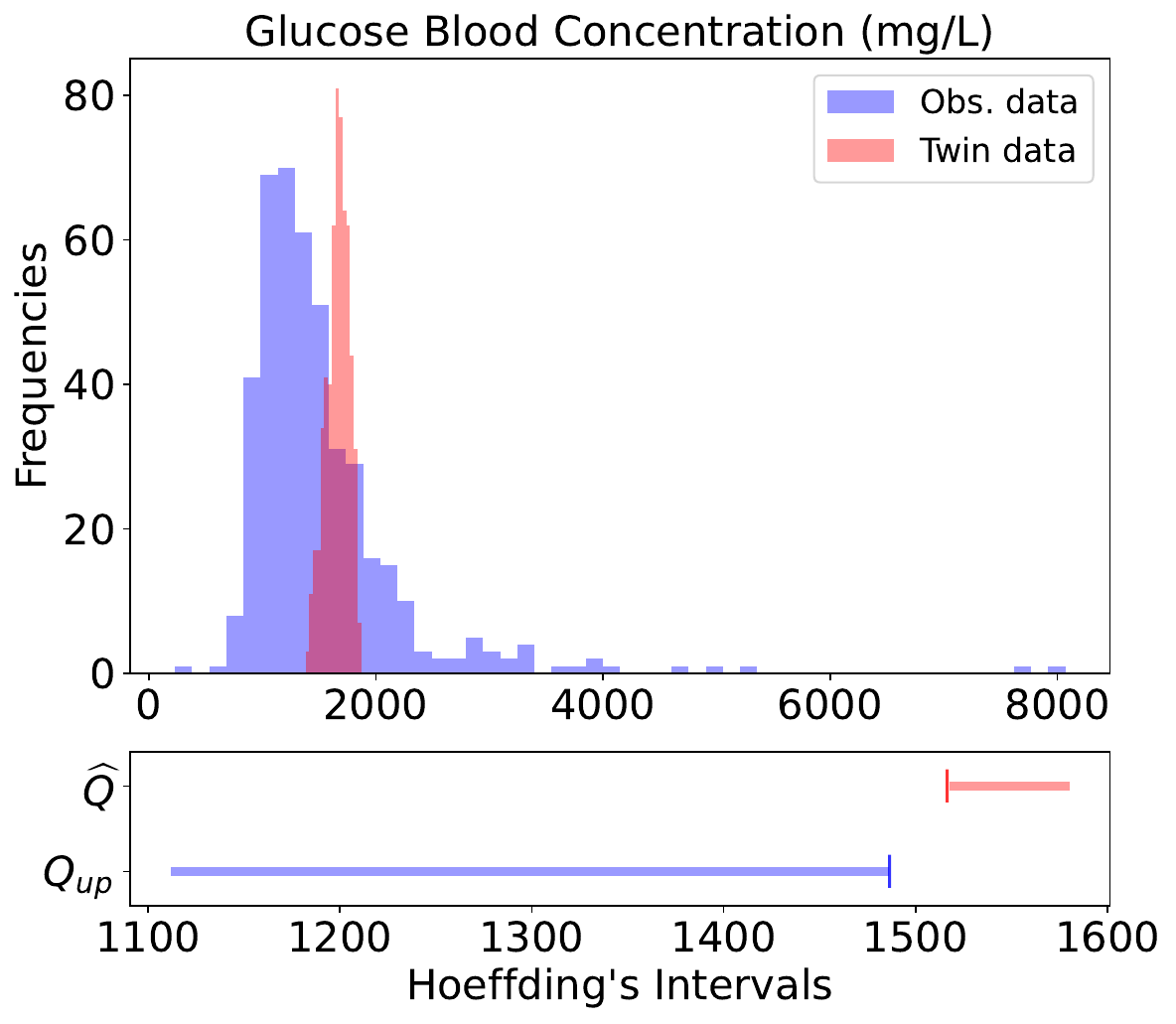}
    \subcaption{Rejected}
    \label{fig:glucoseb-supp}
    \end{subfigure}\\      
    \begin{subfigure}[b]{0.26\textwidth}
    \includegraphics[height=3.7cm]{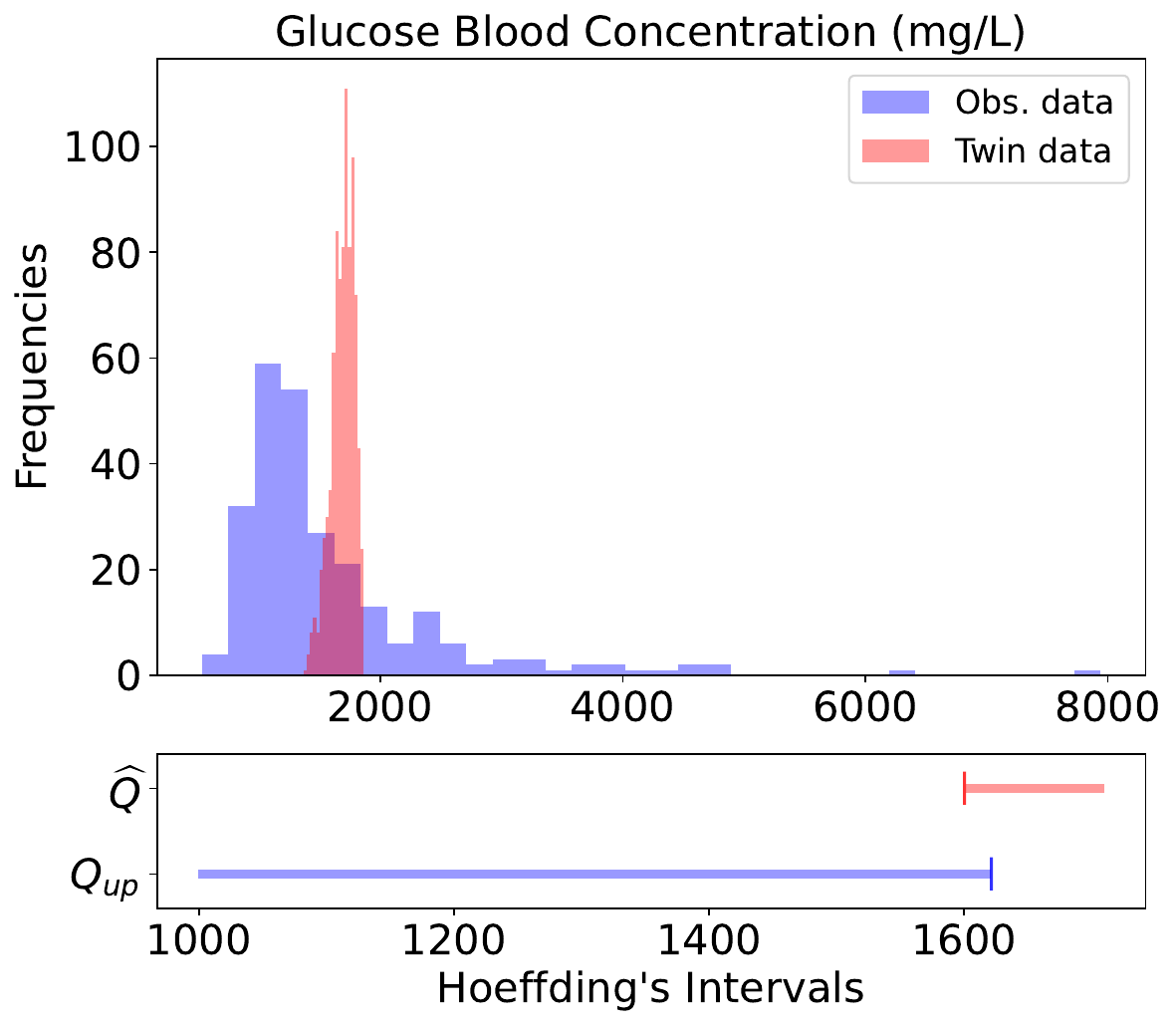}
    \subcaption{Not rejected}
    \label{fig:chloridea}
    \end{subfigure}\hspace{1cm}%
    \begin{subfigure}[b]{0.26\textwidth}
    \includegraphics[height=3.7cm]{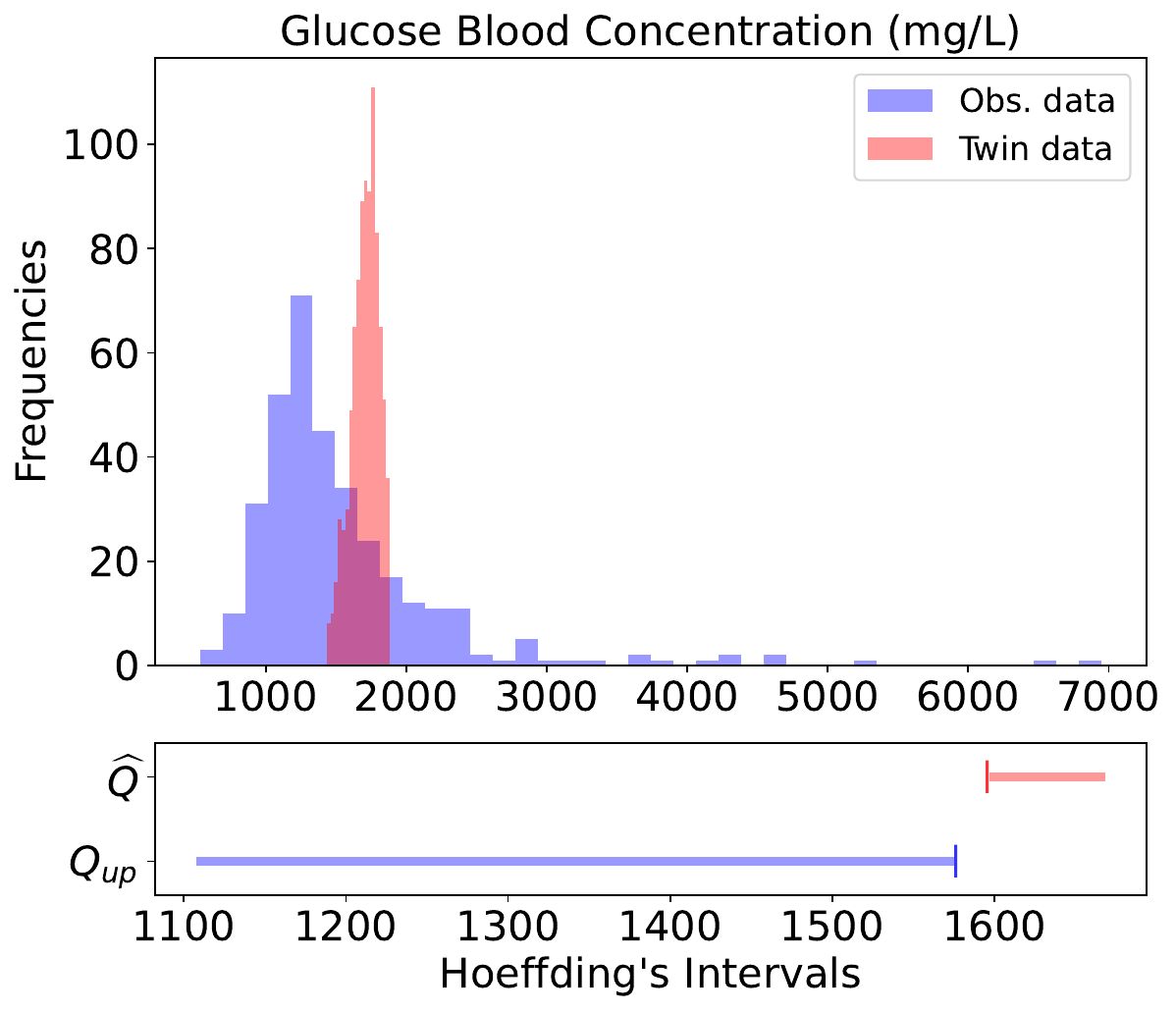}
    \subcaption{Rejected}
    \label{fig:chlorideb}
    \end{subfigure}\\
    \begin{subfigure}[b]{0.26\textwidth}
    \includegraphics[height=3.7cm]{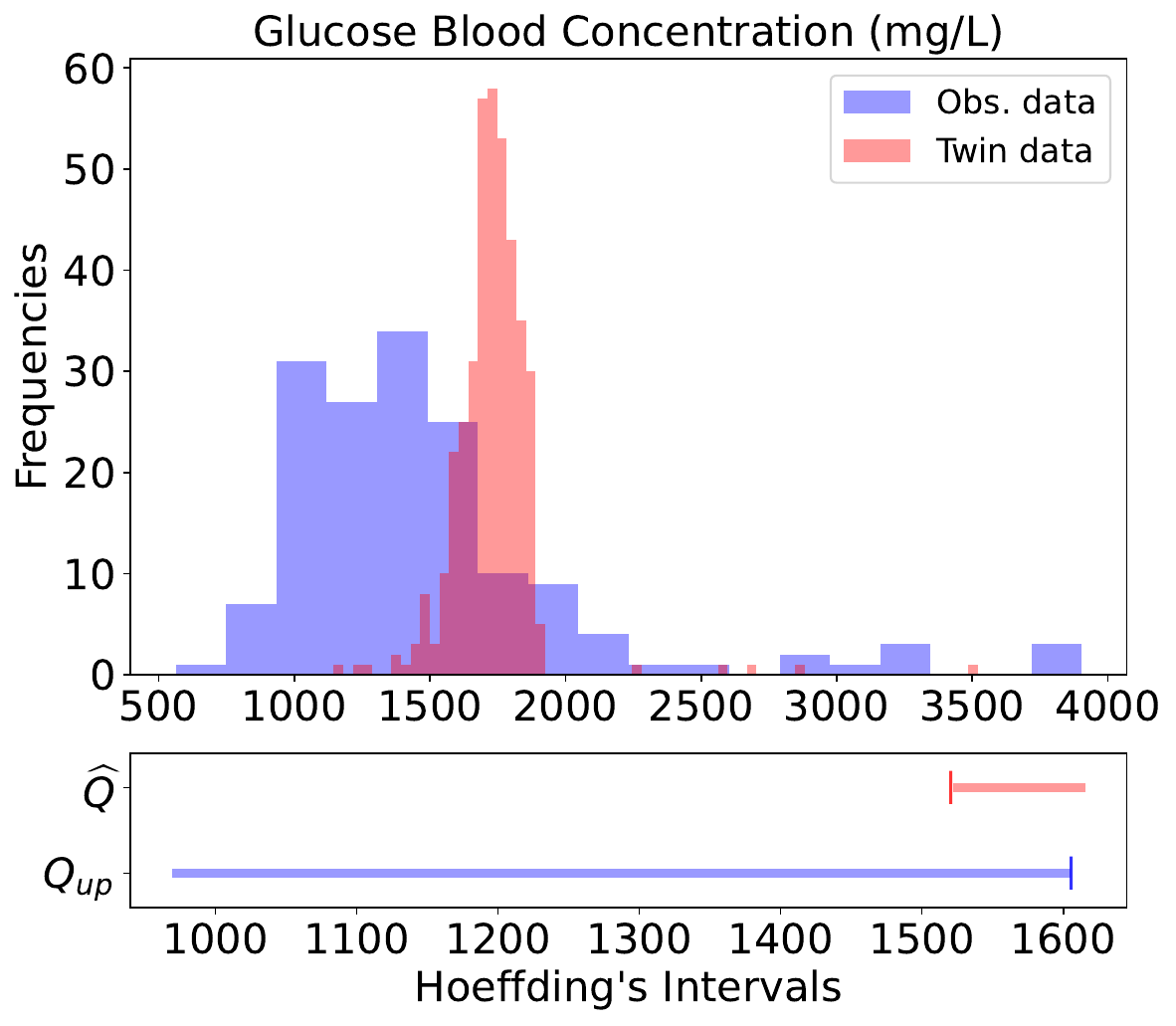}
    \subcaption{Not rejected}
    \label{fig:potassiuma}
    \end{subfigure}\hspace{1cm}%
    \begin{subfigure}[b]{0.26\textwidth}
    \includegraphics[height=3.7cm]{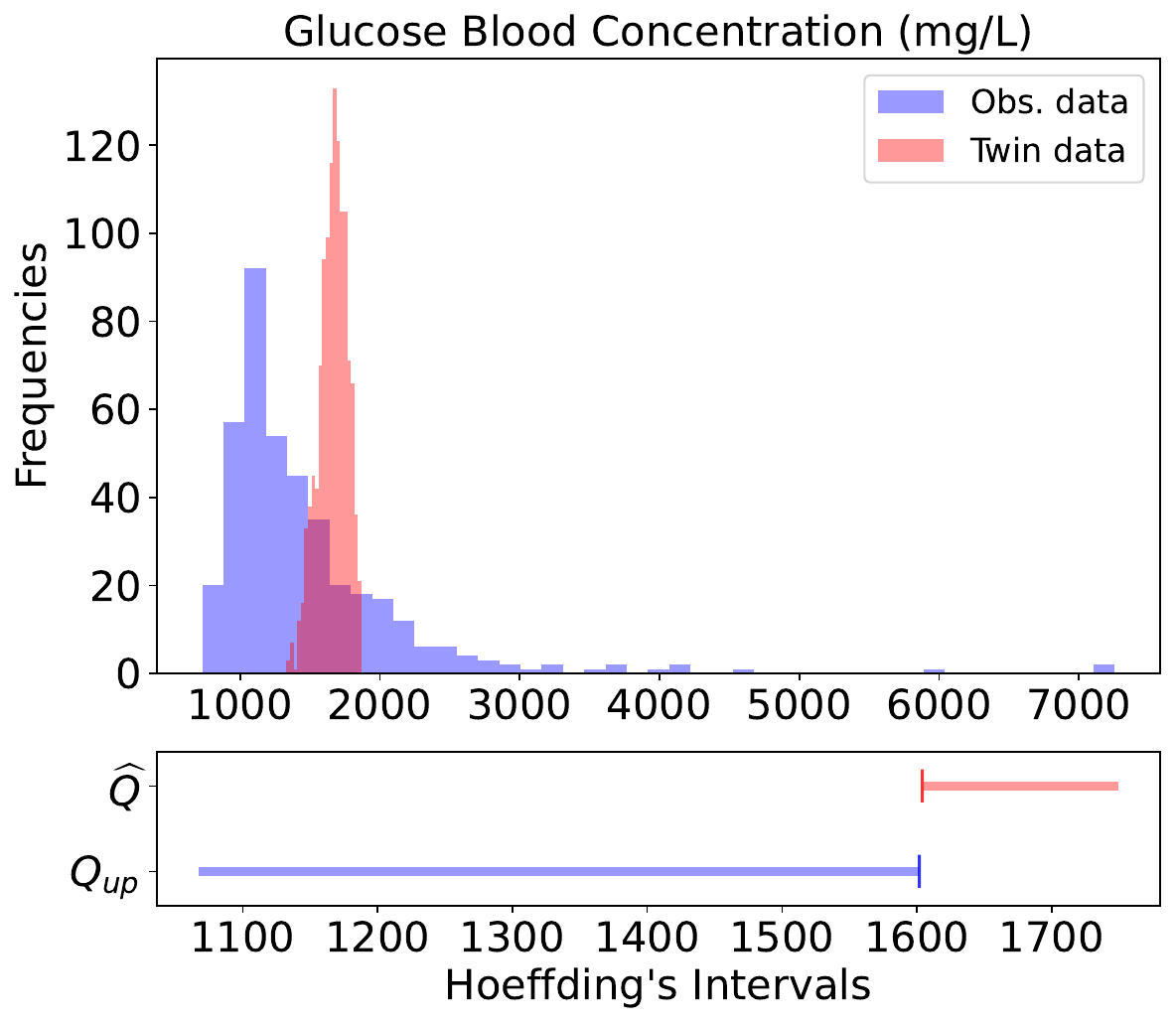}
    \subcaption{Rejected}
    \label{fig:potassiumb}
    \end{subfigure}\\
    \begin{subfigure}[b]{0.26\textwidth}
    \includegraphics[height=3.7cm]{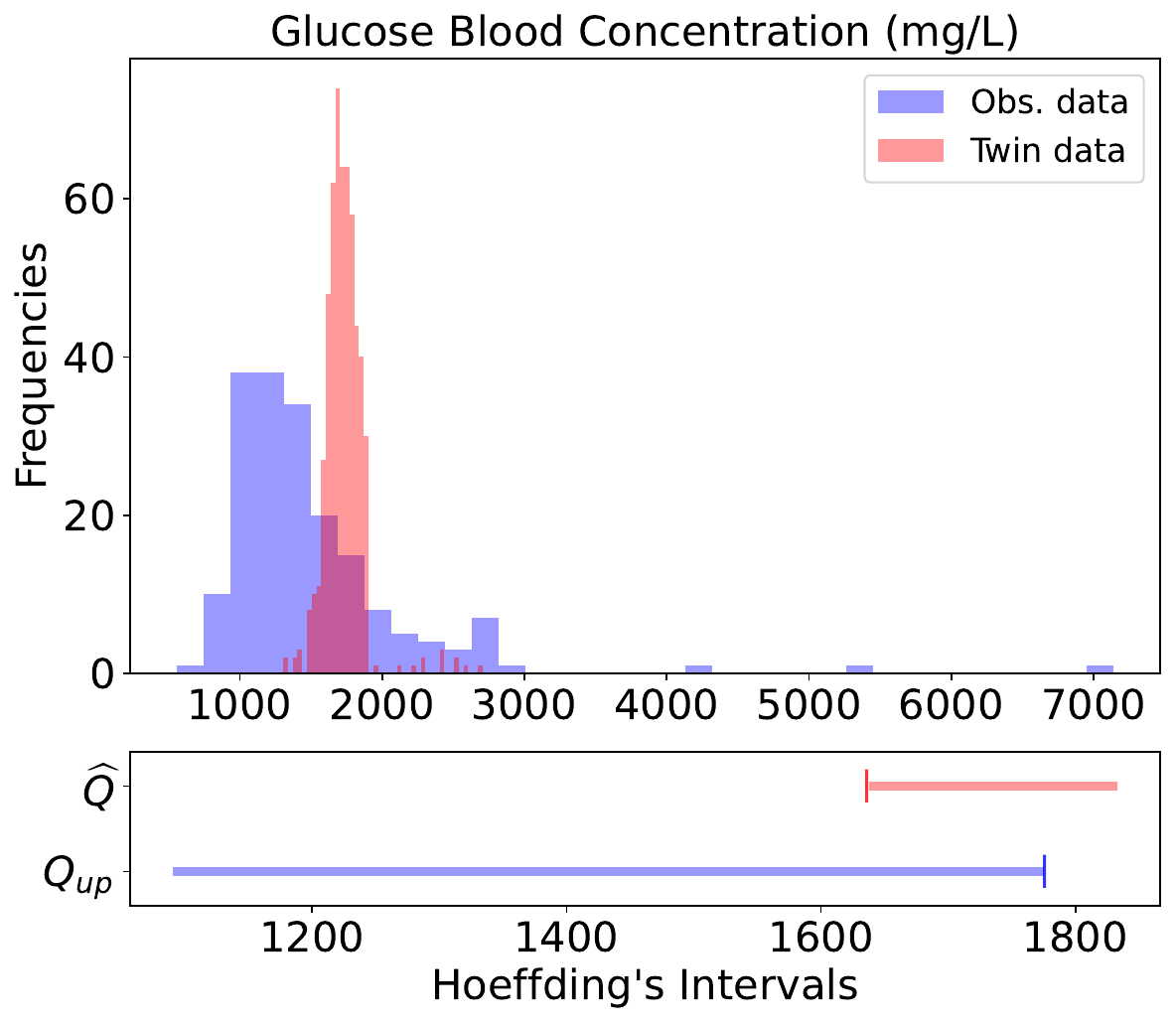}
    \subcaption{Not rejected}
    \label{fig:paco2a}
    \end{subfigure}\hspace{1cm}%
    \begin{subfigure}[b]{0.26\textwidth}
    \includegraphics[height=3.7cm]{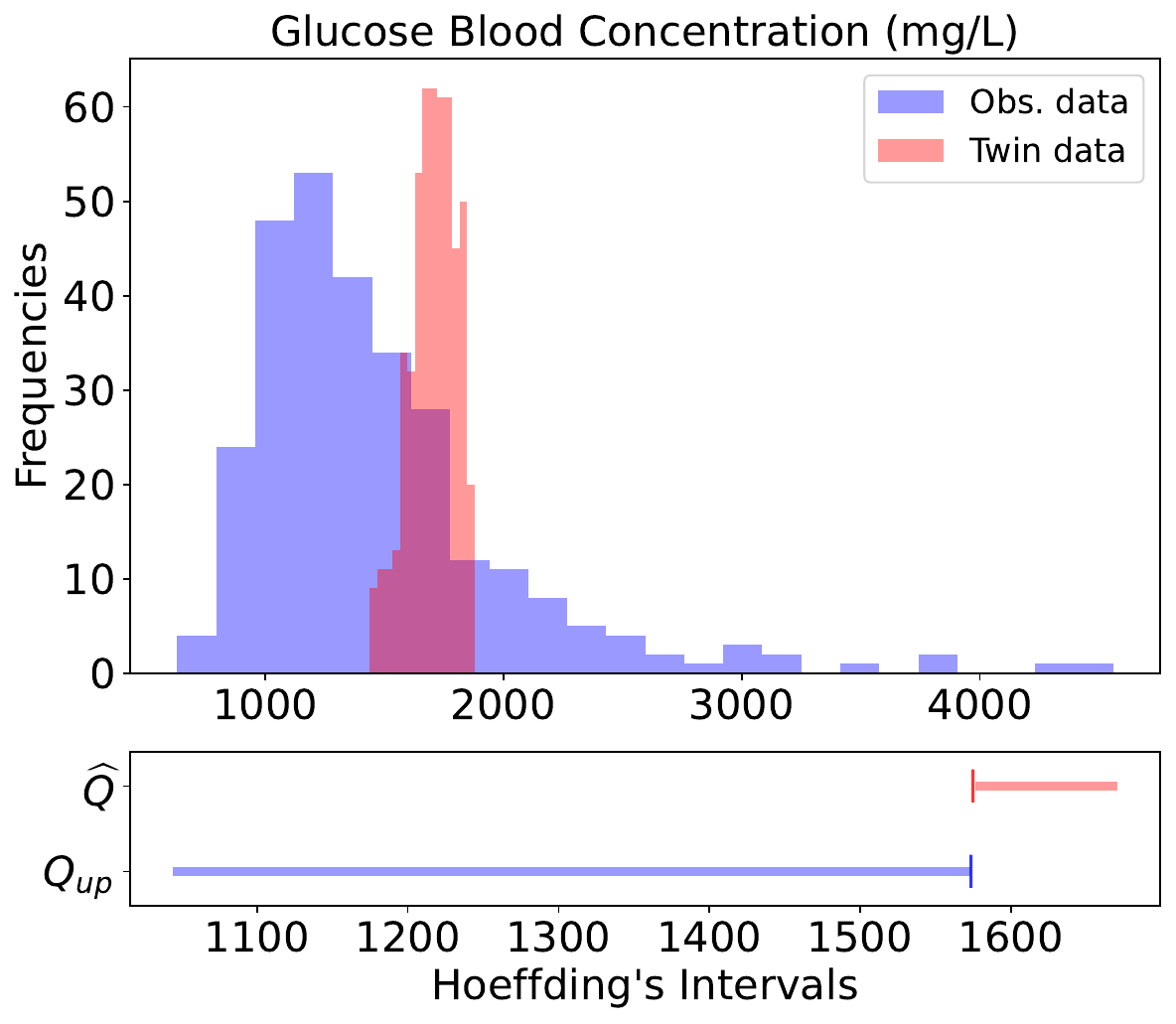}
    \subcaption{Rejected}
    \label{fig:paco2b}
    \end{subfigure}

    \caption{Raw observational data values conditional on $\A_{1:\tx}=\ax_{1:\tx}$ and $\X_{0:\tx}(\A_{1:\tx})\in \B_{0:\tx}$, and from the output of the twin conditional on $\Xt_{0:\tx}(\ax_{1:\tx})\in \B_{0:\tx}$.
    Each row shows two distinct choices of $(\B_{0:\tx}, \ax_{1:\tx})$.
    Below each figure are shown 95\% Hoeffding confidence intervals for $\Qt$ and $\Qup$.
    Unlike Figure \ref{fig:histograms} from the main text, the horizontal axes of the histograms are not truncated, and the first row is in particular an untruncated version of Figure \ref{fig:histograms} from the main text.
    Note however that the scales of the horizontal axes of the confidence intervals differ from those of the histograms, since it is visually more difficult to determine whether or not the confidence intervals overlap when fully zoomed out.} \label{fig:histograms-supplement}
\end{figure}

\subsection{Tightness of bounds and number of data points per hypothesis}
    In this section, we show empirically how both the tightness of the bounds $[\Qlo, \Qup]$ and the number of data points per hypothesis relate to the number of falsifications obtained in our case study.
    Recall that the tightness of $[\Qlo, \Qup]$ is determined by the value of $\Prob(\A_{1:\tx} = \ax_{1:\tx} \mid \X_{0:\N}(\A_{1:\N}) \in \B_{0:\N})$, since we have
    \begin{equation} \label{eq:N-propensity-tightness-relation}
        \frac{\Qup - \Qlo}{\yup - \ylo} = 1 - \Prob(\A_{1:\tx} = \ax_{1:\tx} \mid \X_{0:\N}(\A_{1:\N}) \in \B_{0:\N}).
    \end{equation}
    Here the left-hand side is a number in $[0, 1]$ that quantifies the tightness of the bounds $[\Qlo, \Qup]$ relative to the trivial worst-case bounds $[\ylo, \yup]$, with smaller values meaning tighter bounds. The equation above shows that the higher the value of $\Prob(\A_{1:\tx} = \ax_{1:\tx} \mid \X_{0:\N}(\A_{1:\N}) \in \B_{0:\N})$, the tighter the bounds are.
    
    Figure \ref{fig:scatter-plot} shows the bounds are often informative in practice, with $\Prob(\A_{1:\tx} = \ax_{1:\tx} \mid \X_{0:\N}(\A_{1:\N}) \in \B_{0:\N})$ being reasonably large (and hence the bounds tight, by \eqref{eq:N-propensity-tightness-relation} above) for a significant number of hypotheses we consider.
    However, rejections still occur even when the bounds are reasonably loose (e.g.\ $\Prob(\A_{1:\tx} = \ax_{1:\tx} \mid \X_{0:\N}(\A_{1:\N}) \in \B_{0:\N}) \approx 0.3$), which shows our method can still yield useful information even in this case.
    We moreover observe rejections across a range of different numbers of observational data points used to test each hypothesis, which shows that our method is not strongly dependent on the size of the dataset obtained. 

    \subsection{Sensitivity to $\ylo$ and $\yup$} \label{sec:sensitity-analysis-appendix}

    We investigated the sensitivity of our methodology with respect to our choices of the values $\ylo$ and $\yup$.
    Specifically, we repeated our procedure with the intervals $[\ylo, \yup]$ replaced with $[\ylo\, (1- \Delta/2), \yup\,(1 + \Delta/2)]$ for a range of different values of $\Delta \in \R$.
    Figure \ref{fig:sensitivity-plot-rejections} plots the number of rejections for different values of $\Delta$.
    We observe that for significantly larger $[\ylo, \yup]$ intervals, we do obtain fewer rejections, although this is to be expected since the widths of our both the bounds $[\Qlo, \Qup]$ and our confidence intervals $\qlo{\alpha}$ and $\qup{\alpha}$ obtained using Hoeffding's inequality (see Proposition \ref{prop:hoeffding-confidence-bounds-supp}) grow increasingly large as the width of $[\ylo, \yup]$ grows.
    However, we observe that the number of rejections per outcome is stable for a moderate range of widths of $[\ylo, \yup]$, which indicates that our method is reasonably robust to the choice of $\ylo, \yup$ parameters.

    \begin{figure}[t]
        \centering
        \includegraphics[height=6cm]{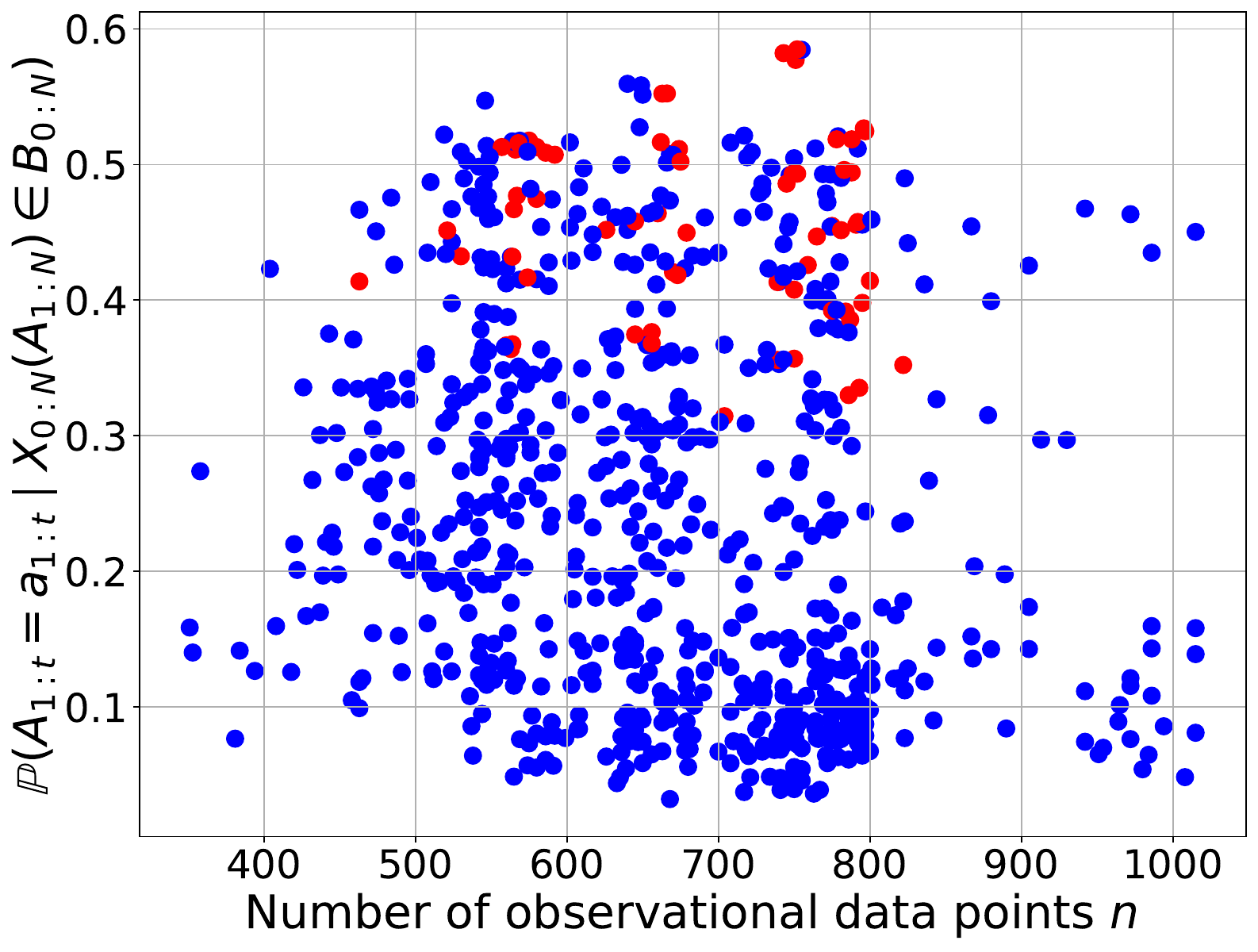}
        \caption{Sample mean estimate of $\Prob(\A_{1:\tx} = \ax_{1:\tx} \mid \X_{0:\N}(\A_{1:\N}) \in \B_{0:\N})$ for each pair of hypotheses $(\Hlo, \Hup)$ corresponding to the same set of parameters $(\tx, \fx, \ax_{1:\tx}, \B_{0:\tx})$ that we tested, along with the corresponding number of observational data points used to test each hypothesis.
        Red points indicate that either $\Hlo$ or $\Hup$ were rejected, while blue points indicate that both $\Hlo$ and $\Hup$ were not rejected.}
        \label{fig:scatter-plot}
    \end{figure}

    \begin{figure}[t]
    \centering
\includegraphics[width=0.6\textwidth]{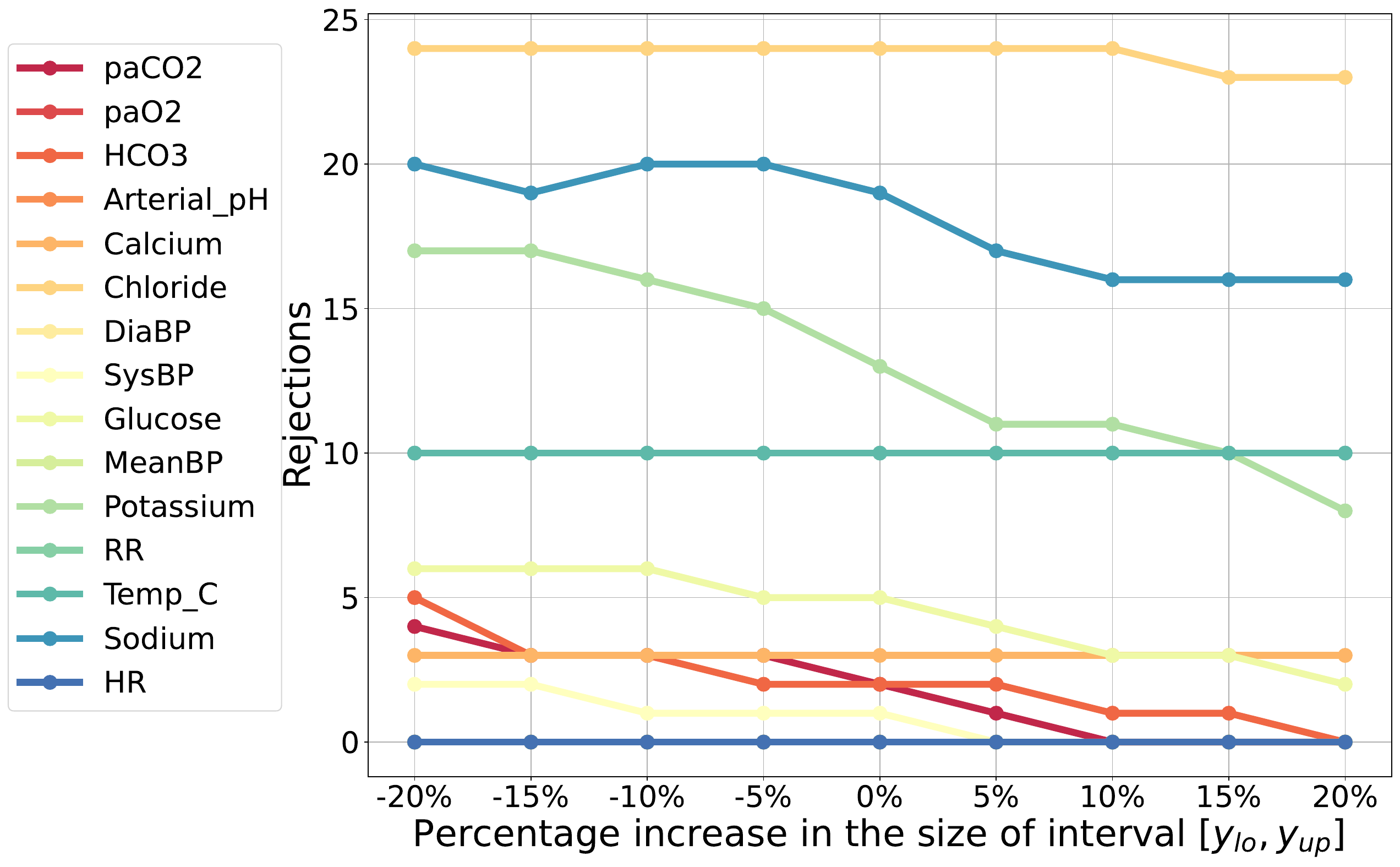}
    \caption{Rejections obtained as the width of the $[\ylo, \yup]$ interval changes. Here, the interval is increased (or decreased) symmetrically on each side.}
    \label{fig:sensitivity-plot-rejections}
\end{figure}

    }
    { %
    }

\setlength{\baselineskip}{0pt} %
\bibliography{references.bib}

\begin{thebibliography}{158}
\providecommand{\natexlab}[1]{#1}
\providecommand{\url}[1]{\texttt{#1}}
\expandafter\ifx\csname urlstyle\endcsname\relax
  \providecommand{\doi}[1]{doi: #1}\else
  \providecommand{\doi}{doi: \begingroup \urlstyle{rm}\Url}\fi

\bibitem[Allamaa et~al.(2022)Allamaa, Patrinos, {Van der Auweraer}, and
  Son]{allamaa2022sim2real}
Jean~Pierre Allamaa, Panagiotis Patrinos, Herman {Van der Auweraer}, and
  Tong~Duy Son.
\newblock Sim2real for autonomous vehicle control using executable digital
  twin.
\newblock \emph{IFAC-PapersOnLine}, 55\penalty0 (24):\penalty0 385--391, 2022.
\newblock ISSN 2405-8963.
\newblock \doi{https://doi.org/10.1016/j.ifacol.2022.10.314}.
\newblock URL
  \url{https://www.sciencedirect.com/science/article/pii/S2405896322023461}.
\newblock 10th IFAC Symposium on Advances in Automotive Control AAC 2022.

\bibitem[Altschuler et~al.(2019)Altschuler, Brunel, and
  Malek]{altschuler2019best}
Jason Altschuler, Victor-Emmanuel Brunel, and Alan Malek.
\newblock Best arm identification for contaminated bandits.
\newblock \emph{Journal of Machine Learning Research}, 20\penalty0
  (91):\penalty0 1--39, 2019.

\bibitem[AMSE(2018)]{amse2018assessing}
AMSE.
\newblock \emph{Assessing Credibility of Computational Modeling through
  Verification and Validation: Application to Medical Devices}.
\newblock AMSE, 2018.

\bibitem[Angelopoulos and Bates(2021)]{conf-bates}
Anastasios~N. Angelopoulos and Stephen Bates.
\newblock A gentle introduction to conformal prediction and distribution-free
  uncertainty quantification.
\newblock \emph{arXiv preprint arXiv:2107.07511}, 2021.

\bibitem[Barricelli et~al.(2019)Barricelli, Casiraghi, and
  Fogli]{barricelli2019survey}
Barbara~Rita Barricelli, Elena Casiraghi, and Daniela Fogli.
\newblock A survey on digital twin: Definitions, characteristics, applications,
  and design implications.
\newblock \emph{IEEE access}, 7:\penalty0 167653--167671, 2019.

\bibitem[Bastani and Bayati(2019)]{bastani2019online}
Hamsa Bastani and Mohsen Bayati.
\newblock Online decision making with high-dimensional covariates.
\newblock \emph{Operations Research}, 68, 11 2019.
\newblock \doi{10.1287/opre.2019.1902}.

\bibitem[Bellemare et~al.(2017)Bellemare, Dabney, and Munos]{distributional-rl}
Marc~G Bellemare, Will Dabney, and R{\'e}mi Munos.
\newblock A distributional perspective on reinforcement learning.
\newblock In \emph{International Conference on Machine Learning}, pages
  449--458, 2017.

\bibitem[Bellinger et~al.(2011)Bellinger, Tuegel, Ingraffea, Eason, and
  Spottswood]{tuegel2011reengineering}
Nicholas Bellinger, Eric~J. Tuegel, Anthony~R. Ingraffea, Thomas~G. Eason, and
  S.~Michael Spottswood.
\newblock Reengineering aircraft structural life prediction using a digital
  twin.
\newblock \emph{International Journal of Aerospace Engineering}, 2011:\penalty0
  154798, 2011.
\newblock \doi{10.1155/2011/154798}.
\newblock URL \url{https://doi.org/10.1155/2011/154798}.

\bibitem[Benjamini and Yekutieli(2001)]{benjamini2001control}
Yoav Benjamini and Daniel Yekutieli.
\newblock The control of the false discovery rate in multiple testing under
  dependency.
\newblock \emph{The Annals of Statistics}, 29\penalty0 (4):\penalty0 1165 --
  1188, 2001.
\newblock \doi{10.1214/aos/1013699998}.
\newblock URL \url{https://doi.org/10.1214/aos/1013699998}.

\bibitem[Beygelzimer and Langford(2008)]{Beygelzimer2008Offset}
Alina Beygelzimer and John Langford.
\newblock The offset tree for learning with partial labels.
\newblock \emph{CoRR}, abs/0812.4044, 2008.
\newblock URL \url{http://arxiv.org/abs/0812.4044}.

\bibitem[Bietti et~al.(2018)Bietti, Agarwal, and
  Langford]{bietti2018contextual}
Alberto Bietti, Alekh Agarwal, and John Langford.
\newblock A contextual bandit bake-off.
\newblock \emph{arXiv preprint arXiv:1802.04064}, 2018.

\bibitem[Bray et~al.(2019)Bray, Webb, Enquobahrie, Vicory, Heneghan, Hubal,
  TerMaath, Asare, and Clipp]{pulse}
Aaron Bray, Jeffrey~B. Webb, Andinet Enquobahrie, Jared Vicory, Jerry Heneghan,
  Robert Hubal, Stephanie TerMaath, Philip Asare, and Rachel~B. Clipp.
\newblock {Pulse Physiology Engine: an Open-Source Software Platform for
  Computational Modeling of Human Medical Simulation}.
\newblock \emph{SN Comprehensive Clinical Medicine}, 1\penalty0 (5):\penalty0
  362--377, 2019.
\newblock \doi{10.1007/s42399-019-00053-w}.
\newblock URL \url{https://doi.org/10.1007/s42399-019-00053-w}.

\bibitem[Brehmer et~al.(2020)Brehmer, Louppe, Pavez, and
  Cranmer]{brehmer2020mining}
Johann Brehmer, Gilles Louppe, Juan Pavez, and Kyle Cranmer.
\newblock Mining gold from implicit models to improve likelihood-free
  inference.
\newblock \emph{Proceedings of the National Academy of Sciences}, 117\penalty0
  (10):\penalty0 5242--5249, 2020.

\bibitem[Breiman(2001)]{breiman2001machine}
Leo Breiman.
\newblock Random forests.
\newblock \emph{Machine Learning}, 45\penalty0 (1):\penalty0 5--32, 2001.
\newblock \doi{10.1023/A:1010933404324}.
\newblock URL \url{https://doi.org/10.1023/A:1010933404324}.

\bibitem[Chandak et~al.(2021)Chandak, Niekum, da~Silva, Learned-Miller,
  Brunskill, and Thomas]{chandak2021universal}
Yash Chandak, Scott Niekum, Bruno~Castro da~Silva, Erik Learned-Miller, Emma
  Brunskill, and Philip~S Thomas.
\newblock Universal off-policy evaluation.
\newblock \emph{arXiv preprint arXiv:2104.12820}, 2021.

\bibitem[Coorey et~al.(2022)Coorey, Figtree, Fletcher, Snelson, Vernon, Winlaw,
  Grieve, McEwan, Yang, Qian, et~al.]{coorey2022health}
Genevieve Coorey, Gemma~A Figtree, David~F Fletcher, Victoria~J Snelson,
  Stephen~Thomas Vernon, David Winlaw, Stuart~M Grieve, Alistair McEwan, Jean
  Yee~Hwa Yang, Pierre Qian, et~al.
\newblock The health digital twin to tackle cardiovascular disease—a review
  of an emerging interdisciplinary field.
\newblock \emph{NPJ Digital Medicine}, 2022.

\bibitem[Cornish et~al.(2023)Cornish, Taufiq, Doucet, and
  Holmes]{cornish2023causalfalsificationdigitaltwins}
Rob Cornish, Muhammad~Faaiz Taufiq, Arnaud Doucet, and Chris Holmes.
\newblock Causal falsification of digital twins, 2023.
\newblock URL \url{https://arxiv.org/abs/2301.07210}.

\bibitem[Corral-Acero et~al.(2020)Corral-Acero, Margara, Marciniak, Rodero,
  Loncaric, Feng, Gilbert, Fernandes, Bukhari, Wajdan,
  et~al.]{corral2020digital}
Jorge Corral-Acero, Francesca Margara, Maciej Marciniak, Cristobal Rodero,
  Filip Loncaric, Yingjing Feng, Andrew Gilbert, Joao~F Fernandes, Hassaan~A
  Bukhari, Ali Wajdan, et~al.
\newblock The 'digital twin' to enable the vision of precision cardiology.
\newblock \emph{European Heart Journal}, 41\penalty0 (48):\penalty0 4556--4564,
  2020.

\bibitem[Cox(1975)]{cox1975note}
David~R Cox.
\newblock A note on data-splitting for the evaluation of significance levels.
\newblock \emph{Biometrika}, 62\penalty0 (2):\penalty0 441--444, 1975.

\bibitem[Dahmen et~al.(2022)Dahmen, Osterloh, and
  Roßmann]{dahmen2022verification}
Ulrich~Richard Dahmen, Tobias Osterloh, and Heinz-Jürgen Roßmann.
\newblock Verification and validation of digital twins and virtual testbeds.
\newblock \emph{International journal of advances in engineering sciences and
  applied mathematics}, 11\penalty0 (1):\penalty0 47--64, 2022.
\newblock ISSN 0975-5616.
\newblock \doi{10.11591/ijaas.v11.i1.pp47-64}.
\newblock URL \url{https://publications.rwth-aachen.de/record/843535}.

\bibitem[Davison and Hinkley(1997)]{davison1997bootstrap}
A.~C. Davison and D.~V. Hinkley.
\newblock \emph{Bootstrap Methods and their Application}.
\newblock Cambridge Series in Statistical and Probabilistic Mathematics.
  Cambridge University Press, 1997.
\newblock \doi{10.1017/CBO9780511802843}.

\bibitem[Deng(2012)]{deng2012mnist}
Li~Deng.
\newblock The mnist database of handwritten digit images for machine learning
  research.
\newblock \emph{IEEE Signal Processing Magazine}, 29\penalty0 (6):\penalty0
  141--142, 2012.

\bibitem[Dua and Graff(2017)]{dua2019uci}
Dheeru Dua and Casey Graff.
\newblock {UCI} machine learning repository, 2017.
\newblock URL \url{http://archive.ics.uci.edu/ml}.

\bibitem[Dudík et~al.(2014{\natexlab{a}})Dudík, Erhan, Langford, and
  Li]{doubly-robust}
Miroslav Dudík, Dumitru Erhan, John Langford, and Lihong Li.
\newblock Doubly robust policy evaluation and optimization.
\newblock \emph{Statistical Science}, 29\penalty0 (4), 2014{\natexlab{a}}.

\bibitem[Dudík et~al.(2014{\natexlab{b}})Dudík, Erhan, Langford, and
  Li]{dudik2014doubly}
Miroslav Dudík, Dumitru Erhan, John Langford, and Lihong Li.
\newblock Doubly robust policy evaluation and optimization.
\newblock \emph{Statistical Science}, 29\penalty0 (4):\penalty0 485--511,
  2014{\natexlab{b}}.
\newblock ISSN 08834237, 21688745.
\newblock URL \url{http://www.jstor.org/stable/43288496}.

\bibitem[Efron(1979)]{efron1979bootstrap}
B.~Efron.
\newblock Bootstrap methods: Another look at the jackknife.
\newblock \emph{The Annals of Statistics}, 7\penalty0 (1):\penalty0 1 -- 26,
  1979.
\newblock \doi{10.1214/aos/1176344552}.
\newblock URL \url{https://doi.org/10.1214/aos/1176344552}.

\bibitem[Farajtabar et~al.(2018{\natexlab{a}})Farajtabar, Chow, and
  Ghavamzadeh]{farajtabar2018more}
Mehrdad Farajtabar, Yinlam Chow, and Mohammad Ghavamzadeh.
\newblock More robust doubly robust off-policy evaluation.
\newblock In Jennifer Dy and Andreas Krause, editors, \emph{Proceedings of the
  35th International Conference on Machine Learning}, volume~80 of
  \emph{Proceedings of Machine Learning Research}, pages 1447--1456. PMLR,
  10--15 Jul 2018{\natexlab{a}}.
\newblock URL \url{https://proceedings.mlr.press/v80/farajtabar18a.html}.

\bibitem[Farajtabar et~al.(2018{\natexlab{b}})Farajtabar, Ghavamzadeh, and
  Chow]{mehrdad2018more}
Mehrdad Farajtabar, Mohammad Ghavamzadeh, and Yinlam Chow.
\newblock More robust doubly robust off-policy evaluation.
\newblock 2018{\natexlab{b}}.

\bibitem[Foffano et~al.(2023)Foffano, Russo, and
  Proutiere]{foffano2023conformal}
Daniele Foffano, Alessio Russo, and Alexandre Proutiere.
\newblock Conformal off-policy evaluation in markov decision processes.
\newblock In \emph{2023 62nd IEEE Conference on Decision and Control (CDC)},
  pages 3087--3094. IEEE, 2023.

\bibitem[Foygel~Barber et~al.(2021)Foygel~Barber, Cand\`es, Ramdas, and
  Tibshirani]{foygel2021limits}
Rina Foygel~Barber, Emmanuel~J Cand\`es, Aaditya Ramdas, and Ryan~J Tibshirani.
\newblock The limits of distribution-free conditional predictive inference.
\newblock \emph{Information and Inference: A Journal of the IMA}, 10\penalty0
  (2):\penalty0 455--482, 2021.

\bibitem[Fujimoto et~al.(2021)Fujimoto, Meger, and Precup]{Fujimoto2021deep}
Scott Fujimoto, David Meger, and Doina Precup.
\newblock A deep reinforcement learning approach to marginalized importance
  sampling with the successor representation.
\newblock In Marina Meila and Tong Zhang, editors, \emph{Proceedings of the
  38th International Conference on Machine Learning}, volume 139 of
  \emph{Proceedings of Machine Learning Research}, pages 3518--3529. PMLR,
  18--24 Jul 2021.
\newblock URL \url{https://proceedings.mlr.press/v139/fujimoto21a.html}.

\bibitem[Galappaththige et~al.(2022)Galappaththige, Gray, Costa, Niederer, and
  Pathmanathan]{galappaththige2022credibility}
Suran Galappaththige, Richard~A Gray, Caroline~Mendonca Costa, Steven Niederer,
  and Pras Pathmanathan.
\newblock Credibility assessment of patient-specific computational modeling
  using patient-specific cardiac modeling as an exemplar.
\newblock \emph{PLoS computational biology}, 18\penalty0 (10):\penalty0
  e1010541, 2022.

\bibitem[Grieves and Vickers(2017)]{grieves2017digital}
Michael Grieves and John Vickers.
\newblock Digital twin: Mitigating unpredictable, undesirable emergent behavior
  in complex systems.
\newblock In \emph{Transdisciplinary Perspectives on Complex Systems}, pages
  85--113. Springer, 2017.

\bibitem[Hall(1988)]{hall1988theoretical}
Peter Hall.
\newblock Theoretical comparison of bootstrap confidence intervals.
\newblock \emph{The Annals of Statistics}, 16\penalty0 (3):\penalty0 927 --
  953, 1988.
\newblock \doi{10.1214/aos/1176350933}.
\newblock URL \url{https://doi.org/10.1214/aos/1176350933}.

\bibitem[Hemmler et~al.(2019)Hemmler, Lutz, Kalender, Reeps, and
  Gee]{hemmler2019patient}
Andr{\'e} Hemmler, Brigitta Lutz, G{\"u}nay Kalender, Christian Reeps, and
  Michael~W Gee.
\newblock Patient-specific in silico endovascular repair of abdominal aortic
  aneurysms: application and validation.
\newblock \emph{Biomechanics and Modeling in Mechanobiology}, 18\penalty0
  (4):\penalty0 983--1004, 2019.

\bibitem[Hern{\'a}n and Robins(2020)]{hernan2020causal}
Miguel~A Hern{\'a}n and James~M Robins.
\newblock \emph{Causal Inference: What If}.
\newblock Chapman and Hall/CRC, Boca Raton, 2020.

\bibitem[Hesterberg(2015)]{hesterberg2015what}
Tim~C. Hesterberg.
\newblock What teachers should know about the bootstrap: Resampling in the
  undergraduate statistics curriculum.
\newblock \emph{The American Statistician}, 69\penalty0 (4):\penalty0 371--386,
  2015.
\newblock \doi{10.1080/00031305.2015.1089789}.
\newblock URL \url{https://doi.org/10.1080/00031305.2015.1089789}.
\newblock PMID: 27019512.

\bibitem[Holland(1986)]{holland1986statistics}
Paul~W. Holland.
\newblock Statistics and causal inference.
\newblock \emph{Journal of the American Statistical Association}, 81\penalty0
  (396):\penalty0 945--960, 1986.
\newblock ISSN 01621459.
\newblock URL \url{http://www.jstor.org/stable/2289064}.

\bibitem[Holm(1979)]{holm1979simple}
Sture Holm.
\newblock A simple sequentially rejective multiple test procedure.
\newblock \emph{Scandinavian Journal of Statistics}, 6\penalty0 (2):\penalty0
  65--70, 1979.
\newblock ISSN 03036898, 14679469.
\newblock URL \url{http://www.jstor.org/stable/4615733}.

\bibitem[Horvitz and Thompson(1952)]{horvitz1952generalization}
D.~G. Horvitz and D.~J. Thompson.
\newblock A generalization of sampling without replacement from a finite
  universe.
\newblock \emph{Journal of the American Statistical Association}, 47\penalty0
  (260):\penalty0 663--685, 1952.
\newblock ISSN 01621459.
\newblock URL \url{http://www.jstor.org/stable/2280784}.

\bibitem[Huang et~al.(2021)Huang, Leqi, Lipton, and
  Azizzadenesheli]{risk-assessment}
Audrey Huang, Liu Leqi, Zachary~C. Lipton, and Kamyar Azizzadenesheli.
\newblock Off-policy risk assessment in contextual bandits.
\newblock \emph{arXiv preprint arXiv:2104.08977}, 2021.

\bibitem[Huang et~al.(2007)Huang, Gretton, Borgwardt, Sch{\"o}lkopf, and
  Smola]{huang2007correcting}
Jiayuan Huang, Arthur Gretton, Karsten Borgwardt, Bernhard Sch{\"o}lkopf, and
  Alex~J Smola.
\newblock Correcting sample selection bias by unlabeled data.
\newblock In \emph{Advances in Neural Information Processing Systems 19}, pages
  601--608, 2007.

\bibitem[Imbens(2020)]{imbens2020potential}
Guido~W Imbens.
\newblock Potential outcome and directed acyclic graph approaches to causality:
  Relevance for empirical practice in economics.
\newblock \emph{Journal of Economic Literature}, 58\penalty0 (4):\penalty0
  1129--79, 2020.

\bibitem[Imbens and Manski(2004)]{imbens2004confidence}
Guido~W Imbens and Charles~F Manski.
\newblock Confidence intervals for partially identified parameters.
\newblock \emph{Econometrica}, 72\penalty0 (6):\penalty0 1845--1857, 2004.

\bibitem[Jans-Singh et~al.(2020)Jans-Singh, Leeming, Choudhary, and
  Girolami]{jans2020digital}
Melanie Jans-Singh, Kathryn Leeming, Ruchi Choudhary, and Mark Girolami.
\newblock Digital twin of an urban-integrated hydroponic farm.
\newblock \emph{Data-Centric Engineering}, 1:\penalty0 e20, 2020.
\newblock \doi{10.1017/dce.2020.21}.

\bibitem[Jiang and Li(2016)]{jiang2016doubly}
Nan Jiang and Lihong Li.
\newblock Doubly robust off-policy value evaluation for reinforcement learning.
\newblock In Maria~Florina Balcan and Kilian~Q. Weinberger, editors,
  \emph{Proceedings of The 33rd International Conference on Machine Learning},
  volume~48 of \emph{Proceedings of Machine Learning Research}, pages 652--661,
  New York, New York, USA, 20--22 Jun 2016. PMLR.
\newblock URL \url{https://proceedings.mlr.press/v48/jiang16.html}.

\bibitem[Jin et~al.(2021)Jin, Ren, and Cand{\`e}s]{jin2021sensitivity}
Ying Jin, Zhimei Ren, and Emmanuel~J Cand{\`e}s.
\newblock Sensitivity analysis of individual treatment effects: A robust
  conformal inference approach.
\newblock \emph{arXiv preprint arXiv:2111.12161}, 2021.

\bibitem[Johnson et~al.(2016)Johnson, Pollard, Shen, Lehman, Feng, Ghassemi,
  Moody, Szolovits, Anthony~Celi, and Mark]{mimic}
Alistair E.~W. Johnson, Tom~J. Pollard, Lu~Shen, Li-wei~H. Lehman, Mengling
  Feng, Mohammad Ghassemi, Benjamin Moody, Peter Szolovits, Leo Anthony~Celi,
  and Roger~G. Mark.
\newblock Mimic-iii, a freely accessible critical care database.
\newblock \emph{Scientific Data}, 3\penalty0 (1):\penalty0 160035, 2016.
\newblock \doi{10.1038/sdata.2016.35}.
\newblock URL \url{https://doi.org/10.1038/sdata.2016.35}.

\bibitem[Jones et~al.(2020)Jones, Snider, Nassehi, Yon, and
  Hicks]{jones2020characterising}
David Jones, Chris Snider, Aydin Nassehi, Jason Yon, and Ben Hicks.
\newblock Characterising the digital twin: A systematic literature review.
\newblock \emph{CIRP Journal of Manufacturing Science and Technology},
  29:\penalty0 36--52, 2020.

\bibitem[Kallus and Uehara(2022)]{kallus2020off}
Nathan Kallus and Masatoshi Uehara.
\newblock Double reinforcement learning for efficient off-policy evaluation in
  markov decision processes.
\newblock \emph{J. Mach. Learn. Res.}, 21\penalty0 (1), jun 2022.
\newblock ISSN 1532-4435.

\bibitem[Kallus and Zhou(2018)]{kallus2018confounding}
Nathan Kallus and Angela Zhou.
\newblock Confounding-robust policy improvement.
\newblock In S.~Bengio, H.~Wallach, H.~Larochelle, K.~Grauman, N.~Cesa-Bianchi,
  and R.~Garnett, editors, \emph{Advances in Neural Information Processing
  Systems}, volume~31. Curran Associates, Inc., 2018.
\newblock URL
  \url{https://proceedings.neurips.cc/paper/2018/file/3a09a524440d44d7f19870070a5ad42f-Paper.pdf}.

\bibitem[Kallus and Zhou(2020)]{kallus2020minimax}
Nathan Kallus and Angela Zhou.
\newblock Minimax-optimal policy learning under unobserved confounding.
\newblock \emph{Management Science}, 67, 10 2020.
\newblock \doi{10.1287/mnsc.2020.3699}.

\bibitem[Kallus et~al.(2021)Kallus, Saito, and Uehara]{kallus2021optimal}
Nathan Kallus, Yuta Saito, and Masatoshi Uehara.
\newblock Optimal off-policy evaluation from multiple logging policies.
\newblock In \emph{International Conference on Machine Learning}, pages
  5247--5256. PMLR, 2021.

\bibitem[Kapteyn et~al.(2021)Kapteyn, Pretorius, and
  Willcox]{kapteyn2021probabilistic}
Michael~G Kapteyn, Jacob~VR Pretorius, and Karen~E Willcox.
\newblock A probabilistic graphical model foundation for enabling predictive
  digital twins at scale.
\newblock \emph{Nature Computational Science}, 1\penalty0 (5):\penalty0
  337--347, 2021.

\bibitem[Keramati et~al.(2020)Keramati, Dann, Tamkin, and
  Brunskill]{keramati2020being}
Ramtin Keramati, Christoph Dann, Alex Tamkin, and Emma Brunskill.
\newblock Being optimistic to be conservative: Quickly learning a cvar policy.
\newblock In \emph{AAAI Conference on Artificial Intelligence}, volume~34,
  pages 4436--4443, 2020.

\bibitem[Khan et~al.(2018)Khan, Dahl, Falkman, and Fabian]{khan2018digital}
Adnan Khan, Martin Dahl, Petter Falkman, and Martin Fabian.
\newblock Digital twin for legacy systems: Simulation model testing and
  validation.
\newblock In \emph{2018 IEEE 14th International Conference on Automation
  Science and Engineering (CASE)}. IEEE, 2018.

\bibitem[Kochunas and Huan(2021)]{kochunas2021digital}
Brendan Kochunas and Xun Huan.
\newblock Digital twin concepts with uncertainty for nuclear power
  applications.
\newblock \emph{Energies}, 14\penalty0 (14):\penalty0 4235, 2021.

\bibitem[Komorowski et~al.(2018)Komorowski, Celi, Badawi, Gordon, and
  Faisal]{ai-clinician}
Matthieu Komorowski, Leo~A. Celi, Omar Badawi, Anthony~C. Gordon, and A.~Aldo
  Faisal.
\newblock The artificial intelligence clinician learns optimal treatment
  strategies for sepsis in intensive care.
\newblock \emph{Nature Medicine}, 24\penalty0 (11):\penalty0 1716--1720, 2018.
\newblock \doi{10.1038/s41591-018-0213-5}.
\newblock URL \url{https://doi.org/10.1038/s41591-018-0213-5}.

\bibitem[Krizhevsky(2009)]{krizhevsky2009learning}
Alex Krizhevsky.
\newblock Learning multiple layers of features from tiny images.
\newblock Technical report, 2009.

\bibitem[Kuipers et~al.(2024)Kuipers, Tumu, Yang, Kazemi, Mangharam, and
  Paoletti]{kuipers2024conformal}
Tom Kuipers, Renukanandan Tumu, Shuo Yang, Milad Kazemi, Rahul Mangharam, and
  Nicola Paoletti.
\newblock Conformal off-policy prediction for multi-agent systems.
\newblock \emph{arXiv preprint arXiv:2403.16871}, 2024.

\bibitem[Kuroki and Pearl(2014)]{kuroki2014measurement}
Manabu Kuroki and Judea Pearl.
\newblock {Measurement bias and effect restoration in causal inference}.
\newblock \emph{Biometrika}, 101\penalty0 (2):\penalty0 423--437, 03 2014.
\newblock ISSN 0006-3444.
\newblock \doi{10.1093/biomet/ast066}.
\newblock URL \url{https://doi.org/10.1093/biomet/ast066}.

\bibitem[Kuzborskij et~al.(2021)Kuzborskij, Vernade, Gy{\"{o}}rgy, and
  Szepesv{\'a}ri]{uncertainty5}
Ilja Kuzborskij, Claire Vernade, Andr{\'{a}}s Gy{\"{o}}rgy, and Csaba
  Szepesv{\'a}ri.
\newblock Confident off-policy evaluation and selection through self-normalized
  importance weighting.
\newblock In \emph{International Conference on Artificial Intelligence and
  Statistics}, pages 640--648, 2021.

\bibitem[Lai et~al.(2023)Lai, Xu, Chen, and Lin]{lai2023generalization}
Jianfa Lai, Manyun Xu, Rui Chen, and Qian Lin.
\newblock Generalization ability of wide neural networks on $\mathbb{R}$, 2023.
\newblock URL \url{https://arxiv.org/abs/2302.05933}.

\bibitem[Lal et~al.(2021)Lal, Li, Cubro, Chalmers, Li, Herasevich, Dong,
  Pickering, Oguz, and Gajic]{DT-patient}
Amos Lal, Guangxi Li, Edin Cubro, Sarah Chalmers, Heyi Li, Vitaly Herasevich,
  Yue Dong, Brian Pickering, Kilickaya Oguz, and Ognjen Gajic.
\newblock Development and verification of a digital twin patient model to
  predict treatment response in sepsis.
\newblock \emph{Critical Care Medicine}, 49:\penalty0 611--611, 01 2021.
\newblock \doi{10.1097/01.ccm.0000730744.82258.38}.

\bibitem[Lambert et~al.(2022)Lambert, Castricato, von Werra, and
  Havrilla]{lambert2022illustrating}
Nathan Lambert, Louis Castricato, Leandro von Werra, and Alex Havrilla.
\newblock Illustrating reinforcement learning from human feedback (rlhf).
\newblock \emph{Hugging Face Blog}, 2022.
\newblock https://huggingface.co/blog/rlhf.

\bibitem[Larrabide et~al.(2012)Larrabide, Kim, Augsburger, Villa-Uriol,
  Rüfenacht, and Frangi]{larrabide2012fast}
Ignacio Larrabide, Minsuok Kim, Luca Augsburger, Maria~Cruz Villa-Uriol, Daniel
  Rüfenacht, and Alejandro~F Frangi.
\newblock Fast virtual deployment of self-expandable stents: method and in
  vitro evaluation for intracranial aneurysmal stenting.
\newblock \emph{Medical Image Analysis}, 16\penalty0 (3):\penalty0 721—730,
  April 2012.
\newblock ISSN 1361-8415.
\newblock \doi{10.1016/j.media.2010.04.009}.
\newblock URL \url{https://doi.org/10.1016/j.media.2010.04.009}.

\bibitem[Lattimore and Szepesvári(2020)]{Lattimore_Szepesvári_2020}
Tor Lattimore and Csaba Szepesvári.
\newblock \emph{Bandit Algorithms}.
\newblock Cambridge University Press, 2020.

\bibitem[Lavori and Dawson(2004)]{lavori2004dynamic}
Philip~W Lavori and Ree Dawson.
\newblock Dynamic treatment regimes: practical design considerations.
\newblock \emph{Clinical trials}, 1\penalty0 (1):\penalty0 9--20, 2004.

\bibitem[Lei and Wasserman(2014)]{lei2014distribution}
Jing Lei and Larry Wasserman.
\newblock Distribution-free prediction bands for non-parametric regression.
\newblock \emph{Journal of the Royal Statistical Society: Series \textup{B}},
  pages 71--96, 2014.

\bibitem[Lei and Cand{\`e}s(2021)]{lei2020conformal}
Lihua Lei and Emmanuel~J Cand{\`e}s.
\newblock Conformal inference of counterfactuals and individual treatment
  effects.
\newblock \emph{Journal of the Royal Statistical Society: Series \textup{B}},
  pages 911--938, 2021.

\bibitem[Li et~al.(2018)Li, Thomas, and Li]{li2018addressing}
Fan Li, Laine~E Thomas, and Fan Li.
\newblock {Addressing Extreme Propensity Scores via the Overlap Weights}.
\newblock \emph{American Journal of Epidemiology}, 188\penalty0 (1):\penalty0
  250--257, 09 2018.
\newblock ISSN 0002-9262.
\newblock \doi{10.1093/aje/kwy201}.
\newblock URL \url{https://doi.org/10.1093/aje/kwy201}.

\bibitem[Li et~al.(2010)Li, Chu, Langford, and Schapire]{li2010contextual}
Lihong Li, Wei Chu, John Langford, and Robert~E. Schapire.
\newblock A contextual-bandit approach to personalized news article
  recommendation.
\newblock In \emph{Proceedings of the 19th International Conference on World
  Wide Web}, WWW '10, page 661–670, New York, NY, USA, 2010. Association for
  Computing Machinery.
\newblock ISBN 9781605587998.
\newblock \doi{10.1145/1772690.1772758}.
\newblock URL \url{https://doi.org/10.1145/1772690.1772758}.

\bibitem[Lin et~al.(2020)Lin, Rudi, Rosasco, and Cevher]{lin2020optimal}
Junhong Lin, Alessandro Rudi, Lorenzo Rosasco, and Volkan Cevher.
\newblock Optimal rates for spectral algorithms with least-squares regression
  over hilbert spaces.
\newblock \emph{Applied and Computational Harmonic Analysis}, 48\penalty0
  (3):\penalty0 868--890, 2020.
\newblock ISSN 1063-5203.
\newblock \doi{https://doi.org/10.1016/j.acha.2018.09.009}.
\newblock URL
  \url{https://www.sciencedirect.com/science/article/pii/S1063520318300174}.

\bibitem[Liu et~al.(2019)Liu, Liu, Anandkumar, and Yue]{liu2019triply}
Anqi Liu, Hao Liu, Anima Anandkumar, and Yisong Yue.
\newblock Triply robust off-policy evaluation, 2019.
\newblock URL \url{https://arxiv.org/abs/1911.05811}.

\bibitem[Liu et~al.(2018)Liu, Li, Tang, and Zhou]{liu2018breaking}
Qiang Liu, Lihong Li, Ziyang Tang, and Dengyong Zhou.
\newblock Breaking the curse of horizon: Infinite-horizon off-policy
  estimation.
\newblock In S.~Bengio, H.~Wallach, H.~Larochelle, K.~Grauman, N.~Cesa-Bianchi,
  and R.~Garnett, editors, \emph{Advances in Neural Information Processing
  Systems}, volume~31. Curran Associates, Inc., 2018.
\newblock URL
  \url{https://proceedings.neurips.cc/paper/2018/file/dda04f9d634145a9c68d5dfe53b21272-Paper.pdf}.

\bibitem[Liu et~al.(2024)Liu, Yao, Ton, Zhang, Guo, Cheng, Klochkov, Taufiq,
  and Li]{liu2024trustworthyllmssurveyguideline}
Yang Liu, Yuanshun Yao, Jean-Francois Ton, Xiaoying Zhang, Ruocheng Guo, Hao
  Cheng, Yegor Klochkov, Muhammad~Faaiz Taufiq, and Hang Li.
\newblock Trustworthy llms: a survey and guideline for evaluating large
  language models' alignment, 2024.
\newblock URL \url{https://arxiv.org/abs/2308.05374}.

\bibitem[Liu et~al.(2020)Liu, Bacon, and Brunskill]{liu2020understanding}
Yao Liu, Pierre-Luc Bacon, and Emma Brunskill.
\newblock Understanding the curse of horizon in off-policy evaluation via
  conditional importance sampling.
\newblock In \emph{Proceedings of the 37th International Conference on Machine
  Learning}, ICML'20. JMLR.org, 2020.

\bibitem[London and Sandler(2019)]{chaudhuri2019london}
Ben London and Ted Sandler.
\newblock {B}ayesian counterfactual risk minimization.
\newblock In Kamalika Chaudhuri and Ruslan Salakhutdinov, editors,
  \emph{Proceedings of the 36th International Conference on Machine Learning},
  volume~97 of \emph{Proceedings of Machine Learning Research}, pages
  4125--4133. PMLR, 09--15 Jun 2019.
\newblock URL \url{https://proceedings.mlr.press/v97/london19a.html}.

\bibitem[Louizos et~al.(2017)Louizos, Shalit, Mooij, Sontag, Zemel, and
  Welling]{louizos2017causal}
Christos Louizos, Uri Shalit, Joris Mooij, David Sontag, Richard Zemel, and Max
  Welling.
\newblock Causal effect inference with deep latent-variable models.
\newblock In \emph{Proceedings of the 31st International Conference on Neural
  Information Processing Systems}, NIPS'17, page 6449–6459, Red Hook, NY,
  USA, 2017. Curran Associates Inc.
\newblock ISBN 9781510860964.

\bibitem[Lu et~al.(2021)Lu, Zhang, Fang, Teshima, and
  Sugiyama]{lu2021rethinking}
Nan Lu, Tianyi Zhang, Tongtong Fang, Takeshi Teshima, and Masashi Sugiyama.
\newblock Rethinking importance weighting for transfer learning.
\newblock \emph{CoRR}, abs/2112.10157, 2021.
\newblock URL \url{https://arxiv.org/abs/2112.10157}.

\bibitem[Lu et~al.(2020)Lu, Liu, Kevin, Wang, Huang, and Xu]{lu2020digital}
Yuqian Lu, Chao Liu, I~Kevin, Kai Wang, Huiyue Huang, and Xun Xu.
\newblock Digital twin-driven smart manufacturing: Connotation, reference
  model, applications and research issues.
\newblock \emph{Robotics and Computer-Integrated Manufacturing}, 61:\penalty0
  101837, 2020.

\bibitem[Manski(1989)]{manski1989anatomy}
Charles~F. Manski.
\newblock Anatomy of the selection problem.
\newblock \emph{The Journal of Human Resources}, 24\penalty0 (3):\penalty0
  343--360, 1989.
\newblock ISSN 0022166X.
\newblock URL \url{http://www.jstor.org/stable/145818}.

\bibitem[Manski(1990)]{manski}
Charles~F. Manski.
\newblock Nonparametric bounds on treatment effects.
\newblock \emph{The American Economic Review}, 80\penalty0 (2):\penalty0
  319--323, 1990.
\newblock ISSN 00028282.
\newblock URL \url{http://www.jstor.org/stable/2006592}.

\bibitem[Manski(1995)]{manski1995identification}
Charles~F Manski.
\newblock \emph{{Identification Problems in the Social Sciences}}.
\newblock Harvard University Press, 1995.

\bibitem[Manski(2003)]{manski2003partial}
Charles~F Manski.
\newblock \emph{Partial Identification of Probability Distributions}.
\newblock Springer, 2003.

\bibitem[Masison et~al.(2021)Masison, Beezley, Mei, Ribeiro, Knapp,
  Sordo~Vieira, Adhikari, Scindia, Grauer, Helba, et~al.]{masison2021modular}
Joseph Masison, Jonathan Beezley, Yu~Mei, Henrique Assis~Lopes Ribeiro, Adam~C
  Knapp, L~Sordo~Vieira, Bandita Adhikari, Yogesh Scindia, Michael Grauer,
  Brian Helba, et~al.
\newblock A modular computational framework for medical digital twins.
\newblock \emph{Proceedings of the National Academy of Sciences}, 118\penalty0
  (20):\penalty0 e2024287118, 2021.

\bibitem[McDaniel and Baird(2019)]{biogears}
Matthew McDaniel and Austin Baird.
\newblock {A Full-Body Model of Burn Pathophysiology and Treatment Using the
  BioGears Engine}.
\newblock In \emph{2019 41st Annual International Conference of the IEEE
  Engineering in Medicine and Biology Society (EMBC)}, pages 261--264, 2019.
\newblock \doi{10.1109/EMBC.2019.8857686}.

\bibitem[McDaniel et~al.(2019)McDaniel, Keller, White, and
  Baird]{sepsis-modelling}
Matthew McDaniel, Jonathan~M. Keller, Steven White, and Austin Baird.
\newblock {A Whole-Body Mathematical Model of Sepsis Progression and Treatment
  Designed in the BioGears Physiology Engine}.
\newblock \emph{Frontiers in Physiology}, 10:\penalty0 1321, 2019.
\newblock ISSN 1664-042X.
\newblock \doi{10.3389/fphys.2019.01321}.
\newblock URL
  \url{https://www.frontiersin.org/article/10.3389/fphys.2019.01321}.

\bibitem[Meng and Wong(1996)]{meng1996simulating}
Xiao-Li Meng and Wing~Hung Wong.
\newblock Simulating ratios of normalizing constants via a simple identity: a
  theoretical exploration.
\newblock \emph{Statistica Sinica}, pages 831--860, 1996.

\bibitem[Metelli et~al.(2021)Metelli, Russo, and
  Restelli]{metelli2021subgaussian}
Alberto~Maria Metelli, Alessio Russo, and Marcello Restelli.
\newblock Subgaussian and differentiable importance sampling for off-policy
  evaluation and learning.
\newblock In M.~Ranzato, A.~Beygelzimer, Y.~Dauphin, P.S. Liang, and J.~Wortman
  Vaughan, editors, \emph{Advances in Neural Information Processing Systems},
  volume~34, pages 8119--8132. Curran Associates, Inc., 2021.
\newblock URL
  \url{https://proceedings.neurips.cc/paper_files/paper/2021/file/4476b929e30dd0c4e8bdbcc82c6ba23a-Paper.pdf}.

\bibitem[Murphy(2005)]{murphy2005experimental}
S.~A. Murphy.
\newblock An experimental design for the development of adaptive treatment
  strategies.
\newblock \emph{Statistics in Medicine}, 24\penalty0 (10):\penalty0 1455--1481,
  2005.
\newblock \doi{https://doi.org/10.1002/sim.2022}.
\newblock URL \url{https://onlinelibrary.wiley.com/doi/abs/10.1002/sim.2022}.

\bibitem[Murphy(2003)]{murphy2003optimal}
Susan~A Murphy.
\newblock Optimal dynamic treatment regimes.
\newblock \emph{Journal of the Royal Statistical Society: Series B (Statistical
  Methodology)}, 65\penalty0 (2):\penalty0 331--355, 2003.

\bibitem[Namkoong et~al.(2020)Namkoong, Keramati, Yadlowsky, and
  Brunskill]{namkoong2020offpolicy}
Hongseok Namkoong, Ramtin Keramati, Steve Yadlowsky, and Emma Brunskill.
\newblock Off-policy policy evaluation for sequential decisions under
  unobserved confounding.
\newblock In H.~Larochelle, M.~Ranzato, R.~Hadsell, M.F. Balcan, and H.~Lin,
  editors, \emph{Advances in Neural Information Processing Systems}, volume~33,
  pages 18819--18831. Curran Associates, Inc., 2020.
\newblock URL
  \url{https://proceedings.neurips.cc/paper/2020/file/da21bae82c02d1e2b8168d57cd3fbab7-Paper.pdf}.

\bibitem[Newey and Robins(2018)]{newey2018cross}
Whitney~K. Newey and James~R. Robins.
\newblock Cross-fitting and fast remainder rates for semiparametric estimation,
  2018.
\newblock URL \url{https://arxiv.org/abs/1801.09138}.

\bibitem[Niederer et~al.(2021)Niederer, Sacks, Girolami, and
  Willcox]{niederer2021scaling}
Steven~A Niederer, Michael~S Sacks, Mark Girolami, and Karen Willcox.
\newblock Scaling digital twins from the artisanal to the industrial.
\newblock \emph{Nature Computational Science}, 1\penalty0 (5):\penalty0
  313--320, 2021.

\bibitem[Osama et~al.(2020)Osama, Zachariah, and Stoica]{osama2020learning}
Muhammad Osama, Dave Zachariah, and Peter Stoica.
\newblock Learning robust decision policies from observational data.
\newblock \emph{arXiv preprint arXiv:2006.02355}, 2020.

\bibitem[Pearl(2009)]{pearl2009causality}
Judea Pearl.
\newblock \emph{Causality}.
\newblock Cambridge University Press, 2 edition, 2009.
\newblock \doi{10.1017/CBO9780511803161}.

\bibitem[Popper(2005)]{popper2005logic}
Karl Popper.
\newblock \emph{The Logic of Scientific Discovery}.
\newblock Routledge, 2005.

\bibitem[Qin and Liu(2013)]{msr}
Tao Qin and Tie{-}Yan Liu.
\newblock Introducing {LETOR} 4.0 datasets.
\newblock \emph{arXiv preprint arXiv:1306.2597}, 2013.

\bibitem[Robins(1986)]{robins1986new}
James Robins.
\newblock A new approach to causal inference in mortality studies with a
  sustained exposure period—application to control of the healthy worker
  survivor effect.
\newblock \emph{Mathematical Modelling}, 7\penalty0 (9):\penalty0 1393--1512,
  1986.
\newblock ISSN 0270-0255.
\newblock \doi{https://doi.org/10.1016/0270-0255(86)90088-6}.
\newblock URL
  \url{https://www.sciencedirect.com/science/article/pii/0270025586900886}.

\bibitem[Romano et~al.(2019)Romano, Patterson, and
  Cand\`es]{romano2019conformalized}
Yaniv Romano, Evan Patterson, and Emmanuel~J. Cand\`es.
\newblock Conformalized quantile regression.
\newblock In \emph{Advances in Neural Information Processing Systems},
  volume~32, pages 3543--3553, 2019.

\bibitem[Romano et~al.(2020)Romano, Barber, Sabatti, and
  Candès]{Romano2020With}
Yaniv Romano, Rina~Foygel Barber, Chiara Sabatti, and Emmanuel~J. Candès.
\newblock With malice toward none: Assessing uncertainty via equalized
  coverage.
\newblock \emph{Harvard Data Science Review}, 2\penalty0 (2), 4 2020.

\bibitem[Rosenbaum(2002)]{rosenbaum2002observational}
Paul~R. Rosenbaum.
\newblock \emph{Observational Studies}.
\newblock Springer, New York, NY, 2002.

\bibitem[Rosenbaum and Rubin(1983)]{rosenbaum1983central}
Paul~R. Rosenbaum and Donald~B. Rubin.
\newblock The central role of the propensity score in observational studies for
  causal effects.
\newblock \emph{Biometrika}, 70\penalty0 (1):\penalty0 41--55, 1983.
\newblock ISSN 00063444.
\newblock URL \url{http://www.jstor.org/stable/2335942}.

\bibitem[Rowland et~al.(2020)Rowland, Harutyunyan, Hasselt, Borsa, Schaul,
  Munos, and Dabney]{rowland2020conditional}
Mark Rowland, Anna Harutyunyan, Hado Hasselt, Diana Borsa, Tom Schaul, R{\'e}mi
  Munos, and Will Dabney.
\newblock Conditional importance sampling for off-policy learning.
\newblock In \emph{International Conference on Artificial Intelligence and
  Statistics}, pages 45--55. PMLR, 2020.

\bibitem[Roy and Oberkampf(2011)]{roy2011comprehensive}
Christopher~J. Roy and William~L. Oberkampf.
\newblock A comprehensive framework for verification, validation, and
  uncertainty quantification in scientific computing.
\newblock \emph{Computer Methods in Applied Mechanics and Engineering},
  200\penalty0 (25):\penalty0 2131--2144, 2011.
\newblock ISSN 0045-7825.
\newblock \doi{https://doi.org/10.1016/j.cma.2011.03.016}.
\newblock URL
  \url{https://www.sciencedirect.com/science/article/pii/S0045782511001290}.

\bibitem[Rubin(1974)]{rubin1974estimating}
Donald~B. Rubin.
\newblock Estimating causal effects of treatments in randomized and
  nonrandomized studies.
\newblock \emph{Journal of Educational Psychology}, 66:\penalty0 688–701,
  1974.
\newblock \doi{https://doi.org/10.1037/h0037350}.

\bibitem[Rubin(2005)]{rubin2005causal}
Donald~B Rubin.
\newblock Causal inference using potential outcomes.
\newblock \emph{Journal of the American Statistical Association}, 100\penalty0
  (469):\penalty0 322--331, 2005.
\newblock \doi{10.1198/016214504000001880}.
\newblock URL \url{https://doi.org/10.1198/016214504000001880}.

\bibitem[Sachdeva et~al.(2020)Sachdeva, Su, and Joachims]{sachdeva2020off}
Noveen Sachdeva, Yi~Su, and Thorsten Joachims.
\newblock Off-policy bandits with deficient support.
\newblock In \emph{Proceedings of the 26th ACM SIGKDD International Conference
  on Knowledge Discovery \& Data Mining}, KDD '20, page 965–975, New York,
  NY, USA, 2020. Association for Computing Machinery.
\newblock ISBN 9781450379984.
\newblock \doi{10.1145/3394486.3403139}.
\newblock URL \url{https://doi.org/10.1145/3394486.3403139}.

\bibitem[Sacks et~al.(2020)Sacks, Brilakis, Pikas, Xie, and
  Girolami]{sacks2020construction}
Rafael Sacks, Ioannis Brilakis, Ergo Pikas, Haiyan Xie, and Mark Girolami.
\newblock Construction with digital twin information systems.
\newblock \emph{Data-Centric Engineering}, 1, 2020.

\bibitem[Saito and Joachims(2022)]{saito2022off}
Yuta Saito and Thorsten Joachims.
\newblock Off-policy evaluation for large action spaces via embeddings.
\newblock In \emph{Proceedings of the 39th International Conference on Machine
  Learning}, pages 19089--19122. PMLR, 2022.

\bibitem[Saito et~al.(2020)Saito, Shunsuke, Megumi, and Yusuke]{saito2020open}
Yuta Saito, Aihara Shunsuke, Matsutani Megumi, and Narita Yusuke.
\newblock Open bandit dataset and pipeline: Towards realistic and reproducible
  off-policy evaluation.
\newblock \emph{arXiv preprint arXiv:2008.07146}, 2020.

\bibitem[Saito et~al.(2021)Saito, Udagawa, Kiyohara, Mogi, Narita, and
  Tateno]{saito2021evaluating}
Yuta Saito, Takuma Udagawa, Haruka Kiyohara, Kazuki Mogi, Yusuke Narita, and
  Kei Tateno.
\newblock Evaluating the robustness of off-policy evaluation.
\newblock In \emph{Proceedings of the 15th ACM Conference on Recommender
  Systems}, RecSys '21, page 114–123, New York, NY, USA, 2021. Association
  for Computing Machinery.
\newblock ISBN 9781450384582.
\newblock \doi{10.1145/3460231.3474245}.
\newblock URL \url{https://doi.org/10.1145/3460231.3474245}.

\bibitem[Schulman et~al.(2017)Schulman, Wolski, Dhariwal, Radford, and
  Klimov]{schulman2017proximal}
John Schulman, Filip Wolski, Prafulla Dhariwal, Alec Radford, and Oleg Klimov.
\newblock Proximal policy optimization algorithms, 2017.

\bibitem[Seymour et~al.(2016)Seymour, Liu, Iwashyna, Brunkhorst, Rea, Scherag,
  Rubenfeld, Kahn, Shankar-Hari, Singer, Deutschman, Escobar, and
  Angus]{seymour2016assessment}
Christopher~W. Seymour, Vincent~X. Liu, Theodore~J. Iwashyna, Frank~M.
  Brunkhorst, Thomas~D. Rea, André Scherag, Gordon Rubenfeld, Jeremy~M. Kahn,
  Manu Shankar-Hari, Mervyn Singer, Clifford~S. Deutschman, Gabriel~J. Escobar,
  and Derek~C. Angus.
\newblock {Assessment of Clinical Criteria for Sepsis: For the Third
  International Consensus Definitions for Sepsis and Septic Shock (Sepsis-3)}.
\newblock \emph{JAMA}, 315\penalty0 (8):\penalty0 762--774, 02 2016.
\newblock ISSN 0098-7484.
\newblock \doi{10.1001/jama.2016.0288}.
\newblock URL \url{https://doi.org/10.1001/jama.2016.0288}.

\bibitem[Shafer and Vovk(2008)]{shafer2008tutorial}
Glenn Shafer and Vladimir Vovk.
\newblock A tutorial on conformal prediction.
\newblock \emph{Journal of Machine Learning Research}, 9\penalty0 (3), 2008.

\bibitem[Shimodaira(2000)]{SHIMODAIRA2000Improving}
Hidetoshi Shimodaira.
\newblock Improving predictive inference under covariate shift by weighting the
  log-likelihood function.
\newblock \emph{Journal of Statistical Planning and Inference}, 90\penalty0
  (2):\penalty0 227--244, 2000.
\newblock ISSN 0378-3758.
\newblock \doi{https://doi.org/10.1016/S0378-3758(00)00115-4}.
\newblock URL
  \url{https://www.sciencedirect.com/science/article/pii/S0378375800001154}.

\bibitem[Singer et~al.(2016)Singer, Deutschman, Seymour, Shankar-Hari, Annane,
  Bauer, Bellomo, Bernard, Chiche, Coopersmith, Hotchkiss, Levy, Marshall,
  Martin, Opal, Rubenfeld, van~der Poll, Vincent, and Angus]{sepsis-criteria}
Mervyn Singer, Clifford~S. Deutschman, Christopher~Warren Seymour, Manu
  Shankar-Hari, Djillali Annane, Michael Bauer, Rinaldo Bellomo, Gordon~R.
  Bernard, Jean-Daniel Chiche, Craig~M. Coopersmith, Richard~S. Hotchkiss,
  Mitchell~M. Levy, John~C. Marshall, Greg~S. Martin, Steven~M. Opal, Gordon~D.
  Rubenfeld, Tom van~der Poll, Jean-Louis Vincent, and Derek~C. Angus.
\newblock {The Third International Consensus Definitions for Sepsis and Septic
  Shock (Sepsis-3)}.
\newblock \emph{JAMA}, 315\penalty0 (8):\penalty0 801--810, 02 2016.
\newblock ISSN 0098-7484.
\newblock \doi{10.1001/jama.2016.0287}.
\newblock URL \url{https://doi.org/10.1001/jama.2016.0287}.

\bibitem[Sondhi et~al.(2020)Sondhi, Arbour, and Dimmery]{sondhi2020balanced}
Arjun Sondhi, David Arbour, and Drew Dimmery.
\newblock Balanced off-policy evaluation in general action spaces.
\newblock In Silvia Chiappa and Roberto Calandra, editors, \emph{Proceedings of
  the Twenty Third International Conference on Artificial Intelligence and
  Statistics}, volume 108 of \emph{Proceedings of Machine Learning Research},
  pages 2413--2423. PMLR, 26--28 Aug 2020.
\newblock URL \url{https://proceedings.mlr.press/v108/sondhi20a.html}.

\bibitem[Stutz et~al.(2022)Stutz, Dvijotham, Cemgil, and
  Doucet]{stutz2021learning}
David Stutz, Krishnamurthy Dvijotham, Ali~Taylan Cemgil, and Arnaud Doucet.
\newblock Learning optimal conformal classifiers.
\newblock \emph{International Conference on Representation Learning}, 2022.

\bibitem[Su et~al.(2019{\natexlab{a}})Su, Dimakopoulou, Krishnamurthy, and
  Dud{\'{\i}}k]{drobust}
Yi~Su, Maria Dimakopoulou, Akshay Krishnamurthy, and Miroslav Dud{\'{\i}}k.
\newblock Doubly robust off-policy evaluation with shrinkage.
\newblock \emph{CoRR}, abs/1907.09623, 2019{\natexlab{a}}.
\newblock URL \url{http://arxiv.org/abs/1907.09623}.

\bibitem[Su et~al.(2019{\natexlab{b}})Su, Wang, Santacatterina, and
  Joachims]{su2019continuous}
Yi~Su, Lequn Wang, Michele Santacatterina, and Thorsten Joachims.
\newblock {CAB}: Continuous adaptive blending for policy evaluation and
  learning.
\newblock In Kamalika Chaudhuri and Ruslan Salakhutdinov, editors,
  \emph{Proceedings of the 36th International Conference on Machine Learning},
  volume~97 of \emph{Proceedings of Machine Learning Research}, pages
  6005--6014. PMLR, 09--15 Jun 2019{\natexlab{b}}.
\newblock URL \url{https://proceedings.mlr.press/v97/su19a.html}.

\bibitem[Su et~al.(2020)Su, Dimakopoulou, Krishnamurthy, and
  Dud\'{\i}k]{su2020doubly}
Yi~Su, Maria Dimakopoulou, Akshay Krishnamurthy, and Miroslav Dud\'{\i}k.
\newblock Doubly robust off-policy evaluation with shrinkage.
\newblock In \emph{Proceedings of the 37th International Conference on Machine
  Learning}, ICML'20. JMLR.org, 2020.

\bibitem[Sugiyama and Kawanabe(2012)]{sugiyama2012machine}
Masashi Sugiyama and Motoaki Kawanabe.
\newblock \emph{Machine Learning in Non-Stationary Environments: Introduction
  to Covariate Shift Adaptation}.
\newblock The MIT Press, 2012.
\newblock ISBN 9780262017091.
\newblock URL \url{http://www.jstor.org/stable/j.ctt5hhbtm}.

\bibitem[Sugiyama et~al.(2007)Sugiyama, Krauledat, and
  M\"uller]{sugiyama2007covariate}
Masashi Sugiyama, Matthias Krauledat, and Klaus-Robert M\"uller.
\newblock Covariate shift adaptation by importance weighted cross validation.
\newblock \emph{Journal of Machine Learning Research}, 8\penalty0 (5):\penalty0
  985--1005, 2007.

\bibitem[Sugiyama et~al.(2008)Sugiyama, Nakajima, Kashima, Buenau, and
  Kawanabe]{sugiyama2008direct}
Masashi Sugiyama, Shinichi Nakajima, Hisashi Kashima, Paul~von Buenau, and
  Motoaki Kawanabe.
\newblock Direct importance estimation with model selection and its application
  to covariate shift adaptation.
\newblock In \emph{Advances in Neural Information Processing Systems 20}, pages
  1433--1440, 2008.

\bibitem[Swaminathan and
  Joachims(2015{\natexlab{a}})]{swaminathan2015counterfactual}
Adith Swaminathan and Thorsten Joachims.
\newblock Counterfactual risk minimization: Learning from logged bandit
  feedback.
\newblock In \emph{Proceedings of the 32nd International Conference on
  International Conference on Machine Learning - Volume 37}, ICML'15, page
  814–823. JMLR.org, 2015{\natexlab{a}}.

\bibitem[Swaminathan and Joachims(2015{\natexlab{b}})]{swaminathan2015the}
Adith Swaminathan and Thorsten Joachims.
\newblock The self-normalized estimator for counterfactual learning.
\newblock In C.~Cortes, N.~Lawrence, D.~Lee, M.~Sugiyama, and R.~Garnett,
  editors, \emph{Advances in Neural Information Processing Systems}, volume~28.
  Curran Associates, Inc., 2015{\natexlab{b}}.
\newblock URL
  \url{https://proceedings.neurips.cc/paper_files/paper/2015/file/39027dfad5138c9ca0c474d71db915c3-Paper.pdf}.

\bibitem[Swaminathan and Joachims(2015{\natexlab{c}})]{uncertainty3}
Adith Swaminathan and Thorsten Joachims.
\newblock Batch learning from logged bandit feedback through counterfactual
  risk minimization.
\newblock \emph{Journal of Machine Learning Research}, 16\penalty0
  (52):\penalty0 1731--1755, 2015{\natexlab{c}}.

\bibitem[Swaminathan and Joachims(2015{\natexlab{d}})]{uncertainty4}
Adith Swaminathan and Thorsten Joachims.
\newblock The self-normalized estimator for counterfactual learning.
\newblock In \emph{Advances in Neural Information Processing Systems},
  volume~28, 2015{\natexlab{d}}.

\bibitem[Swaminathan et~al.(2017{\natexlab{a}})Swaminathan, Krishnamurthy,
  Agarwal, Dud{\'\i}k, Langford, Jose, and Zitouni]{swaminathan2016off}
Adith Swaminathan, Akshay Krishnamurthy, Alekh Agarwal, Miroslav Dud{\'\i}k,
  John Langford, Damien Jose, and Imed Zitouni.
\newblock Off-policy evaluation for slate recommendation.
\newblock In \emph{Advances in Neural Information Processing Systems},
  2017{\natexlab{a}}.

\bibitem[Swaminathan et~al.(2017{\natexlab{b}})Swaminathan, Krishnamurthy,
  Agarwal, Dud\'{\i}k, Langford, Jose, and Zitouni]{swaminathan2017off}
Adith Swaminathan, Akshay Krishnamurthy, Alekh Agarwal, Miroslav Dud\'{\i}k,
  John Langford, Damien Jose, and Imed Zitouni.
\newblock Off-policy evaluation for slate recommendation.
\newblock In \emph{Proceedings of the 31st International Conference on Neural
  Information Processing Systems}, NIPS'17, page 3635–3645, Red Hook, NY,
  USA, 2017{\natexlab{b}}. Curran Associates Inc.
\newblock ISBN 9781510860964.

\bibitem[Tan(2006)]{tan2006distributional}
Zhiqiang Tan.
\newblock A distributional approach for causal inference using propensity
  scores.
\newblock \emph{Journal of the American Statistical Association}, 101\penalty0
  (476):\penalty0 1619--1637, 2006.

\bibitem[Taufiq et~al.(2022)Taufiq, Ton, Cornish, Teh, and
  Doucet]{taufiq2022conformal}
Muhammad~Faaiz Taufiq, Jean-Francois Ton, Rob Cornish, Yee~Whye Teh, and Arnaud
  Doucet.
\newblock Conformal off-policy prediction in contextual bandits.
\newblock In Alice~H. Oh, Alekh Agarwal, Danielle Belgrave, and Kyunghyun Cho,
  editors, \emph{Advances in Neural Information Processing Systems}, 2022.
\newblock URL \url{https://openreview.net/forum?id=IfgOWI5v2f}.

\bibitem[Taufiq et~al.(2023{\natexlab{a}})Taufiq, Bl\"obaum, and
  Minorics]{taufiq2023manifold}
Muhammad~Faaiz Taufiq, Patrick Bl\"obaum, and Lenon Minorics.
\newblock Manifold restricted interventional shapley values.
\newblock In Francisco Ruiz, Jennifer Dy, and Jan-Willem van~de Meent, editors,
  \emph{Proceedings of The 26th International Conference on Artificial
  Intelligence and Statistics}, volume 206 of \emph{Proceedings of Machine
  Learning Research}, pages 5079--5106. PMLR, 25--27 Apr 2023{\natexlab{a}}.
\newblock URL \url{https://proceedings.mlr.press/v206/taufiq23a.html}.

\bibitem[Taufiq et~al.(2023{\natexlab{b}})Taufiq, Doucet, Cornish, and
  Ton]{taufiq2023marginal}
Muhammad~Faaiz Taufiq, Arnaud Doucet, Rob Cornish, and Jean-Francois Ton.
\newblock Marginal density ratio for off-policy evaluation in contextual
  bandits.
\newblock In \emph{Thirty-seventh Conference on Neural Information Processing
  Systems}, 2023{\natexlab{b}}.
\newblock URL \url{https://openreview.net/forum?id=noyleECBam}.

\bibitem[Taufiq et~al.(2024)Taufiq, Ton, and
  Liu]{taufiq2024achievablefairnessdatautility}
Muhammad~Faaiz Taufiq, Jean-Francois Ton, and Yang Liu.
\newblock Achievable fairness on your data with utility guarantees, 2024.
\newblock URL \url{https://arxiv.org/abs/2402.17106}.

\bibitem[Tchetgen et~al.(2020)Tchetgen, Ying, Cui, Shi, and
  Miao]{tchetgen2020introduction}
Eric J~Tchetgen Tchetgen, Andrew Ying, Yifan Cui, Xu~Shi, and Wang Miao.
\newblock An introduction to proximal causal learning, 2020.

\bibitem[Thomas and Brunskill(2016)]{thomas2016data}
Philip~S. Thomas and Emma Brunskill.
\newblock Data-efficient off-policy policy evaluation for reinforcement
  learning.
\newblock In \emph{Proceedings of the 33rd International Conference on
  International Conference on Machine Learning - Volume 48}, ICML'16, page
  2139–2148. JMLR.org, 2016.

\bibitem[Thomas et~al.(2015)Thomas, Theocharous, and Ghavamzadeh]{uncertainty2}
Philip~S. Thomas, Georgios Theocharous, and Mohammad Ghavamzadeh.
\newblock High-confidence off-policy evaluation.
\newblock In \emph{AAAI Conference on Artificial Intelligence}, 2015.

\bibitem[Tibshirani and Efron(1993)]{tibshirani1993introduction}
Robert~J Tibshirani and Bradley Efron.
\newblock An introduction to the bootstrap.
\newblock \emph{Monographs on statistics and applied probability}, 57:\penalty0
  1--436, 1993.

\bibitem[Tibshirani et~al.(2019)Tibshirani, Barber, Cand\`es, and
  Ramdas]{tibshirani2020conformal}
Ryan~J. Tibshirani, Rina~Foygel Barber, Emmanuel~J. Cand\`es, and Aaditya
  Ramdas.
\newblock Conformal prediction under covariate shift.
\newblock In \emph{Advances in Neural Information Processing Systems}, 2019.

\bibitem[Tsiatis et~al.(2019)Tsiatis, Davidian, Holloway, and
  Laber]{tsiatis2019dynamic}
Anastasios~A Tsiatis, Marie Davidian, Shannon~T Holloway, and Eric~B Laber.
\newblock \emph{Dynamic treatment regimes: Statistical methods for precision
  medicine}.
\newblock Chapman and Hall/CRC, 2019.

\bibitem[Uehara et~al.(2022)Uehara, Shi, and
  Kallus]{uehara2022reviewoffpolicyevaluationreinforcement}
Masatoshi Uehara, Chengchun Shi, and Nathan Kallus.
\newblock A review of off-policy evaluation in reinforcement learning, 2022.
\newblock URL \url{https://arxiv.org/abs/2212.06355}.

\bibitem[Voloshin et~al.(2021)Voloshin, Le, Jiang, and
  Yue]{voloshin2021empirical}
Cameron Voloshin, Hoang~Minh Le, Nan Jiang, and Yisong Yue.
\newblock Empirical study of off-policy policy evaluation for reinforcement
  learning.
\newblock In \emph{Thirty-fifth Conference on Neural Information Processing
  Systems Datasets and Benchmarks Track (Round 1)}, 2021.
\newblock URL \url{https://openreview.net/forum?id=IsK8iKbL-I}.

\bibitem[Vovk(2012)]{vovk2012}
Vladimir Vovk.
\newblock Conditional validity of inductive conformal predictors.
\newblock In Steven C.~H. Hoi and Wray Buntine, editors, \emph{Proceedings of
  the Asian Conference on Machine Learning}, volume~25 of \emph{Proceedings of
  Machine Learning Research}, pages 475--490, Singapore Management University,
  Singapore, 04--06 Nov 2012. PMLR.
\newblock URL \url{https://proceedings.mlr.press/v25/vovk12.html}.

\bibitem[Vovk et~al.(2005)Vovk, Gammerman, and Shafer]{vovk2005algorithmic}
Vladimir Vovk, Alexander Gammerman, and Glenn Shafer.
\newblock \emph{Algorithmic Learning in a Random World}.
\newblock Springer Science \& Business Media, 2005.

\bibitem[Wang et~al.(2017{\natexlab{a}})Wang, Agarwal, and
  Dud\'{\i}k]{adaptive-ope}
Yu-Xiang Wang, Alekh Agarwal, and Miroslav Dud\'{\i}k.
\newblock Optimal and adaptive off-policy evaluation in contextual bandits.
\newblock In \emph{International Conference on Machine Learning}, page
  3589–3597, 2017{\natexlab{a}}.

\bibitem[Wang et~al.(2017{\natexlab{b}})Wang, Agarwal, and
  Dud\'{\i}k]{wang2017optimal}
Yu-Xiang Wang, Alekh Agarwal, and Miroslav Dud\'{\i}k.
\newblock Optimal and adaptive off-policy evaluation in contextual bandits.
\newblock In \emph{Proceedings of the 34th International Conference on Machine
  Learning - Volume 70}, ICML'17, page 3589–3597. JMLR.org,
  2017{\natexlab{b}}.

\bibitem[Xie et~al.(2019{\natexlab{a}})Xie, Ma, and Wang]{ope-rl}
Tengyang Xie, Yifei Ma, and Yu-Xiang Wang.
\newblock Towards optimal off-policy evaluation for reinforcement learning with
  marginalized importance sampling.
\newblock In H.~Wallach, H.~Larochelle, A.~Beygelzimer, F.~d\textquotesingle
  Alch\'{e}-Buc, E.~Fox, and R.~Garnett, editors, \emph{Advances in Neural
  Information Processing Systems}, volume~32. Curran Associates, Inc.,
  2019{\natexlab{a}}.
\newblock URL
  \url{https://proceedings.neurips.cc/paper/2019/file/4ffb0d2ba92f664c2281970110a2e071-Paper.pdf}.

\bibitem[Xie et~al.(2019{\natexlab{b}})Xie, Ma, and Wang]{xie2019advances}
Tengyang Xie, Yifei Ma, and Yu-Xiang Wang.
\newblock Towards optimal off-policy evaluation for reinforcement learning with
  marginalized importance sampling.
\newblock In H.~Wallach, H.~Larochelle, A.~Beygelzimer, F.~d\textquotesingle
  Alch\'{e}-Buc, E.~Fox, and R.~Garnett, editors, \emph{Advances in Neural
  Information Processing Systems}, volume~32. Curran Associates, Inc.,
  2019{\natexlab{b}}.
\newblock URL
  \url{https://proceedings.neurips.cc/paper/2019/file/4ffb0d2ba92f664c2281970110a2e071-Paper.pdf}.

\bibitem[Xu and Gretton(2024)]{xu2024kernel}
Liyuan Xu and Arthur Gretton.
\newblock Kernel single proxy control for deterministic confounding, 2024.

\bibitem[Xu et~al.(2021)Xu, Kanagawa, and Gretton]{xu2021deep}
Liyuan Xu, Heishiro Kanagawa, and Arthur Gretton.
\newblock Deep proxy causal learning and its application to confounded bandit
  policy evaluation.
\newblock In A.~Beygelzimer, Y.~Dauphin, P.~Liang, and J.~Wortman Vaughan,
  editors, \emph{Advances in Neural Information Processing Systems}, 2021.
\newblock URL \url{https://openreview.net/forum?id=0FDxsIEv9G}.

\bibitem[Xu et~al.(2020)Xu, Dong, Li, He, and Li]{xu2020contextual}
Xiao Xu, Fang Dong, Yanghua Li, Shaojian He, and Xin Li.
\newblock Contextual-bandit based personalized recommendation with time-varying
  user interests.
\newblock \emph{Proceedings of the AAAI Conference on Artificial Intelligence},
  34:\penalty0 6518--6525, 04 2020.
\newblock \doi{10.1609/aaai.v34i04.6125}.

\bibitem[Yadlowsky et~al.(2022)Yadlowsky, Namkoong, Basu, Duchi, and
  Tian]{yadlowsky2022bounds}
Steve Yadlowsky, Hongseok Namkoong, Sanjay Basu, John Duchi, and Lu~Tian.
\newblock {Bounds on the conditional and average treatment effect with
  unobserved confounding factors}.
\newblock \emph{The Annals of Statistics}, 50\penalty0 (5):\penalty0 2587 --
  2615, 2022.
\newblock \doi{10.1214/22-AOS2195}.
\newblock URL \url{https://doi.org/10.1214/22-AOS2195}.

\bibitem[Yin et~al.(2021)Yin, Shi, Wang, and Blei]{yin2021conformal}
Mingzhang Yin, Claudia Shi, Yixin Wang, and David~M Blei.
\newblock Conformal sensitivity analysis for individual treatment effects.
\newblock \emph{arXiv preprint arXiv:2112.03493}, 2021.

\bibitem[Zhang and Bareinboim(2019)]{bareinboim}
Junzhe Zhang and Elias Bareinboim.
\newblock Near-optimal reinforcement learning in dynamic treatment regimes.
\newblock In H.~Wallach, H.~Larochelle, A.~Beygelzimer, F.~d\textquotesingle
  Alch\'{e}-Buc, E.~Fox, and R.~Garnett, editors, \emph{Advances in Neural
  Information Processing Systems}, volume~32. Curran Associates, Inc., 2019.
\newblock URL
  \url{https://proceedings.neurips.cc/paper/2019/file/8252831b9fce7a49421e622c14ce0f65-Paper.pdf}.

\bibitem[Zhang et~al.(2023)Zhang, Shi, and Luo]{zhang2023conformal}
Yingying Zhang, Chengchun Shi, and Shikai Luo.
\newblock Conformal off-policy prediction.
\newblock In \emph{International Conference on Artificial Intelligence and
  Statistics}, pages 2751--2768. PMLR, 2023.

\end{thebibliography}

\end{document}